\documentclass{article}
\usepackage{graphicx} 
\usepackage{appendix}
\usepackage{booktabs} 
\usepackage{algorithm}
\usepackage{algorithmic}
\usepackage{amsmath}
\usepackage{amssymb}
\usepackage{hyperref}
\usepackage[most]{tcolorbox}
\usepackage{subcaption}
\usepackage{enumitem}
\usepackage{array}
\usepackage{threeparttable}

\usepackage{amsmath, amssymb, amsthm}

\usepackage{xcolor}

\usepackage{fancyhdr}
\DeclareMathOperator*{\argmax}{arg\,max}
\newtheorem{theorem}{Theorem}
\newtheorem{corollary}{Corollary}
\newtheorem{definition}{Definition}
\newtheorem{example}{Example}
\newtheorem{lemma}[theorem]{Lemma}
\newtheorem{proposition}[theorem]{Proposition}
\newtheorem*{proofsketch}{Proof sketch}

\newenvironment{keywords}
  {\par\vspace{\baselineskip}\noindent\textbf{Keywords:} }
  {}

\makeatletter
\newcommand{\STATEx}{\item[]}
\makeatother

\def\email#1{{\texttt{#1}}}

\title{Probabilistic and reinforced mining of association rules}
\author{Yongchao Huang \footnote{Author email: \email{yongchao.huang@abdn.ac.uk}}}
\date{March 2025}

\begin{document}

\maketitle

\begin{abstract}
    This work introduces 4 novel probabilistic and reinforcement-driven methods for association rule mining (ARM): Gaussian process-based association rule mining (GPAR), Bayesian association rule mining (BARM), multi-armed bandit based association rule mining (MAB-ARM), and reinforcement learning based association rule mining (RLAR). These methods depart fundamentally from traditional frequency-based algorithms such as Apriori, FP-Growth, and Eclat, offering enhanced capabilities for incorporating prior knowledge, modeling uncertainty, item dependencies, probabilistic inference and adaptive search strategies. GPAR employs Gaussian processes to model item co-occurrence via feature representations, enabling principled inference, uncertainty quantification, and efficient generalization to unseen itemsets without retraining. BARM adopts a Bayesian framework with priors and optional correlation structures, yielding robust uncertainty quantification through full posterior distributions over item presence probabilities. MAB-ARM, including its Monte Carlo tree search (MCTS) companion, utilizes an upper confidence bound (UCB) strategy for efficient and adaptive exploration of the itemset space, while RLAR applies a deep Q-network (DQN) to learn a generalizable policy for identifying high-quality rules. Collectively, these approaches improve the flexibility and robustness of ARM, particularly for discovering rare or complex patterns and operating on small datasets. Empirical results on synthetic and real-world datasets demonstrate their effectiveness, while also highlighting trade-offs in computational complexity and interpretability. These innovations mark a significant shift from static, frequency-driven paradigms, offering some prior and dependency-informed, uncertainty-aware or scalable ARM frameworks for diverse application domains such as retail, geography, finance, medical diagnostics, and other risk-sensitive scenarios.

\end{abstract}

\tableofcontents

\section{Introduction} \label{sec:intro}

Association rule (AR \cite{Agrawal1993mining,Pei2009}) mining is a fundamental, unsupervised data mining technique for discovering frequent patterns, correlations, or associations \footnote{Recognizing that association rules identify \textit{correlations} rather than \textit{causal} relationships.} among items within large, unlabeled datasets. For example, in market basket analysis, transactional data such as supermarket purchase records can be analysed with AR to understand purchasing behaviors. These relationships are typically represented as implications of the form \{$A$\} → \{$B$\}, indicating that the occurrence of itemset $A$ often leads to the occurrence of itemset $B$. For example, in market basket analysis, where the goal is to identify purchasing behaviors such as which products are often bought together, the rule \{bread\} → \{milk\} suggests that customers who buy bread are also likely to purchase milk. AR is rule-based and the strength of a rule is evaluated using frequency-based measures such as \textit{support}, which measures the frequency of the itemset in the dataset, and \textit{confidence}, which represents the conditional probability of the consequent given the antecedent, and \textit{lift}, which assesses the strength of a rule over the random occurrence of the antecedent and the consequent. Relationships that exceed user-specified thresholds for these measures are extracted \footnote{For example, to find a rule like \{diapers\} → \{beer\}, we first identify if {diapers, beer} is a frequent itemset (high support) and then check the confidence, ensuring it meets the user-specified threshold. This process is fundamental to generating actionable insights, such as recommending product placements (e.g. placing these items near each other to encourage combined purchases.) or promotions, potentially increasing sales.}. To efficiently extract rules from datasets (e.g. transaction records), several algorithms have been developed (see Appendix.\ref{app:AR} for more details): the \textit{Apriori} algorithm \cite{Agrawal1994fast} employs a bottom-up approach, identifying frequent items in a database, generating candidate itemsets and pruning those that do not meet a predefined minimum support threshold. The \textit{frequent pattern growth} (FP-Growth \cite{Han2000mining}) algorithm represents the database using a compact data structure called an \textit{FP-tree} and extracts frequent itemsets directly from this structure, avoiding the candidate generation step of Apriori and thus improving efficiency. The \textit{equivalence class transformation} (Eclat \cite{Zaki2000scalable}) algorithm utilizes a depth-first search strategy to identify frequent itemsets by intersecting transaction lists, which can be more efficient in certain contexts compared to Apriori.

Association rule mining has diverse applications across various domains. In retail market basket analysis \cite{Pei2009,Ünvan2021}, it assists in identifying products that customers frequently purchase together, enabling retailers to optimise product placements and promotional strategies. In healthcare and clinical settings \cite{Srinivas2001,Lakshmi2017,Ramasamy2020}, AR is useful in identifying associations among symptoms and diseases, aiding in disease diagnosis and treatment planning \cite{Pei2009}. In biology, AR has been used to identify transcription factor interactions in genomic regions \cite{Ceddia2019}. AR can also be used in web usage mining to help inform user behavior \cite{Tripathi2017Web, Metwally2005}, website design and build personalized content recommendations. In finance, AR can be used for feature selection \cite{Srivastava2024} and to detect fraudulent activities by identifying unusual patterns in transaction data \cite{Sanchez2009fraud}. Association rules can also be effectively utilized for classification tasks \footnote{This branch of research, which combines AR with classification, is known as \textit{associative classification} (AC \cite{Han2012,Mao2012,Mitra2012}). By selecting high-quality rules through pruning and ranking, association rule-based classifiers often achieve higher accuracy, frequently surpassing traditional decision tree approaches.} \cite{Pei2009} by identifying relationships between features and class labels; these discovered rules can then predict labels for previously unseen cases, as in active learning. Despite its utility, traditional association rule mining faces several challenges, including the generation of numerous redundant or irrelevant rules, difficulty in selecting appropriate thresholds without domain expertise, and the risk of misinterpreting statistical associations as causal relationships. 

This work explores the development of 4 novel methods, namely Gaussian process-based ARM (GPAR), Bayesian ARM (BARM), multi-armed bandit-based ARM (MAB-ARM), and reinforcement learning-based ARM (RLAR), to enhance the flexibility and robustness of ARM in complex datasets, moving beyond traditional search algorithms. Unlike traditional frequency-based or deterministic search algorithms, these approaches either incorporate probabilistic modeling or reinforcement learning to better handle uncertainty, integrate prior knowledge into the rule discovery process and apply adaptive search. Specifically, GPAR employs Gaussian processes to model item co-occurrence probabilities by representing each item as a feature vector (e.g. size, shape, color, price, or manufacturing information), enabling principled inference through posterior marginalization and sampling. BARM applies Bayesian inference to derive full posterior distributions over item presence probabilities, offering robust uncertainty quantification and support for correlated structures. MAB-ARM employs an UCB strategy (along with a MCTS extension) for adaptive exploration of the itemset space, while RLAR uses a DQN to learn optimal itemset selection policies. These methods not only improve the quality of discovered rules by reducing spurious associations but also uncover more meaningful and actionable insights from data. Drawing on probabilistic and reinforcement approaches widely adopted in learning and decision-making tasks such as Bayesian inference and generative modeling \cite{Zoubin2015PML,Jordan1999introduction,Winkler1972introduction,Gelman2013BDA,Murphy2012PML,Bishop2023DL}, this work contributes some new methodologies for dependency-informed, uncertainty-aware and data-efficient association rule mining.

\section{Related work} \label{sec:related_work}

Traditional association rule mining (ARM) has been extensively studied, with foundational algorithms such as Apriori \cite{Agrawal1994fast}, FP-Growth \cite{Han2000mining}, and Eclat \cite{Zaki2000scalable} relying on frequency measures such as support, confidence, and lift to identify associations within large datasets. These methods have found broad application across domains including retail \cite{Pei2009,Ünvan2021}, healthcare \cite{Srinivas2001,Pei2009,Lakshmi2017,Ramasamy2020}, biology \cite{Ceddia2019}, web \cite{Tripathi2017Web,Metwally2005}, finance \cite{Sanchez2009fraud,Srivastava2024}, and classification tasks \cite{Han2012,Mao2012,Mitra2012}. However, traditional ARM approaches encounter some challenges, notably generating redundant, suspicious or irrelevant rules, struggling with appropriate metrics threshold selection \footnote{Traditional frequency-based ARM methods are generally sensitive to threshold selection. For example, \cite{Gonzalez2020BRM} argued that frequentist thresholds often lead to significant yet infrequent associations.}, and the inability to capture uncertainty and rarity adequately \cite{Pei2009,Lakshmi2017,Ramasamy2020,Ünvan2021}. Their measures of uncertainty, such as confidence, are based on empirical frequency counts and fail to incorporate relevant prior knowledge.

These limitations have prompted research into probabilistic and Bayesian enhancements. Tian et al. \cite{Tian2013BAR} introduced Bayesian association rule mining (BAR), which combines traditional Apriori techniques with \textit{Bayesian networks}. BAR defines novel measures, i.e. Bayesian confidence (BC) and Bayesian lift (BL), to measure the conditional dependencies between items more robustly than traditional frequency metrics. BAR outputs best rules according to BC and BL. Their method was demonstrated effectively on clinical phenotype datasets, revealing rules that capture nuanced relationships otherwise missed by frequentist methods. González et al. \cite{Gonzalez2020BRM} proposed a Bayesian rule mining (BRM) framework which expands the taxonomy of AR mining to incorporate belief-based criteria for rule selection. BRM employs an 'increasing belief' criterion based on recursive Bayesian updating rather than the conventional frequency-based minimum-support threshold. It recursively applies Bayes' theorem to evaluate rules, progressively updates the belief in a rule's validity, enabling it to effectively extract rare associations without suffering from support dilution which is inherent in frequency-based methods. This approach has demonstrated effectiveness in identifying rare and significant associations from datasets.

Bayesian rule mining has also been used in classification and time series modelling to enhance interpretability and accuracy. Beygelzimer et al. \cite{Dominique2012} introduced a Bayesian framework for classification rule mining in quantitative databases, which defines a model space and prior distribution over classification rules. These mined rules then serve as features for classifiers. This approach enables the efficient mining of locally optimal classification rules without the need for univariate preprocessing of attributes, demonstrating resilience to spurious patterns and improving predictive performance over traditional rule-based classifiers. Similarly, Wang et al. \cite{Wang2017BRL} introduced Bayesian rule sets (BRS), a machine learning algorithm that constructs classifiers composed of a small number of short rules. These models, structured in restricted disjunctive normal form, offer enhanced interpretability by producing concise rule sets that describe specific classes. The BRS framework incorporates user-defined priors to shape the model according to domain-specific interpretability requirements and has been effectively applied in contexts such as in-vehicle personalized recommender systems. Hüwel and Beecks \cite{Huwel2023} introduced a novel method which uses the Apriori algorithm to efficiently identify frequent kernel components in GPs for time series modelling. Their approach leverages the strengths of frequent itemset mining to uncover prevalent patterns in GP models, thereby facilitating a deeper understanding of the underlying data structures and enhancing the interpretability and analysis of complex time series data. 

While these Bayesian ARM methods demonstrate the potential of probabilistic frameworks, they predominantly rely on discrete Bayesian structures or belief networks, which limits their capacity to exploit continuous probabilistic modeling for capturing complex item dependencies. Moreover, no ARM methods have extensively explored the integration of Gaussian Processes (GPs), the formulation of a generalized Bayesian framework, or the application of MAB and RL techniques to enhance rule discovery. This work aims to bridge these gaps by e.g. introducing a GP-based mining method (GPAR) which incorporates continuous probabilistic modeling into ARM. Unlike previous probabilistic approaches that depend on discrete Bayesian structures or networks, GPAR represents items as feature vectors, employs GPs to model continuous latent variables capturing item correlations, and computes co-occurrence probabilities by sampling the marginalized joint posterior distribution. This GPAR methodology offers a principled approach to explicitly handle uncertainty, potentially improving the quality of mined association rules (e.g. identifying rare but significant associations), and enhancing the flexibility and robustness of rule discovery in complex datasets (e.g. reducing redundant rule generation). Extending GPAR, we build a general Bayesian framework (BARM) for ARM that principally integrates prior knowledge and uncertainty quantification to support a wide range of rule mining scenarios.

In parallel, reinforcement learning-based itemset mining has gained attention, notably in the generic itemset mining based on reinforcement learning (GIM-RL) framework \cite{GIM2022} which presents a unified framework where an agent, implemented as a DQN, is trained via reinforcement learning to extract diverse itemset types by iteratively adding or removing items, guided by a reward function reflecting the relevance to a specified target type. This approach demonstrates general effectiveness across various tasks, including the mining of high utility itemsets (HUIs), frequent itemsets (FIs), and association rules (ARs), and introduces a novel agent transfer mechanism to leverage knowledge from a source dataset for efficient mining on a related target dataset. We also employed the DQN approach in our RLAR method due to its simplicity and universality. Rather than general itemset construction, RLAR learns a task-specific, generalizable policy optimized for discovering high-quality association rules across a spectrum of support thresholds. In contrast to GIM-RL’s task-agnostic itemset extraction approach, RLAR focuses on adaptive rule optimization, prioritizes rare but significant rules, enhances rule diversity via managing uncertainty in action selection, and refrains from relying on transfer learning or extending beyond the ARM domain. Inspired by the reinforcement approach, we also crafted a multi-armed bandit (MAB) based and a Monte Carlo tree search (MCTS) method to provide adaptive exploration strategies that dynamically adjust to dataset characteristics and improve the efficiency of rule discovery processes.

\section{Association rule mining (ARM)} \label{sec:ARM}

Association rule (AR \footnote{More details about AR can be found in Appendix.\ref{app:AR}.}) mining is an unsupervised learning method used to uncover frequent patterns from unlabeled datasets, particularly in \textit{transactional data} such as supermarket purchase records. Introduced by Agrawal et al. \cite{Agrawal1993mining}, AR learning identifies relationships among items, typically represented as implications of the form \{$A$\} → \{$B$\}, where $A$, $B$ $\subseteq \mathcal{I}$, and $A$ $\cap$ $B$ = $\emptyset$. $A$ and $B$ are called \textit{itemsets} which can consist of one or more items. Each itemset is a collection of items, and each \textit{item} is represented by a binary attribute \footnote{Binary attributes provide a straightforward way to represent the presence or absence of items in transactions. This representation simplifies the formulation and interpretation of association rules, such as $\{i_j\} \Rightarrow \{i_k\}$, which indicates that the presence of item $i_j$ implies the presence of item $i_k$ in a transaction. Note that, real-world datasets can contain categorical or numerical attributes with multiple (more than two) possible values (e.g. a hierarchical structure or taxonomy). To apply traditional AR mining techniques to such data, these attributes are typically transformed into a binary format through a process called \textit{discretization} or \textit{binarization}, generating the \textit{extended} transactions. For example, a categorical attribute with multiple levels can be converted into multiple binary attributes, each representing the presence or absence of a \textit{specific} category. This transformation allows the use of standard AR mining algorithms, but can lead to an explosion in the number of attributes and potential loss of information. There are techniques, e.g. \cite{Srikant1997}, developed to handle hierarchical and quantitative attributes.} or variable. The largest itemset which contains all (unique) items is denoted as $\mathcal{I} = \{i_1, i_2, \ldots, i_M\}$, with $M=|\mathcal{I}|$ being the total number of unique items being considered. $\mathcal{I}$ and its sub itemsets are finite sets of literals \cite{Srikant1997}, e.g. $\mathcal{I}= \{i_1, i_2, \ldots, i_7\}$ = \{milk, bread, butter, eggs, beer, diapers, fruit\} \footnote{Note that, mathematically speaking, a \textit{set} has non-repeated and un-ordered elements. Therefore, the set $\mathcal{I}$ should be theoretically non-indexable.}. Each item $i_k$, e.g. $i_1=\text{milk}$, is a binary \textit{attribute} \footnote{\textit{Attribute} here refers to the fact that $i_k$ is an binary attribute (or variable) of the itemset $\mathcal{I}$, not an attribute of the item itself. In traditional AR, there is no notation of item attributes, we shall introduce \textit{features} to represent an item later in GPAR.}, meaning in any transaction, it is either present (1) or absent (0). For example, in market basket analysis, $i_k=1$ means item $k$ is purchased in this transaction. The implication \{$A$\} → \{$B$\} is referred to as an \textit{association rule}, where $A$ is the \textit{antecedent} (or left-hand side, LHS), and $B$ is the \textit{consequent} (or right-hand side, RHS). Every rule is composed by two different, disjoint itemsets. The statement \{$A$\} → \{$B$\} is often read as \textit{if} $A$ \textit{then} $B$, implying the \textit{co-occurrence} of the antecedent and consequent. For example, the rule \{bread\} → \{milk\} suggests that customers who buy bread also tend to buy milk. We denote $\mathcal{T} = \{\mathbf{t}_1, \mathbf{t}_2, \ldots, \mathbf{t}_N\}$ as a set \footnote{As the transaction database $\mathcal{T}$ can be arranged in time, maybe interpreting $\mathcal{T}$ as a \textit{DataFrame} or \textit{array} or even a \textit{dictionary} (with datetime being the key) makes it more meaningful.} of \textit{transactions} with $N=|\mathcal{T}|$ being the total number of transactions. Each transaction $t_k$ represents a subset of items purchased together - it can be represented either in literal form such as $t=$ \{milk, bread, fruit\} (a 3-itemset), or be encoded as a binary vector \footnote{Note, the literal representation of a transaction can have varying lengths across transactions, while in the database its encoded representation, i.e. the binary vector, has fixed length $M$ with ordered elements.}, e.g. $\mathbf{t}_1 = [1, 1, 0, 0, 0, 0, 1]$ assuming the items order \{milk, bread, butter, eggs, beer, diapers, fruit\}, which has the same length $M$ as the universe itemset $\mathcal{I}$. A transaction $t \in T$ satisfies \cite{Agrawal1993mining} an itemset $A$ if $t[k] = 1$ for all items $i_k$ in $A$.

The strength of a rule is evaluated using \textit{frequency-based} measures. Transactions data as an example, the \textit{support} of an itemset $A$ $\subseteq \mathcal{I}$ is defined as the proportion of transactions in $\mathcal{T}$ that contain $A$:
\[
\text{\textit{support}}(A) = \frac{|\{ t \in \mathcal{T} \mid A \subseteq t \}|}{|D|}
\]
and the \textit{confidence} of a rule \{$A$\} → \{$B$\} is the conditional probability that a transaction contains the consequent B given that it contains the antecedent $A$:
\[
\text{\textit{confidence}}(\text{\{$A$\} → \{$B$\}}) = p(B \mid A) = \frac{\text{\textit{support}}(A \cup B)}{\text{\textit{support}}(A)}
\]
and the \textit{lift} of a rule measures the strength of the implication relative to the expected co-occurrence of the antecedent and the consequent under independence:
\[
\text{\textit{lift}}(A \rightarrow B) = \frac{\text{\textit{confidence}}(A \rightarrow B)}{\text{\textit{support}}(B)} = \frac{\text{\textit{support}}(A \cup B)}{\text{\textit{support}}(A) \times \text{\textit{support}}(B)}.
\]

These metrics ensure that discovered rules are statistically significant: support quantifies prevalence, confidence captures reliability, and lift highlights deviation from independence (non-random associations). Notably, association rules uncover statistical correlations, not causal relationships, which is a critical distinction to avoid erroneous interpretations (e.g. concluding that bread purchases cause milk purchases).

To efficiently extract association rules from large datasets, several algorithms have been developed \footnote{More details about these 3 AR mining algorithms please see Appendix.\ref{app:AR}.}. The \textit{Apriori} algorithm \cite{Agrawal1994fast} uses a bottom-up approach, generating candidate itemsets and pruning those below a minimum support threshold. The \textit{FP-Growth} algorithm \cite{Han2000mining} constructs a compressed representation of the dataset using an FP-tree, enabling frequent itemset mining without candidate generation, which improves efficiency. The \textit{Eclat} algorithm \cite{Zaki2000scalable} employs a depth-first search strategy with vertical data representation, using transaction ID intersections for for frequent itemset discovery. These methods have broad applications, including market basket analysis for retail optimisation, disease-symptom pattern detection in healthcare, and personalized recommendations in web usage mining. However, traditional AR mining also presents challenges, such as rule redundancy, difficulty in selecting meaningful thresholds without domain knowledge, and the risk of mistaking statistical associations for causal insights, etc.

\section{GPAR: Gaussian process based ARM} \label{sec:GPAR}

Traditional AR mining identifies frequent itemsets and rules (e.g. $A$ → $B$) from transactional data based on frequency metrics such as support and confidence, without considering item similarities or uncertainty. GPAR reframes this unsupervised task as a supervised, probabilistic framework. GPAR leverages Gaussian processes (GPs) to model item relationships: each item $k$ from the item basket $\{1, 2, \dots, M\}$ is represented by a $d$-dimensional feature vector $\mathbf{x}_k \in \mathbb{R}^d$ which encode attributes such as size, shape, color, or price, etc. These vectors serve as inputs to a GP which outputs the continuous, scalar-valued, latent variables $z_k$ whose sign represents the membership of item $k$ in transactions. Inference is performed to obtain the GP posterior, incorporating prior knowledge about item similarities and transaction data; by sampling from the marginalised posterior, we can evaluate each itemset, i.e. estimating its co-occurrence probabilities (i.e. support) and confidence. By modeling the relation between input feature vectors $\mathbf{x}_k$ and the latent variable $z_k$ using GP, we craft the unsupervised data mining problem into a supervised one. In the following, we first briefly re-visit the inner workings of GP, and then describe the GPAR procedure in steps.

\subsection*{\textit{Gaussian process: preliminaries}}

Gaussian processes (GPs \cite{Rasmussen2004GP,Rasmussen2006GPbook}) are non-parametric, probabilistic models that define distributions over functions \footnote{More details about GP can be found in Appendix.\ref{app:GP_more_details}.}. A Gaussian process (GP) models a collection of random variables, any finite number of which have a joint Gaussian distribution. GPs are used for regression, classification, and uncertainty quantification. A GP is fully specified by its mean function $m(x)$, and covariance function (kernel) $k(\mathbf{x}, \mathbf{x}')$. For any finite set of input points $\{\mathbf{x}_1, \mathbf{x}_2, \ldots, \mathbf{x}_n\}$ with $\mathbf{x}_j \in \mathbb{R}^d$, the corresponding function values $\mathbf{f} = [f(\mathbf{x}_1), f(\mathbf{x}_2), \ldots, f(\mathbf{x}_n)]^\top$, where \footnote{Here for simplicity we use single-output GP, i.e. $f(\mathbf{x})$ is scalar-valued.} $f: \mathbb{R}^d \to \mathbb{R}$, are jointly distributed as a multivariate normal (MVN) distribution \cite{Rasmussen2006GPbook}:
\[
\mathbf{f} \sim \mathcal{N}(\mathbf{m}, K)
\]
where $\mathbf{m} = [m(\mathbf{x}_1), m(\mathbf{x}_2), \ldots, m(\mathbf{x}_n)]^\top$ is the mean vector \footnote{Zero mean prior, i.e. $\mathbf{m}(\mathbf{x}) = \mathbf{0}$, is commonly used.} and $K$ the covariance matrix with entries $K_{ij} = k(\mathbf{x}_i, \mathbf{x}_j)$. In GP, the kernel $k(\cdot,\cdot)$ acts as a covariance function, it captures the similarity between inputs, enabling GP to model non-linear relationships \footnote{One can choose a kernel function with known properties to reflect prior information, e.g. smoothness or periodicity to better represent data (the kernel hyper-parameters are optimised when training GP).}. The kernel needs to be positive definite to ensure the resulting covariance matrix $K=[k(\mathbf{x}_i,\mathbf{x}_j)]_{i,j=1}^M$ is positive definite (i.e. the eigenvalues of the covariance matrix $K$ are all non-negative). The choice of kernel function $k(\mathbf{x}, \mathbf{x}')$ significantly impact its performance (e.g. item similarity, flexibility, expressiveness, scalability), as kernel encodes assumptions about the function's properties, such as smoothness and periodicity as well as prior assumptions about the data. Radial basis function (RBF) kernel \footnote{The RBF kernel is also termed the \textit{Gaussian} or \textit{squared exponential} (SE) kernel \cite{Rasmussen2006GPbook}.}, for example:
\begin{equation} \tag{cc.\ref{eq:RBF_kernel}}
    k(\mathbf{x}, \mathbf{x}') = \sigma_f^2 \exp\left(-\frac{\|\mathbf{x} - \mathbf{x}'\|^2}{2\ell^2}\right)
\end{equation}
where $\sigma_f^2$ is the magnitude and $\ell$ the length scale parameter - they are hyper-parameters controlling the function's variance and smoothness, respectively. The RBF kernel yields high similarity for close points, and decays as distance increases.

Given a chosen form of the mean and covariance functions \footnote{More generally, training a GP involves choosing different functional forms for mean and covariance functions as well as finding the optimal hyper-parameters of these functions and noise level.}, we want to  make inferences about all of the hyper-parameters (e.g. those defining mean, kernel and noise) in the light of the data \cite{Rasmussen2004GP}. Training a GP  is done via maximizing the \textit{log marginal likelihood} \footnote{This log marginal likelihood can also be obtained directly by observing $\mathbf{y} \sim \mathcal{N}(0, K + \sigma_n^2 I)$. Maximizing this log marginal likelihood trades off model fit and complexity, see details in Appendix.\ref{app:GP_more_details}.} using training data \cite{Rasmussen2006GPbook}:
\begin{equation} \tag{cc.Eq.\ref{eq:GP_log_likelihood}} \label{eq:GP_log_likelihood_cc}
\log p(\mathbf{y} \mid X, \boldsymbol{\theta}) = 
    -\frac{1}{2} (\mathbf{y} - \mathbf{m})^\top (K + \sigma_n^2 I)^{-1} (\mathbf{y} - \mathbf{m})
    - \frac{1}{2} \log |K + \sigma_n^2 I| 
    - \frac{n}{2} \log 2\pi
\end{equation}
where $K = k(X, X)$ is the noise-free kernel matrix, $K_y = K + \sigma_n^2 I \in \mathbb{R}^{n \times n}$ is the covariance matrix of the noisy observations, $\boldsymbol{\theta}$ denotes all kernel hyper-parameters (including noise variance $\sigma_n^2$). 

The hyper-parameters $\boldsymbol{\theta} = (\ell,\sigma_f^2,\sigma_n^2)$, in the case of RBF for example, are learned by maximizing the log marginal likelihood, which is typically performed using gradient-based methods (as analytical derivatives are available) without a separate validation set:
\[
\boldsymbol{\theta}^* = \arg\max_{\boldsymbol{\theta}} \log p(\mathbf{y} \mid X, \boldsymbol{\theta})
\]

After fitting the GP, we can make predictions. In regression, for example, given training data $\{(\mathbf{x}_i, y_i)\}_{i=1}^N$ with inputs $\mathbf{x}_i$ and response $y_i$, the goal is to predict the function value $f(\mathbf{x}_*)$ at a new input $\mathbf{x}_*$ with uncertainty estimates. Assuming a zero mean prior and incorporating Gaussian noise with variance $\sigma_n^2$, the \textit{predictive distribution} for $f(\mathbf{x}_*)$ is Gaussian with mean and variance given by \cite{Rasmussen2006GPbook}:
\begin{equation} \tag{cc.Eq.\ref{eq:posterior_predictive_signle_test_point}} \label{eq:posterior_predictive_signle_test_point_cc}
    \begin{aligned}
    \mu(\mathbf{x}_*) &= \mathbf{k}_*^\top (K + \sigma_n^2 I)^{-1} \mathbf{y}, \\
    \sigma^2(\mathbf{x}_*) &= k(\mathbf{x}_*, \mathbf{x}_*) - \mathbf{k}_*^\top (K + \sigma_n^2 I)^{-1} \mathbf{k}_*
    \end{aligned}
\end{equation}
where $\mathbf{k}_* = [k(\mathbf{x}_1, \mathbf{x}_*), k(\mathbf{x}_2, \mathbf{x}_*), \ldots, k(\mathbf{x}_n, \mathbf{x}_*)]^\top$ is the covariance vector between the training inputs and test input $\mathbf{x}_*$. $K$ is the covariance matrix of the training inputs, and $\mathbf{y} = [y_1, y_2, \ldots, y_n]^\top$ is the vector of observed labels. 

GPs supplies not only point predictions but also uncertainty quantification, providing a principled way to handle noise \cite{Rasmussen2006GPbook}; it also provides flexibility of encoding prior knowledge - we can flexibly choose mean and covariance functions to represent prior knowledge. This principled treatment of uncertainty and prior knowledge representation make GPs highly applicable in time series forecasting \cite{Huwel2023, Rasmussen2006GPbook}, active learning \cite{Rodrigues2014,Riis2022}, and Bayesian optimisation \cite{McIntire2016,Wang2024}. Given its capacity in handling uncertainty and encoding prior knowledge, we employ a GP in association rule mining to model how item attributes (e.g. size, color, shape, price, etc) influence co-occurrence probabilities \footnote{Note that, although GP naturally supplies uncertainty quantification, in the context of GPAR, we only use its posterior for probability estimation (no notion of uncertainty over the estimated probability for a generated rule); we are not interested in making and quantifying predictions as in \ref{eq:posterior_predictive_signle_test_point_cc}.}. By capturing the similarities between items through appropriate kernel functions, GPs can provide a flexible, probabilistic alternative to simple, empirical frequency-based models. 

\subsection{Data representations}
In GPAR, each item $k$ in the universal itemset $\mathcal{I} = \{1, 2, \dots, M\}$ is represented by a feature vector \footnote{To avoid confusion, we use \textit{feature} vector instead of \textit{attribute} vector here to describe the item, as attribute has been used in traditional AR to reflect that fact that each item $i_k$ is a binary attribute of the itemset $\mathcal{I}$ indicating membership or appearance of item $k$.} (a $d$-tuple) $\mathbf{x}_k \in \mathbb{R}^d$, a $d$-dimensional vector that encodes descriptive attributes of the item, such as size, shape, color, price, category, etc. For example, an item might have a feature vector like 
\[
\mathbf{x}_k = \left[
\overset{\text{size}}{0.1},\ 
\overset{\text{shape}}{0.5},\ 
\overset{\text{color}}{1},\ 
\overset{\text{price}}{0.9},\
...
\right]
\]
where each element is a normalized value corresponding to a specific feature. The distance between two feature vectors $\mathbf{x}_i$ and $\mathbf{x}_j$ is fed into a kernel whose output (scalar-valued) quantifies the dissimilarity (covariance) between items $i$ and $j$.

The collection of all feature vectors for the $M$ items forms the feature matrix $X = [\mathbf{x}_1, \mathbf{x}_2, \dots, \mathbf{x}_M]^\top \in \mathbb{R}^{M \times d}$. This matrix serves as the input (often referred to as input indices) to the GP model (specifically, the kernel function), based on which the kernel function generates the covariance matrix for the GP model. The mean vector (assumed to be zero in our case) and the covariance matrix fully define the GP model and the distribution of the GP outputs - the continuous latent variables that capture underlying patterns among items.

Transactions in GPAR are represented similarly to traditional AR mining. Each transaction $\{\mathbf{t}_j\} \subseteq \mathcal{I}$ is encoded as a binary vector $\mathbf{t}_j = [i_1, i_2, \dots, i_M]^\top$, where $i_k \in \{0, 1\}$ indicates the presence ($i_k = 1$) or absence ($i_k = 0$) of item $k$ in transaction $\mathbf{t}_j$. These binary membership variables are linked to latent continuous variables $z_k$, one for each item $k$, through a quantization process \footnote{This quantization strategy converts continuous outputs into binary labels and is widely used across machine learning applications. For example, in binary linear classifiers like generalized linear models (GLMs), the predicted label is computed as \(\text{\textit{sign}}(\text{GLM}(\mathbf{x}))\), where \(\text{GLM}(\mathbf{x})\) generates a continuous score that is thresholded by the sign function \cite{Boyd_Vandenberghe_2018}. Similarly, in sigmoid-based binary classification, the label is determined by \(\text{sign}(\sigma(\mathbf{w}^\top \mathbf{x} + b) - 0.5)\), where \(\sigma(\cdot)\) is the sigmoid function producing a probability, and 0.5 acts as the standard decision threshold. This technique also extends to neural networks, where quantization is applied to binarize continuous activations or weights, enhancing efficiency or enabling discrete predictions.}: $i_k = \mathbb{I}(z_k > 0)$, where $\mathbb{I}(\cdot)$ is the indicator function ($\mathbb{I}(x) = 1$ if $x > 0$, and 0 otherwise). $\mathcal{T} = \{\mathbf{t}_1, \mathbf{t}_2, \dots, \mathbf{t}_N\}$ is the collection of all observed transactions $\mathbf{t}_j$, where $N$ is the total number of transactions. $\mathcal{T}$ is therefore a label or response matrix, i.e. $\mathcal{T} \in \mathbb{R}^{N \times M}$.

\subsection{The GP model, parameter estimation and posterior inference}
The vector of latent variables $\mathbf{z} = [z_1, z_2, \dots, z_M]^\top$ is modeled as a multivariate Gaussian \footnote{GPAR uses a static GP with static mean and covariance functions. We are not fitting a multivariate normal (MVN) distribution to the data, but over $z_k$, with their covariances produced by the kernel which consumes item attributes.} with a GP prior: $\mathbf{z} \sim \mathcal{N}(\mathbf{0}, K)$. The covariance matrix $K \in \mathbb{R}^{M \times M}$ encodes item dependencies: diagonal elements $K[i, i]$ reflect the model's confidence in predicting the presence of item $i$; off-diagonal elements $K[i, j]$ capture the joint influence and co-occurrence with other items. This covariance matrix is defined by a kernel function, e.g. the radial basis function (RBF) kernel $K_{M \times M} [i, j] = k(\mathbf{x}_i, \mathbf{x}_j) = \sigma_f^2 \exp\left( -\frac{\|\mathbf{x}_i - \mathbf{x}_j\|^2}{2\ell^2} \right)$ in which $\sigma_f^2$ represents the variance magnitude, and $\ell$ is the length scale. This kernel encodes item similarity: a smaller distance between feature vectors $\mathbf{x}_i$ and $\mathbf{x}_j$ results in a higher covariance $K_{ij}$, which may suggest a greater likelihood of items $i$ and $j$ co-occurring in transactions \footnote{Whether similar items, e.g. milk from different brands, co-occur frequently in transactions remains questionable though. There are arguments about whether supermarkets should place items similar to each other together ('taxonomy'), or they should make strategic placement based on consumer behaviour evidence such as put the 'beer and diapers' \cite{Berry1997Book}, which could be very different based on their feature vectors, together. Of course, one can take any empirically implied associations (from transaction history) as a factor into the feature vector.}. There are other kernels \cite{Rasmussen2006GPbook} such as linear, Ornstein–Uhlenbeck, Matérn, periodic, rational quadratic, etc, available to model different phenomena.

To obtain the values of kernel parameters, we maximize the log-likelihood \footnote{It is called log \textit{marginal} likelihood which emphasizes the non-parametric nature of GP \cite{Rasmussen2004GP}.} of all transactions $\mathcal{T} = \{\mathbf{t}_1, \mathbf{t}_2, \dots, \mathbf{t}_N\}$:
\begin{equation} \label{eq:GP_training_MLE}
    \boldsymbol{\theta} = \argmax_{\boldsymbol{\theta}} \log p(\mathcal{T} \mid X, \boldsymbol{\theta})
\end{equation}
where \footnote{Compared to (\ref{eq:GP_log_likelihood_cc}), here we have also assumed zero mean for the GP.}
\begin{equation*}
\log p(\mathcal{T} \mid X, \boldsymbol{\theta}) =
\sum_{j=1}^N \log p(\mathbf{t}_j \mid X, \boldsymbol{\theta}) = \sum_{j=1}^N \left[
    -\frac{1}{2} \mathbf{t}_j^\top K_t^{-1} \mathbf{t}_j 
    - \frac{1}{2} \log |K_t| 
    - \frac{d}{2} \log 2\pi
    \right]
\end{equation*}
where $\log p(\mathcal{T} \mid X, \boldsymbol{\theta})$ is the log-likelihood of the transaction data $\mathcal{T}$ given the feature matrix $X$ and parameters $\boldsymbol{\theta}$, and the likelihood $p(\mathbf{t}_j | \mathbf{z})$ is approximated\footnote{This approximation is important for computational feasibility. For a transaction $\mathbf{t}_j$, $\mathbf{t}_j$ indicates which items are present. The likelihood $p(\mathbf{t}_j | \mathbf{z})$ tells how likely the observed transaction $\mathbf{t}_j$ (a binary vector showing which items are present) is, given the latent variables $\mathbf{z}$. Since $\mathbf{t}_j$ is binary and $\mathbf{z}$ is continuous, computing the exact likelihood $p(\mathbf{t}_j | \mathbf{z})$ would require integrating over all possible $\mathbf{z}$ configurations that match $\mathbf{t}_j$, i.e. $p(i_k = 1 | z_k) = p(z_k > 0)$, which, for a Gaussian $z_k \sim \mathcal{N}(0, K_{kk})$, is 0.5 without conditioning, but the joint distribution and correlations in $K$ make exact computation complex, so we treat $p(\mathbf{t}_j | \mathbf{z})$ as proportional to the probability density of $\mathbf{z}$ under the Gaussian process, conditioned on $z_k > 0$ for $k \in \mathbf{t}_j$.} via $z_k$. $K_t = K + \sigma_n^2 I \in \mathbb{R}^{n \times n}$ is the covariance matrix of the noisy observations; $K = k(X, X)$ is the noise-free kernel matrix generated by the kernel function with chosen parameters $\sigma_f$ and $\ell$. $\boldsymbol{\theta}=(\sigma_f,\ell,\sigma_n)$ denotes all kernel hyper-parameters to be optimised. The overall log-likelihood is decomposed into summation of the log likelihood at each transaction, as we assume independence between transactions $\mathbf{t}_i$ and $\mathbf{t}_j$ given $X$ and $\boldsymbol{\theta}$. The quadratic term $-\frac{1}{2} \mathbf{t}_j^\top K_t^{-1} \mathbf{t}_j$ measures the 'distance' of $\mathbf{t}_j$ from the mean (zero) under the covariance $K_t$, while the normalization term $-\frac{1}{2} \log |K_t|$ involving the determinant of the covariance matrix serves as a penalty for complexity.

For each transaction, we set $z_k = 1$ where items are present and 0 elsewhere. The maximum likelihood optimisation involves plugging the feature matrix $X$ and the label data $\mathbf{t}_j, i=1,2,...,N$, into the above formula to evaluate the overall likelihood at some chosen kernel parameters; the choice of kernel parameters can be guided by e.g. gradient descent. One notes that, calculating the overall likelihood matrix inversion, which costs $\mathcal{O}(M^3)$ (see \textit{GP training} in Appendix.\ref{app:GP_more_details}), making this optimisation process computationally intensive for large $M$. After estimating $\sigma$ and $\ell$, the posterior distribution over $\mathbf{z}$, conditioned on the response matrix $\mathcal{T}$ and the feature matrix $X = [\mathbf{x}_1, \mathbf{x}_2, \dots, \mathbf{x}_M]^\top$, is derived: $\mathbf{z} \sim \mathcal{N}(\mathbf{0}, K_t)$ in which $K_t$ is obtained using the updated covariance function, enabling probabilistic inference.

\subsection{Rule evaluation using marginalised posterior and Monte Carlo sampling}
After parameter inference, conventional uses of GP would utilize the trained GP to make predictions at new data points and propagate parameter uncertainties into predictions. In the context of GPAR, aligning with traditional AR target, we are interested in sampling from the marginal posterior to obtain probabilistic estimates of rules (e.g. likelihood of frequent itemsets) based on \textit{existing} items. To mine rules, we examine all \footnote{Enumerating all possible combination of items from $M$ items costs $\mathcal{O}(2^M)$. This is same for traditional AR and GPAR.} possible itemsets $I \neq \phi$ and $I \subseteq \mathcal{I}$; for each itemset $I$, GPAR estimates its co-occurrence probability (the joint probability):
\begin{equation} \label{eq:rule_inference_via_integral_of_marginalised_joint}
    p(I) = p(\forall k \in I, z_k > 0) = \int_0^\infty \cdots \int_0^\infty \mathcal{N}(z; 0, K_I) dz
\end{equation}
where $K_I$ is the submatrix of the overall covariance matrix $K_{M \times M}$, corresponding to items in $I$. Eq.\ref{eq:rule_inference_via_integral_of_marginalised_joint} gives the probability $p(I)$ that all items in the itemset are present (i.e. all their latent variables are positive), leveraging the fact that GP models the latent variables (and any subset of them) for the items as a multivariate normal distribution.

Computing a multivariate probability integral in Eq.\ref{eq:rule_inference_via_integral_of_marginalised_joint} is analytically intractable for more than a few dimensions \footnote{The joint probability $p(\text{all } z_k > 0 \text{ for } k \in I)$ requires integrating over a hyper-rectangle, which has no closed-form solution. See Appendix.\ref{app:marginal_posterior_sampling} for an extended discussion on this.}; it can be numerically computed via \textit{Monte Carlo sampling} from the marginal posterior $\mathbf{z}_I \sim \mathcal{N}(\mathbf{0}, K_I)$, which leads to the approximation 
\begin{equation} \label{eq:GPAR_marginal_posterior_MC_approx}
    p(I) \approx \frac{1}{S} \sum_{s=1}^S \mathbb{I}(\mathbf{z}_{s,I} > 0)
\end{equation}
given $S$ samples $\mathbf{z}_{s,I}$. With a user-specified threshold $min\_prob$, if $p(I) > min\_prob$, rules $A$ → $B$ (where $A \cup B = I$, $A \cap B = \emptyset$) are generated. We further examine the confidence $\text{conf}(A \rightarrow B) = p(I) / p(A)$, if $\text{conf}(A \rightarrow B) > min\_conf$, then the rule $A \rightarrow B$ is identified as frequent (or significant). The complexity of enumerating all itemsets and sampling scales exponentially \textit{w.r.t.} no. of items $M=|\mathcal{I}|$ as $\mathcal{O}(2^M (M^3 + S M^2))$, limiting GPAR to small to medium datasets (e.g. $M \leq 15$, see later complexity analysis). 

\subsection{Rule generation beyond existing items}
In addition to the capabilities of traditional association rule mining, GPAR offers an extra bonus usage: the trained GP model enables probabilistic inference \footnote{This is termed \textit{posterior predictive} in e.g. probabilistic programming languages, which makes predictions at new data point and propagates parameter uncertainties to predictions.} for new rules involving items outside the universal itemset \(\mathcal{I}\). In traditional AR mining, the analysis is restricted to a fixed universal itemset $\mathcal{I} = \{1, 2, \dots, M\}$, which contains all items present in the dataset \footnote{In traditional AR mining, the task is to mine all frequent rules in an unsupervised manner, so there is no training and test split of the data.}. This means that rules (e.g. $A \rightarrow B$, where $A$ and $B$ are subsets of $\mathcal{I}$) can only be generated for items already observed in the data. If a previously unobserved, new item is introduced, e.g. a new product is imported to a store, the entire dataset must be reprocessed to refresh the rule-mining framework, which can be computationally costly and impractical in dynamic settings. In contrast, GPAR mining provides a powerful advantage: once trained, the GP model can seamlessly extend to new items without reprocessing the entire dataset, making GPAR uniquely valuable for real-world applications such as retail and e-commerce where new items are regularly added, offering flexibility and efficiency beyond the limitations of traditional AR mining. For example, if a store adds a new type of coffee to its inventory, GPAR can use its feature vector (e.g. price, brand, roast type, etc) to infer rules such as 'if a customer buys tea, they might also buy this new coffee' without needing to collect new transaction data or retrain the model.

GPAR’s ability to handle new items stems from its use of a Gaussian process - a probabilistic model that leverages \textit{feature-based} representations of items and their relationships. Each item $k \in \mathcal{I}$ in the dataset is associated with a feature vector $\mathbf{x}_k \in \mathbb{R}^d$, which captures attributes such as size, price, category, or other relevant properties. These feature vectors serve as the inputs to the GP model, which learns to associate them with latent variables $z_k$ that indicate the presence (membership) or importance of each item in transactions \footnote{GPAR compute the co-occurrence (joint) probabilities directly rather than individual predictions. In classical GP regression or classification, the goal is to predict a target variable $z_k$ given a feature vector $\mathbf{x}_k$. The GP learns a distribution over possible functions that map inputs to outputs, allowing for predictions of $z_k$ for new inputs. This is useful for tasks where independent predictions are sufficient (e.g. stock price prediction, risk estimation, etc.). AR mining aims to understand dependencies and co-occurrences, rather than individual occurrence. Therefore, GPAR doesn't predict individual $z_k$ in isolation; it  models the joint distribution of these latent variables, which allows it to estimate the co-occurrence probability of items appearing together in transactions. This is a departure from classical GP usage.}. During training, the GP uses a kernel function (e.g. RBF or custom designed) to capture similarities between the feature vectors of existing items. Kernel-based similarity is core to GP’s reasoning - it quantifies how similar the new item is to existing items based on their feature vectors. For example, when a new item $M+1$, which is outside the original set $\mathcal{I}$, is introduced, it is assigned a feature vector $\mathbf{x}_{M+1}$ based on its attributes, and a covariance vector of length $M$ linking the new item to the trained model is computed. If the new item $\mathbf{x}_{M+1}$ is close to the feature vectors of items that frequently co-occur in transactions, the GP will assign $z_{M+1}$ a distribution that reflects this similarity, suggesting that the new item is likely to follow similar co-occurrence patterns. To predict the co-occurrence of items in the augmented itemset universe $\mathcal{I} = \{1, 2, \dots, M, M+1\}$, we can augment the trained covariance matrix $K_{M \times M}$ by incorporating this covariance vector, which enables use to evaluate the probability of new rules via Monte Carlo sampling with the new covariance matrix \footnote{Similar to classic GP usage, a trivial task is to plug this covariance vector into its posterior prediction formulas (Eq.\ref{eq:posterior_predictive_signle_test_point}) to predict the (Gaussian) distribution of the latent variable $z_{M+1}$ for the new item, which probabilistically predicts the individual value of $\mathbf{z}_{M+1}$.}, i.e. sampling from the joint posterior distribution of the latent variables $(z_1,z_2,...,z_M,z_{M+1})$. This similarity-driven approach ensures that the inferences are grounded in the learned relationships among the original items. The generalization capacity \footnote{This is a general property of supervised algorithms.} of GP also ensures consistency of informaiton flow. As a non-parametric model, GP does not assume a fixed functional form for the relationships between items. Instead, it relies on the data-driven covariance structure defined by the kernel. This allows the GP to make predictions for unseen inputs (e.g. the feature vector of a new item) by extrapolating from the patterns observed in the training data. As long as the new item’s feature vector $\mathbf{x}_{M+1}$ can be compared to the existing ones, the GP can infer its latent variable $z_{M+1}$ and integrate it into the rule-mining framework. Therefore, the accuracy of GPAR rule inference depends on two factors: feature relevance and kernel choice. The feature vectors must meaningfully represent the properties that drive co-occurrence patterns (i.e. ensuring items likely to co-occur are close in the new feature space, even if they’re dissimilar in physical attributes), and the kernel function must effectively capture similarities that align with the data’s underlying structure. We discuss kernel design in Section.\ref{sec:custom_kernel_design}. 

\subsection{The GPAR algorithm} \label{sec:GPAR_algorithm}
The full GPAR algorithm is presented in Algo.\ref{algo:GPAR}. We discuss its computational and memory complexity, as well as comparison with classic AR methods.

\begin{algorithm}[H]
\footnotesize
\caption{GPAR: Gaussian process-based association rule mining}
\label{algo:GPAR}
\textbf{Input:} a set of items $\{1, 2, \dots, M\}$, each represented by a $d$-dimensional feature vector $\mathbf{x}_i$, forming an items feature matrix $X \in \mathbb{R}^{M \times d}$; a set of transactions $\mathcal{T} = \{\mathbf{t}_1, \mathbf{t}_2, \dots, \mathbf{t}_N\}$, where each $\mathbf{t}_j \in \mathbb{R}^M$; a kernel function $k(\cdot, \cdot)$, e.g. RBF; a minimum probability threshold $min\_prob$; a minimum confidence threshold $min\_conf$; number of Monte Carlo samples $S$. \\
\textbf{Output:} a set of association rules $\mathcal{R}$ where each rule $r \in \mathcal{R}$ is of the form $A$ → $B$ with $A, B \subseteq \{1, 2, \dots, M\}$, satisfying the probability and confidence thresholds.

\vspace{1mm}\hrule\vspace{1mm}

\begin{algorithmic}[1]
\STATE Compute kernel matrix $K \in \mathbb{R}^{M \times M}$ with $K[i,j]=k(\mathbf{x}_i,\mathbf{x}_j)$, $\mathbf{x}_i,\mathbf{x}_j\in X$. \hfill\textit{$\mathcal{O}(M^2 d)$}
\STATE optimise kernel parameters by minimizing the negative log-likelihood of the transactions $\mathcal{T}$ over $T$ iterations. E.g. length scale $\ell = \arg\min_{\ell} -\log p(\mathcal{T} | X, \ell)$. \hfill\textit{$\mathcal{O}(T (M^2d + NM^2 + M^3)$}
\STATE Update the covariance matrix $K=\text{RBF\_kernel}(X,\ell)$ with the estimated length scale $\ell$. \hfill\textit{$\mathcal{O}(M^2 d)$}
\STATE Initialize an empty set $\mathcal{R}=\emptyset$ to store association rules. \hfill\textit{$\mathcal{O}(1)$}
\FOR{$m=2,\dots,M$}
    \STATE Generate all possible size $m$ itemsets $I_m = \binom{\{1, 2, \dots, M\}}{m}$. \hfill\textit{$\mathcal{O}(\mathcal{C}_m^M)$}
    \FOR{itemset $I\in I_m$}
        \STATE Estimate co-occurrence probability using Monte Carlo sampling:
        $p(I)=\frac{1}{S}\sum_{s=1}^{S}\mathbb{I}(\mathbf{z}_s>0)$, $\mathbf{z}_s\sim\mathcal{N}(0,K_I)$. \hfill\textit{$\mathcal{O}(m^3+S m^2)$}
        \IF{$p(I)>min\_prob$}
            \FOR{each \textit{(antecedent,consequent)} split $(A,B)$ of $\mathcal{I}$}
                \STATE Estimate the antecedent probability:
                $p(A)=\frac{1}{S}\sum_{s=1}^{S}\mathbb{I}(\mathbf{z}_s>0)$, $\mathbf{z}_s\sim\mathcal{N}(0,K_A)$. \hfill\textit{$\mathcal{O}((m-1)^3+S(m-1)^2)$}
                \STATE Compute confidence: $\text{conf}(A\rightarrow B)=\frac{p(I)}{p(A)}$. \hfill\textit{$\mathcal{O}(1)$}
                \IF{$\text{conf}(A\rightarrow B)>min\_conf$}
                    \STATE Add rule $A\rightarrow B$ to $\mathcal{R}$. \hfill\textit{$\mathcal{O}(1)$}
                \ENDIF
            \ENDFOR
        \ENDIF
    \ENDFOR
\ENDFOR
\STATE Return the rules set $\mathcal{R}$. \hfill\textit{$\mathcal{O}(1)$}
\end{algorithmic}
\end{algorithm}

\paragraph{Computational complexity}
Suppose we have $M$ unique items (e.g. cake, bread, milk, beer, etc), each item can be represented by a $d$-dimensional vector (e.g. shape, size, color, etc), this setting gives the itemset matrix $IS_{M \times d}$ and covariance matrix $\Sigma_{M \times M}$. We also have $N$ transaction records, each transaction record is a vector of length $M$ with each element representing the binary value of the scalar $z_k$ (i.e. the membership of item $\mathcal{I}$).

For GP inference, the major computation lies in operations such as kernel computation and inversion of the covariance matrix. Kernel computation involves calculating the pairwise item distance, which takes $\mathcal{O}(M^2d)$. Optimizing the kernel hyper-parameters (e.g. length scale and magnitude in RBF) using e.g. \textit{L-BFGS-B} \cite{Zhu1997BFGS} involves searching through the parameter space, computing the covariance and evaluating corresponding log likelihood (inverting the covariance matrix), over $T$ iterations the optimisation process takes $\mathcal{O}(T(M^2d + NM^2 + M^3))$ - the first refers to the cost of computing $\Sigma_{M \times M}$ using pairwise kernel distance, the second refers to the cost of plugging in the $N$ transactions data, each is a $M$-length membership vector, the third term corresponds to the cost induced by the one time LU (Cholesky) decomposition of the covariance matrix $\Sigma_{M \times M}$ in each iteration.

After fitting the GP, we evaluate the probability of different mining rules, i.e. rule inference. The number of total combinations of rules to mine (i.e. total number of itemsets) is $\sum_{m=2}^M \mathcal{C}^M_m = 2^M - M -1$, which is $\mathcal{O}(2^M)$. For each itemset of size $m \in (1,M]$, we compute its probability of co-occurrence \footnote{In our implementation, we further identify its antecedent set of size $m-1$, which further induces a factor of $m$ to the cost. We can set a marginal probability threshold to reduce the number of this calculation though.}, i.e. the marginalised joint probability. We use \textit{Monte Carlo sampling} \footnote{For standard Gaussian distributions, if $m = 1$, $p(z_i > 0) = 0.5$; if $m = 2$, the standard bivariate Gaussian probability $p(z_i > 0, z_j > 0) = \frac{1}{4} + \frac{1}{2\pi} \arcsin(\rho)$, with $\rho$ being the correlation coefficient. For 3 or more dimensions, the multivariate normal CDF is not analytically tractable; numerical methods such as MC sampling, quadrature methods \cite{Hennig_Osborne_Kersting_2022}, etc, need to be sought.} to generate $S$ samples from a sub multivariate Gaussian $p(\mathbf{x}; \boldsymbol{\mu}_{m \times 1},\boldsymbol{\Sigma}_{m \times m})$ with Cholesky decomposition of $\boldsymbol{\Sigma}_{m \times m}$ \cite{Rasmussen2006GPbook}, which costs $\mathcal{O}(m^3 + Sm^2)$. Therefore, the total cost \footnote{In Algo.\ref{algo:GPAR}, there is a subsequent step for splitting the itemset into (antecedent,consequent), which induces similar cost $\mathcal{O}((m-1)^3+S(m-1)^2)$.} is $\mathcal{O}(\sum_{m=2}^M \mathcal{C}^M_m \times (m^3 + Sm^2)) \leq \mathcal{O}(2^M (M^3 + S M^2))$, which dominates \footnote{Indeed, our later experiments in Section.\ref{sec:GPAR_experiments} empirically verify that the largest cost of GPAR lies in the Monte Carlo sampling when performing rules inference.} the overall cost of GPAR due to the factor $\mathcal{O}(2^M)$.

The overall complexity, combining GP inference and marginal density evaluation, is $\mathcal{O}(M^2d + T(M^2d + NM^2 + M^3) + \sum_{m=2}^M \mathcal{C}^M_m \times (m^3 + Sm^2))$, which is roughly dominated by the last term $\mathcal{O}(2^M (M^3 + SM^2))$. We observe that, the complexity is linear in the number of transactions $N$, number of optimisation iterations $T$, and dimension of item vector $d$, but has exponential dependence on the number of items $M$ to be considered. That is, GPAR is extremely sensitive to $M$, for $M=20$, for example, the computational cost becomes infeasible as it contains mining more than $2^{20} \approx$ 1 million itemsets. While the item representation dimension $d$ only appears in kernel computation, which is relatively minor, the number of transactions $N$ mainly affects the log likelihood evaluation, which is $\mathcal{O}(TNM^2)$ - it can be a bottleneck for very large $N$, but not as impactful as $M$. The number of items $M$ has impact on the overall cost through the exponential factor $2^M$ and the cubic term $M^3$. The number of Monte Carlo samples $S$ used to evaluating the marginal density may increase exponentially as  dimension due to \textit{curse of dimensionality}, if high accuracy is demanded \footnote{One may use approximate sampling to reduce $S$.}.

As a comparison, the Apriori implementation generally costs $\mathcal{O}(NM)$ for binary matrix conversion, and $\mathcal{O}(2^M)$ for itemset enumeration \footnote{Similar to GPAR, the number of rules to mine can be reduced using a minimum support threshold.} (detailed complexity analysis can be found in Appendix.\ref{app:traditional_AR_mining_algos}). As both Apriori and GPAR both suffer from the exponential growth of items (which is inevitable without further actions), to make a fair comparison, we ignore the factor $2^M$ and just refer to the complexity of GPAR as $\mathcal{O}(M^3 + SM^2)$, and the complexity of Apriori as $\mathcal{O}(NM)$. We see that Apriori is more scalable for large number of items $M$ due to pruning, whereas GPAR evaluates all itemsets without early termination, making it less practical. In general, we expect GPAR to be feasible only for small $M$ (e.g. $M \leq 15$).  

\paragraph{Memory usage} 
The covariance matrix $K \in \mathbb{R}^{M \times M}$, where $M$ is the number of items, requires $\mathcal{O}(M^2)$ space, dominating the memory footprint due to its quadratic growth; For example, with $M = 15$, this matrix occupies approximately $15 \times 15 \times 8 = 1,800$ bytes (assuming 8 bytes per float). The transaction data $\mathcal{T} \in \mathbb{R}^{N \times M}$, with $N$ transactions, consumes $\mathcal{O}(NM)$ space, which scales linearly with both $N$ and $M$; for $N = 1,000$ and $M = 15$, this amounts to $1,000 \times 15 \times 8 = 120000$ bytes. During Monte Carlo sampling, the algorithm generates $S$ samples for each itemset of size $m$, with a sub-covariance matrix $K_I \in \mathbb{R}^{m \times m}$ and samples $\mathbf{z}_s \in \mathbb{R}^m$, requiring up to $\mathcal{O}(m^2 + Sm)$ space per itemset evaluation; for a maximum $m = M$, $S = 1,000$, and $M = 15$, this is approximately $(15^2 + 1,000 \times 15) \times 8 = 121,800$ bytes per itemset. However, as itemsets are evaluated sequentially, only one itemset’s samples are stored at a time, capping this cost at $\mathcal{O}(M^2 + SM)$. Additionally, the set of association rules $\mathcal{R}$ grows with the number of rules, potentially up to $\mathcal{O}(2^M)$, but each rule (storing indices of items) is small, typically $\mathcal{O}(M)$ per rule, making its contribution less significant compared to $K$ and $\mathcal{T}$. Overall, GPAR’s memory usage is dominated by the kernel matrix ($\mathcal{O}(M^2)$) and transaction data ($\mathcal{O}(NM)$), with Monte Carlo sampling adding a manageable $\mathcal{O}(M^2 + SM)$ per itemset; while feasible for small $M \leq 15$, the quadratic and linear scalings with $M$ and $N$ respectively can become prohibitive for larger datasets, necessitating careful memory management or optimisation strategies.

\paragraph{Compare with traditional AR mining method}
A comparison of GPAR and traditional association rule mining is made in Table.\ref{tab:gpar_comparison}. While GPAR offers advanced probabilistic modeling, its scalability and interpretability lag behind traditional methods for large datasets. The interpretability of GPAR is worse than traditional association rule mining because GPAR generates probabilistic rules based on latent variables (i.e. using GP to model the membership variables $z_k$), which are more complex and less intuitive. In contrast, traditional methods produce straightforward \textit{if-then} statements (correlation, not causation) derived from frequency counts, making them easier to understand.

\begin{table}[h]
\centering
\tiny
\caption{Comparison: Traditional Association Rule Mining \textit{vs} GPAR}
\label{tab:gpar_comparison}
\setlength{\tabcolsep}{3pt}
\renewcommand{\arraystretch}{1.1}
\begin{threeparttable}
\begin{tabular}{p{3cm} p{4.2cm} p{4.2cm}}
\toprule
\textbf{} & \textbf{Traditional AR Mining} & \textbf{GPAR with Feature Vectors} \\
\midrule
\textbf{Data Type} & Discrete (e.g. binary transactions) & Discrete, modeled via continuous latent variables \\
\textbf{Prior Knowledge} & Minimal, data-driven & High, can encode in kernel function \\
\textbf{Uncertainty Handling} & Limited, based on frequency counts & High, probabilistic framework \\
\textbf{Scalability} & High with optimized algorithms (Apriori, FP-Growth, Eclat) & Low, computationally complex for large $M$ \\
\textbf{Inference of Unobserved Itemsets} & Re-processing, high-cost & Direct marginalisation, low-cost \\
\textbf{Interpretability} & High, rules are if-then statements & Lower, probabilistic modeling of latent variables \\
\textbf{Computational Cost} & Linear to quadratic in transactions & Cubic in number of items for inference \\
\bottomrule
\end{tabular}
\begin{tablenotes}
\item[1] Both are subject to exponential growth of itemsets, although pruning can reduce this complexity.
\end{tablenotes}
\end{threeparttable}
\end{table}

\subsection{Custom kernel design} \label{sec:custom_kernel_design}

The choice of kernel significantly impacts item similarity representation, flexibility in modeling complex data patterns, and scalability to large datasets. Conventional kernels, such as the radial basis function (RBF), assume smooth item similarity, but may fail to adequately capture nuanced relationships (e.g. complementary \textit{vs} substitute items) important for effective rule mining. For example, the RBF kernel $k(\mathbf{x}_i, \mathbf{x}_j) = \exp\left( -\frac{\|\mathbf{x}_i - \mathbf{x}_j\|^2}{2\ell^2} \right)$ hints that items with similar feature vectors (i.e. small $\|\mathbf{x}_i - \mathbf{x}_j\|$) have high covariance, implying that their latent variables are highly correlated and thus likely to co-occur in transactions. However, this implication doesn’t always hold in practice - consider these two scenarios: (1) \textit{Substitute} products. Two milk products from different brands might have very similar feature vectors (e.g. nutritional content, packaging, etc), making them close in feature space (if these features are used to encode the products). But because they serve the same purpose, people are unlikely to buy both in a single transaction, meaning their co-occurrence should be low, not high as indicated by the output of an RBF kernel. (2) \textit{Complement} products. The classic 'beer and diapers' example shows that dissimilar items (in terms of features like product category or use case) can co-occur frequently due to complementary needs, despite being far apart in feature space. In order to make these component products 'close', as measured by a RBF kernel, to reflect the fact that they may co-occur frequently, we have to carefully craft the features, which essentially encodes prior human knowledge or experiences into data representation. In both cases, the RBF kernel’s behavior (i.e. covariance decreases as distance increases) works against capturing the true transaction patterns. This misalignment suggests we need a kernel that better reflects co-occurrence dynamics.

For AR mining use, we would an expect an ideal kernel to sharply peak at certain optimal distance $d_0$ centered at the most frequent items (which captures the frequent co-occurrence of complementary items), and exhibit some inverse-distance behaviour, i.e. similar items co-occur less frequently in a purchase. We aim to design a valid PSD kernel reflecting these requirements. In the following, we propose several kernel candidates that are both valid and invalid. An extended discussion on kernel design can be found in Appendix.\ref{app:kernel_designs}. 

\paragraph{An exponetiated IMQ (EIMQ) kernel} Combining the spirits of the RBF and the inverse multiquadratic (IMQ, see Appendix.\ref{app:IMQ_kernel}) kernels, we trivially propose the EIMQ kernel:
\begin{equation} \label{eq:EIMQ_kernel_cc} \tag{cc.Eq.\ref{eq:EIMQ_kernel}}
k(\mathbf{x}_i, \mathbf{x}_j) = \exp \left(- \frac{1}{(\|\mathbf{x}_i - \mathbf{x}_j\|^2 + c^2 )^{\beta}} \right), \quad c > 0, \, \beta > 0
\end{equation}
where $c > 0$ is a small constant to avoid singularity (division by zero); $\beta$ controls the decay rate of the kernel. EIMQ encourages similar items (small distance) have lower covariance, while dissimilar items (larger distance) have higher covariance: the induced covariance approaches a finite limit of 1 as the distance between $\mathbf{x}_i$ and $\mathbf{x}_j$ increases; it approaches zero \footnote{To be exact, $k(\mathbf{x}_i, \mathbf{x}_j) \rightarrow e^{-c^{-2\beta}}$ as $\|\mathbf{x}_i - \mathbf{x}_j\|^2 \rightarrow 0$.} as the distance decreases. This behaviour directly addresses the issue with substitute items: similar milk products would have low covariance \footnote{This behaviour is similar to other kernel-based methods such as SVGD in which the kernel models the repulsive and attractive forces - strong repulsion at small distances in this case.}, reflecting their low likelihood of co-occurrence, while dissimilar items (e.g. beer and diapers) might co-occur more often \footnote{Of course this depends on feature representations.}. This intuition for dissimilar items aligns with what we expected for complementary items such as beer and diapers - although they are not infinitely dissimilar. Three examples of the EIMQ kernel are shown in Fig.\ref{fig:EIMQ_kernel_examples}). However, the positive semi-definite (PSD) property, as required to be a valid kernel, is satisfied (see Appendix.\ref{app:EIMQ_kernel} an analysis) for the EIMQ kernel, which prevents its use in GPAR\footnote{Symmetry and PSD are required for a kernel to be valid and useful in GP modelling. To determine whether a kernel is PSD, one can instantiate the kernel and check if the induced covariance matrix is PSD, i.e. for any finite set of inputs $\{ \mathbf{x}_1, \ldots, \mathbf{x}_n\}$, the kernel matrix $K$ with entries $K_{ij} = k(\mathbf{x}_i, \mathbf{x}_j)$ must be symmetric and PSD. Empirically checking a covariance matrix instance is not a principled way to verify kernel PSD but a straightforward approach to detect potential violation if it's not PSD.}. Further, Eq.\ref{eq:EIMQ_kernel} overgeneralizes by implying all dissimilar items have higher covariance, whereas some extremely dissimilar items might not co-occur at all.

\begin{figure}
    \centering
    \includegraphics[width=0.35\linewidth]{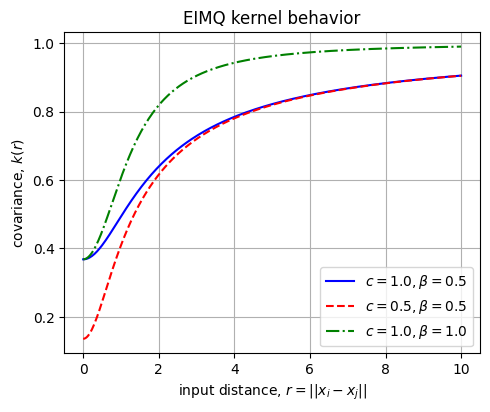}
    \caption{Instantiations of the EIMQ kernel.}
    \label{fig:EIMQ_kernel_examples}
\end{figure}

\paragraph{An absolute RBF kernel} To encourage the co-occurrence of complementary items which can be distant in their feature space, we inherit the spirit of RBF kernel and simply remove the negative sign from the exponent, making similarity grow with distance:
\begin{equation} \label{eq:absolute_RBF}
    k(\mathbf{x}_i, \mathbf{x}_j) = \exp\left( \frac{(\|\mathbf{x}_i - \mathbf{x}_j\|)^2}{2\ell^2} \right)
\end{equation}

This trivial idea based, absolute RBF kernel is not PSD in general - its outputs grow exponentially with distance, giving a kernel matrix that grow unbounded and fail the definiteness condition \footnote{As exponential is unbounded above, the matrix is prone to having large off-diagonal entries, which breaks PSD.}. A 1D example with $[x_1=0, x_2=1, x_3=2, x_4=3]$ is shown in Fig.\ref{fig:absolute_RBF_kernel}, in which negative eigenvalues breaks PSD - a Hermitian kernel matrix is positive semi-definite if and only if all its eigenvalues are non-negative. Its counterpart, i.e. a RBF kernel with same parameter setting ($\ell$=1.0), conversely demonstrates PSD with non-negative eigenvalues in Fig.\ref{fig:RBF}.

\begin{figure}[htbp]
    \centering
    \begin{subfigure}[b]{0.45\linewidth}
        \centering
        \includegraphics[width=\linewidth]{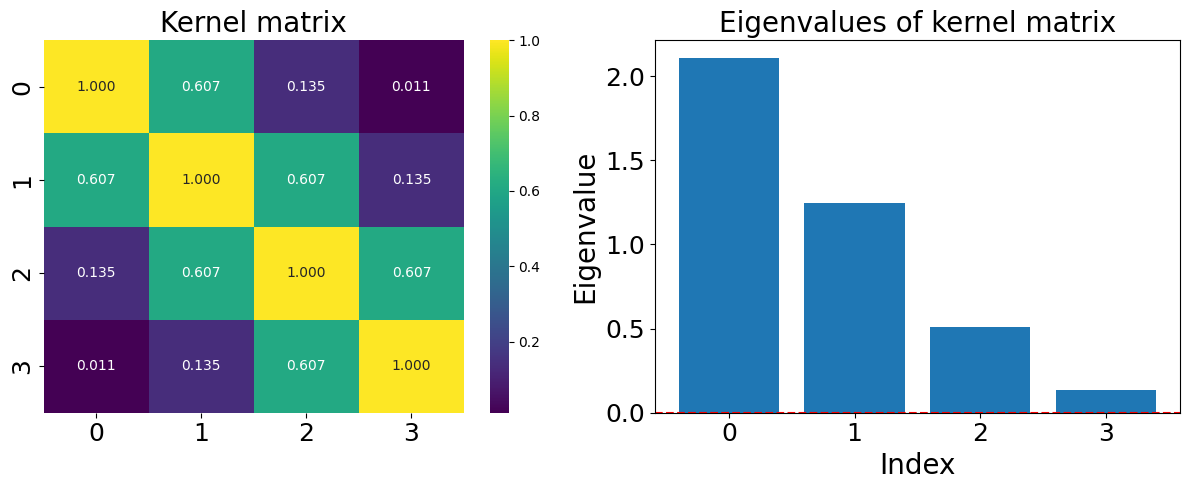}
        \caption{RBF ($\ell$=1.0)}
        \label{fig:RBF}
    \end{subfigure}
    \begin{subfigure}[b]{0.45\linewidth}
        \centering
        \includegraphics[width=\linewidth]{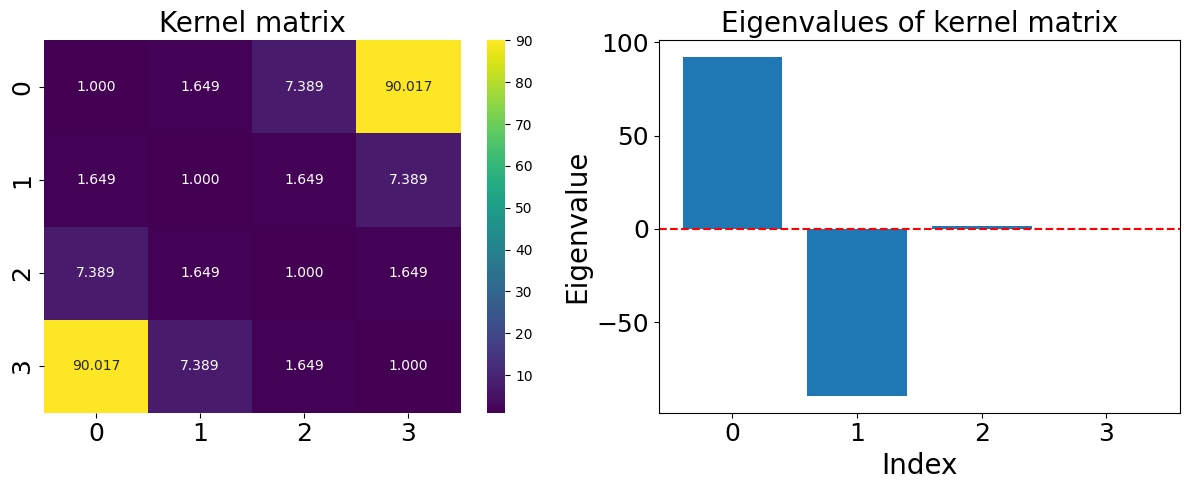}
        \caption{Absolute RBF ($\ell$=1.0)}
        \label{fig:absolute_RBF_kernel}
    \end{subfigure}
    \caption{A comparison of example RBF and absolute RBF kernels.}
    \label{fig:comparison_EBF_absolute_RBF}
\end{figure}

\paragraph{A shifted RBF kernel} 

We consider another variant of the RBF kernel designed to peak at a tunable optimal distance $d_0$, representing complementary items that are neither excessively similar nor dissimilar:
\begin{equation}\label{eq:shifted_RBF}
    k(\mathbf{x}_i, \mathbf{x}_j) = \exp\left(-\frac{(\|\mathbf{x}_i - \mathbf{x}_j\| - d_0)^2}{2\ell^2}\right)
\end{equation}
where $d_0>0$ is a hyperparameter at which covariance is maximized, and can be empirically optimized from data. The length-scale $\ell$ controls the rate at which covariance decays around $d_0$.

This kernel captures intuitive item interactions: covariance peaks at $\|\mathbf{x}_i - \mathbf{x}_j\| = d_0$ (e.g. complementary products such as beer and diapers), while it decreases both at small distances, modeling substitutable items (e.g. two similar milk brands), and at large distances, preventing overgeneralization. Unlike the standard RBF kernel, which peaks at zero distance (identical feature vectors), the shifted RBF kernel is maximized at a nonzero \textit{optimal} or \textit{preferred} distance $d_0$. Nonetheless, it remains a stationary \footnote{A stationary kernel \cite{Rasmussen2006GPbook} is one whose covariance function depends only on the relative position between inputs, not on their absolute locations, i.e. $k(\mathbf{x}, \mathbf{x}') = k(\mathbf{x} - \mathbf{x}')$. For stationary covariance functions such as RBF, we can write $k(\mathbf{x},\mathbf{x}') = k(\mathbf{x}-\mathbf{x}')$.} radial kernel, depending solely on the relative distance $r = \|\mathbf{x}_i - \mathbf{x}_j\|$, and can thus be succinctly expressed as:
\[
k(r) = \exp\left(-\frac{(r - d_0)^2}{2\ell^2}\right)
\]

\begin{figure}[H]
    \centering
    \begin{subfigure}[b]{0.3\linewidth}
        \centering
        \includegraphics[width=\linewidth]{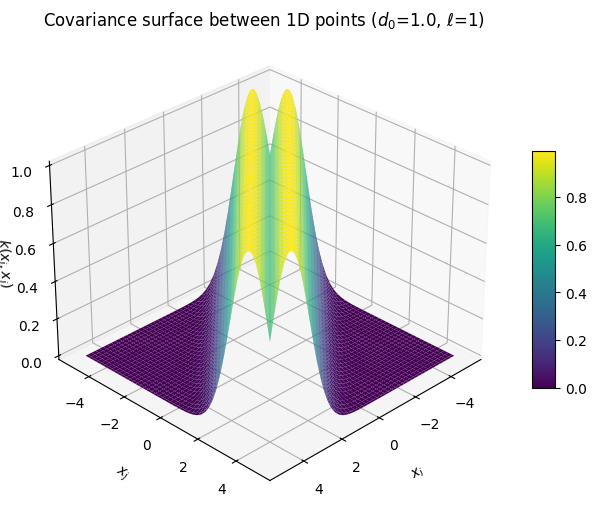}
    \end{subfigure}
    \hspace{0.2cm}
    \begin{subfigure}[b]{0.3\linewidth}
        \centering
        \includegraphics[width=\linewidth]{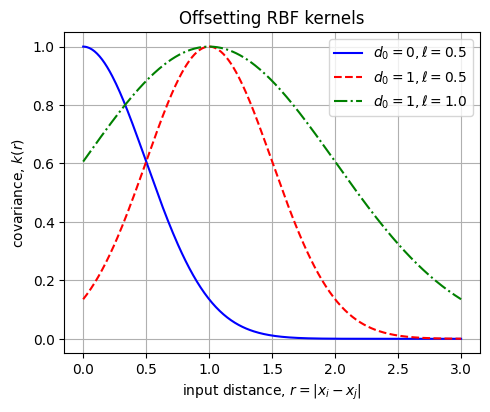}
    \end{subfigure}
    \caption{The shifted RBF kernel with 1D inputs.}
    \label{fig:shifted_RBF}
\end{figure}

The behaviour of the shifted RBF kernel in 1D is shown in Fig.\ref{fig:shifted_RBF}. This shifted RBF kernel, however, is generally not positive semi-definite (PSD). According to \textit{Schoenberg's theorem} (Theorem.\ref{theorem:Schoenbergs} in Appendix.\ref{app:kernel_designs}) and Corollary.\ref{corollary:PSD_and_radial_funcs}, a radial kernel of the form $k(r)=\exp(-\phi(r))$ is PSD if and only if $\phi(r)$ is conditionally negative definite (CND). Equivalently, a radial kernel $f(r^2)$ is PSD if and only if it is completely monotonic in $r$. The standard RBF kernel $k(r)=\exp(-r^2)$ meets this criterion (see Example.\ref{example:RBF_kernel}), while kernels such as $k(r)=\exp(-(r - d_0)^2)$ do not. Specifically, the shifted RBF kernel is neither monotonic nor completely monotonic, since it peaks at $r=d_0$, violating conditions for positive definiteness. Empirical tests (e.g. computing Gram matrices for simple datasets) confirm that shifted RBF kernels typically yield negative eigenvalues, particularly for positive $d_0$ and small length-scale $\ell$.

As the shifted RBF kernel does not satisfy \textit{Mercer's condition} (Theorem.\ref{theorem:mercer} in Appendix.\ref{app:kernel_designs}), it cannot be directly employed as a valid kernel in standard kernel methods such as Gaussian processes or SVMs, unless $d_0=0$ (recovering the standard RBF kernel). A practical solution to this issue involves converting the non-PSD kernel into a PSD kernel through an empirical eigen-decomposition approach. Specifically, given the covariance matrix $K$ computed from training data, we perform eigen-decomposition $K_{M \times M} = Q_{M \times M} \Lambda_{M \times M} Q^\top_{M \times M}$, keeping only positive eigenvalues $\{\lambda_k\}_{k=1}^m$ and extracting the $i$-th row $\{q_{ik}\}_{k=1}^m$ of $K$ for data point $\mathbf{x}_i$, obtaining an \textit{empirical} feature mapping:
\[
\phi(\mathbf{x}_i) = [\sqrt{\lambda_1} q_{i1},\,\ldots,\,\sqrt{\lambda_m} q_{im}]^\top
\]
yielding a new, data-driven PSD kernel:
\[
k^{\text{new}}(\mathbf{x}_i, \mathbf{x}_j) = \phi(\mathbf{x}_i)^\top \phi(\mathbf{x}_j)
\]

Collecting these features into a matrix $\Phi \in \mathbb{R}^{M \times m}$, we reconstruct a PSD covariance matrix as:
\[
K^{\text{new}}_{M \times M} = Q_{M \times m} \Lambda_{m \times m} Q^\top_{m \times M} = (Q_{M \times m} \Lambda^{1/2}_{m \times m})(Q_{M \times m} \Lambda^{1/2}_{m \times m})^\top = \Phi_{M \times m} \Phi^\top_{m \times M}
\]

This data-driven kernel construction ensures PSD conditions are satisfied, albeit with additional computational complexity from empirical kernel learning \footnote{The eigen-decomposition step adds computational overhead, especially for large $M$. The complexity is $\mathcal{O}(M^3)$ for the eigen-decomposition, but this is a one-time cost when computing $K$, and subsequent operations (e.g. in Monte Carlo sampling) operate on smaller submatrices.}. The shifted RBF kernel peaks at $\|\mathbf{x}_i - \mathbf{x}_j\| = d_0$, capturing complementary relationships between items that are neither too similar nor too dissimilar. This leads to a different GP model and has impact on later joint probability estimation compared to the standard RBF kernel. During GP training, $d_0$ can be tuned based on the dataset. For example, $d_0$ could be set based on the average pairwise distance between feature vectors in $X$. To estimate a good $d_0$, we can optimize $d_0$ in a similar fashion to optimize $\ell$ via maximum likelihood \footnote{In our later experiments, we tested different kernels on \textit{Synthetic 1} dataset. Unfortunately, for this dataset we found the MLE optimised $d_0$=0, which reduces it to RBF.}.

\paragraph{A neural tangent kernel (NTK)}
NTK \cite{Jacot2020NTK} is an empirical, data-driven kernel derived from deep neural networks. It is defined through the gradients of the neural network outputs \textit{w.r.t.} its parameters. Specifically, consider a neural network \footnote{The neural network can be in different state depending on the training stages, i.e. it can be a function of time $f(\mathbf{x}; \theta(t))$.} $f(\mathbf{x}; \theta)$, where $\mathbf{x}$ denotes the input and $\theta$ the network parameters (e.g. weights and biases). Different from standard Gaussian process kernels, a NTK is defined using the gradient of the output of the randomly initialized neural net with respect to its parameters \cite{Jacot2020NTK,Arora2019NTK,Yang2020NTK}:
\begin{equation} \label{eq:NTK}
    K(\mathbf{x}, \mathbf{x}') = \lim_{\text{width} \to \infty} \left\langle \frac{\partial f(\mathbf{x}; \theta)}{\partial \theta}, \frac{\partial f(\mathbf{x}'; \theta)}{\partial \theta} \right\rangle
\end{equation}
As the network width (the number of neurons per layer) approaches infinity, the NTK emerges as a deterministic and stable kernel \footnote{See statements in Lemma.\ref{lem:NTK_kernel_flow} and Theorem.\ref{thm:NTKntk_equiv} in Appendix.\ref{app:NN_equivalent_to_NTK}, as well as \cite{Jacot2020NTK,Arora2019NTK}.}. In some literature, e.g. \cite{Jacot2020NTK,Arora2019NTK}, the NTK associated with a depth-$L$, infinitely wide neural net is denoted as $\Theta_\infty^{(L)}$. The following proposition states that the NTK expressed by Eq.\ref{eq:NTK} is PSD: 
\begin{proposition}[Positive‐definiteness of NTK on parameter gradients]
Let 
\[
k(\mathbf{x},\mathbf{x}') \;=\;\bigl\langle\nabla_\theta f(\mathbf{x}),\,\nabla_\theta f(\mathbf{x}')\bigr\rangle
\]
be the empirical NTK of a finite‐width network. Then for any finite set $\{\mathbf{x}_i\}_{i=1}^M$ and real coefficients $\{c_i\}_{i=1}^M$, the Gram form
\[
\sum_{i,j=1}^M c_i\,c_j\,k(\mathbf{x}_i,\mathbf{x}_j)
\]
is non‑negative, i.e.\ $k$ is positive‑semidefinite (PSD).  Moreover, in the infinite‑width limit $k\to\Theta_\infty^{(L)}$ with $L$ being the depth of the network, this remains true, and in fact becomes strictly positive‑definite (PD) on an unit sphere (see Prop.2 of \cite{Jacot2020NTK}).
\end{proposition}

\begin{proofsketch}
For a neural network with finite width, take any set of points $\{\mathbf{x}_1, \dots, \mathbf{x}_M\}$ and coefficients \(\{c_1, \dots, c_M\}\), we want to check if the following is PSD:
\[ \sum_{i,j=1}^M c_i c_j k(\mathbf{x}_i, \mathbf{x}_j) = \sum_{i,j=1}^M c_i c_j \left\langle \nabla_\theta f(\mathbf{x}_i), \nabla_\theta f(\mathbf{x}_j) \right\rangle 
= \left\langle \sum_{i=1}^M c_i \nabla_\theta f(\mathbf{x}_i), \sum_{j=1}^M c_j \nabla_\theta f(\mathbf{x}_j) \right\rangle \]
Since the two sums are identical, denote $\mathbf{v} = \sum_{i=1}^n c_i \nabla_\theta f(\mathbf{x}_i)$, we arrive at $\left\langle \mathbf{v}, \mathbf{v} \right\rangle = \|\mathbf{v}\|^2$. The squared norm of any vector $\mathbf{v}$ is always non-negative: $\|\mathbf{v}\|^2 \geq 0 $. Thus, $k(\mathbf{x}, \mathbf{x}')$ is PSD as it is a dot product kernel in the space of parameter gradients.
In the infinite‑width limit, this dot‑product structure carries over to the NTK $\Theta_\infty^{(L)}$, and by a more detailed spectral argument one shows that, for non‑polynomial, Lipschitz activations and $L\ge2$, the restriction of $\Theta_\infty^{(L)}$ to the unit sphere $\mathbb S^{n_0 - 1} = \left\{ \mathbf{x} \in \mathbb{R}^{n_0} : \mathbf{x}^T \mathbf{x} = 1 \right\}$, with $n_0$ being the input dimension, is strictly positive‑definite (Prop.2 in \cite{Jacot2020NTK}).
\end{proofsketch}

NTK encapsulates the behavior of the neural network in the infinite-width regime \footnote{NTK captures the behavior of a fully-trained wide neural net under weaker condition \cite{Arora2019NTK}. The equivalence between a fully trained, sufficiently wide neural network and a NTK regression predictor is discussed in Appendix.\ref{app:NN_equivalent_to_NTK}. Training a neural net follows kernel gradient flow, see statements in Lemma.\ref{lem:NTK_kernel_flow} in Appendix.\ref{app:NN_equivalent_to_NTK}, as well as \cite{Jacot2020NTK,Arora2019NTK}.}. In the infinite-width limit, training a sufficiently wide neural network via gradient descent becomes equivalent to kernel regression with the NTK \footnote{See statements in Theorem.\ref{thm:NTKntk_equiv} in Appendix.\ref{app:NN_equivalent_to_NTK}, as well as \cite{Jacot2020NTK,Arora2019NTK}}. Therefore, NTK bridges deep learning and kernel methods, establishing an insightful connection to GPs, wherein the NTK defines a GP prior over neural network functions. This connection elucidates neural network generalization and convergence properties, providing a theoretical foundation for analyzing and designing neural architectures\cite{Jacot2020NTK,Arora2019NTK,Yang2020NTK}. For GPAR, NTK can be employed as the covariance function, offering advantages including expressiveness, feature learning, and data-driven flexibility. Specifically, NTKs capture complex, non-linear relationships between item features, which may help identify intricate co-occurrence patterns overlooked by standard kernels \footnote{It is reported \cite{Arora2019NTK} that, the NTK associated with an 11-layer convolutional network (i.e. CNTK) featuring global average pooling achieves a 77\% classification accuracy, which is 10\% higher than the best reported performance of a Gaussian process with a fixed kernel on CIFAR-10 \cite{Novak2019GP}.}. Additionally, NTKs implicitly learn feature representations through the neural network's structure (e.g. MLP, CNN), adapting directly to data without manual feature engineering. Moreover, by adjusting network architecture parameters (e.g. depth and activation functions), the NTK can be customized to suit specific characteristics of transaction data. Here we use a simple architecture, i.e. a wide, 3-layer shallow network with \textit{ReLU} activation \footnote{This architecture is useful and important because, according to Hornik et al. \cite{Hornik1993universal,Rasmussen2006GPbook}, an infinitely wide neural network with a single hidden layer can act as a universal approximator for a wide range of activation functions (excluding polynomials), enabling it to approximate any continuous function given sufficient neurons. Also note, in the limit
of infinite width, a single hidden-layer neural network with \textit{i.i.d.} random parameters is a function drawn from a Gaussian process \cite{Neal1996,Arora2019NTK}.}, whose hidden layer has $m$ neurons. The induced NTK can be expressed in the following analytic form:
\begin{equation} \label{eq:NN_3layer_NTK_cc} \tag{cc.Eq.\ref{eq:NN_3layer_NTK}}
  k^{NTK}(\mathbf{x}, \mathbf{x}') = m \cdot \left\{ \frac{\|\mathbf{x}\| \|\mathbf{x}'\|}{2\pi} \left( \sqrt{1 - \rho^2} + \rho \arccos(-\rho) \right) + \left( \frac{1}{4} + \frac{\arcsin(\rho)}{2\pi} \right) \mathbf{x}^\top \mathbf{x}' \right\}
\end{equation}

Using NTK in GPAR presents challenges as well. First, as in all kernel-based methods, computational cost can be significant, as calculating and storing the NTK matrix scales poorly with large datasets, adding burden to the rule mining process. Scalability may be addressed through low-rank approximation. Second, as in all neural network methods, NTKs offer reduced interpretability compared to traditional kernels, complicating explanations of how item features influence rule discovery.

\paragraph{A 3-layer \textit{erf}-activated neural network kernel}
Williams \cite{Williams1998NNkernel} and Neal \cite{Neal1996} demonstrated that in the limit of an infinite number of hidden units, a Bayesian neural network with certain priors on its weights converges to a Gaussian process. For a specific architecture consisting of an input layer, a single hidden layer with $m$ units, and a linear output layer with scalar output, a well-defined, analytic kernel function can be derived with mild assumptions. To construct such neural net kernel, we augment the inputs $\mathbf{x} \in \mathbb{R}^d$ with a constant bias term to form $\tilde{\mathbf{x}} = (1,x^1,\dots,x^d)^T \in \mathbb{R}^{d+1}$. The input-to-output mapping is:
\begin{equation} \label{3layer_NN_kernel_input_to_output_cc} \tag{cc.Eq.\ref{3layer_NN_kernel_input_to_output}}
    f(\mathbf{x}) = b + \sum_{j=1}^{m} v_j \sigma(\mathbf{x}; \mathbf{w}_j)
\end{equation}
where $m$ is again the number of hidden neurons, $b$ the bias term. $v_j$ are the hidden-to-output weights. $\{\mathbf{w}_j\}_{j=1}^m$ is the collection of input-to-hidden weight vectors; $\mathbf{w}_j$ is the $d+1$-dimensional weight vector for the $i$-th hidden unit, including a first weight element for the augmented bias input. $\sigma(\mathbf{x}; \mathbf{w}_j)$ is the activation function of the $j$-th hidden unit. 

The $m$ hidden neurons employ an activation function which, in our specific case, is an error function, i.e. $\sigma(z) = \text{erf}(z) = \frac{2}{\sqrt{\pi}} \int_0^z e^{-t^2} dt$. To arrive at a tractable analytic formulation of the neural net kernel, we further make following assumptions: the activation is bounded, and all weight vectors $\mathbf{w}_j$ are assumed to follow a common zero-mean, isotropic Gaussian distribution, i.e. $\mathbf{w}_j \sim \mathcal{N}(\mathbf{0}, \Sigma)$ with diagonal covariance matrix $\Sigma = \text{diag}(\sigma_0^2, \sigma_1^2, \dots, \sigma_d^2)$; the bias and the hidden-to-output weights have independent zero-mean Gaussian priors. The resulting covariance function (derivation see Appendix.\ref{app:NN_kernel_derivation}), often termed the \textit{neural network kernel} or \textit{arcsin kernel}, is given by \cite{Williams1998NNkernel,Rasmussen2006GPbook}:
\begin{equation} \label{eq:NN_kernel_erf_activation_cc} \tag{cc.Eq.\ref{eq:NN_kernel_erf_activation}}
    k^{NN}(\mathbf{x}, \mathbf{x}') = \frac{2}{\pi} \sin^{-1}\left(\frac{2\tilde{\mathbf{x}}^T \Sigma \tilde{\mathbf{x}}'}{\sqrt{(1+2\tilde{\mathbf{x}}^T \Sigma \tilde{\mathbf{x}})(1+2\tilde{\mathbf{x}}'^T \Sigma \tilde{\mathbf{x}}')}}\right)
\end{equation}

This kernel provides a non-stationary covariance. The hyper-parameters in $\Sigma$ encodes assumptions about the input-to-hidden weights (and bias), they control the properties of the learned function: $\sigma_0^2$ influences the offset of the activation functions, while $\sigma_i^2$ (for $i \geq 1$) determines the relevance or scaling of the $i$-th input feature, affecting how quickly the function varies along that dimension. An additional overall scale factor and a bias variance term (representing $\mathbb{E}[b^2]$ from the network's output bias) can also be incorporated into this kernel.

\subsection{GPAR experiments} \label{sec:GPAR_experiments}

\subsubsection{Experimental setup and environments}

We evaluate the performances of AR mining algorithms on 3 datasets: \textit{Synthetic 1}, \textit{Synthetic 2}, and a real-world \textit{UK Accidents dataset} \cite{road_safety_data}. Seven methods are tested on \textit{Synthetic 1}: GPAR with 4 kernel variants (RBF, shifted RBF, neural network, NTK), Apriori, FP-Growth, and Eclat. For \textit{Synthetic 2} and \textit{UK Accidents}, 4 methods are evaluated: RBF-based GPAR, Apriori, FP-Growth, and Eclat. Each method is assessed across a range of minimum support thresholds (0.1 to 0.5), with a fixed minimum confidence of 0.5, measuring \textit{runtime}, \textit{memory usage}, \textit{no. of frequent itemsets}, and \textit{no. of rules generated}. More details about experiments can be found in Section.\ref{app:GPAR_experimental_setup}.

Experiments on \textit{Synthetic 1} and \textit{Synthetic 2} datasets were conducted on a laptop with an 11th Gen Intel® Core™ i7-1185G7 processor at 3.00 GHz, with 8 CPUs across 4 cores and a memory hierarchy with 192 KiB L1 data cache, 5 MiB L2 cache, and 12 MiB L3 cache. The \textit{UK Accidents} dataset experiments ran in \textit{Google Colab} on an Intel(R) Xeon(R) CPU at 2.00GHz with 48 processors across 24 cores, a 39,424 KB cache, and 350 GB of memory, of which 337 GB was free. Both systems supported AVX512 and included mitigations for Spectre and Meltdown, ensuring efficient computation for the respective tasks.

\subsubsection{Data and processing methods}

Each dataset comprises two distinct matrices: a \textit{feature matrix} and a \textit{transaction matrix}. The feature matrix contains $M$-rows of $d$-dimensional vectors representing item characteristics, such as shape, size, color, and price for each of the $M$ items. In the transaction matrix, an item is labeled as 1 in an $m$-length vector if its latent variable $z > 0$, and 0 otherwise, with each row representing a transaction and indicating the presence (1) or absence (0) of items. For example, a dataset with $M=10$ items includes bread ($\mathbf{x}_B = [\text{large, round, white, \$2}]$), milk ($\mathbf{x}_m = [\text{medium, cylindrical, white, \$3}]$), eggs ($\mathbf{x}_e = [\text{small, oval, brown, \$4}]$), and others, where a transaction record like [1, 0, 1, 0, 0, 0, 0, 0, 0, 0] indicates the presence of items 0 (bread) and 2 (eggs) and the absence of others. Item relationships are modeled using a selected kernel to compute a covariance matrix $K$, enabling the estimation of the probability $p(z_b > 0, z_m > 0, z_e > 0)$ from transaction data to assess the likelihood of bread, milk, and eggs being purchased together, focusing on their joint distribution while marginalizing over other items.

\textit{Synthetic 1} comprises 1000 transactions of 10 items, with a feature matrix $X \in \mathbb{R}^{10 \times 10}$ sampled from a standard normal distribution. Transactions are generated using a multivariate normal distribution with an RBF kernel, retaining items with positive latent values. Synthetic 2 includes 1000 transactions of 15 items, organized into three Gaussian clusters (food, drinks, consumables) with a feature matrix $X \in \mathbb{R}^{15 \times 10}$. Transactions are sampled from multiple clusters to emulate diverse purchasing patterns (see Appendix \ref{app:synthetic_data_tests} for details). The \textit{UK Accidents dataset} (2005-2015) contains 1,048,575 records, filtered to 5231 fatal and serious accidents, with a transaction matrix $\mathcal{T}_{5231 \times 39}$ constructed from seven categorical columns (e.g. \textit{Accident\_Severity}, \textit{Weather\_Conditions}) and a feature matrix $X_{39 \times 10}$ derived from spatial and temporal features (e.g. \textit{Longitude}, \textit{Hour\_of\_Day}) (see Appendix \ref{app:accident_data_tests}).

For GPAR, frequent itemsets are identified via Monte Carlo sampling (100 samples), followed by confidence and lift computation for rule generation. Apriori and FP-Growth utilize the \textit{Mlxtend} library \cite{raschkas_2018_mlxtend}, employing standard pruning techniques, while Eclat uses a vertical database format with recursive transaction ID intersections. Performance is presented in tabular form and visualized graphically, comparing metrics across support thresholds.

\subsubsection{Results}
On \textit{Synthetic 1}, GPAR with neural network kernel achieves the fastest runtime (0.133s at support 0.1), followed by NTK (0.449s), while RBF and shifted RBF variants \footnote{Unfortunately, for this dataset the MLE optimised $d_0$=0, which reduces the shifted RBF to RBF.} are slower (8.930-10.179s), as shown in Table \ref{tab:synthetic1_comparison_runtime_mainText}. Apriori and Eclat are consistently efficient (average runtime <0.635s), with FP-Growth slightly slower (2.019-2.967s). Memory usage is negligible for all GPAR variants and Apriori/FP-Growth, but Eclat peaks at 18.149 MB (support 0.1). GPAR (RBF/shifted RBF) generates the most frequent itemsets (1013 at 0.1) and rules (54036), far exceeding neural network (140 itemsets, 322 rules) and NTK (305 itemsets, 1485 rules), as detailed in Table \ref{tab:synthetic1_comparison_noRules_mainText}. Apriori/FP-Growth consistently produce 1023 itemsets and 57002 rules (0.1-0.4), dropping at 0.5, while Eclat maintains 1033 itemsets across all thresholds. Rule analysis shows GPAR (RBF/shifted RBF) uncovers complex patterns with high lift (e.g. 12.1212 for `item\_0, item\_8, item\_6, item\_9 $\to$ item\_1, item\_2, item\_3, item\_4, item\_7'), while Apriori/FP-Growth/Eclat focus on simpler, high-confidence rules (e.g. confidence 0.9981 for `item\_7'). These trends are illustrated in Fig. \ref{fig:synthetic1_performance_comparison_mainText}, highlighting trade-offs between computational efficiency and pattern discovery.

\textit{Synthetic 2} highlights RBF-based GPAR's ability to infer latent patterns, identifying rules with high confidence (e.g. 13.0000 for `[0, 3, 4, 6, 8, 12, 13] $\to$ [2, 11]') despite zero support counts, leveraging item feature similarities. However, its runtime (231.274s at 0.1) and memory usage (276.090 MB) are significantly higher than Apriori (0.006s), FP-Growth (0.046s), and Eclat (0.087s), which show negligible memory usage. GPAR generates far more itemsets (10215 at 0.1) and rules (1200793) than traditional methods (e.g. Apriori/FP-Growth: 24 itemsets, 18 rules), which produce simpler rules with lower lift (e.g. 1.4632 for `[item\_11] $\to$ [item\_12]' for Apriori). Full results are in Appendix \ref{app:synthetic_data_tests}.

For the \textit{UK Accidents} dataset, RBF-based GPAR is computationally intensive (1027.79s, 341.23 MB at 0.1), while Apriori, FP-Growth, and Eclat are efficient (0.01-3.47s, memory peaking at 22.95 MB for Eclat), as shown in Table \ref{tab:accident_comparison_frequentItemsets_mainText}. GPAR generates the most itemsets (10366) and rules (16869) at 0.1, diminishing at higher thresholds, compared to Apriori/FP-Growth (277 itemsets, 2052 rules) and Eclat (479 itemsets, 5794 rules). GPAR identifies nuanced patterns (e.g. `Accident\_Severity=2 $\to$ Weather\_Conditions=1', support 0.4610), associating severe accidents with fine weather, while Eclat excels in lift (9.8493 for rules involving wet road conditions and junctions), indicating strong associations with specific contexts. Results from the \textit{UK Accidents} dataset are detailed in Appendix \ref{app:accident_data_tests}.

\begin{table}[h]
\centering
\scriptsize
\begin{threeparttable}
\caption{Runtime performance (in \textit{seconds}) (\textit{Synthetic 1})}
\label{tab:synthetic1_comparison_runtime_mainText}
\begin{tabular}{p{1cm} p{1cm} p{1.8cm} p{1.5cm} p{1cm} p{1cm} p{1.5cm} p{1cm}}
\toprule
Min Support & GPAR (RBF) & GPAR (shifted RBF) & GPAR (neural net) & GPAR (NTK) & Apriori (AVG) & FP-Growth (AVG) & Eclat (AVG) \\
\midrule
0.1 & 10.179 & 8.930 & 0.133 & 0.449 & 0.483 & 2.967 & 0.556 \\
0.2 & 8.448 & 8.697 & 0.046 & 0.077 & 0.505 & 2.734 & 0.555 \\
0.3 & 3.295 & 2.658 & 0.017 & 0.030 & 0.565 & 2.622 & 0.575 \\
0.4 & 0.215 & 0.174 & 0.002 & 0.016 & 0.570 & 2.700 & 0.558 \\
0.5 & 0.013 & 0.034 & 0.005 & 0.003 & 0.211 & 2.019 & 0.635 \\
\bottomrule
\end{tabular}
\end{threeparttable}
\end{table}

\begin{table}[h]
\centering
\scriptsize
\begin{threeparttable}
\caption{No. of rules generated (\textit{Synthetic 1})}
\label{tab:synthetic1_comparison_noRules_mainText}
\begin{tabular}{p{1cm} p{1cm} p{1.8cm} p{1.5cm} p{1cm} p{1cm} p{1.5cm} p{1cm}}
\toprule
Min Support & GPAR (RBF) & GPAR (shifted RBF) & GPAR (neural net) & GPAR (NTK) & Apriori (AVG) & FP-Growth (AVG) & Eclat (AVG) \\
\midrule
0.1 & 54036 & 54036 & 322 & 1485 & 57002 & 57002 & 57002 \\
0.2 & 51479 & 51328 & 59 & 159 & 57002 & 57002 & 57002 \\
0.3 & 19251 & 16067 & 18 & 26 & 57002 & 57002 & 57002 \\
0.4 & 782 & 530 & 0 & 4 & 57002 & 57002 & 57002 \\
0.5 & 2 & 8 & 0 & 0 & 20948 & 20948 & 57002 \\
\bottomrule
\end{tabular}
\end{threeparttable}
\end{table}

\begin{table}[h]
\centering
\scriptsize
\begin{threeparttable}
\caption{No. of frequent itemsets generated (\textit{UK Accidents})}
\label{tab:accident_comparison_frequentItemsets_mainText}
\begin{tabular}{p{1.5cm} p{1.5cm} p{1.5cm} p{1.5cm} p{1.5cm}}
\toprule
Min Support & GPAR & Apriori & FP-Growth & Eclat \\
\midrule
0.1 & 10366 & 277 & 277 & 479 \\
0.2 & 770 & 153 & 153 & 235 \\
0.3 & 4 & 75 & 75 & 97 \\
0.4 & 0 & 55 & 55 & 69 \\
0.5 & 0 & 39 & 39 & 69 \\
\bottomrule
\end{tabular}
\begin{tablenotes}
\item[1] Total number of accident records: 5231.
\item[2] Numbers are counts of frequent itemsets generated.
\item[3] We use min\_conf=0.5 to filter the generated rules. 
\end{tablenotes}
\end{threeparttable}
\end{table}

\begin{figure}[H]
    \centering
    \includegraphics[width=1.0\textwidth]{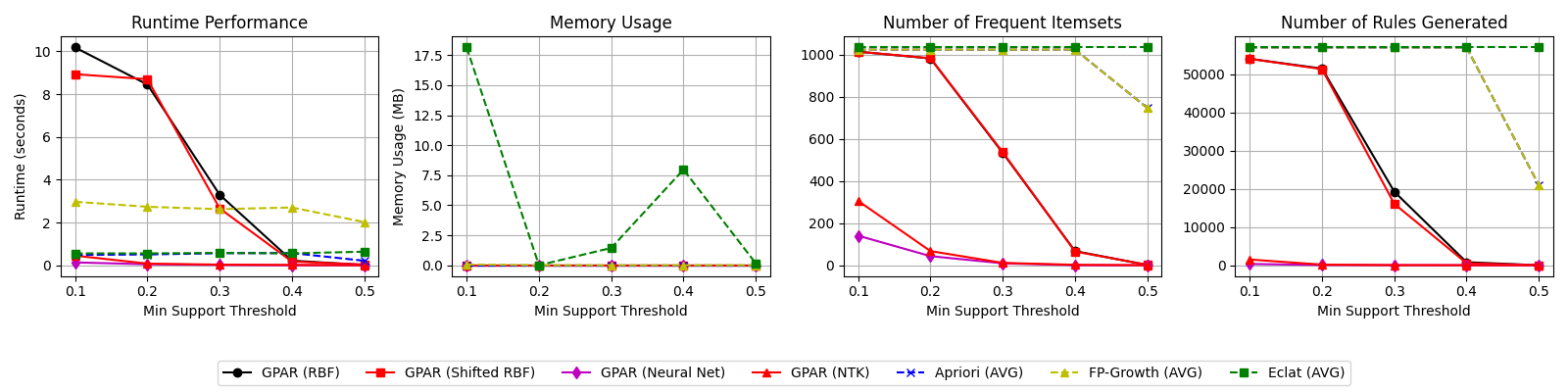}
    \caption{Performances of 7 algorithms on \textit{Synthetic 1} dataset across min support thresholds.}
    \label{fig:synthetic1_performance_comparison_mainText}
\end{figure}

\subsubsection{Discussions}
In GPAR, each item $\mathcal{I}$ is associated with a feature vector $\mathbf{x}_i$, capturing attributes such as size, shape, color, price, etc. These attributes provide a measure of natural similarity between items, which can be leveraged for modeling relationships between the corresponding continuous, scalar-valued, latent variable $z_k$ (i.e. the raw, un-quantized membership variable) whose sign (positive or negative) represents the presence of item $k$ in a transaction (e.g. $z_k$ > 0 implies presence). Assuming a GP prior (i.e. mean and covariance structures), the dependencies between items can be captured based on their attribute similarities. The joint distribution of all $z_k$ is multivariate Gaussian with mean vector $\mu$ and covariance matrix $K$, where $K[i,j] = k(\mathbf{x}_i, \mathbf{x}_j)$ and $k$ is a kernel function (e.g. RBF, linear, Matérn, etc) based on the feature vectors. The covariance matrix $K$ captures the dependencies, with higher covariance between items that co-occur more frequently. For a subset $\mathcal{I}_s \in \mathcal{I}$ of $s$ items, the probability that all items in $\mathcal{I}_s$ are present in a transaction, i.e. their co-occurrence probability (marginal joint probability), is $p(\text{all } z_i > 0 \text{ for } i \in \mathcal{I}_s)$. This is the marginal probability of the multivariate Gaussian distribution over the subset of variables corresponding to $\mathcal{I}_s$, which can be computed by marginalizing over the remaining items (i.e. integrating out $z_j$ for $j \notin \mathcal{I}_s$). The marginalization process involves computing the joint probability over the subset $\mathcal{I}_s$, which is standard \footnote{Computing the probability that an subset $\mathcal{I}_s$ of items is present (all $z_k > 0$ for $k \in \mathcal{I}_s$) involves integrating over a hyper-rectangle in the multivariate Gaussian, which is feasible but computationally intensive for large sets, especially without closed-form solutions.} in multivariate Gaussian distributions but can be computationally intensive for large $s$ due to high-dimensional integration - this can be got around by Monte Carlo sampling. 

As discussed in Section.\ref{sec:GPAR_algorithm}, scalability is the bottleneck of GPAR. Traditional association rule mining algorithms such as Apriori and FP-Growth are optimized for large datasets, with complexities that scale with the number of transactions $N$ and items $M$. GPAR's computational cost, potentially cubic \footnote{Cost can be reduced by using e.g. sparse or low-rank approximation of the covariance matrix though.} in the number of items $M$ for Gaussian process inference, limits its applicability to small to medium-sized datasets. Added complexity due to probabilistic computations (estimating the posterior probabilities using Monte-Carlo sampling) further imposes computational burden. In addition, same as other ARM algorithms, when mining frequent itemsets, GPAR is further bottlenecked by evaluating exponentially many subsets. Heuristics or pruning based on partial evaluations might be necessary.

Despite the computational challenge, GPAR offers several benefits. While traditional ARM methods use empirical probability (e.g. frequency counts), GPAR provides a principled framework to handle probability. GPAR allows prior knowledge to be incorporated: the covariance function can encode domain knowledge, such as known similarities between items, enhancing model generalization. For example, in retail, prior knowledge about product categories could inform the kernel, leading to more interpretable rules. GPAR is able to capture complex patterns. Gaussian processes can model non-linear relationships, potentially capturing more nuanced associations than traditional methods, such as rare but significant patterns which could be missed by support thresholds, as those appear in retail or medical symptom-disease associations.

GPAR could be suitable for small to medium-sized datasets, or where computational resources are not a constraint, and probabilistic insights are valuable. It can also be valuable for specialized domains where the probabilistic interpretation is vital, for example medical diagnosis (symptom-disease associations) or financial risk assessment. If we have some prior knowledge which are believed to be insightful, GPAR can be used to incorporate them to uncover novel patterns or test hypotheses as well. However, for large-scale applications such as market basket analysis in supermarkets, traditional methods are likely to be more efficient. The trade-off is between the richness of probabilistic modeling and computational feasibility.

Overall, the GPAR approach is a novel extension to the family of ARM algorithms, offering a probabilistic framework to handle uncertainty and incorporate prior knowledge. By modeling item relationships with Gaussian processes and latent variables, it can capture complex patterns. However, practical limitations, particularly scalability for large datasets, suggest it is best suited for smaller datasets or specialized applications where these challenges can be managed. Further research could explore approximations or structured covariance functions to enhance efficiency.

\section{BARM: a Bayesian framework for association rule mining} \label{sec:BARM}

We generalise the GPAR framework, which uses a GP to model latent membership variables, to a fully Bayesian approach that directly models item co-occurrence probabilities using probabilistic distributions. This framework, termed Bayesian association rule mining (BARM), treats the presence of items in transactions as binary outcomes governed by latent probabilities, which are assigned prior distributions and updated via Bayesian inference given the transaction data. In the following, we first re-visit the principles of Bayesian modelling, then introduce the BARM methodology in steps.

\subsection*{\textit{Bayesian inference: preliminaries}}

Bayesian modeling provides a probabilistic framework for reasoning under uncertainty, which is particularly well-suited for association rule mining as it allows for the incorporation of prior knowledge and the quantification of uncertainty in item co-occurrence probabilities. At its core, Bayesian inference updates the probability distribution of parameters of interest using observed data through Bayes' theorem, expressed as:
\begin{equation} \label{eq:posterior}
    p(\boldsymbol{\theta} | D) = \frac{p(D | \boldsymbol{\theta}) p(\boldsymbol{\theta})}{p(D)}
\end{equation}
where $\boldsymbol{\theta}$ represents the parameters (e.g. item presence probabilities in BARM), $D$ denotes the observed data (i.e. transaction records), $p(\boldsymbol{\theta})$ is the prior distribution reflecting initial beliefs, $p(D | \boldsymbol{\theta})$ is the likelihood of the data given the parameters, and $p(D) = \int p(D | \boldsymbol{\theta}) p(\boldsymbol{\theta}) d\boldsymbol{\theta}$ is the marginal likelihood, often intractable in complex models. Exact inference of $p(\boldsymbol{\theta} | D)$ can be realised if an analytical form can be derived (e.g. conjugate prior and likelihood); when analytical computation of the posterior $p(\boldsymbol{\theta} | D)$ is infeasible, approximate methods such as Markov Chain Monte Carlo (MCMC \cite{Bayesian_signal_processing_Joseph}, e.g. Metropolis–Hastings \cite{Metropolis1953}, Gibbs \cite{Geman1984}, Langevin \cite{Roberts1996}, Hamiltonian \cite{Duane1987}, slice sampling \cite{Neal2003slice}, SMC \cite{Doucet2001}, etc) sampling or variational inference (e.g. mean-field VI \cite{Jordan1999introduction}, stochastic VI \cite{Hoffman2013}, black-box VI \cite{Ranganath2014}, SVGD \cite{Liu2016SVGD}, EParVI \cite{EParVI2024}, SPH-ParVI \cite{SPH_ParVI2024}, MPM-ParVI \cite{MPM_ParVI2024}, etc) are employed to draw samples from the posterior, enabling the estimation of quantities such as co-occurrence probabilities in BARM - although these methods can be computationally intensive in high-dimensional or large-scale settings. Bayesian inference offers some advantages for ARM, including the ability to quantify uncertainty through posterior distributions, incorporate domain knowledge via priors, and perform continuous updates as new transactions are observed, thereby enhancing the robustness and interpretability of mined rules, particularly for rare or uncertain patterns.

\subsection{Setting the scene}

We have available a set of $M$ items $\{1, 2, \dots, M\}$, each item $i$ is represented by a $d$-dimensional feature vector $\mathbf{x}_i=[x_{i1},x_{i1},...,x_{id}]$, forming a feature matrix $X \in \mathbb{R}^{M \times d}$. Also available are transaction records which consist of $N$ transactions $\mathcal{T} = \{\mathbf{t}_1, \mathbf{t}_2, \dots, \mathbf{t}_N\}$, in which each transaction $\mathbf{t}_j = [z_1,z_1,...,z_M]$ is a binary vector with each element being the latent presence variable $z_i \in \{0, 1\}$ indicating the presence ($z_i=1$) or absence ($z_i=0$) of item $i$. Using a user-specified minimum probability threshold min\_prob and a minimum confidence threshold min\_conf, we aim to generate a set of association rules $\mathcal{R}$, where each rule $r \in \mathcal{R}$ is of the form $A \rightarrow B$, with $A, B \subseteq \{1, 2, \dots, M\}$, satisfying the specified thresholds.

\subsection{Model specification and inference} 

Instead of modeling item presence via a GP with latent variables $z$, BARM directly models the probability of item presence using a probabilistic approach. We illustrate the key elements of the Bayesian model in the following.

\paragraph{Priors for item presence probability} For each item $i$, let $p_i = p(z_i=1) = p(\text{item } i \text{ is present})$ be the probability that item $i$ appears in a transaction. We assume $p_i \in [0, 1]$ is a random variable with a prior distribution. For example, we can assign a Beta prior to each $p_i$, i.e. $p_i \sim \text{Beta}(\alpha_i, \beta_i)$, where $ \alpha_i > 0 $ and $ \beta_i > 0 $ are hyper-parameters (controlling the shape of Beta distribution), reflecting our initial beliefs about the item’s presence frequency (e.g. it becomes uniform prior with $\alpha_i = \beta_i = 1$). A proper intro about the Beta distribution is provided in Appendix.\ref{app:Beta_dist}.

\paragraph{Item dependencies} For an itemset $I \subseteq \{1, 2, \dots, M\}$, the joint probability of all items in $I$ being present, $p(I) = p(z_1=1,z_2=1,...,z_M=1) = p(\mathbf{t}[i] = 1 \ \forall i \in I)$, is modeled using a dependency structure, e.g. assuming covariance relations between the latent presence variable $z_i$. Similar to kernel-based methods, to capture correlations between items, we introduce a correlation parameter $\rho_{ij}$ for each pair of items $(i, j)$, which can be informed by the feature matrix $X$. This dependence can be characterised by a copula-like structure or a pairwise covariance structure. For example, we can employ a RBF kernel to model pairwise correlations: $\rho_{ij} = k(\mathbf{x}_i, \mathbf{x}_j) = \exp(-||\mathbf{x}_i - \mathbf{x}_j||^2 / \ell^2)$ with $\ell$ again being a length scale parameter. Instead of manually specifying the value of $\ell$, we can treat it as a parameter and assign a prior to it, e.g. $\ell \sim \text{Gamma}(a, b)$, to principally incorporate uncertainty in the dependency structure. 

\paragraph{Likelihood} Given a transaction $\mathbf{t}_j$, the likelihood of observing $\mathbf{t}_j$ is modeled as a product of \textit{Bernoulli} distributions, adjusted for correlations:
\begin{equation} \label{eq:BARM_likelihood}
    p(\mathbf{t}_j | \{p_i\}, \{\rho_{ij}\}) = \prod_{i \in I_j} p_i \prod_{i \notin I_j} (1 - p_i) \cdot g(\{\rho_{ij}\})
\end{equation}
where $I_j = \{i : \mathbf{t}_j[i] = 1\}$ is the set of items present in transaction $j$, and $g(\{\rho_{ij}\})$ is a correlation adjustment factor (e.g. derived from a Gaussian copula or a simplified pairwise model).

\paragraph{Posterior} Given the aforementioned prior and likelihood, we can obtain the posterior distribution of $p_i$ and any hyper-parameters (e.g. $\ell$), given the transaction data $\mathcal{T}$ and the item attributes $X$:
\begin{equation} \label{eq:BARM_posterior}
p(\{p_i\}, \ell | \mathcal{T}, X) \propto p(\mathcal{T} | \{p_i\}, \ell, X) \cdot \prod_{i=1}^M p(p_i) \cdot p(\ell)
\end{equation}
where $p(\mathcal{T} | \{p_i\}, \ell, X) = \prod_{j=1}^N p(\mathbf{t}_j | \{p_i\}, \ell, X)$ is the overall likelihood, assuming independence of transactions. Each atomic likelihood is computed by Eq.\ref{eq:BARM_likelihood}.

\paragraph{Approximate Bayesian inference and co-occurrence probabilities} Computing the posteriors in Eq.\ref{eq:BARM_posterior}, unfortunately, in many cases is analytically intractable (unless conjugate prior and likelihood are used, and the correlation structure is removed), we can use \textit{Markov chain Monte Carlo} (MCMC) sampling \footnote{Off-the-shelf implementations of MCMC sampling algorithms can be found in many libraries across languages, e.g. \textit{PyMC} \cite{Abrilpla2023PyMC}, \textit{PINTS} \cite{Clerx2019Pints}, \textit{Stan} \cite{Carpenter2017Stan}, and \textit{Turing.jl} \cite{Fjelde2024Turing}, etc.} such as Metropolis-Hastings (MH \cite{MH_Robert}) or Gibbs \cite{Gibbs_Geman} sampling, or variational inference (VI) methods to generate samples from the posterior. The resulted samples yield the joint posterior over all $\{p_i\}$, so for any itemset $I \in \mathcal{I}$, the required subset $\{p_i \mid i \in I\}$ is readily available in each sample, along with $\ell$. To compute the joint co-occurrence probability $p(I)$, we first perform MCMC sampling to generate samples $(\{p_i\}, \ell)$ from the posterior distribution $p(\{p_i\}, \ell | \mathcal{T}, X)$, where $\{p_i\}$ represents the set of item presence probabilities for all items $i = 1, 2, \dots, M$, and $\ell$ is the length scale parameter for the correlation function. This posterior is a \textit{joint distribution} over all $p_i$ (for all items) and $\ell$, i.e. all the $S$ MCMC samples are drawn from the full joint posterior $p(\{p_i\}, \ell | \mathcal{T}, X)$, and each MCMC sample is a tuple $(\{p_1^{(s)}, p_2^{(s)}, \dots, p_M^{(s)}\}, \ell^{(s)})$. For a given itemset $I \subseteq \{1, 2, \dots, M\}$, the subset of probabilities $\{p_i^{(s)} \mid i \in I\}$ is directly available as a subset of this joint sample. However, to compute $p(I | \{p_i^{(s)}\}, \ell^{(s)})$, which is the probability that all items in $I$ are present (i.e. $p(\mathbf{t}[i] = 1 \ \forall i \in I)$), we need a model for the joint distribution of item presences given $\{p_i^{(s)}\}$ and $\ell^{(s)}$. In this BARM framework, the item presences are modeled as Bernoulli random variables, but with a correlation structure governed by $\ell^{(s)}$. Specifically, the correlation function $k(\mathbf{x}_i, \mathbf{x}_j) = \exp(-||\mathbf{x}_i - \mathbf{x}_j||^2 / (\ell^{(s)})^2)$ defines pairwise correlations $\rho_{ij}^{(s)}$, and a simplified dependency model (e.g. a Gaussian copula or a pairwise interaction model) is used to approximate the joint probability. Thus, for each MCMC sample $s$, $p(I | \{p_i^{(s)}\}, \ell^{(s)})$ is computed as:
\begin{equation} \label{eq:BARM_subset_prob_estimate_via_MCMC_posterior_samples}
    p(I | \{p_i^{(s)}\}, \ell^{(s)}) = p(\mathbf{t}[i] = 1 \ \forall i \in I | \{p_i^{(s)}\}, \ell^{(s)}) \approx \prod_{i \in I} p_i^{(s)} \cdot g(\{\rho_{ij}^{(s)} \mid i, j \in I\})
\end{equation}
where $g(\cdot)$ adjusts for correlations among items in $I$, and its computation depends on the size of $I$, denoted $m = |I|$, leading to the complexity $\mathcal{O}(S \cdot m)$ for $S$ samples. The challenge lies in accurately modeling the correlation adjustment $g(\cdot)$, which may require approximations for tractability, especially for larger itemsets.

Note the difference between the likelihood equation for a transaction $\mathbf{t}_j$ (Eq.\ref{eq:BARM_likelihood}) and the co-occurrence probability $p(I | \{p_i^{(s)}\}, \ell^{(s)})$ (Eq.\ref{eq:BARM_subset_prob_estimate_via_MCMC_posterior_samples}). Eq.\ref{eq:BARM_subset_prob_estimate_via_MCMC_posterior_samples} is designed to compute the probability that all items in a specific itemset $I$ are present (i.e. $\mathbf{t}[i] = 1 \ \forall i \in I$) for a given MCMC sample. $\{p_i^{(s)}\}$ are the sampled item presence probabilities, and $\ell^{(s)}$ determines the pairwise correlations $\rho_{ij}^{(s)}$. This equation focuses solely on the items in $I$, modeling their joint presence by taking the product of their individual probabilities $\prod_{i \in I} p_i^{(s)}$ and adjusting for their correlations via $g(\{\rho_{ij}^{(s)} \mid i, j \in I\})$. The term $g(\cdot)$ accounts for dependencies among the items in $I$, and the equation does not consider items outside $I$ as the goal is to estimate the probability of the specific event where all items in $I$ are present, irrespective of the state of other items. In contrast, the likelihood equation for a transaction $\mathbf{t}_j$, issued by Eq.\ref{eq:BARM_likelihood}, models the probability of observing the entire transaction $\mathbf{t}_j$, which is a binary vector of length $M$ indicating the presence (1) or absence (0) of each item. $I_j = \{i : \mathbf{t}_j[i] = 1\}$ is the set of items present in transaction $j$, and the likelihood accounts for all items: $\prod_{i \in I_j} p_i$ captures the probability of items that are present, while $\prod_{i \notin I_j} (1 - p_i)$ accounts for items that are absent. The correlation adjustment $g(\{\rho_{ij}\})$ applies to the entire transaction, considering dependencies across all pairs of items. This equation is used during Bayesian inference (MCMC sampling) to compute the likelihood of the observed data $\mathcal{T}$, enabling the (joint) posterior update of $\{p_i\}$ and $\ell$. Therefore, the difference arises from their distinct purposes: the co-occurrence probability focuses on a subset of items $I$ to estimate the probability of their joint presence for association rule mining, thus only including terms for items in $I$ (i.e. $\mathbf{t}[i] = 1 \ \forall i \in I$). Conversely, the likelihood equation models the full transaction $\mathbf{t}_j$, requiring terms for both present and absent items to compute the probability of the entire observed vector, which is necessary for posterior inference. Both equations incorporate a correlation adjustment $g(\cdot)$, but the scope of items considered and the events being modeled (joint presence of a subset versus the full transaction) are different. This difference ensures that the co-occurrence probability aligns with the needs of rule mining, while the likelihood supports the Bayesian updating of model parameters.

We have now computed the joint probability $p(I)$ for each posterior sample; if we draw $S$ samples from the posterior, we estimate $p(I)$ by averaging \footnote{Averaging is essentially integrating over the correlation structure.}:
\begin{equation} \label{eq:BARM_cooccurrence_probs}
    p(I) \approx \frac{1}{S} \sum_{s=1}^S p(I | \{p_i^{(s)}\}, \ell^{(s)})
\end{equation}
where $\{p_i^{(s)}\}, \ell^{(s)}$ are the $s$-th posterior samples. $p(I | \{p_i^{(s)}\}, \ell^{(s)})$ is calculated using Eq.\ref{eq:BARM_subset_prob_estimate_via_MCMC_posterior_samples}.

\subsection{The BARM algorithm}

We present the BARM algorithm in Algo.\ref{algo:BARM}. BARM replaces the GP model in GPAR with a direct Bayesian model of item presence probabilities, using \textit{Beta} priors for individual item probabilities and a correlation structure informed by item features to capture dependencies between items \footnote{A dependency-free version of BARM, which simplifies by setting $g(\cdot) = 1$, removes the feature matrix $X$, and uses analytical posterior for $\{p_i\}$, is provided in Section.\ref{sec:BARM_variant}.}. By employing MCMC sampling, BARM updates these probabilities with transaction data, estimating co-occurrence probabilities and deriving association rules in a manner similar to GPAR but with enhanced uncertainty quantification and flexibility to incorporate prior knowledge. While BARM may incur higher computational costs due to MCMC sampling, it avoids the memory-intensive covariance matrix operations of GPAR, offering a scalable and interpretable alternative for association rule mining in a Bayesian framework.

\begin{algorithm}[H]
\footnotesize
\caption{BARM: Bayesian association rule mining}
\label{algo:BARM}
\textbf{Input:} a set of items $\{1, 2, \dots, M\}$, each represented by a $d$-dimensional feature vector $\mathbf{x}_i$, forming an items feature matrix $X \in \mathbb{R}^{M \times d}$; a set of transactions $\mathcal{T} = \{\mathbf{t}_1, \mathbf{t}_2, \dots, \mathbf{t}_N\}$, where each $\mathbf{t}_j \in \{0, 1\}^M$; prior parameters for item presence probabilities $\{\alpha_i, \beta_i\}_{i=1}^M$; a correlation function $k(\cdot, \cdot)$, e.g. RBF-based; a minimum probability threshold $min\_prob$; a minimum confidence threshold $min\_conf$; number of MCMC (or ParVI) samples $S_{\text{MCMC}}$; number of posterior samples $S$ for probability estimation. \\
\textbf{Output:} a set of association rules $\mathcal{R}$ where each rule $r \in \mathcal{R}$ is of the form $A \rightarrow B$ with $A, B \subseteq \{1, 2, \dots, M\}$, satisfying the probability and confidence thresholds.

\vspace{1mm}\hrule\vspace{1mm}

\begin{algorithmic}[1]
\STATE Define prior distributions: $p_i \sim \text{Beta}(\alpha_i, \beta_i)$ for each item $i$, and length scale $\ell \sim \text{Gamma}(a, b)$ for the correlation function $k(\mathbf{x}_i, \mathbf{x}_j) = \exp(-||\mathbf{x}_i - \mathbf{x}_j||^2 / \ell^2)$. \hfill\textit{$\mathcal{O}(1)$}
\STATE Perform Bayesian inference using MCMC to sample from the posterior $p(\{p_i\}, \ell | \mathcal{T}, X)$ in Eq.\ref{eq:BARM_posterior}, generating $S_{\text{MCMC}}$ samples of $\{p_i, \ell\}$. \hfill\textit{$\mathcal{O}(S_{\text{MCMC}} \cdot N \cdot M)$}
\STATE Initialize an empty set $\mathcal{R} = \emptyset$ to store association rules. \hfill\textit{$\mathcal{O}(1)$}
\FOR{$m = 2, \dots, M$}
    \STATE Generate all possible size $m$ itemsets $I_m = \binom{\{1, 2, \dots, M\}}{m}$. \hfill\textit{$\mathcal{O}(\mathcal{C}_m^M)$}
    \FOR{itemset $I \in I_m$}
        \STATE Estimate co-occurrence probability $p(I)$ using posterior samples:
        $p(I) = \frac{1}{S} \sum_{s=1}^{S} p(I | \{p_i^{(s)}\}, \ell^{(s)})$, where $\{p_i^{(s)}\}, \ell^{(s)}$ are subsampled from MCMC samples. \hfill\textit{$\mathcal{O}(S \cdot m)$}
        \IF{$p(I) > min\_prob$}
            \FOR{each \textit{(antecedent, consequent)} split $(A, B)$ of $I$}
                \STATE Estimate antecedent probability $p(A)$ using posterior samples:
                $p(A) = \frac{1}{S} \sum_{s=1}^{S} p(A | \{p_i^{(s)}\}, \ell^{(s)})$. \hfill\textit{$\mathcal{O}(S \cdot (m-1))$}
                \STATE Compute confidence: $\text{conf}(A \rightarrow B) = \frac{p(I)}{p(A)}$. \hfill\textit{$\mathcal{O}(1)$}
                \IF{$\text{conf}(A \rightarrow B) > min\_conf$}
                    \STATE Add rule $A \rightarrow B$ to $\mathcal{R}$. \hfill\textit{$\mathcal{O}(1)$}
                \ENDIF
            \ENDFOR
        \ENDIF
    \ENDFOR
\ENDFOR
\STATE Return the rules set $\mathcal{R}$. \hfill\textit{$\mathcal{O}(1)$}
\end{algorithmic}
\end{algorithm}

\paragraph{Computational and memory complexity}
Step 2 uses MCMC to sample from the posterior of $\{p_i\}$ and $\ell$, replacing the GP optimisation in GPAR. MCMC typically requires multiple iterations to converge (e.g. MH or Gibbs sampling), the dominant cost is the evaluation of the likelihood over $ S_{\text{MCMC}} $ iterations and $N$ transactions, with $M$ parameters to update. The complexity is $\mathcal{O}(S_{\text{MCMC}} \cdot N \cdot M)$. Here we have ignored the cost of pre-computing the correlation adjustment factor $g(\cdot)$, therefore the cost of likelihood evaluation (Eq.\ref{eq:BARM_likelihood}) is dominated by the Bernoulli product $ \prod_{i \in I_j} p_i \prod_{i \notin I_j} (1 - p_i)$, which scales with $M$. If we are to take into account of the cost of computing $g(\cdot)$, the correlation function $ k(\mathbf{x}_i, \mathbf{x}_j) $ depends on the feature vectors $\mathbf{x}_i$ and $\mathbf{x}_j$, costing $ \mathcal{O}(d) $ per pair, but since $ g(\cdot) $ typically involves summing over pairs, the cost per transaction is approximately $ \mathcal{O}(M^2 d) $ in the worst case. Across $ N $ transactions, the likelihood evaluation costs $\mathcal{O}(N \cdot M^2 d)$ per MCMC iteration. The overall cost of MCMC sampling would therefore be $\mathcal{O}(S_{\text{MCMC}} \cdot (N \cdot M + N \cdot M^2 d))$, which is roughly $\mathcal{O}(S_{\text{MCMC}} \cdot N \cdot M^2 d)$.

Steps 4–18 mirror GPAR’s structure for generating itemsets and rules but use posterior samples to estimate probabilities: co-occurrence probabilities $p(I)$ and antecedent probabilities $p(A)$ are computed by averaging over $S$ posterior samples, with a cost of $\mathcal{O}(S \cdot m)$ and $\mathcal{O}(S \cdot (m-1))$, respectively, as they involve evaluating probabilities for $m$ or $m-1$ items per sample (evaluating the product over $ m = |I| $ items costs $ \mathcal{O}(m) $ per sample, and with $ S $ samples, the cost is $ \mathcal{O}(S \cdot m)$). We need to evaluate $\mathcal{O}(2^M)$ itemsets, leading to the total of $ \mathcal{O}(2^M \cdot S \cdot M) $. Again, we have assumed the correlation adjustment factor  $ g(\cdot) $ is either precomputed or simplified to a linear cost (e.g. using a sparse approximation); if we are to consider the cost of computing $g(\cdot)$, it costs $ \mathcal{O}(m^2) $ per sample, and the cost of the rule inference would be $ \mathcal{O}(2^M \cdot S \cdot (M+M^2)) $ or $ \mathcal{O}(2^M \cdot S \cdot M^2) $ in the worst case.

The total cost of BARM sums up to $\mathcal{O}(S_{\text{MCMC}} \cdot N \cdot M + 2^M \cdot S \cdot M)$, or $\mathcal{O}(S_{\text{MCMC}} \cdot N \cdot M^2 \cdot d + 2^M \cdot S \cdot M^2)$ if considering computing the pairwise correlation adjustment. Unlike GPAR, BARM avoids the $\mathcal{O}(M^3)$ cost of GP covariance matrix operations, making it potentially more scalable for large $M$, though MCMC sampling can be expensive if $S_{\text{MCMC}}$ is large.

Storing the feature matrix $X$ requires $\mathcal{O}(M d)$, transaction matrix $\mathcal{T}$ requires $\mathcal{O}(NM)$, posterior samples for $\{p_i\}$ and $\ell$ requires $\mathcal{O}(S_{\text{MCMC}} \cdot M)$, making the memory usage manageable but growing with $S_{\text{MCMC}}$. 

\subsection{Comparison with GPAR}  
Unlike GPAR, which uses a GP to model latent variables $z$ and Monte Carlo sampling to estimate probabilities, BARM directly models item presence probabilities $p_i$ with Bayesian priors, avoiding the need for a covariance matrix $K$ ($\mathcal{O}(M^2)$ in GPAR) and its associated Cholesky decomposition ($\mathcal{O}(M^3)$). In terms of uncertainty quantification, uncertainty in GPAR arises from its marginalised posterior and associated Monte Carlo sampling and doesn't account for hyper-parameter uncertainties (kernel hyper-parameters are point estimated using transaction data); BARM represents uncertainty as a full posterior distribution over $p(I)$, derived from the variability across posterior samples of $\{p_i\}$ and $\ell$ and accounting for uncertainties in both the membership beliefs and hyper-parameters, which provides a comprehensive view of uncertainty and important for interpreting rare or uncertain rules
- it allows users to assess the reliability of low probability estimates. In terms of flexibility, BARM applies priors on $p_i$ and $\ell$, allowing for the incorporation of domain knowledge (e.g. expected item frequencies), whereas GPAR represents prior knowledge via the chosen kernel and tunes it using data through optimisation, which may overfit to the transaction data \footnote{In Bayesian modelling, prior can be viewed as a kind of regularisation, maximizing a posterior (MAP) is in fact maximizing the likelihood (MLE) plus the prior, which reduces the chance of overfitting.}. In terms of computational efficiency, although BARM’s MCMC sampling can be computationally expensive compared to GPAR’s GP kernel optimisation (e.g. via maximum likelihood), its linear in $M$ complexity \footnote{There is no optimisation in this BARM setting, one just needs to generate samples from the (intractable) full posterior to estimate the posterior probability for each candidate rule during the rule generation stage; in contrast, GPAR requires tuning the hyper-parameter (which is fast though) and sampling the marginalised posterior.} per itermset is cheaper than the Monte Carlo posterior sampling used in GPAR which costs $\mathcal{O}(M^3)$ per itermset, potentially scaling better for large $M$ if $S_{\text{MCMC}}$ is controlled.

BARM's full Bayesian approach lends it power in rare rule inference. Rare rules often have low co-occurrence probabilities, making them sensitive to small changes in model parameters. In GPAR, the lack of uncertainty quantification in $\ell$ and its reliance on a point estimate for $p(I)$ can lead to misleading conclusions about such rules, as the user cannot distinguish between a reliably low probability and an estimate with high uncertainty. BARM’s posterior distribution over $p(I)$ provides a more robust framework for interpreting such rules, as it explicitly captures the range of possible probabilities and their associated uncertainty, enabling better informed decision-making in applications where rare rules are of interest (e.g. identifying unusual purchasing patterns, symptom-diagnosis association, or accident correlations). 

Overall, BARM offers greater flexibility by incorporating prior knowledge through assigned priors on hyper-parameters, potentially improving robustness and uncertainty quantification in rule mining, and its overall computational complexity may be less than GPAR. GPAR can also provide a notation of variation of uncertainty (although we used the average posterior samples in Section.\ref{sec:GPAR_algorithm}). BARM, on the other hand, requires careful prior specification to avoid bias (e.g. in presence of small transaction data). More discussions can be found in Appendix.\ref{app:compare_GPAR_and_BARM}.

\subsection{The item dependency-free BARM} \label{sec:BARM_variant}

We can simplify the BARM framework by removing the dependency adjustment (i.e. setting $g(\cdot) = 1$) in both the likelihood and co-occurrence probability computations. This modification eliminates the need for modeling pairwise correlations between items, which in turn removes the dependency on the feature vector representations $\mathbf{x}_i$ and the associated correlation function $k(\mathbf{x}_i, \mathbf{x}_j) = \exp(-||\mathbf{x}_i - \mathbf{x}_j||^2 / \ell^2)$.

By setting $g(\{\rho_{ij}\}) = 1$, we assume that the presence of items in a transaction is independent \footnote{This strong assumption seems unreasonable for ARM, however, it does simplify the computation and provides some insights.}, meaning the joint probability of item presences factorizes directly as the product of their individual probabilities. The original BARM likelihood Eq.\ref{eq:BARM_likelihood} is therefore reduced to:
\[
  p(\mathbf{t}_j | \{p_i\}) = \prod_{i \in I_j} p_i \prod_{i \notin I_j} (1 - p_i)
\]
which is now a product of independent Bernoulli distributions for each item. And the co-occurrence probability Eq.\ref{eq:BARM_subset_prob_estimate_via_MCMC_posterior_samples} becomes:
\begin{equation} \label{eq:BARM_cooccurrence_prob}
    p(I | \{p_i^{(s)}\}) = \prod_{i \in I} p_i^{(s)}
\end{equation}
assuming independence among the items in $I$.

As the correlation function $k(\mathbf{x}_i, \mathbf{x}_j)$ and length scale $\ell$ are no longer needed to compute $\rho_{ij}$, the feature matrix $X \in \mathbb{R}^{M \times d}$ and the prior on $\ell$ (i.e. $\ell \sim Gamma(a, b)$) can be removed from the model. This simplifies the input to just the transaction data $\mathcal{T}$. The posterior distribution now involves only $\{p_i\}$, i.e. $p(\{p_i\} | \mathcal{T})$, and the MCMC sampling step no longer needs to sample $\ell$. Further, if Beta priors are used, as the likelihood is a product of independent Bernoullis, the posterior for each $p_i$ can be computed analytically (as Beta prior is conjugate to Bernoulli likelihood), which eliminates the need for MCMC sampling. 

Given the simplified likelihood and Beta prior \footnote{Beta distribution is often used as a prior distribution to describe the distribution of a probability value \cite{Mackay2003Information}. E.g. to represent ignorance of prior parameter values we can use Beta(1,1), Beta(0,0), Beta(1/2,1/2), etc.} $p_i \sim \text{Beta}(\alpha_i, \beta_i)$, the posterior for each $p_i$ can be derived analytically due to conjugacy \footnote{In Bayesian statistics, conjugacy implies that the posterior distribution, after updating with observed data, remains within the same Beta family, facilitating analytical computation of the posterior parameters. Conjugacy simplifies posterior computation, as the parameters of the posterior can be updated analytically based on the data, without requiring numerical methods like MCMC sampling. Beta distribution (of the first kind) is the conjugate prior probability distribution for Bernoulli, binomial, negative binomial, and geometric distributions. }, i.e. Beta distribution is the conjugate prior for the Bernoulli likelihood (or its extension to binomial data). For item $i$, let $n_i = \sum_{j=1}^N \mathbf{t}_j[i]$ be the number of transactions where item $i$ is present, and $N - n_i$ the number where it is absent. For a set of $ N $ transactions, the presence of item $ i $ in each transaction $ \mathbf{t}_j[i] $ is modeled as a Bernoulli trial with success probability $ p_i $. The likelihood of observing the data $ \{\mathbf{t}_j[i]\}_{j=1}^N $ (i.e. the number of times item $ i $ is present, $ n_i = \sum_{j=1}^N \mathbf{t}_j[i] $, and absent, $ N - n_i $) is given by the binomial likelihood:
\[
p(\{\mathbf{t}_j[i]\}_{j=1}^N | p_i) = \binom{N}{n_i} p_i^{n_i} (1 - p_i)^{N - n_i}
\]
where $ \binom{N}{n_i} $ is the binomial coefficient - a constant with respect to $ p_i $, it can be ignored in the posterior computation up to a normalizing constant. The posterior distribution is proportional to the product of the likelihood and the prior:
\[
p(p_i | \{\mathbf{t}_j[i]\}_{j=1}^N) \propto p(\{\mathbf{t}_j[i]\}_{j=1}^N | p_i) \cdot p(p_i)
\]
Substituting the prior and likelihood expressions gives:
\[
p(p_i | \{\mathbf{t}_j[i]\}_{j=1}^N) \propto \left[ p_i^{n_i} (1 - p_i)^{N - n_i} \right] \cdot \frac{p_i^{\alpha_i - 1} (1 - p_i)^{\beta_i - 1}}{B(\alpha_i, \beta_i)}
\]
where $ B(\alpha_i, \beta_i) = \frac{\Gamma(\alpha_i) \Gamma(\beta_i)}{\Gamma(\alpha_i + \beta_i)} $ is the Beta function, a normalizing constant for the prior.
The \textit{pdf} of the Beta distribution is \footnote{See Appendix.\ref{app:Beta_dist} an intro about Beta distribution.} $ p(p_i) = \frac{p_i^{\alpha_i - 1} (1 - p_i)^{\beta_i - 1}}{B(\alpha_i, \beta_i)} $, and the likelihood $ p_i^{n_i} (1 - p_i)^{N - n_i} $ has the same functional form. Combining these, the unnormalized posterior becomes \footnote{The Beta distribution is conjugate to the binomial likelihood, meaning that if the prior distribution for a probability $p$ is Beta, the posterior distribution after observing binomial data is also Beta, with updated parameters.}:
\[
p(p_i | \{\mathbf{t}_j[i]\}_{j=1}^N) \propto p_i^{n_i} (1 - p_i)^{N - n_i} \cdot p_i^{\alpha_i - 1} (1 - p_i)^{\beta_i - 1} = p_i^{n_i + \alpha_i - 1} (1 - p_i)^{N - n_i + \beta_i - 1}
\]
which is the kernel of a Beta distribution with parameters $ \alpha_i + n_i $ and $ \beta_i + N - n_i $. The normalizing constant is therefore $ B(\alpha_i + n_i, \beta_i + N - n_i) $, so the posterior is again Beta distributed:
\begin{equation} \label{eq:BARM_Beta_posterior}
    p_i | \mathcal{T} \sim \text{Beta}(\alpha_i + n_i, \beta_i + N - n_i)
\end{equation}
The conjugacy holds because the Beta prior and the Bernoulli (or binomial) likelihood share the same exponential family form, where the sufficient statistics (counts of successes $ n_i $ and failures $ N - n_i $) update the prior parameters directly. This analytical update eliminates the need for iterative sampling methods such as MCMC, as the posterior parameters are cirectly computed in closed form. 

The updated BARM algorithm without dependency adjustment is presented in Algo.\ref{algo:denpendency_free_BARM}. 

\begin{algorithm}[H]
\footnotesize
\caption{BARM: Bayesian association rule mining (without item dependency)}
\label{algo:denpendency_free_BARM}
\textbf{Input:} a set of items $\{1, 2, \dots, M\}$; a set of transactions $\mathcal{T} = \{\mathbf{t}_1, \mathbf{t}_2, \dots, \mathbf{t}_N\}$, where each $\mathbf{t}_j \in \{0, 1\}^M$; prior parameters for item presence probabilities $\{\alpha_i, \beta_i\}_{i=1}^M$; a minimum probability threshold $min\_prob$; a minimum confidence threshold $min\_conf$; number of samples $S$ for probability estimation. \\
\textbf{Output:} a set of association rules $\mathcal{R}$ where each rule $r \in \mathcal{R}$ is of the form $A \rightarrow B$ with $A, B \subseteq \{1, 2, \dots, M\}$, satisfying the probability and confidence thresholds.

\vspace{1mm}\hrule\vspace{1mm}

\begin{algorithmic}[1]
\STATE Compute posterior parameters for each item $i$: $n_i = \sum_{j=1}^N \mathbf{t}_j[i]$, then $p_i | \mathcal{T} \sim \text{Beta}(\alpha_i + n_i, \beta_i + N - n_i)$. \hfill\textit{$\mathcal{O}(N \cdot M)$}
\STATE Sample $S$ values $\{p_i^{(s)}\}_{s=1}^S$ from the posterior $\text{Beta}(\alpha_i + n_i, \beta_i + N - n_i)$ for each item $i$. \hfill\textit{$\mathcal{O}(S \cdot M)$}
\STATE Initialize an empty set $\mathcal{R} = \emptyset$ to store association rules. \hfill\textit{$\mathcal{O}(1)$}
\FOR{$m = 2, \dots, M$}
    \STATE Generate all possible size $m$ itemsets $I_m = \binom{\{1, 2, \dots, M\}}{m}$. \hfill\textit{$\mathcal{O}(\mathcal{C}_m^M)$}
    \FOR{itemset $I \in I_m$}
        \STATE Estimate co-occurrence probability $p(I)$ using posterior samples:
        $p(I) = \frac{1}{S} \sum_{s=1}^{S} \prod_{i \in I} p_i^{(s)}$. \hfill\textit{$\mathcal{O}(S \cdot m)$}
        \IF{$p(I) > min\_prob$}
            \FOR{each \textit{(antecedent, consequent)} split $(A, B)$ of $I$}
                \STATE Estimate antecedent probability $p(A)$ using posterior samples:
                $p(A) = \frac{1}{S} \sum_{s=1}^{S} \prod_{i \in A} p_i^{(s)}$. \hfill\textit{$\mathcal{O}(S \cdot (m-1))$}
                \STATE Compute confidence: $\text{conf}(A \rightarrow B) = \frac{p(I)}{p(A)}$. \hfill\textit{$\mathcal{O}(1)$}
                \IF{$\text{conf}(A \rightarrow B) > min\_conf$}
                    \STATE Add rule $A \rightarrow B$ to $\mathcal{R}$. \hfill\textit{$\mathcal{O}(1)$}
                \ENDIF
            \ENDFOR
        \ENDIF
    \ENDFOR
\ENDFOR
\STATE Return the rules set $\mathcal{R}$. \hfill\textit{$\mathcal{O}(1)$}
\end{algorithmic}
\end{algorithm}

Compared to the original BARM in Algo.\ref{algo:BARM}, in the dependence-free BARM, we set  $ g(\cdot) = 1 $ and removed the feature matrix $X$, correlation function $k(\cdot, \cdot)$, and length scale $\ell$, as they are no longer needed without dependency modeling. Removing the correlation adjustment reduces the likelihood model to independent Bernoullis, which preserves the conjugacy property and eliminates the need for feature vectors and correlation modeling. This simplification allows for an analytical posterior computation using Beta-Bernoulli conjugacy, significantly reducing computational overhead. However, this comes at the cost of losing the ability to capture real-world dependencies (e.g. complementary items such as bread and butter being more likely to appear together), which may affect the quality of mined rules in datasets where items exhibit strong correlations. However, it makes the framework more computationally efficient and easier to implement.

\paragraph{Complexity} Step 1 computes the posterior parameters for each $p_i$ analytically using the conjugate Beta-Bernoulli update, with complexity $\mathcal{O}(N \cdot M)$ to count the occurrences $n_i$. Step 2 samples $S$ values from each posterior $Beta(\alpha_i + n_i, \beta_i + N - n_i)$, replacing the MCMC sampling step in the original BARM, with complexity $\mathcal{O}(S \cdot M)$. Step 7 computes the co-occurrence probability $p(I)$ of a sub-itemset $I$ as the average of $\prod_{i \in I} p_i^{(s)}$ over $S$ samples, with complexity $\mathcal{O}(S \cdot m)$. Step 9 computes $p(A)$ similarly for the antecedent $A$, with complexity $\mathcal{O}(S \cdot (m-1))$. 

The removal of MCMC sampling reduces the inference cost from $\mathcal{O}(S_{\text{MCMC}} \cdot N \cdot M)$ to $\mathcal{O}(N \cdot M)$. The total complexity is now dominated by the itemset evaluation loop, which remains $\mathcal{O}(2^M \cdot S \cdot M)$. The absence of correlation computations makes each probability estimation simpler. The total cost of this dependency-free BARM sums up to $\mathcal{O}(N \cdot M + 2^M \cdot S \cdot M)$, which is quite manageable compared to the original BARM cost $\mathcal{O}(S_{\text{MCMC}} \cdot N \cdot M + 2^M \cdot S \cdot M)$.

\subsection{BARM experiment}
\label{sec:BARM_experiments}

Here we evaluate the performance of the item dependency-free BARM method on the \textit{Synthetic 1} dataset which was used in the GPAR experiment (Section.\ref{sec:GPAR_experiments}). This experiment was conducted on the same laptop (specs see Section.\ref{sec:GPAR_experiments}) as used for the GPAR synthetic experiments. We utilized Python with the \textit{PyMC} library (version 5.23.0) for MCMC sampling, alongside \textit{NumPy}, \textit{Pandas}, and \textit{Psutil} for data processing and metric recording.

\subsubsection{Data and processing methods}

The \textit{Synthetic 1} dataset, identical to that used in the GPAR experiments, comprises 1000 transactions involving 10 items, with a feature matrix $X \in \mathbb{R}^{10 \times 10}$ sampled from a standard normal distribution and a transaction matrix $\mathcal{T} \in \{0, 1\}^{1000 \times 10}$. Each transaction $\mathbf{t}_j$ is a binary vector indicating the presence (1) or absence (0) of each item, generated using a multivariate normal distribution with an RBF kernel, retaining items with positive latent values (see Appendix.\ref{app:synthetic_data_tests} for details). For the dependency-free BARM, the feature matrix $X$ was not utilized, as the method assumes item independence, eliminating the need for correlation modeling via item features.

We assigned uniform Beta priors, i.e. $\text{Beta}(1, 1)$, to each item’s presence probability to reflect a lack of prior knowledge about item frequencies. The likelihood was defined as a product of independent Bernoulli distributions for each transaction. We generated samples of the analytical posterior $\text{Beta}(\alpha_i + n_i, \beta_i + N - n_i)$ (Eq.\ref{eq:BARM_Beta_posterior}) via direct sampling \footnote{Sampling from a Beta distribution is straightforward using standard random number generators, as the Beta distribution is well-supported in numerical libraries such \textit{NumPy}.}, drawing 1000 posterior samples for each item’s presence (posterior) probability. The co-occurrence probability for itemset $I$, i.e. $p(I)$ (the support), is approximated by the expected value of $p(I \mid \{p_i^{(s)}\})$ over the $S=1000$ posterior samples, i.e. averaging Eq.\ref{eq:BARM_cooccurrence_prob}:
\begin{equation} \label{eq:BARM_cooccurrence_prob_avg}
    p(I) \approx \frac{1}{S} \sum_{s=1}^S p(I \mid \{p_i^{(s)}\}) = \frac{1}{S} \sum_{s=1}^S \prod_{i \in I} p_i^{(s)}
\end{equation}
Alternatively, instead of averaging, one can obtain a sample-based distribution of $p(I)$ using these $S=1000$ posterior samples.

As in GPAR experiments, we recorded the performance metrics, i.e. runtime (in seconds), memory usage (in MB), number of frequent itemsets, and number of rules generated, for minimum support thresholds of 0.1, 0.2, 0.3, 0.4, and 0.5, with a fixed minimum confidence threshold of 0.5. Frequent itemsets were identified by comparing co-occurrence probabilities against the minimum support threshold, and rules were generated by evaluating all possible antecedent-consequent splits, computing confidence as the ratio of itemset probability to antecedent probability. Rules meeting the minimum confidence threshold were retained, and lift was calculated as the ratio of itemset probability to the product of antecedent and consequent probabilities.

\subsubsection{Results}

The performance of the dependency-free BARM method, utilizing direct sampling from the analytical Beta posterior, on the Synthetic 1 dataset is summarized in Table.\ref{tab:BARM_metrics_Synthetic1}. At a minimum support of 0.1, BARM required 1.1821 seconds and negligible memory, generating 876 frequent itemsets and 12029 rules. As the support threshold increased, computational demands decreased significantly: at 0.2, runtime was 0.2314 seconds with negligible memory usage, producing 385 itemsets and 3024 rules. This reduction reflects stricter criteria for frequent itemsets, limiting the number of valid rules generated. Both GPAR and BARM can produce uncertainty quantification for each co-occurrence probability $p(I)$, as well as probabilistic confidence and lift, by sampling from the GP or Bayesian posterior distributions, yielding full distributions for these quantities; however, in our results, we only report the sample averages, as in Eq.\ref{eq:BARM_cooccurrence_prob_avg}, for simplicity.

\begin{table}[H]
\centering
\scriptsize
\begin{threeparttable}
\caption{Performance metrics of \textbf{BARM} on \textit{Synthetic 1} dataset}
\label{tab:BARM_metrics_Synthetic1}
\begin{tabular}{p{2cm} p{2cm} p{2cm} p{2cm} p{2cm}}
\toprule
Min Support & Runtime (s) & Memory (MB) & Frequent Itemsets & Rules \\
\midrule
0.1 & 1.1821 & 0.0 & 876 & 12029 \\
0.2 & 0.2314 & 0.0 & 385 & 3024 \\
0.3 & 0.0734 & 0.0 & 165 & 802 \\
0.4 & 0.0481 & 0.0 & 45 & 90 \\
0.5 & 0.0451 & 0.0 & 44 & 88 \\
\bottomrule
\end{tabular}
\begin{tablenotes}
\item[1] Total number of transactions: 1000.
\item[2] Minimum confidence used: 0.5.
\item[3] Direct sampling from analytical Beta posterior with 1000 samples.
\item[4] Co-occurrence probabilities estimated with 1000 posterior samples ($S = 1000$).
\end{tablenotes}
\end{threeparttable}
\end{table}

Compared to GPAR variants and classic ARM methods (Apriori, FP-Growth, Eclat) on Synthetic 1 (see Tables.\ref{tab:synthetic1_comparison_runtime} to.\ref{tab:synthetic1_comparison_noRules}), BARM exhibits competitive performance characteristics. At a minimum support of 0.1, BARM’s runtime (1.1821 seconds) is higher than GPAR’s neural network kernel (0.133 seconds), NTK (0.449 seconds), Apriori (0.483 seconds), and Eclat (0.556 seconds), but lower than FP-Growth (2.967 seconds) and significantly lower than GPAR RBF/shifted RBF (10.179/8.930 seconds). BARM’s memory usage (0.0 MB) is comparable to GPAR variants, Apriori, and FP-Growth (0.0–0.063 MB), and lower than Eclat’s peak (18.149 MB). BARM generates fewer itemsets (876) and rules (12029) than GPAR RBF/shifted RBF (both \footnote{The shifted RBF based GPAR is the same as the RBF based GPAR as its shifted distance $d_0=0$ after optimisation.} with 1013 itemsets, 54036 rules), Apriori/FP-Growth (both with 1023 itemsets, 57002 rules), and Eclat (1033 itemsets, 57002 rules), but more than GPAR neural network (140 itemsets, 322 rules) and NTK (305 itemsets, 1485 rules). At higher support thresholds (0.4–0.5), BARM’s runtime (0.0451–0.0481 seconds) approaches that of GPAR neural network (0.002–0.005 seconds) and classic methods (0.211–0.635 seconds), while its itemset (44–45) and rule counts (88–90) are comparable to GPAR RBF/shifted RBF (1–67 itemsets, 2–782 rules), though classic methods maintain higher counts (746–1033 itemsets, 20948–57002 rules).

Tables.\ref{tab:BARM_mined_rules_rankedBySupport_minSupport01_Synthetic1},.\ref{tab:BARM_mined_rules_rankedByConfidence_minSupport01_Synthetic1}, and.\ref{tab:BARM_mined_rules_rankedByLift_minSupport01_Synthetic1} present the top 10 association rules mined by BARM at a minimum support of 0.1, ranked by support, confidence, and lift, respectively. Table.\ref{tab:BARM_mined_rules_rankedBySupport_minSupport01_Synthetic1} shows high-support rules such as \(\text{item}_2 \to \text{item}_9\) (support 0.5356, confidence 0.7308, lift 1.0000) and \(\text{item}_9 \to \text{item}_2\) (support 0.5356, confidence 0.7329, lift 1.0000), indicating frequent but near-independent co-occurrences, which is consistent with BARM’s items independence assumption. In contrast, GPAR RBF/shifted RBF (Table.\ref{tab:RBF_GPAR_mined_rules_rankedBySupport_minSupport01_Synthetic1}) identifies rules such as \(\text{item}_5 \to \text{item}_8, \text{item}_6\) (support 0.5200, confidence 0.9123, lift 2.3392), while Apriori/FP-Growth (Tables.\ref{tab:Apriori_mined_rules_rankedBySupport_minSupport01_Synthetic1},.\ref{tab:FP-Growth_mined_rules_rankedBySupport_minSupport01_Synthetic1}) focus on \(\text{item}_7 \to \text{item}_5\) (support 0.6620, confidence 0.9220, lift 1.2700). Table.\ref{tab:BARM_mined_rules_rankedByConfidence_minSupport01_Synthetic1} lists rules with confidence around 0.7329, such as \(\text{item}_1, \text{item}_4, \text{item}_6, \text{item}_9 \to \text{item}_2\) (support 0.1916, confidence 0.7329), far below GPAR RBF/shifted RBF’s top confidence of 2.2727 (Table.\ref{tab:RBF_GPAR_mined_rules_rankedByConfidence_minSupport01_Synthetic1}) or Apriori/FP-Growth’s 0.9981 (Tables.\ref{tab:Apriori_mined_rules_rankedByConfidence_minSupport01_Synthetic1},.\ref{tab:FP-Growth_mined_rules_rankedByConfidence_minSupport01_Synthetic1}). Table.\ref{tab:BARM_mined_rules_rankedByLift_minSupport01_Synthetic1} shows near-trivial lift values (e.g. 1.0001 for \(\text{item}_0, \text{item}_2, \text{item}_3, \text{item}_6 \to \text{item}_4, \text{item}_7\)), reflecting weak associations due to the independence assumption, unlike GPAR RBF/shifted RBF’s high lift (12.1212, Table.\ref{tab:RBF_GPAR_mined_rules_rankedByLift_minSupport01_Synthetic1}) or NTK’s 8.5470 (Table.\ref{tab:NTK_GPAR_mined_rules_rankedByLift_minSupport01_Synthetic1}).

\begin{table}[H]
\centering
\scriptsize
\begin{threeparttable}
\caption{Top 10 association rules mined by \textbf{BARM} (ranked in descending order by \textit{\color{red}{support}})}
\label{tab:BARM_mined_rules_rankedBySupport_minSupport01_Synthetic1}
\begin{tabular}{r >{\raggedright\arraybackslash}p{4cm} >{\raggedright\arraybackslash}p{4cm} p{1cm} p{1cm} p{1cm}}
\toprule
\# & Antecedent & Consequent & \color{red}{Supp} & Conf & Lift \\
\midrule
1 & item\_2 & item\_9 & 0.5356 & 0.7308 & 1.0000 \\
2 & item\_9 & item\_2 & 0.5356 & 0.7329 & 1.0000 \\
3 & item\_2 & item\_5 & 0.5312 & 0.7247 & 1.0000 \\
4 & item\_5 & item\_2 & 0.5312 & 0.7329 & 1.0000 \\
5 & item\_2 & item\_6 & 0.5312 & 0.7248 & 1.0000 \\
6 & item\_6 & item\_2 & 0.5312 & 0.7329 & 1.0000 \\
7 & item\_5 & item\_9 & 0.5296 & 0.7307 & 1.0000 \\
8 & item\_9 & item\_5 & 0.5296 & 0.7247 & 1.0000 \\
9 & item\_6 & item\_9 & 0.5297 & 0.7308 & 1.0000 \\
10 & item\_9 & item\_6 & 0.5297 & 0.7248 & 1.0000 \\
\bottomrule
\end{tabular}
\begin{tablenotes}
\item[1] Total number of transactions: 1000.
\item[2] Minimum support used in mining: 0.1.
\item[3] Minimum confidence used: 0.5.
\item[4] Items correspond to synthetic dataset indices.
\end{tablenotes}
\end{threeparttable}
\end{table}

\begin{table}[H]
\centering
\scriptsize
\begin{threeparttable}
\caption{Top 10 association rules mined by \textbf{BARM} (ranked in descending order by \textit{\color{red}{confidence}})}
\label{tab:BARM_mined_rules_rankedByConfidence_minSupport01_Synthetic1}
\begin{tabular}{r >{\raggedright\arraybackslash}p{6cm} >{\raggedright\arraybackslash}p{3cm} p{1cm} p{1cm} p{1cm}}
\toprule
\# & Antecedent & Consequent & Supp & \color{red}{Conf} & Lift \\
\midrule
1 & item\_1, item\_4, item\_6, item\_9 & item\_2 & 0.1916 & 0.7329 & 1.0000 \\
2 & item\_0, item\_1, item\_4, item\_6, item\_9 & item\_2 & 0.1381 & 0.7329 & 1.0000 \\
3 & item\_1, item\_4, item\_5, item\_6, item\_9 & item\_2 & 0.1388 & 0.7329 & 1.0000 \\
4 & item\_1, item\_3, item\_4, item\_6, item\_9 & item\_2 & 0.1371 & 0.7329 & 1.0000 \\
5 & item\_1, item\_4, item\_6, item\_7, item\_9 & item\_2 & 0.1374 & 0.7329 & 1.0000 \\
6 & item\_0, item\_1, item\_4, item\_5, item\_6, item\_9 & item\_2 & 0.1001 & 0.7329 & 1.0000 \\
7 & item\_1, item\_4, item\_6 & item\_2 & 0.2622 & 0.7329 & 1.0000 \\
8 & item\_1, item\_4, item\_9 & item\_2 & 0.2643 & 0.7329 & 1.0000 \\
9 & item\_0, item\_1, item\_4, item\_6 & item\_2 & 0.1890 & 0.7329 & 1.0000 \\
10 & item\_1, item\_4, item\_5, item\_6 & item\_2 & 0.1900 & 0.7329 & 1.0000 \\
\bottomrule
\end{tabular}
\begin{tablenotes}
\item[1] Total number of transactions: 1000.
\item[2] Minimum support used in mining: 0.1.
\item[3] Minimum confidence used in filtering candidate rules: 0.5.
\item[4] Items correspond to synthetic dataset indices.
\end{tablenotes}
\end{threeparttable}
\end{table}

\begin{table}[H]
\centering
\scriptsize
\begin{threeparttable}
\caption{Top 10 association rules mined by \textbf{BARM} (ranked in descending order by \textit{\color{red}{lift}})}
\label{tab:BARM_mined_rules_rankedByLift_minSupport01_Synthetic1}
\begin{tabular}{r >{\raggedright\arraybackslash}p{5cm} >{\raggedright\arraybackslash}p{5cm} p{1cm} p{1cm} p{1cm}}
\toprule
\# & Antecedent & Consequent & Supp & Conf & \color{red}{Lift} \\
\midrule
1 & item\_0, item\_2, item\_3, item\_6 & item\_4, item\_7 & 0.1380 & 0.5035 & 1.0001 \\
2 & item\_0, item\_2, item\_3, item\_5, item\_6 & item\_4, item\_7 & 0.1000 & 0.5035 & 1.0001 \\
3 & item\_1, item\_4, item\_7 & item\_0, item\_3 & 0.1826 & 0.5160 & 1.0001 \\
4 & item\_1, item\_4, item\_6, item\_7 & item\_0, item\_3 & 0.1324 & 0.5160 & 1.0001 \\
5 & item\_0, item\_3, item\_6 & item\_4, item\_7 & 0.1883 & 0.5035 & 1.0001 \\
6 & item\_1, item\_2, item\_4, item\_7 & item\_0, item\_6 & 0.1355 & 0.5226 & 1.0001 \\
7 & item\_0, item\_3, item\_5, item\_6 & item\_4, item\_7 & 0.1364 & 0.5035 & 1.0001 \\
8 & item\_1, item\_2, item\_4, item\_7 & item\_0, item\_3 & 0.1338 & 0.5160 & 1.0001 \\
9 & item\_1, item\_3, item\_4, item\_7 & item\_0, item\_6 & 0.1324 & 0.5226 & 1.0001 \\
10 & item\_1, item\_4, item\_5, item\_7 & item\_0, item\_6 & 0.1340 & 0.5226 & 1.0001 \\
\bottomrule
\end{tabular}
\begin{tablenotes}
\item[1] Total number of transactions: 1000.
\item[2] Minimum support used in mining: 0.1.
\item[3] Minimum confidence used in filtering candidate rules: 0.5.
\item[4] Items correspond to synthetic dataset indices.
\item[5] Lift values are close to 1.0000, making the ordering sensitive to small differences.
\end{tablenotes}
\end{threeparttable}
\end{table}

Table.\ref{tab:BARM_analysis_minSupport01_Synthetic1} details the top 10 rules with \textit{item 2} as the consequent, ranked by confidence, for a minimum support level 0.1. All rules exhibit a confidence around 0.7329, with the top rule, \([1, 4, 6, 9] \to [2]\), having a support of 0.1916 and a support count of 496 transactions \footnote{This mismatch between the Bayesian posterior probability (i.e. the support) and the empirical frequency 496/1000 is very interesting.}. Other rules, such as \([1, 4, 6] \to [2]\) (support 0.2622, count 527), indicate frequent co-occurrences involving item 2. The example transactions, consistently including all items (e.g. \([0, 1, 2, 3, 4, 5, 6, 7, 8, 9]\)), reflect the dense nature of the Synthetic 1 dataset. Comparatively, GPAR RBF/shifted RBF (Table.\ref{tab:RBF_GPAR_analysis_minSupport01_Synthetic1}) yields higher-confidence rules (e.g. \([0, 1, 3, 4, 8] \to [2, 7]\), confidence 2.1176, support 0.3600), while Apriori/FP-Growth (Table.\ref{tab:Apriori_analysis_minSupport01_Synthetic1}) achieve near-perfect confidence (0.9855) for complex rules, and Eclat (Table.\ref{tab:Eclat_analysis_minSupport01_Synthetic1}) offers high confidence (1.0571) for simpler rules such as \([0, 1, 9] \to [2]\).

\begin{table}[H]
\centering
\scriptsize
\begin{threeparttable}
\caption{Top 10 association rules mined by \textbf{BARM} with item 2 on RHS (min support = 0.1, rules ranked by \color{red}{confidence})}
\label{tab:BARM_analysis_minSupport01_Synthetic1}
\begin{tabular}{r >{\raggedright\arraybackslash}p{2.6cm} p{1cm} p{1cm} p{1cm} >{\raggedright\arraybackslash}p{1cm} >{\raggedright\arraybackslash}p{5cm}}
\toprule
\# & Rule & \color{red}{Conf} & Supp & Supp Count & Example Trans.\ Indices & Example Transactions \\
\midrule
1 & [1, 4, 6, 9] $\to$ [2] & 0.7329 & 0.1916 & 496 & [1, 2, 3] & [[0, 1, 2, 3, 4, 5, 6, 7, 8, 9], [0, 1, 2, 3, 4, 5, 6, 7, 8, 9], [0, 1, 2, 3, 4, 5, 6, 7, 8, 9]] \\
2 & [0, 1, 4, 6, 9] $\to$ [2] & 0.7329 & 0.1381 & 481 & [1, 2, 3] & [[0, 1, 2, 3, 4, 5, 6, 7, 8, 9], [0, 1, 2, 3, 4, 5, 6, 7, 8, 9], [0, 1, 2, 3, 4, 5, 6, 7, 8, 9]] \\
3 & [1, 4, 5, 6, 9] $\to$ [2] & 0.7329 & 0.1388 & 483 & [1, 2, 3] & [[0, 1, 2, 3, 4, 5, 6, 7, 8, 9], [0, 1, 2, 3, 4, 5, 6, 7, 8, 9], [0, 1, 2, 3, 4, 5, 6, 7, 8, 9]] \\
4 & [1, 3, 4, 6, 9] $\to$ [2] & 0.7329 & 0.1371 & 485 & [1, 2, 3] & [[0, 1, 2, 3, 4, 5, 6, 7, 8, 9], [0, 1, 2, 3, 4, 5, 6, 7, 8, 9], [0, 1, 2, 3, 4, 5, 6, 7, 8, 9]] \\
5 & [1, 4, 6, 7, 9] $\to$ [2] & 0.7329 & 0.1374 & 487 & [1, 2, 3] & [[0, 1, 2, 3, 4, 5, 6, 7, 8, 9], [0, 1, 2, 3, 4, 5, 6, 7, 8, 9], [0, 1, 2, 3, 4, 5, 6, 7, 8, 9]] \\
6 & [0, 1, 4, 5, 6, 9] $\to$ [2] & 0.7329 & 0.1001 & 474 & [1, 2, 3] & [[0, 1, 2, 3, 4, 5, 6, 7, 8, 9], [0, 1, 2, 3, 4, 5, 6, 7, 8, 9], [0, 1, 2, 3, 4, 5, 6, 7, 8, 9]] \\
7 & [1, 4, 6] $\to$ [2] & 0.7329 & 0.2622 & 527 & [1, 2, 3] & [[0, 1, 2, 3, 4, 5, 6, 7, 8, 9], [0, 1, 2, 3, 4, 5, 6, 7, 8, 9], [0, 1, 2, 3, 4, 5, 6, 7, 8, 9]] \\
8 & [1, 4, 9] $\to$ [2] & 0.7329 & 0.2643 & 515 & [1, 2, 3] & [[0, 1, 2, 3, 4, 5, 6, 7, 8, 9], [0, 1, 2, 3, 4, 5, 6, 7, 8, 9], [0, 1, 2, 3, 4, 5, 6, 7, 8, 9]] \\
9 & [0, 1, 4, 6] $\to$ [2] & 0.7329 & 0.1890 & 507 & [1, 2, 3] & [[0, 1, 2, 3, 4, 5, 6, 7, 8, 9], [0, 1, 2, 3, 4, 5, 6, 7, 8, 9], [0, 1, 2, 3, 4, 5, 6, 7, 8, 9]] \\
10 & [1, 4, 5, 6] $\to$ [2] & 0.7329 & 0.1900 & 507 & [1, 2, 3] & [[0, 1, 2, 3, 4, 5, 6, 7, 8, 9], [0, 1, 2, 3, 4, 5, 6, 7, 8, 9], [0, 1, 2, 3, 4, 5, 6, 7, 8, 9]] \\
\bottomrule
\end{tabular}
\begin{tablenotes}
\item[1] Total number of transactions: 1000.
\item[2] Minimum support used in mining: 0.1.
\item[3] Minimum confidence used in filtering candidate rules: 0.5.
\item[4] Items correspond to synthetic dataset indices.
\end{tablenotes}
\end{threeparttable}
\end{table}

Examining the tables with item 2 as the consequent, a notable discrepancy exists between the support (posterior probability) and the observed empirical frequencies (actual transaction count /1000) in probabilistic methods, as shown in Fig.\ref{fig:support_vs_count}. For frequency-based methods such as Apriori, FP-Growth, and Eclat (Tables.\ref{tab:Apriori_analysis_minSupport01_Synthetic1},.\ref{tab:FP-Growth_analysis_minSupport01_Synthetic1},.\ref{tab:Eclat_analysis_minSupport01_Synthetic1}), the support exactly matches the empirical frequencies (i.e. count/total number of transactions), as these methods compute support directly as the fraction of transactions containing the itemset. For example, Apriori’s top rule has support 0.4760 and count 476, yielding count/1000 = 0.4760. In contrast, probabilistic methods such as BARM and GPAR variants (Tables.\ref{tab:BARM_analysis_minSupport01_Synthetic1},.\ref{tab:RBF_GPAR_analysis_minSupport01_Synthetic1},.\ref{tab:shifted_RBF_GPAR_analysis_minSupport01_Synthetic1},.\ref{tab:nn_GPAR_analysis_minSupport01_Synthetic1},.\ref{tab:NTK_GPAR_analysis_minSupport01_Synthetic1}) show significant mismatches. For BARM, the top rule’s support is 0.1916, but count/1000 is 0.496, indicating underestimation. GPAR (RBF/shifted RBF) shows support 0.3600 \textit{vs} count/1000=0.488, while GPAR (neural net) has the largest gap (e.g. support 0.1000 \textit{vs} count/1000=0.542). GPAR (NTK) also underestimates (e.g. support 0.1000 \textit{vs} count/1000=0.496). This discrepancy arises because probabilistic methods estimate support via posterior sampling (e.g. Beta for BARM, Gaussian processes for GPAR), which smooths probabilities based on model assumptions, unlike the empirical counts of frequency-based methods. Probabilistic methods’ systematic underestimation of the support is a concern.

\begin{figure}[H]
    \centering
    \includegraphics[width=0.5\linewidth]{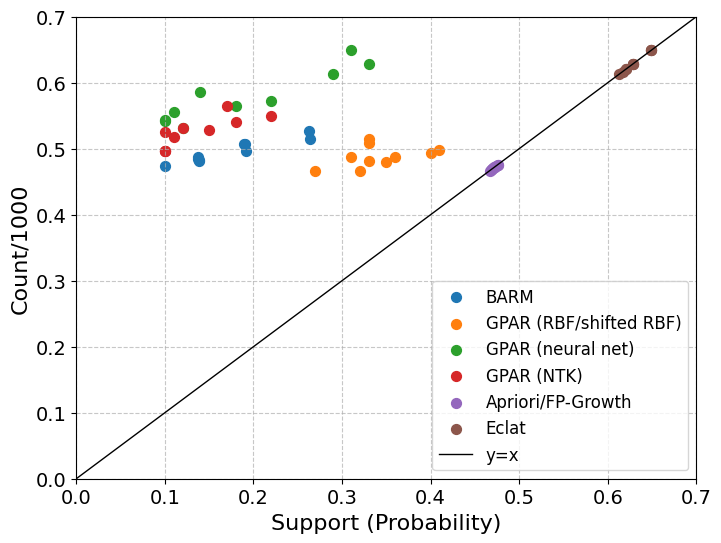}
    \caption{Empirical frequencies \textit{vs} support for the top 10 rules (ordered by confidence) with item 2 as consequent.}
    \label{fig:support_vs_count}
\end{figure}

\section{MAB-ARM: multi-armed bandit based association rule mining} 
\label{sec:MAB_ARM}

We present a third method for association rule mining using multi-armed bandit (MAB). The MAB framework is well-suited for sequential decision-making under uncertainty, and by modeling each transaction as an $M$-vector of Bernoulli variables, we frame the problem of mining association rules as an MAB problem. 

\subsection{Inspiration} 

In ARM, transactions are represented as Bernoulli vectors: each transaction $\mathbf{t}_j \in \{0, 1\}^M$ is an $M$-vector where each element $\mathbf{t}_j[i] \in \{0, 1\}$ represents the presence (1) or absence (0) of item $i$ in this transaction $j$. This can be modeled as the outcome of $M$ independent Bernoulli random variables, where the $i$-th variable has success probability $p_i = p(\mathbf{t}_j[i] = 1)$, representing the probability of item $i$ being present in a transaction. We therefore have the multi-armed bandit analogy: in a classical MAB problem, there are $K$ arms, each associated with an unknown reward distribution (e.g. Bernoulli for binary rewards), and the goal is to maximize cumulative rewards by choosing which arm to pull at each time step, balancing exploration (learning about the arms) and exploitation (choosing the best-known arm). We can apply this sequential decision-making game to generating the transaction Bernoulli vectors $\mathbf{t}_j$, by treating transactions as observations, and the $M$ items as $M$ arms in an MAB problem. In each round $j=1,2,...,N$ of the game, each 'pull' of an arm $i$ corresponds to observing whether item $i$ is present in a transaction, akin to receiving a reward (1 if present, 0 if absent). The reward distribution for arm $i$ is Bernoulli with parameter $p_i$, which is initially unknown.

However, the standard MAB approach focuses on individual arms, while AR mining requires estimating joint probabilities, e.g. $p(I) = p(\mathbf{t}[i] = 1 \ \forall i \in I)$ for an itemset $I$, not just individual item probabilities. We can therefore treat it as a combinatorial MAB problem, in which we are interested in 'pulling' subsets of arms (itemsets) rather than individual arms. This combinatorial nature raises complexity: evaluating subsets of items (itemsets) scales exponentially with $M$ ($\mathcal{O}(2^M)$).

Unlike a typical MAB setup where the agent actively chooses which arm to pull, here the transactions $\mathcal{T} = \{\mathbf{t}_1, \mathbf{t}_2, \dots, \mathbf{t}_N\}$ are given as a fixed dataset. We can simulate the MAB process by treating the selection of itemsets to evaluate as the 'arms' to pull, where pulling an arm corresponds to estimating the co-occurrence probability of an itemset using the transaction data. Also, in MAB, we have the exploration-exploitation trade-off, i.e. we balance exploring arms to estimate their reward distributions and exploiting arms with high estimated rewards. For AR mining, exploration corresponds to sampling transactions to estimate itemset probabilities, and exploitation corresponds to focusing on itemsets with high estimated co-occurrence probabilities to generate rules.

\subsection{MAB-ARM: methodology}

To adapt the MAB framework for AR mining, we use a combinatorial MAB approach where each 'arm' corresponds to an itemset $I \subseteq \{1, 2, \dots, M\}$. However, since evaluating all $2^M$ itemsets is computationally infeasible for large $M$, we will use an MAB strategy to prioritize the evaluation of promising itemsets, balancing exploration (evaluating new itemsets to estimate their co-occurrence probabilities) and exploitation (focusing on itemsets with high estimated probabilities that are likely to yield valid association rules).

\paragraph{Arms as itemsets}
Each arm corresponds to an itemset $I$. The 'reward' for pulling arm $I$ is an estimate of the co-occurrence probability $p(I)$, which we compute using the transaction data $\mathcal{T}$. As there are $2^M$ possible itemsets, we start with smaller itemsets (e.g. size 2) and use the MAB strategy to guide the exploration of larger itemsets.

\paragraph{Reward estimation} For an itemset $I$, the co-occurrence probability is estimated as:
 \[
 \hat{p}(I) = \frac{1}{N} \sum_{j=1}^N \mathbb{I}(\mathbf{t}_j[i] = 1 \ \forall i \in I),
 \]
 where $N$ is the number of transactions, and $\mathbb{I}(\mathbf{t}_j[i] = 1 \ \forall i \in I) = 1$ if all items in $I$ are present in transaction $\mathbf{t}_j$, and 0 otherwise. This estimate serves as the 'reward' for arm $I$, reflecting the empirical frequency of $I$ in the dataset.

\paragraph{MAB strategy} We use the upper confidence bound (UCB) strategy to balance exploration and exploitation \footnote{In the contexts of game theory and reinforcement learning, exploration allows the agent (i.e. the decision-maker) to improve its knowledge about each action and hopefully leading to a long-term benefit. Exploitation allows the agent to choose a greedy action to try to get the most reward for short-term benefit. However, a pure greedy action selection can lead to sub-optimal behaviour. The exploration-exploitation dilemma says that an agent can not choose to both explore and exploit at the same time, and UCB score represents the uncertainty in balancing exploration and exploitation.}. For each itemset $I$, after $t$ total pulls (evaluations of itemsets), let $n_I(t)$ be the number of times $I$ has been evaluated, and $\hat{p}_I(t)$ be the current estimate of $p(I)$. The UCB score for arm $I$ is:
\[
\text{UCB}_I(t) = \hat{p}_I(t) + \sqrt{\frac{2 \log t}{n_I(t)}},
\]
where the second term is the exploration bonus, encouraging the evaluation of itemsets that have been sampled less frequently. At each step, we select the itemset $I$ with the highest UCB score to evaluate, update its probability estimate, and use the estimate to generate association rules if it meets the thresholds. More details about the UCB method can be found in Appendix.\ref{app:UCB}.

\paragraph{Rules generation} For an itemset $I$ with $\hat{p}(I) > min\_prob$, evaluate all splits into antecedent $A$ and consequent $B$, estimate $\hat{p}(A)$, compute confidence $\text{conf}(A \rightarrow B) = \frac{\hat{p}(I)}{\hat{p}(A)}$, and add the rule to $\mathcal{R}$ if $\text{conf}(A \rightarrow B) > min\_conf$.

To manage the exponential number of itemsets, we can apply a pruning strategy to reduce the search space: we use the MAB strategy to prioritize itemsets likely to yield rules, effectively pruning the search space. For example, if an itemset $I$ has a low estimated probability after several evaluations, its UCB score will decrease, and the algorithm will focus on other itemsets.

\subsection{The MAB-ARM algorithm} 

The MAB-ARM algorithm is presented in Algo.\ref{algo:MAB-ARM}.

\begin{algorithm}[H]
\footnotesize
\caption{MAB-ARM: multi-armed bandit association rule mining}
\label{algo:MAB-ARM}
\textbf{Input:} a set of items $\{1, 2, \dots, M\}$; a set of transactions $\mathcal{T} = \{\mathbf{t}_1, \mathbf{t}_2, \dots, \mathbf{t}_N\}$, where each $\mathbf{t}_j \in \{0, 1\}^M$; a minimum probability threshold $min\_prob$; a minimum confidence threshold $min\_conf$; maximum number of itemset evaluations $T_{\text{max}}$; minimum itemset size $m_{min}$; maximum itemset size $m_{max}$. \\
\textbf{Output:} a set of association rules $\mathcal{R}$ where each rule $r \in \mathcal{R}$ is of the form $A \rightarrow B$ with $A, B \subseteq \{1, 2, \dots, M\}$, satisfying the probability and confidence thresholds.

\vspace{1mm}\hrule\vspace{1mm}

\begin{algorithmic}[1]
\STATE Initialize an empty set $\mathcal{R} = \emptyset$ to store association rules. \hfill\textit{$\mathcal{O}(1)$}
\STATE Initialize a set of candidate itemsets $\mathcal{I} = \emptyset$. For $m = m_{\text{min}}$ to $m_{max}$, add all size $m$ itemsets $I_m = \binom{\{1, 2, \dots, M\}}{m}$ to $\mathcal{I}$. \hfill\textit{$\mathcal{O}(\sum_{m=m_{\text{min}}}^{m_{\text{max}}} \mathcal{C}_m^M)$}
\STATE Initialize counters: for each $I \in \mathcal{I}$, set $n_I = 0$ (number of evaluations) and $\hat{p}_I = 0$ (estimated probability). \hfill\textit{$\mathcal{O}(|\mathcal{I}|)$}
\FOR{$t = 1$ to $T_{\text{max}}$}
    \STATE Compute UCB scores for each $I \in \mathcal{I}$: $\text{UCB}_I(t) = \hat{p}_I + \sqrt{\frac{2 \log t}{n_I}}$ (if $n_I = 0$, set $\text{UCB}_I(t) = \infty$). \hfill\textit{$\mathcal{O}(|\mathcal{I}|)$}
    \STATE Select itemset $I^* = \arg\max_{I \in \mathcal{I}} \text{UCB}_I(t)$. \hfill\textit{$\mathcal{O}(|\mathcal{I}|)$}
    \STATE Update $n_{I^*} = n_{I^*} + 1$, and estimate $\hat{p}_{I^*} = \frac{1}{N} \sum_{j=1}^N \mathbb{I}(\mathbf{t}_j[i] = 1 \ \forall i \in I^*)$. \hfill\textit{$\mathcal{O}(N \cdot m)$}
    \IF{$\hat{p}_{I^*} > min\_prob$}
        \FOR{each \textit{(antecedent, consequent)} split $(A, B)$ of $I^*$}
            \STATE Estimate antecedent probability: $\hat{p}(A) = \frac{1}{N} \sum_{j=1}^N \mathbb{I}(\mathbf{t}_j[i] = 1 \ \forall i \in A)$. \hfill\textit{$\mathcal{O}(N \cdot (m-1))$}
            \STATE Compute confidence: $\text{conf}(A \rightarrow B) = \frac{\hat{p}_{I^*}}{\hat{p}(A)}$. \hfill\textit{$\mathcal{O}(1)$}
            \IF{$\text{conf}(A \rightarrow B) > min\_conf$}
                \STATE Add rule $A \rightarrow B$ to $\mathcal{R}$. \hfill\textit{$\mathcal{O}(1)$}
            \ENDIF
        \ENDFOR
    \ENDIF
    \STATE Optionally prune $\mathcal{I}$: remove itemsets with low $\text{UCB}_I(t)$ or $\hat{p}_I$ after sufficient evaluations. \hfill\textit{$\mathcal{O}(|\mathcal{I}|)$}
\ENDFOR
\STATE Return the rules set $\mathcal{R}$. \hfill\textit{$\mathcal{O}(1)$}
\end{algorithmic}
\end{algorithm}

The MAB-ARM algorithm takes the transaction data $\mathcal{T}$, thresholds min\_prob and min\_conf, and parameters $T_{\text{max}}$ (maximum number of itemset evaluations), $m_{min}$, and $m_{max}$ (range of itemset sizes to consider) as inputs. Unlike GPAR and BARM, MAB-ARM does not require feature vectors $\mathbf{x}_i$ - it directly uses the transaction data to estimate probabilities, and similar to classic methods, these probabilities are observed empirical frequencies. We first initialise a candidate set of itemsets $\mathcal{I}$ within the size range $[m_{\text{min}}, m_{\text{max}}]$, reducing the search space compared to evaluating all $2^M$ itemsets (Step 2). Also initialised are the counters for each itemset for tracking the number of evaluations and estimated probabilities (Step 3). Steps 4–18 is the main MAB loop, in which we iterate $T_{\text{max}}$ times, using the UCB strategy to select and evaluate itemsets. Step 5 computes UCB scores at current iteration $t$, prioritizing unexplored itemsets (via $\text{UCB}_I = \infty$ for $n_I = 0$) and balancing exploration and exploitation. Step 7 estimates the co-occurrence probability $\hat{p}_{I}$ for the selected itemset, costing $\mathcal{O}(N \cdot m)$ to scan $N$ transactions. Steps 8–13 generate rules if $\hat{p}_{I} > \text{min\_{prob}}$, similar to GPAR and BARM. Step 17 optionally prunes itemsets with low UCB scores or probabilities, further reducing the search space for next iteration.

\paragraph{Computational complexity} 
The initialization stage costs $\mathcal{O}(\sum_{m=m_{\text{min}}}^{m_{\text{max}}} \mathcal{C}_m^M)$, depending on the range of itemset sizes. The main loop takes $T_{\text{max}}$ iterations, with each iteration costing $\mathcal{O}(|\mathcal{I}|)$ for UCB computation and selection, and $\mathcal{O}(N \cdot m)$ for probability estimation, leading to a total of $\mathcal{O}(T_{\text{max}} \cdot (|\mathcal{I}| + N \cdot M))$. The overall cost $\mathcal{O}(T_{\text{max}} \cdot (\sum_{m=m_{\text{min}}}^{m_{\text{max}}} \mathcal{C}_m^M + N \cdot M)) \leq \mathcal{O}(T_{\text{max}} \cdot (2^M + N \cdot M))$ can be more efficient than GPAR’s $\mathcal{O}(2^M \cdot (M^3 + S M^2))$, and comparable to BARM’s $\mathcal{O}(2^M \cdot S \cdot M)$ for the rule mining phase, as $T_{\text{max}}$ can be set to a reasonable value to limit evaluations. Detailed complexity analysis can be found in Appendix.\ref{app:MAB-ARM_complexity}.

\subsection{Comparison with GPAR and BARM}

MAB-ARM uses a combinatorial multi-armed bandit approach to AR mining, treating itemsets as arms and using a UCB strategy to balance exploration and exploitation in evaluating co-occurrence probabilities. Compared to GPAR and BARM, MAB-ARM offers a more efficient search through the itemset space by limiting evaluations and pruning low probability itemsets, while its UCB-based uncertainty quantification helps prioritize uncertain or rare rules for further exploration, though it lacks the full posterior uncertainty representation of BARM or the dependency modeling of GPAR and BARM. This makes MAB-ARM particularly suitable for large datasets where computational efficiency and targeted exploration of rare rules are priorities, at the cost of potentially over-simplifying item relationships.

\paragraph{Modelling} 
GPAR uses a GP model to model latent variables $\mathbf{z}$, with item dependencies captured via a kernel-based covariance matrix $K$, while BARM models item presence probabilities $p_i$ with Beta priors, incorporating dependencies via a correlation function (optional) and performing Bayesian inference using analytical solution or with MCMC. MAB-ARM treats itemsets as arms in a combinatorial MAB problem, guided by UCB exploration and exploitation mechanism, it directly estimates the co-occurrence probabilities of combinatorial rules from transaction data (i.e. empirical frequencies $\hat{p}(I) = \frac{1}{N} \sum \mathbb{I}(\mathbf{t}_j[i] = 1 \forall i \in I)$) without modeling dependencies or using priors, instead of GPAR and BARM’s sampling from posterior distributions.

\paragraph{Notion of uncertainty} 
GPAR provides both distributional (via Monte Carlo sampling) and point (via averaging) estimates for co-occurrence probabilities, with uncertainty limited to sampling variability (no uncertainty notation with kernel hyper-parameters), while BARM offers a full posterior distribution over co-occurrence probabilities, capturing uncertainty in both $\{p_i\}$ and hyper-parameters (e.g. kernel length-scale $\ell$), which is advantageous for interpreting rare rules. BARM can also supply point estimate via averaging.
MAB-ARM quantifies uncertainty through UCB scores, which include an exploration term $\sqrt{\frac{2 \log t}{n_I}}$. This term reflects the uncertainty in the estimate $\hat{p}_I$: a larger term indicates higher uncertainty due to fewer evaluations ($n_I$). While this does not provide a full posterior distribution like BARM, it allows MAB-ARM to prioritize itemsets with uncertain estimates for further evaluation, helping the discovery of rare rules.

\paragraph{Computational efficiency} 
GPAR has high computational cost due to GP inference ($\mathcal{O}(M^3)$ for covariance operations) and Monte Carlo sampling for each itemset, while BARM is cheaper than GPAR if an analytical posterior is pursued; it can be expensive if MCMC sampling based inference or item dependency structure is exploited. MAB-ARM uses the MAB exploration-exploitation mechanism to generate rules, which is more efficient when limiting the number of itemset evaluations $T_{\text{max}}$. The UCB strategy prunes the search space, making it scalable for larger $M$. However, all these 3 methods suffer from an exponential term $2^M$ with their complexity, with MAB-ARM potentially suffer less by limiting the length of itemsets for generating rules.

\paragraph{Modelling item dependency} 
GPAR naturally takes into account item dependencies; BARM has the option to incorporate item dependencies in correcting its likelihood, posterior and rule evaluation. MAB-ARM assumes independence in probability estimation (when calculating the UCB scores), relying solely on empirical frequencies from the transaction data, which simplifies computation but may miss insights from item relationships.

\subsection{MAB-ARM experiment}

\subsubsection{Experimental setup}

The MAB-ARM experiment was conducted on the \textit{Synthetic 1} dataset, comprising 1000 transactions involving 10 items, identical to the dataset used in the GPAR and BARM experiments. Each transaction is a binary vector indicating the presence or absence of items, generated using a multivariate normal distribution with an RBF kernel, retaining items with positive latent values. The dataset also includes a feature matrix $X \in \mathbb{R}^{10 \times 10}$ sampled from a standard normal distribution, although MAB-ARM does not utilize this feature matrix, focusing solely on transaction data.

The MAB-ARM algorithm employs an upper confidence bound (UCB, details see Appendix.\ref{app:UCB}) strategy to balance exploration and exploitation in itemset evaluation. The settings are as follows: 
(1) Number of MAB iterations ($T_{\text{max}}$): set to 1000 evaluations to balance thorough exploration with computational efficiency. Increasing $T_{\text{max}}$ would enhance exploration at the cost of longer runtime, while decreasing it would accelerate computation but potentially miss frequent itemsets.
(2) Itemset size: set to sizes 2 to 10 to cover all possible itemset lengths (as the dataset contains 10 items). Larger itemsets exponentially increase the search space, pruning strategy may be used to improve search efficiency.
(3) Pruning strategy: itemsets with low UCB scores are pruned every 100 iterations, retaining those with fewer than 10 evaluations to ensure sufficient exploration of less frequent itemsets. The pruning frequency and threshold can be adjusted to trade off between computational speed and the completeness of itemset discovery.
(4) Support calculation: MAB-ARM uses empirical frequencies to compute support, defined as the fraction of transactions containing an itemset (count/1000). As in classic AR methods, the support calculated in this way directly reflects observed transaction frequencies, avoiding the underestimation seen in probabilistic methods such as GPAR and BARM due to posterior smoothing.
(5) Confidence threshold: a minimum confidence of 0.5 was applied to filter association rules, which is consistent with the GPAR and BARM experiments. Confidence is calculated as the ratio of the itemset’s support to the antecedent’s support, and lift is computed as the ratio of the itemset’s support to the product of the antecedent and consequent supports.

The experiment was conducted on the same computational environment as the GPAR and BARM experiments, using a laptop with specifications detailed in the GPAR experiment section. Performance metrics recorded include runtime (seconds), memory usage (MB), number of frequent itemsets, and number of association rules generated, evaluated across minimum support thresholds of 0.1, 0.2, 0.3, 0.4, and 0.5, and a fixed minimum confidence of 0.5.

\subsubsection{Results}

\begin{table}[H]
\centering
\scriptsize
\begin{threeparttable}
\caption{Performance metrics of \textbf{MAB-ARM} on \textit{Synthetic 1} dataset}
\label{tab:MABARM_metrics_Synthetic1}
\begin{tabular}{p{1.8cm} p{1.2cm} p{2cm} p{2.2cm} p{2cm}}
\toprule
Min support & Runtime (s) & Memory (MB) & Frequent itemsets & Rules \\
\midrule
0.1 & 1.6599 & 0.000 & 1000 & 50372 \\
0.2 & 1.6036 & 13.625 & 1000 & 50372 \\
0.3 & 1.6251 & 0.625 & 1000 & 50372 \\
0.4 & 1.6126 & 6.375 & 1000 & 50372 \\
0.5 & 0.9390 & 0.125 & 736 & 20948 \\
\bottomrule
\end{tabular}
\begin{tablenotes}
\item[1] Total number of transactions: 1000.
\item[2] Minimum confidence used: 0.5.
\item[3] MAB-ARM with $T_{\text{max}} = 1000$, itemset sizes 2 to 10.
\end{tablenotes}
\end{threeparttable}
\end{table}

The performance of the MAB-ARM method on the Synthetic 1 dataset is summarized in Table.\ref{tab:MABARM_metrics_Synthetic1}. We observe that, at a minimum support of 0.1, MAB-ARM required 1.6599 seconds with negligible memory usage (0.0 MB, as recorded by \textit{Psutil}), generating 1000 frequent itemsets and 50372 rules. Runtime remained relatively stable across support thresholds 0.1 to 0.4 (1.6036-1.6599 seconds), with a notable decrease to 0.9390 seconds at 0.5, reflecting fewer qualifying itemsets and rules due to stricter support criteria. Memory usage exhibited spikes at support thresholds 0.2 (13.625 MB) and 0.4 (6.375 MB), likely due to temporary storage of candidate itemsets during UCB evaluations, but was minimal otherwise (0.0-0.625 MB). The number of frequent itemsets was consistently 1000 for support thresholds 0.1 to 0.4, dropping to 736 at 0.5, indicating that MAB-ARM effectively identifies frequent itemsets until the support threshold becomes highly restrictive. Similarly, the rule count remained constant at 50372 for thresholds 0.1 to 0.4, decreasing to 20948 at 0.5, reflecting the reduced number of frequent itemsets available for rule generation. The stability in itemset and rule counts across lower support thresholds suggests that MAB-ARM’s UCB-based selection, combined with periodic pruning, efficiently identifies frequent itemsets up to a high support threshold, after which stricter criteria naturally reduce the output.

In a second run, we investigated the algorithm’s behavior by increasing the key parameter $T_{max}$ and saving an extra variable $n_I$ which records the number of times an itemset has been evaluated. The the results are shown in Table.\ref{tab:MABARM_metrics_Synthetic1_secondRun}. First, we noted that memory usage is highly sensitive to the variables stored for analysis, rendering variations up to 100 MB statistically insignificant for detecting performance differences. Second, increasing the number of iterations $T_{\text{max}}$ from 1000 to $2^{10} = 1024$ (to ensure comprehensive evaluation of all itemsets) while tracking the number of evaluations per itemset $n_I$, which increased memory demands. This adjustment resulted in a modest increase in frequent itemsets (from 1000 to 1013 at thresholds 0.1 to 0.4) and a substantial rise in rules generated (from 50,372 to 57,024 at thresholds 0.1 to 0.4), with runtime extending to 2.3471 seconds at 0.1 and memory usage peaking at 40.625 MB at 0.2. Plot of evaluation counts (Fig.\ref{fig:evaluation_counts}) reveals a relatively even distribution, with a small number of itemsets evaluated twice. 

\begin{table}[H]
\centering
\scriptsize
\begin{threeparttable}
\caption{Performance metrics of \textbf{MAB-ARM} on \textit{Synthetic 1} dataset (2nd run)}
\label{tab:MABARM_metrics_Synthetic1_secondRun}
\begin{tabular}{p{1.8cm} p{1.2cm} p{2cm} p{2.2cm} p{2cm}}
\toprule
Min support & Runtime (s) & Memory (MB) & Frequent itemsets & Rules \\
\midrule
0.1 & 2.3471 & 0.000  & 1013 & 57024 \\
0.2 & 2.0871 & 40.625 & 1013 & 57024 \\
0.3 & 1.8898 & 2.124  & 1013 & 57024 \\
0.4 & 1.9054 & 14.875 & 1013 & 57024 \\
0.5 & 1.0720 & 0.875  & 736  & 20970 \\
\bottomrule
\end{tabular}
\begin{tablenotes}
\item[1] Total number of transactions: 1000.
\item[2] Minimum confidence used: 0.5.
\item[3] MAB-ARM with $T_{\text{max}} = 1024$, itemset sizes 2 to 10.
\end{tablenotes}
\end{threeparttable}
\end{table}

\begin{figure}[H]
\centering
\includegraphics[width=0.5\textwidth]{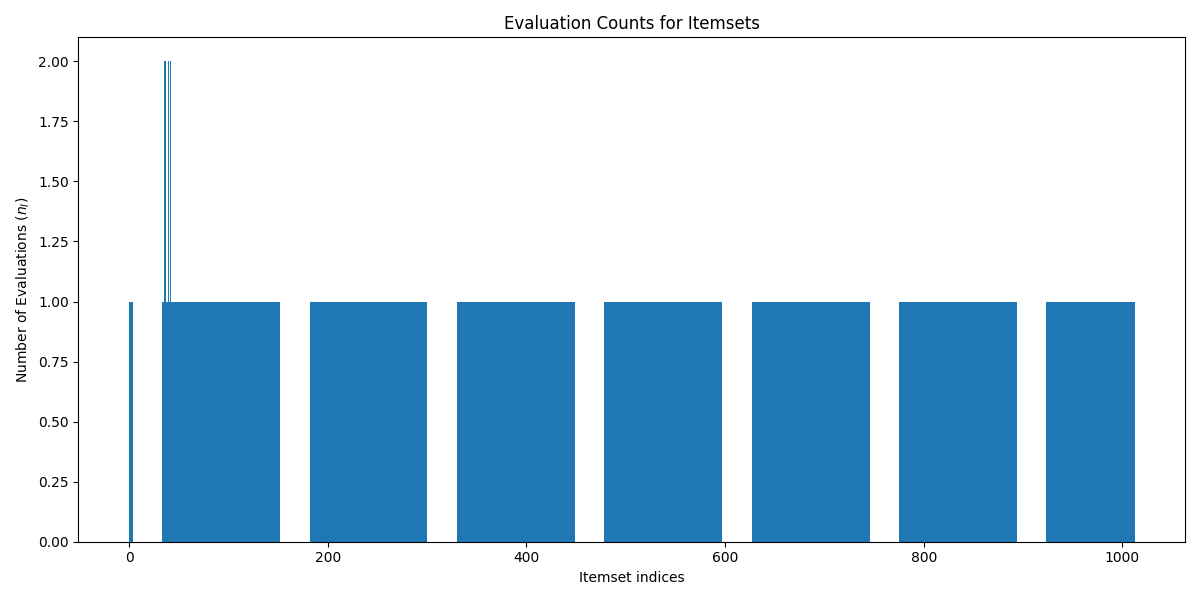}
\caption{Evaluation counts for itemsets with $T_{\text{max}} = 1024$.}
\label{fig:evaluation_counts}
\end{figure}

\paragraph{Comparison with other methods}

Comparing MAB-ARM with GPAR variants (RBF, shifted RBF, neural network, NTK), BARM, and classic ARM methods (Apriori, FP-Growth, Eclat) on the Synthetic 1 dataset (see Tables.\ref{tab:synthetic1_comparison_runtime},.\ref{tab:synthetic1_comparison_memory},.\ref{tab:synthetic1_comparison_frequentItemsets},.\ref{tab:synthetic1_comparison_noRules},.\ref{tab:BARM_metrics_Synthetic1}), MAB-ARM exhibits distinct performance characteristics:

\begin{enumerate}
    \item \textbf{Runtime}: at a minimum support of 0.1, MAB-ARM’s runtime (1.6599 seconds) is significantly faster than GPAR RBF/shifted RBF (10.179/8.930 seconds) and FP-Growth (2.967 seconds), but slower than GPAR with neural network kernel (0.133 seconds), GPAR with NTK kernel (0.449 seconds), Apriori (0.483 seconds), Eclat (0.556 seconds), and BARM (1.1821 seconds). The fixed $T_{\text{max}} = 1000$ evaluations ensures consistent exploration across thresholds, contributing to stable but higher runtimes compared to pruning-based methods such as Apriori or Eclat, which leverage efficient candidate generation and pruning.
    \item \textbf{Memory usage}: MAB-ARM is memory efficient, its memory usage is generally low (0.0-0.625 MB), except for notable peaks at support thresholds 0.2 (13.625 MB) and 0.4 (6.375 MB), possibly due to temporary storage of candidate itemsets during UCB score calculations. This performance is comparable to GPAR variants, Apriori, FP-Growth, and BARM, which maintain negligible memory usage (0.0-0.063 MB). Eclat, however, exhibits higher peaks (18.149 MB at support level 0.1, 7.969 MB at 0.4). The memory spikes in MAB-ARM suggest potential optimisation opportunities, such as more aggressive pruning or streamlined candidate management.
    \item \textbf{Number of frequent itemsets}: MAB-ARM generates 1000 frequent itemsets at support thresholds 0.1 to 0.4, dropping to 736 at 0.5, aligning closely with Apriori/FP-Growth (1023 itemsets at 0.1--0.4, 746 at 0.5) and GPAR RBF/shifted RBF (1013 at 0.1, decreasing to 1/2 at 0.5). Eclat consistently produces the highest number of itemsets (1033 across all thresholds), while GPAR with neural network kernel (140 at 0.1, 0 at 0.4-0.5) and NTK (305 at 0.1, 0-2 at 0.4-0.5) generate significantly fewer. BARM produces fewer itemsets (876 at 0.1, 44 at 0.5) than MAB-ARM, reflecting its conservative posterior-based estimation. MAB-ARM’s high itemset count results from its UCB-driven exploration, which evaluates a broad range of itemsets of sizes 2 to 10.
    \item \textbf{Number of rules}: MAB-ARM generates 50372 rules at support thresholds 0.1 to 0.4, a bit fewer than the counts of Apriori/FP-Growth/Eclat (57002 rules at 0.1-0.4) and GPAR RBF/shifted RBF (54036 at 0.1). At a support of 0.5, MAB-ARM produces 20948 rules, matching Apriori/FP-Growth (20948) but falling below Eclat (57002). GPAR with neural network kernel (322 at 0.1, 0 at 0.4-0.5) and NTK kernel (1485 at 0.1, 0-4 at 0.4-0.5) generate significantly fewer rules, while BARM produces fewer still (12029 at 0.1, 88 at 0.5). MAB-ARM’s robust rule count is driven by its ability to identify frequent itemsets efficiently across itemset sizes 2 to 10.
\end{enumerate}

\paragraph{Rule quality}

Table.\ref{tab:MABARM_mined_rules_rankedBySupport_minSupport01_Synthetic1} presents the top 10 MAB-ARM rules at a minimum support of 0.1, ranked by \textit{support}. High-support rules, such as \(\text{item}_7 \to \text{item}_5\) (support 0.662, confidence 0.9220, lift 1.2700) and \(\text{item}_5 \to \text{item}_7\) (support 0.662, confidence 0.9118, lift 1.2700), indicate frequent co-occurrences with moderate lift, suggesting meaningful but not highly specific associations. These rules align closely \footnote{Table.\ref{tab:MABARM_mined_rules_rankedBySupport_minSupport01_Synthetic1} is indeed the same as Table.\ref{tab:Apriori_mined_rules_rankedBySupport_minSupport01_Synthetic1} and Table.\ref{tab:FP-Growth_mined_rules_rankedBySupport_minSupport01_Synthetic1}.} with Apriori/FP-Growth’s top rules (e.g. \(\text{item}_7 \to \text{item}_5\), support 0.6620, confidence 0.9220, lift 1.2700, see Table.\ref{tab:Apriori_mined_rules_rankedBySupport_minSupport01_Synthetic1} and Table.\ref{tab:FP-Growth_mined_rules_rankedBySupport_minSupport01_Synthetic1}) but differ from Eclat (Table.\ref{tab:Eclat_mined_rules_rankedBySupport_minSupport01_Synthetic1}) and GPAR RBF/shifted RBF’s focus on higher-lift rules (e.g. \(\text{item}_5 \to \text{item}_8, \text{item}_6\), support 0.5200, confidence 0.9123, lift 2.3392, see Table.\ref{tab:RBF_GPAR_mined_rules_rankedBySupport_minSupport01_Synthetic1} and Table.\ref{tab:shifted_RBF_GPAR_mined_rules_rankedBySupport_minSupport01_Synthetic1}). BARM’s top rules, such as \(\text{item}_2 \to \text{item}_9\) (support 0.5356, confidence 0.7308, lift 1.0000), exhibit near-trivial lift due to its independence assumption, highlighting MAB-ARM’s advantage in capturing empirical dependencies.

\begin{table}[H]
\centering
\scriptsize
\begin{threeparttable}
\caption{Top 10 association rules mined by \textbf{MAB-ARM} (ranked in descending order by \textit{\color{red}{support}})}
\label{tab:MABARM_mined_rules_rankedBySupport_minSupport01_Synthetic1}
\begin{tabular}{r >{\raggedright\arraybackslash}p{2cm} >{\raggedright\arraybackslash}p{2cm} p{0.6cm} p{0.6cm} p{0.6cm}}
\toprule
\# & Antecedent & Consequent & \color{red}{Supp} & Conf & Lift \\
\midrule
1 & item\_7 & item\_5 & 0.662 & 0.9220 & 1.2700 \\
2 & item\_5 & item\_7 & 0.662 & 0.9118 & 1.2700 \\
3 & item\_7 & item\_8 & 0.662 & 0.9220 & 1.2877 \\
4 & item\_8 & item\_7 & 0.662 & 0.9246 & 1.2877 \\
5 & item\_0 & item\_5 & 0.660 & 0.9141 & 1.2591 \\
6 & item\_5 & item\_0 & 0.660 & 0.9091 & 1.2591 \\
7 & item\_0 & item\_8 & 0.656 & 0.9086 & 1.2690 \\
8 & item\_8 & item\_0 & 0.656 & 0.9162 & 1.2690 \\
9 & item\_0 & item\_7 & 0.655 & 0.9072 & 1.2635 \\
10 & item\_7 & item\_0 & 0.655 & 0.9123 & 1.2635 \\
\bottomrule
\end{tabular}
\begin{tablenotes}
\item[1] Total number of transactions: 1000.
\item[2] Minimum support used in mining: 0.1.
\item[3] Minimum confidence used: 0.5.
\item[4] Items correspond to synthetic dataset indices.
\end{tablenotes}
\end{threeparttable}
\end{table}

Table.\ref{tab:MABARM_mined_rules_rankedByConfidence_minSupport01_Synthetic1} lists the top 10 rules ranked by \textit{confidence}, with rules such as \(\text{item}_1, \text{item}_4, \text{item}_5, \text{item}_8 \to \text{item}_7\) (confidence 0.9981, support 0.529, lift 1.3901) achieving near-perfect confidence. These are comparable to Apriori/FP-Growth’s high-confidence rules (e.g. confidence 0.9981 for \(\text{item}_4, \text{item}_5, \text{item}_1, \text{item}_8 \to \text{item}_7\), see Table.\ref{tab:Apriori_mined_rules_rankedByConfidence_minSupport01_Synthetic1} and Table.\ref{tab:FP-Growth_mined_rules_rankedByConfidence_minSupport01_Synthetic1}), and significantly exceed BARM’s maximum confidence (0.7329 for \(\text{item}_1, \text{item}_4, \text{item}_6, \text{item}_9 \to \text{item}_2\), see Table.\ref{tab:BARM_mined_rules_rankedByConfidence_minSupport01_Synthetic1}). GPAR with NTK achieves highest confidence (5.5 in Table.\ref{tab:NTK_GPAR_mined_rules_rankedByConfidence_minSupport01_Synthetic1}), GPAR with neural network kernel (3.6667 in Table.\ref{tab:nn_GPAR_mined_rules_rankedByConfidence_minSupport01_Synthetic1}), GPAR RBF/shifted RBF (e.g. 2.2727 for \(\text{item}_1, \text{item}_2, \text{item}_3, \text{item}_6, \text{item}_7, \text{item}_9 \to \text{item}_0, \text{item}_4\), see Table.\ref{tab:RBF_GPAR_analysis_minSupport01_Synthetic1} and Table.\ref{tab:shifted_RBF_GPAR_mined_rules_rankedByConfidence_minSupport01_Synthetic1}) also achieve high confidence, reflecting GPAR's ability to model complex dependencies via Gaussian processes.

\begin{table}[H]
\centering
\scriptsize
\begin{threeparttable}
\caption{Top 10 association rules mined by \textbf{MAB-ARM} (ranked in descending order by \textit{\color{red}{confidence}})}
\label{tab:MABARM_mined_rules_rankedByConfidence_minSupport01_Synthetic1}
\begin{tabular}{r >{\raggedright\arraybackslash}p{5cm} >{\raggedright\arraybackslash}p{1.2cm} p{0.6cm} p{0.6cm} p{0.6cm}}
\toprule
\# & Antecedent & Consequent & Supp & \color{red}{Conf} & Lift \\
\midrule
1 & item\_1, item\_4, item\_5, item\_8 & item\_7 & 0.529 & 0.9981 & 1.3901 \\
2 & item\_1, item\_5, item\_8, item\_9 & item\_7 & 0.528 & 0.9981 & 1.3901 \\
3 & item\_0, item\_1, item\_2, item\_5, item\_8 & item\_7 & 0.526 & 0.9981 & 1.3901 \\
4 & item\_0, item\_1, item\_4, item\_5, item\_8 & item\_7 & 0.520 & 0.9981 & 1.3901 \\
5 & item\_1, item\_2, item\_3, item\_5, item\_8 & item\_7 & 0.519 & 0.9981 & 1.3901 \\
6 & item\_0, item\_1, item\_5, item\_8, item\_9 & item\_7 & 0.518 & 0.9981 & 1.3901 \\
7 & item\_0, item\_1, item\_2, item\_5, item\_6, item\_8 & item\_7 & 0.516 & 0.9981 & 1.3901 \\
8 & item\_1, item\_5, item\_6, item\_8, item\_9 & item\_7 & 0.514 & 0.9981 & 1.3901 \\
9 & item\_1, item\_2, item\_3, item\_5, item\_6, item\_8 & item\_7 & 0.514 & 0.9981 & 1.3901 \\
10 & item\_1, item\_3, item\_5, item\_8, item\_9 & item\_7 & 0.513 & 0.9981 & 1.3900 \\
\bottomrule
\end{tabular}
\begin{tablenotes}
\item[1] Total number of transactions: 1000.
\item[2] Minimum support used in mining: 0.1.
\item[3] Minimum confidence used: 0.5.
\item[4] Items correspond to synthetic dataset indices.
\end{tablenotes}
\end{threeparttable}
\end{table}

Table.\ref{tab:MABARM_mined_rules_rankedByLift_minSupport01_Synthetic1} shows the top 10 rules ranked by \textit{lift}, with rules such as \(\text{item}_2, \text{item}_3, \text{item}_4, \text{item}_5 \to \text{item}_0, \text{item}_1, \text{item}_6, \text{item}_9\) (lift 1.7612, support 0.468, confidence 0.9123) indicating strong associations. These lift values are lower than GPAR RBF/shifted RBF’s maximum lift (12.1212 \(\text{item}_0, \text{item}_8, \text{item}_6, \text{item}_9 \to \text{item}_1, \text{item}_2, \text{item}_3, \text{item}_4, \text{item}_7\) in Table.\ref{tab:RBF_GPAR_mined_rules_rankedByLift_minSupport01_Synthetic1} and Table.\ref{tab:shifted_RBF_GPAR_mined_rules_rankedByLift_minSupport01_Synthetic1}), GPAR NTK’s (8.5470 for \(\text{item}_8, \text{item}_1, \text{item}_4 \to \text{item}_2, \text{item}_5, \text{item}_6\) in Table.\ref{tab:NTK_GPAR_mined_rules_rankedByLift_minSupport01_Synthetic1}), and GPAR neural net's (6.4103 in Table.\ref{tab:nn_GPAR_mined_rules_rankedByLift_minSupport01_Synthetic1}), but they significantly exceed BARM’s near-trivial lift values (e.g. 1.0001 for \(\text{item}_0, \text{item}_2, \text{item}_3, \text{item}_6 \to \text{item}_4, \text{item}_7\) in Table.\ref{tab:BARM_mined_rules_rankedByLift_minSupport01_Synthetic1}). Apriori/FP-Growth’s top lift (1.7982 for \(\text{item}_6, \text{item}_0, \text{item}_1, \text{item}_9, \text{item}_7 \to \text{item}_2, \text{item}_3, \text{item}_8, \text{item}_5, \text{item}_4\) in Table.\ref{tab:Apriori_mined_rules_rankedByLift_minSupport01_Synthetic1}) is slightly higher than MAB-ARM’s, reflecting their exhaustive search over all possible itemsets.

\begin{table}[H]
\centering
\scriptsize
\begin{threeparttable}
\caption{Top 10 association rules mined by \textbf{MAB-ARM} (ranked in descending order by \textit{\color{red}{lift}})}
\label{tab:MABARM_mined_rules_rankedByLift_minSupport01_Synthetic1}
\begin{tabular}{r >{\raggedright\arraybackslash}p{4cm} >{\raggedright\arraybackslash}p{4cm} p{0.6cm} p{0.6cm} p{0.6cm}}
\toprule
\# & Antecedent & Consequent & Supp & Conf & \color{red}{Lift} \\
\midrule
1 & item\_2, item\_3, item\_4, item\_5 & item\_0, item\_1, item\_6, item\_9 & 0.468 & 0.9123 & 1.7612 \\
2 & item\_0, item\_1, item\_6, item\_9 & item\_2, item\_3, item\_4, item\_5 & 0.468 & 0.9035 & 1.7612 \\
3 & item\_4, item\_5, item\_6, item\_9 & item\_0, item\_1, item\_2, item\_3 & 0.468 & 0.9213 & 1.7581 \\
4 & item\_0, item\_1, item\_2, item\_3 & item\_4, item\_5, item\_6, item\_9 & 0.468 & 0.8931 & 1.7581 \\
5 & item\_4, item\_6, item\_7, item\_9 & item\_0, item\_1, item\_2, item\_3 & 0.469 & 0.9196 & 1.7550 \\
6 & item\_0, item\_1, item\_2, item\_3 & item\_4, item\_6, item\_7, item\_9 & 0.469 & 0.8950 & 1.7550 \\
7 & item\_2, item\_3, item\_4, item\_8 & item\_0, item\_1, item\_6, item\_9 & 0.469 & 0.9089 & 1.7547 \\
8 & item\_0, item\_1, item\_6, item\_9 & item\_2, item\_3, item\_4, item\_8 & 0.469 & 0.9054 & 1.7547 \\
9 & item\_0, item\_2, item\_3, item\_4 & item\_1, item\_5, item\_6, item\_9 & 0.468 & 0.9159 & 1.7511 \\
10 & item\_1, item\_5, item\_6, item\_9 & item\_0, item\_2, item\_3, item\_4 & 0.468 & 0.8948 & 1.7511 \\
\bottomrule
\end{tabular}
\begin{tablenotes}
\item[1] Total number of transactions: 1000.
\item[2] Minimum support used in mining: 0.1.
\item[3] Minimum confidence used: 0.5.
\item[4] Items correspond to synthetic dataset indices.
\end{tablenotes}
\end{threeparttable}
\end{table}

Table.\ref{tab:MABARM_analysis_minSupport01_Synthetic1} lists the top 10 rules with \(\text{item}_2\) as the consequent, ranked by \textit{confidence}, at a minimum support of 0.1. The top rule, \([1, 3, 4, 5, 6, 9] \to [2]\) (confidence 0.9855, support 0.476, support count 476), exemplifies MAB-ARM’s ability to identify highly confident rules with substantial support. MAB-ARM’s reliance on empirical frequencies for \textit{support} calculation ensures that the support of each rule precisely matches the observed transaction frequency (count/1000). This eliminates the discrepancy observed in probabilistic methods such as GPAR and BARM, as previously illustrated in Fig.\ref{fig:support_vs_count}. BARM underestimates support. For example, the top rule with \(\text{item}_2\) as the consequent (Rule \([1, 4, 6, 9] \to [2]\) in Table.\ref{tab:BARM_analysis_minSupport01_Synthetic1}) has a support of 0.1916 but a count/1000 of 0.496, indicating significant underestimation due to Bayesian smoothing. All GPAR variants also show similar discrepancies: GPAR RBF/shifted RBF reports a support of 0.3600 \textit{vs} count/1000 = 0.488 (with confidence 2.1176) for Rule \([0, 1, 3, 4, 8] \to [2, 7]\) in Table.\ref{tab:RBF_GPAR_analysis_minSupport01_Synthetic1} and Table.\ref{tab:shifted_RBF_GPAR_analysis_minSupport01_Synthetic1}. GPAR with neural network kernel shows support 0.1000 \textit{vs} count/1000 = 0.542 for the top Rule \([8, 9, 6] \to [2]\) in Table.\ref{tab:nn_GPAR_analysis_minSupport01_Synthetic1}. GPAR with NTK kernel has support 0.1000 \textit{vs} count/1000 = 0.496 for Rule \([8, 1, 4, 6] \to [2, 5]\) in Table.\ref{tab:NTK_GPAR_analysis_minSupport01_Synthetic1}. MAB-ARM’s empirical frequency-based support aligns closely with Apriori/FP-Growth which also use empirical frequencies (e.g. Apriori: support 0.4760, count 476 for Rule \([6, 1, 3, 8, 9, 7, 4] \to [2]\) in Table.\ref{tab:Apriori_analysis_minSupport01_Synthetic1}), this evidenced accuracy makes them reliable for applications requiring precise frequency-based insights, such as market basket analysis. Eclat, however, its supports overestimate the corresponding empirical frequencies, as observed in Table.\ref{tab:Eclat_analysis_minSupport01_Synthetic1}.

\begin{table}[H]
\centering
\scriptsize
\begin{threeparttable}
\caption{Top 10 association rules mined by \textbf{MAB-ARM} with item 2 on RHS (min support = 0.1, rules ranked by \color{red}{confidence})}
\label{tab:MABARM_analysis_minSupport01_Synthetic1}
\begin{tabular}{r >{\raggedright\arraybackslash}p{2.9cm} p{0.6cm} p{0.5cm} p{0.6cm} >{\raggedright\arraybackslash}p{1.8cm} >{\raggedright\arraybackslash}p{5cm}}
\toprule
\# & Rule & \color{red}{Conf} & Supp & Supp Count & Example Trans.\ Indices & Example Transactions \\
\midrule
1 & [1, 3, 4, 5, 6, 9] $\to$ [2] & 0.9855 & 0.476 & 476 & [1, 2, 3] & [[0, 1, 2, 3, 4, 5, 6, 7, 8, 9], [0, 1, 2, 3, 4, 5, 6, 7, 8, 9], [0, 1, 2, 3, 4, 5, 6, 7, 8, 9]] \\
2 & [1, 3, 4, 6, 7, 8, 9] $\to$ [2] & 0.9855 & 0.476 & 476 & [1, 2, 3] & [[0, 1, 2, 3, 4, 5, 6, 7, 8, 9], [0, 1, 2, 3, 4, 5, 6, 7, 8, 9], [0, 1, 2, 3, 4, 5, 6, 7, 8, 9]] \\
3 & [1, 3, 4, 5, 6, 7, 9] $\to$ [2] & 0.9854 & 0.474 & 474 & [1, 2, 3] & [[0, 1, 2, 3, 4, 5, 6, 7, 8, 9], [0, 1, 2, 3, 4, 5, 6, 7, 8, 9], [0, 1, 2, 3, 4, 5, 6, 7, 8, 9]] \\
4 & [0, 1, 3, 4, 6, 9] $\to$ [2] & 0.9854 & 0.473 & 473 & [1, 2, 3] & [[0, 1, 2, 3, 4, 5, 6, 7, 8, 9], [0, 1, 2, 3, 4, 5, 6, 7, 8, 9], [0, 1, 2, 3, 4, 5, 6, 7, 8, 9]] \\
5 & [1, 3, 4, 5, 6, 8, 9] $\to$ [2] & 0.9854 & 0.472 & 472 & [1, 2, 3] & [[0, 1, 2, 3, 4, 5, 6, 7, 8, 9], [0, 1, 2, 3, 4, 5, 6, 7, 8, 9], [0, 1, 2, 3, 4, 5, 6, 7, 8, 9]] \\
6 & [0, 1, 3, 4, 6, 7, 9] $\to$ [2] & 0.9853 & 0.469 & 469 & [1, 2, 3] & [[0, 1, 2, 3, 4, 5, 6, 7, 8, 9], [0, 1, 2, 3, 4, 5, 6, 7, 8, 9], [0, 1, 2, 3, 4, 5, 6, 7, 8, 9]] \\
7 & [0, 1, 3, 4, 6, 8, 9] $\to$ [2] & 0.9853 & 0.469 & 469 & [1, 2, 3] & [[0, 1, 2, 3, 4, 5, 6, 7, 8, 9], [0, 1, 2, 3, 4, 5, 6, 7, 8, 9], [0, 1, 2, 3, 4, 5, 6, 7, 8, 9]] \\
8 & [0, 1, 3, 4, 5, 6, 9] $\to$ [2] & 0.9853 & 0.468 & 468 & [1, 1, 2, 3] & [[0, 1, 2, 3, 4, 5, 6, 7, 8, 9], [0, 1, 2, 3, 4, 5, 6, 7, 8, 9], [0, 1, 2, 3, 4, 5, 6, 7, 8, 9]] \\
9 & [3, 4, 5, 6, 9] $\to$ [2] & 0.9839 & 0.488 & 488 & [1, 2, 3] & [[0, 1, 2, 3, 4, 5, 6, 7, 8, 9], [0, 1, 2, 3, 4, 5, 6, 7, 8, 9], [0, 1, 2, 3, 4, 5, 6, 7, 8, 9]] \\
10 & [3, 4, 6, 7, 8, 9] $\to$ [2] & 0.9838 & 0.487 & 487 & [1, 2, 3] & [[0, 1, 2, 3, 4, 5, 6, 7, 8, 9], [0, 1, 2, 3, 4, 5, 6, 7, 8, 9], [0, 1, 2, 3, 4, 5, 6, 7, 8, 9]] \\
\bottomrule
\end{tabular}
\begin{tablenotes}
\item[1] Total number of transactions: 1000.
\item[2] Minimum support used in mining: 0.1.
\item[3] Minimum confidence used: 0.5.
\item[4] Items correspond to synthetic dataset indices.
\end{tablenotes}
\end{threeparttable}
\end{table}

\paragraph{Discussion}

MAB-ARM offers a balanced approach to association rule mining, leveraging UCB to prioritize itemset evaluations, achieving competitive runtime and memory efficiency compared to classic methods while generating a high number of frequent itemsets and rules. Its empirical frequency-based support calculation ensures accurate support estimation, avoiding the underestimation inherent in probabilistic methods like BARM and GPAR. This makes MAB-ARM particularly suitable for applications where precise transaction frequencies are needed. The algorithm’s performance is robust across support thresholds, with stable itemset and rule counts at lower thresholds (0.1-0.4) due to effective UCB-driven exploration and periodic pruning.

Compared to GPAR, MAB-ARM is significantly faster than the RBF/shifted RBF variants (8.930-10.179 seconds at 0.1) but slower than the neural network (0.133 seconds) and NTK (0.449 seconds) kernels, offering rule counts (50372 at 0.1) comparable to RBF/shifted RBF (54036). Against classic methods, MAB-ARM’s runtime is higher than Apriori (0.483 seconds) and Eclat (0.556 seconds) but competitive with FP-Growth (2.967 seconds), while its rule count is slightly lower than Apriori/FP-Growth/Eclat (57002), despite enumerating itemsets of sizes 2 to 10. BARM, with a runtime of 1.1821 seconds and fewer rules (12029 at 0.1), is outperformed by MAB-ARM in both rule quantity and quality, particularly in lift, due to BARM’s independence assumption.

The fixed $T_{\text{max}} = 1000$ ensures computational feasibility for medium-sized datasets, but may limit exploration of very large itemsets compared to Apriori/FP-Growth/Eclat or GPAR RBF/shifted RBF’s high-lift rules. Future improvements include adaptive $T_{\text{max}}$ based on dataset characteristics, integration of feature-based priors to leverage item similarities, or optimized pruning to further reduce memory usage, potentially bridging the gap with GPAR’s probabilistic modeling while maintaining MAB-ARM’s empirical accuracy and efficiency.

\subsection{An improved MAB-ARM method}

The naive MAB-ARM algorithm (Algo.\ref{algo:MAB-ARM}) has an inefficiencies in its rule discovery process: once an itemset’s probability is evaluated based on empirical frequency, this value remains static in subsequent iterations, despite its UCB score decreasing over time to encourage exploration of new rules. If its empirical frequency (evaluated probability) is high, it retains a high UCB score in later iterations (with marginal decrease), leading repeated evaluations in Steps 6-16, which encompass UCB score computation to rule generation,  and fail to update its probability (because its empirical frequency is fixed based on the transaction data), thereby wasting computational resources. This repetitive assessment of previously evaluated itemsets and their subsets, without altering the probabilities of derived sub-itemset rules, undermines the algorithm’s performance, diverting effort from exploring potentially more promising itemsets.

To address this inefficiency, we propose an improved MAB-ARM algorithm (Algo.\ref{algo:EMAB-ARM}) integrating the Apriori principle, which states that if an itemset is frequent, all its supersets (mother itemsets) are at least as frequent. Upon evaluating a current itemset, the algorithm updates the probabilities of all $2^M$ possible itemsets, adjusting any superset’s probability to match that of the current itemset if it is lower, while leaving higher probabilities unchanged. This associative update  lifts up the UCB scores of underestimated itemsets, enabling the algorithm to prioritize those with greater frequency potential and reducing redundant evaluations by leveraging hierarchical itemset relationships. Applied to the same \textit{Synthetic 1} dataset, the resulting runtimes and memory usages are comparable to (marginally smaller than) those in Table.\ref{tab:MABARM_metrics_Synthetic1_secondRun} with evidence not being statistically significant due to the small dataset used. The frequencies of visits for each itemset, as shown in Fig.\ref{fig:evaluation_counts_improved}, however, show some interesting pattern as compared to Fig.\ref{fig:evaluation_counts}, as a result of the impact of the associative probability update. The improved algorithm may help redistribute evaluation efforts and reduce the bias induced by the static probability assignments and aggressive pruning within limited number of iterations.

\begin{algorithm}[H]
\footnotesize
\caption{EMAB-ARM: enhanced multi-armed bandit association rule mining}
\label{algo:EMAB-ARM}
\textbf{Input:} a set of items $\{1, 2, \dots, M\}$; a set of transactions $\mathcal{T} = \{\mathbf{t}_1, \mathbf{t}_2, \dots, \mathbf{t}_N\}$, where each $\mathbf{t}_j \in \{0, 1\}^M$; a minimum probability threshold $min\_prob$; a minimum confidence threshold $min\_conf$; maximum number of itemset evaluations $T_{\text{max}}$; minimum itemset size $m_{min}$; maximum itemset size $m_{max}$. \\
\textbf{Output:} a set of association rules $\mathcal{R}$ where each rule $r \in \mathcal{R}$ is of the form $A \rightarrow B$ with $A, B \subseteq \{1, 2, \dots, M\}$, satisfying the probability and confidence thresholds.

\vspace{1mm}\hrule\vspace{1mm}

\begin{algorithmic}[1]
\STATE Initialize an empty set $\mathcal{R} = \emptyset$ to store association rules. \hfill\textit{$\mathcal{O}(1)$}
\STATE Initialize a set of candidate itemsets $\mathcal{I} = \emptyset$. For $m = m_{\text{min}}$ to $m_{max}$, add all size $m$ itemsets $I_m = \binom{\{1, 2, \dots, M\}}{m}$ to $\mathcal{I}$. \hfill\textit{$\mathcal{O}(\sum_{m=m_{\text{min}}}^{m_{\text{max}}} \mathcal{C}_m^M)$}
\STATE Initialize counters: for each $I \in \mathcal{I}$, set $n_I = 0$ (number of evaluations) and $\hat{p}_I = 0$ (estimated probability). \hfill\textit{$\mathcal{O}(|\mathcal{I}|)$}
\FOR{$t = 1$ to $T_{\text{max}}$}
    \STATE Compute UCB scores for each $I \in \mathcal{I}$: $\text{UCB}_I(t) = \hat{p}_I + \sqrt{\frac{2 \log t}{n_I}}$ (if $n_I = 0$, set $\text{UCB}_I(t) = \infty$). \hfill\textit{$\mathcal{O}(|\mathcal{I}|)$}
    \STATE Select itemset $I^* = \arg\max_{I \in \mathcal{I}} \text{UCB}_I(t)$. \hfill\textit{$\mathcal{O}(|\mathcal{I}|)$}
    \STATE Update $n_{I^*} = n_{I^*} + 1$, and estimate $\hat{p}_{I^*} = \frac{1}{N} \sum_{j=1}^N \mathbb{I}(\mathbf{t}_j[i] = 1 \ \forall i \in I^*)$. \hfill\textit{$\mathcal{O}(N \cdot m)$}
    \STATE \textcolor{red}{Associative probability update: for each $I \in \mathcal{I}$ where $I^* \subseteq I$ and $|I| > |I^*|$, if $\hat{p}_I < \hat{p}_{I^*}$, set $\hat{p}_I = \hat{p}_{I^*}$}. \hfill\textit{$\mathcal{O}(2^M)$}
    \IF{$\hat{p}_{I^*} > min\_prob$}
        \FOR{each \textit{(antecedent, consequent)} split $(A, B)$ of $I^*$}
            \STATE Estimate antecedent probability: $\hat{p}(A) = \frac{1}{N} \sum_{j=1}^N \mathbb{I}(\mathbf{t}_j[i] = 1 \ \forall i \in A)$. \hfill\textit{$\mathcal{O}(N \cdot (m-1))$}
            \STATE Compute confidence: $\text{conf}(A \rightarrow B) = \frac{\hat{p}_{I^*}}{\hat{p}(A)}$. \hfill\textit{$\mathcal{O}(1)$}
            \IF{$\text{conf}(A \rightarrow B) > min\_conf$}
                \STATE Add rule $A \rightarrow B$ to $\mathcal{R}$. \hfill\textit{$\mathcal{O}(1)$}
            \ENDIF
        \ENDFOR
    \ENDIF
    \STATE Optionally prune $\mathcal{I}$: remove itemsets with low $\text{UCB}_I(t)$ or $\hat{p}_I$ after sufficient evaluations. \hfill\textit{$\mathcal{O}(|\mathcal{I}|)$}
\ENDFOR
\STATE Return the rules set $\mathcal{R}$. \hfill\textit{$\mathcal{O}(1)$}
\end{algorithmic}
\end{algorithm}

\begin{figure}[H]
\centering
\includegraphics[width=0.5\textwidth]{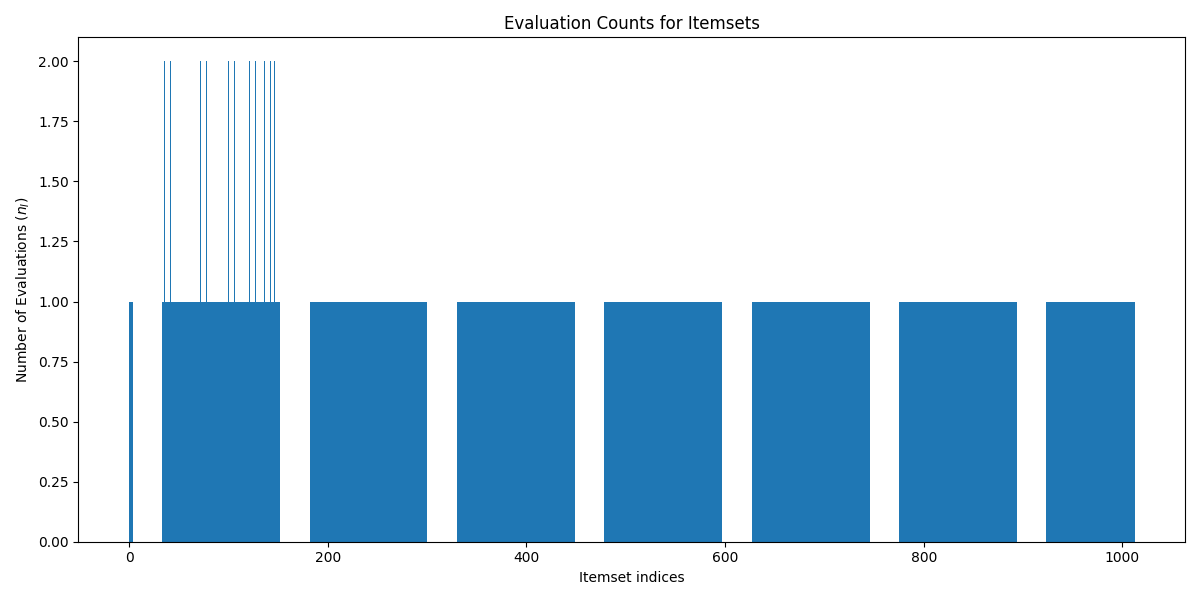}
\caption{Evaluation counts for itemsets with $T_{\text{max}} = 1024$ (improved MAB-ARM).}
\label{fig:evaluation_counts_improved}
\end{figure}

\subsection{An Monte Carlo tree search (MCTS) method for association rule mining} \label{subsec:MCTS-ARM}

MCTS is a heuristic search algorithm widely used in decision-making problems, particularly in game-playing AI such as \textit{AlphaGo} \cite{Silver2016AlphaGo}. In the context of association rule mining, MCTS can be adapted to search for itemsets that generate high-quality association rules, treating the itemset space as a tree where nodes represent itemsets. We formulate the MCTS based ARM framework as follows:

\paragraph{Tree structure} Each node in the tree represents an itemset, starting with an empty itemset at the root. Child nodes are generated by adding one item to the parent itemset, forming supersets. For example, from the empty set {}, children include {A}, {B}, {C}, and from {A}, children include {A,B}, {A,C}, etc. Edges represent the action of adding a specific item to the parent itemset. The depth of a node corresponds to the size of the itemset (number of items). A maximum depth can be set to limit exploration to itemsets of manageable size, based on dataset characteristics or computational constraints.

\paragraph{The MCTS -based search procedure} MCTS operates through four phases: selection, expansion, simulation, and backpropagation. In the \textit{selection phase}, MCTS employs UCB to choose the most promising child node to explore \footnote{UCT (upper confidence bound 1, UCB1, applied to trees) is utilized in MCTS to balance exploitation and exploration during child node selection, addressing the challenge of prioritizing moves with high average win rates while ensuring under-explored moves are evaluated. The UCT formula, \(\frac{w_i}{n_i} + c \sqrt{\frac{\ln N_i}{n_i}}\), integrates exploitation through the first term (high for moves with a high win ratio) and exploration through the second term (high for moves with few simulations), with \(w_i\), \(n_i\), \(N_i\), and \(c\) representing wins, simulations for the node, total parent simulations, and an empirically tuned exploration parameter (typically \(\sqrt{2}\)) respectively. This approach, derived from UCB1 by \cite{Auer2002FiniteTimeMAB} and adapted by \cite{Kocsis2006MABMC}, enhances decision-making efficiency in multi-stage models such as Markov decision processes (MDP) \cite{Chang2005MDP}, making it suitable for navigating the complex itemset space in association rule mining.}:
\[
\text{UCB} = \bar{r} + c \sqrt{\frac{\ln N}{n}}
\]
where \(\bar{r}\) is the average reward of the node, based on past evaluations. \(N\) is the total number of visits to the parent node. \(n\) is the number of visits to the child node. \(c\) is an exploration parameter (e.g. 0.5–2.0, typically set to \(\sqrt{2}\) or tuned empirically, which balances exploration and exploitation. Similar to the MAB-ARM procedure, we start from the tree root, recursively select the child node with the highest UCB value until a leaf node (unexpanded or terminal) is reached. This ensures the algorithm balances exploiting itemsets with high average rewards and exploring less-visited itemsets. 

In the \textit{expansion phase}, if the selected leaf node has not been fully expanded, we create new child nodes by adding valid items to the itemset, constrained by e.g. maximum itemset size. We can also apply pruning to avoid expanding unpromising itemsets. For example, we can use the anti-monotonicity property of support: if an itemset’s support is below the minimum threshold, do not expand it further, as its supersets will also be infrequent. This mirrors traditional association rule mining strategies such as Apriori.

In the \textit{simulation phase}, MCTS performs random rollout. From the newly expanded node, we perform a simulation by randomly generating a sequence of item additions to create a larger itemset, up to a predefined depth or until a terminal condition. Then we evaluate the simulated itemset by computing its potential to generate association rules. The potential can be quantified by a reward function which can be defined as support-based:
\[
\text{Reward} = 
\begin{cases}
\text{support}(I), & \text{if } \text{support}(I) \geq \text{min\_support} \\
\delta \in \{0, -1\}, & \text{otherwise}
\end{cases}
\]
or we can generate candidate rules (e.g. \(X \rightarrow Y\) where \(X \cup Y\) is the itemset) and compute the maximum confidence:
\[
Reward = \max(\text{confidence}(X \rightarrow Y)) \text{ for all valid } Y
\]
or the number of rules with confidence above a threshold
\[
\text{Reward} = \text{no. of rules with } \text{confidence} \geq \text{min\_conf}
\]
or just combine support and confidence to prioritize frequent and reliable rules: 
\[
\text{reward = support × confidence}
\]
For example, for an itemset \(X\), we compute support as the fraction of transactions containing \(X\). If support $\geq$ min\_support, we continue to generate rules \(X \setminus Y \rightarrow Y\) for all non-empty \(Y \subseteq X\). The we calculate reward = maximum confidence of valid rules, or 0 if no rules meet min\_confidence. The reward is dynamically updated after each simulation, contributing to the average reward (\(\bar{r}\)) used in the UCB formula.

In the \textit{backpropagation phase}, we propagate the reward from the simulation back through the tree, updating the visit count (\(n\)) and total reward for each node in the path. The average reward (\(\bar{r}\)) is re-calculated as:
\[
\bar{r} = \frac{\text{total reward}}{\text{visit count}}
\]
Then we update UCB values using the updated \(\bar{r}\) and visit counts, which impacts future selections. The MCTS-ARM algorithm, using the maximum confidence reward, is presented in Algo.\ref{algo:MCTS-ARM} in Appendix.\ref{app:MCTS-ARM}.

Similar to MAB-ARM, we set a maximum itemset size to prevent combinatorial explosion of rules and run MCTS for a fixed number of iterations (e.g. 10000). Pruning happens in multiple stages, for example, we can skip expansion of itemsets with support below min\_support; during simulation, we can discard rules with confidence below min\_confidence to focus on promising itemsets. Using a cache of evaluated itemsets, we can also avoid exploring itemsets that are supersets of already-evaluated low-reward itemsets.

The advantages of using MCTS for itemset search lie in: (1) efficient exploration. UCB ensures a balance between exploring new itemsets and exploiting known high-reward ones, reducing the need to evaluate all possible combinations. (2) Flexible reward design and dynamic adaption. The reward function can be customized to prioritize different rule quality metrics (e.g. support, confidence, lift); rewards are updated dynamically, allowing the algorithm to adapt to the dataset’s structure and focus on promising regions of the itemset space. (3) Scalability. By structuring the itemset space as a tree and incorporating pruning, MCTS has the potential to scale to large datasets while discovering high-quality rules. 

However, it also faces several challenges. First, the computational cost of MCTS grows with the number of items in the dataset. This can be potentially mitigated by limiting the search tree’s depth and number of iterations, adopting efficient representations such as bit-vectors (as seen in GIM-RL \cite{GIM2022}), and parallelizing simulations to exploit modern multi-core systems. Second, reward design plays a important role in guiding the search; poorly designed rewards can lead to the discovery of suboptimal rules. To mitigate this, hybrid reward functions, e.g. using the product of support and confidence, and reward normalization techniques can be employed to ensure balanced and meaningful evaluations. Third, overfitting to simulation rollouts is a risk, particularly when relying on random simulations that may not reflect true reward distributions. This issue can be alleviated by integrating domain knowledge to bias rollouts towards more informative itemsets or by using heuristic-guided simulations to improve the fidelity of reward estimation.

\section{RLAR: reinforcement learning based association rule mining} \label{sec:RLAR}

Reinforcement learning (RL) is a machine learning framework \footnote{The other two are supervised and un-supervised learning.} where an agent learns to make sequential decisions by interacting with an environment to maximize a cumulative reward by trial and error. Fitting to the context of AR mining, the agent aims to identify itemsets that generate high-quality association rules (e.g. $A \rightarrow B$) by exploring the combinatorial space of possible itemsets. RL offers a promising approach to ARM by learning to identify valuable itemsets through reward-guided trials.

\subsection*{\textit{Reinforcement learning: preliminaries}}

In RL, the interaction between agent and environment is formalized as a \textit{Markov decision process} (MDP), defined by a tuple $\langle \mathcal{S}, \mathcal{A}, \mathcal{P}, \mathcal{R}, \gamma \rangle$, where $\mathcal{S}$ is the set of states, $\mathcal{A}$ is the set of actions, $\mathcal{P}(s' | s, a)$ is the transition probability to the next state $s'$, $\mathcal{R}(s, a, s')$ is the reward function, and $\gamma \in [0, 1)$ is the discount factor balancing immediate and future rewards. At each time step $t$, the agent observes the current state $s_t \in \mathcal{S}$, selects an action $a_t \in \mathcal{A}$ based on a policy $\pi(a | s)$, receives a reward $r_t$, and transitions to a new state $s_{t+1}$. The objective is to learn an optimal policy $\pi^*$ that maximizes the expected discounted return, defined as $G_t = \sum_{k=0}^\infty \gamma^k r_{t+k}$. The action-value function, or Q-function, $Q^\pi(s, a) = \mathbb{E}_\pi[G_t | s_t = s, a_t = a]$, represents the expected return for taking action $a$ in state $s$ and following policy $\pi$. In our RLAR method, RL is applied to explore itemsets, where states represent itemset configurations, actions correspond to toggling items, and rewards are derived from the quality of association rules (based on support and confidence). The RLAR environment uses the data, i.e. a binary transaction matrix, to compute these rewards, enabling the agent to learn a policy for generating high-quality rules efficiently.

The deep Q-network (DQN) method, a value-based RL approach employed in our RLAR framework, uses a neural network to approximate the Q-function for large state-action spaces, overcoming the limitations of traditional Q-learning. Q-learning, the discrete precursor to DQN, updates a tabular Q-function using the Bellman equation: $Q(s, a) \leftarrow Q(s, a) + \alpha \left[ r + \gamma \max_{a'} Q(s', a') - Q(s, a) \right]$, where $\alpha$ is the learning rate, $r$ is the immediate reward, and $s'$ is the next state. This update iteratively improves Q-values toward the optimal $Q^*(s, a)$. However, Q-learning is infeasible for RLAR’s high-dimensional state space (e.g. itemset combinations). DQN addresses this by parameterizing the Q-function as $Q(s, a; \theta)$, where $\theta$ represents neural network weights, trained to minimize the loss $L(\theta) = \mathbb{E} \left[ \left( r + \gamma \max_{a'} Q(s', a'; \theta^-) - Q(s, a; \theta) \right)^2 \right]$, using a target network with fixed weights $\theta^-$. In RLAR, DQN predicts Q-values for toggling items in an itemset, guiding the agent to select actions that maximize rule quality. Experience replay, where experiences $(s, a, r, s')$ are stored in a buffer (with size $B$) and sampled in mini-batches (with mini-batch size $B_m$), stabilizes training by reducing correlation. The target network and replay buffer ensure robust learning, making DQN suitable for navigating the complex state-action space of association rule mining to discover frequent and confident rules.

\subsection{RLAR: methodology}

Drawing inspiration from the RL-based generic itemset mining (GIM-RL \cite{GIM2022}) method, we configure the RLAR framework as follows:

\paragraph{The RL environment} We define the state and action spaces, as well as the transition and reward mechanisms for using a RL agent for ARM. 

The \textit{state space} $\mathcal{S}$ consists of all possible itemsets represented as bit-vectors of length $M$ (the total number of items in the dataset), where each position $i \in \{1, 2, \dots, M\}$ is 1 if item $i$ is included in the itemset and 0 otherwise. For example, for items $\mathcal{I} = \{A, B, C\}$, the itemset $\{A, C\}$ is encoded as $[1, 0, 1]$. To enrich the state representation, additional features can also be included: the support of the current itemset, defined as $\text{supp}(X) = \frac{1}{N} \sum_{j=1}^N \mathbb{I}(\mathbf{t}_j[i] = 1 \ \forall i \in X)$; the maximum confidence of rules derived from the itemset, computed as $\max(\{\text{conf}(A \rightarrow B) \mid A \cup B = X, B \neq \emptyset\})$; and the itemset size $|X|$ to enforce constraints such as maximum itemset size $m_{\text{max}}$. Thus, a state $s \in \mathcal{S}$ is a tuple $(v, \text{supp}, \text{max\_conf}, |X|)$, where $v \in \{0, 1\}^M$ is the bit-vector, and the additional features are real-valued or integer scalars.

The \textit{action space} $\mathcal{A}$ comprises actions that modify the current itemset by adding or removing a single item. For $M$ items, there are up to $M$ actions: for each item $i \in \{1, 2, \dots, M\}$, the agent can choose to add item $i$ (if $v[i] = 0$) or remove item $i$ (if $v[i] = 1$). Alternatively, the action space can be restricted to only adding items, as in GIM-RL \cite{GIM2022}, reducing $\mathcal{A}$ to the set of actions that add an item not currently in the itemset, i.e. $\mathcal{A}_s = \{ \text{add } i \mid v[i] = 0 \}$, where $s = (v, \text{supp}, \text{max\_conf}, |X|)$. Applying an action updates the state: adding item $i$ sets $v[i] = 1$, removing item $i$ sets $v[i] = 0$, and the new state reflects the updated itemset and its computed metrics (support, maximum confidence, and size).

\paragraph{DQN architecture} In RLAR, the deep Q-network (DQN) employs a fully connected multi-layer perceptron (MLP) to approximate the Q-function, $Q(s, a; \theta)$, where $s$ represents the state, $a$ denotes the action, and $\theta$ are the network parameters. The network architecture consists of an input layer of size $M + k$, where $M$ is the length of the itemset bit-vector and $k$ accounts for additional features such as support and itemset size; 3 hidden layers, each with 128 to 256 neurons and ReLU activation; and an output layer with $M$ neurons, each corresponding to the Q-value for adding or removing an item. Note the depth of the MLP can impact its capability in decision-making. 

\paragraph{Reward design} Reward function is key for guiding the DQN towards discovering high-quality ARs. We define a composite reward incorporating the ARM metrics support, confidence, and lift.
Support is the proportion of transactions containing an itemset, indicating its frequency, Computed as $\text{supp}(X) = \frac{1}{N} \sum_{j=1}^N \mathbb{I}(\mathbf{t}_j[i] = 1 \ \forall i \in X)$, where $N$ is the number of transactions, and \(\mathbf{t}_j\) is the $j$-th transaction.
Confidence is the conditional probability that a transaction containing the antecedent also contains the consequent, reflecting rule reliability. For an itemset $X$, generate rules $A \rightarrow B$ where $A \cup B = X$, $B \neq \emptyset$. Compute $\text{conf}(A \rightarrow B) = \frac{\text{supp}(X)}{\text{supp}(A)}$.
Lift is the ratio of observed to expected co-occurrence of antecedent and consequent, measuring rule strength beyond independence. It is computed as $\text{lift}(A \rightarrow B) = \frac{\text{conf}(A \rightarrow B)}{\text{supp}(B)} = \frac{\text{supp}(X)}{\text{supp}(A) \cdot \text{supp}(B)}$.

The reward is designed blending these metrics. If $\text{supp}(X) < min\_supp$, a negative reward (e.g. -1) is returned to penalize infrequent itemsets. Otherwise, for all valid rules $A \rightarrow B$ from $X$, we first compute a rule score:
\[
\text{score}(A \rightarrow B) = w_1 \cdot \text{conf}(A \rightarrow B) + w_2 \cdot \text{lift}(A \rightarrow B)
\]
where $w_1, w_2$ are weight hyper-parameters (e.g. $w_1 = 0.5, w_2 = 0.5$). If $\text{conf}(A \rightarrow B) \geq min\_conf$, we include the rule in a set of valid rules.

The reward is then calculated as:
\begin{equation} \label{eq:RLAR_reward_design}
    \text{reward} = \begin{cases} 
                    -1 & \text{if } \text{supp}(X) < min\_supp, \\
                    \max(\{\text{score}(A \rightarrow B) \mid \text{conf}(A \rightarrow B) \geq min\_conf\}) & \text{if valid rules exist}, \\
                    0 & \text{otherwise}.
                    \end{cases}
\end{equation}
This reward function encourages the agent to find frequent itemsets with high-confidence and high-lift rules, balancing rule reliability and interestingness. To ensure support, confidence and lift have comparable scales, we can normalize their values to be within [0, 1] (e.g. by dividing by their maximum possible values).

\paragraph{Training the DQN} 

Training a decision-making agent (DQN) for ARM involves many episodes, each comprising multiple steps during which the network weights and action rewards are dynamically updated. An episode begins with an empty itemset, represented as a bit-vector $v = [0, 0, \dots, 0]$, and consists of a sequence of actions that modify the itemset, such as adding or removing an item. The episode terminates when a condition (or a composite condition) is met, such as reaching the maximum itemset size ($|X| = m_{\text{max}}$), encountering an itemset with support below the minimum threshold ($\text{supp}(X) < min\_supp$), detecting a low-reward itemset, or completing a fixed number of steps (e.g. 500, as in GIM-RL \cite{GIM2022}). The environment provides rewards as feedbacks based on the quality of the itemset (e.g. computing support, confidence, and lift of the rules resulted from the itemset, see reward design), guiding the DQN to prioritize itemsets that towards yielding high-quality ARs.

The Q-net with parameters $\theta$ and a target network with parameters $\theta^-$ are initialised. The Q-network parameters $\theta$ are initialised randomly, with the initial parameters of the target network set as $\theta^- = \theta$. A replay buffer of size 10000 is employed to store transitions $(s_t, a_t, r_t, s_{t+1})$, from which mini-batches of 32 transitions are sampled to reduce data correlation. For each of the 5000 episodes, the agent observes the current state $s_t$ (itemset and features), selects an action $a_t$ using an \textit{epsilon-greedy} policy \cite{RLbook2018}, applies the action to obtain a new itemset (i.e. the state transition), computes the reward $r_t$, and observes the next state $s_{t+1}$. These transitions are stored in the replay buffer, and the network is updated via gradient descent which minimizes the temporal-difference error \cite{RLbook2018}:
\begin{equation} \label{eq:RLAR_temporal_diff}
    L(\theta) = \mathbb{E} \left[ \left( r_t + \gamma \max_{a'} Q(s_{t+1}, a'; \theta^-) - Q(s_t, a_t; \theta) \right)^2 \right]
\end{equation}
where $\gamma = 0.99$ is a discount factor \cite{RLbook2018}. The target network is updated every 1000 steps to stabilize learning, and training stops after the fixed number of episodes or when a sufficient number of high-quality rules are identified, ensuring effective learning of a policy that prioritizes valuable itemsets.

\paragraph{Exploration and exploitation}
The DQN employs an \textit{epsilon-greedy} policy \cite{RLbook2018} to balance exploration and exploitation. It first selects a random action with probability $\epsilon$, or chooses the action with the highest Q-value, i.e. $a = \arg\max_a Q(s, a; \theta)$. To shift from exploration to exploitation, \(\epsilon\) is decayed over time, for example, from 1.0 to 0.1 over 10000 steps, enabling the agent to focus on high-value itemsets as training progresses. Unlike MCTS, which explicitly uses UCB for exploration, the DQN's epsilon-greedy strategy inherently manages the trade-off between exploring new itemsets and exploiting known high-reward ones, ensuring effective navigation of the itemset space in the RLAR framework.

\paragraph{Rules extraction} After training the DQN, we extract rules by running the trained DQN to generate itemsets (using greedy action selection). For each itemset $X$ with $\text{supp}(X) \geq min\_supp$, we generate rules $A \rightarrow B$, and include rules in the output set $\mathcal{R}$ if $\text{conf}(A \rightarrow B) \geq min\_conf$.

\subsection{The RLAR algorithm}
The main RLAR workflow is outlined in the pseudo code below, with the algorithm presented in Algo.\ref{algo:RLAR}.

{\centering \textbf{RLAR pseudocode}\par}

\begin{lstlisting}
Initialize Q-network Q(s, a; $\theta$), target network Q(s, a; $\theta^-$), replay buffer D
Initialize $\varepsilon$ = 1.0, $\mathcal{R}$ = $\emptyset$
For episode = 1 to max_episodes:
    s = empty itemset, step = 0
    While step < max_steps and not terminal:
        With probability $\varepsilon$: select random action a
        Else: select a = argmax_a Q(s, a; $\theta$)
        Execute a, get new itemset X, compute reward r, next state s'
        Store (s, a, r, s') in D
        Sample mini-batch from D
        Compute target: y = r + $\gamma$ max_a' Q(s', a'; $\theta^-$) for non-terminal s'
        Update $\theta$ by minimizing loss: (y - Q(s, a; $\theta$))^2
        s = s', step += 1
    Decay $\varepsilon$
    Periodically update $\theta^-$ = $\theta$
Extract rules $\mathcal{R}$ from high-reward itemsets using trained Q-network
Return $\mathcal{R}$
\end{lstlisting}

\begin{algorithm}[htbp]
\scriptsize
\caption{RLAR: Reinforcement learning-based association rule mining with DQN}
\label{algo:RLAR}
\textbf{Input:} A set of items $\{1, 2, \dots, M\}$; a set of transactions $\mathcal{T} = \{\mathbf{t}_1, \mathbf{t}_2, \dots, \mathbf{t}_N\}$, where each $\mathbf{t}_j \in \{0, 1\}^M$; a minimum support threshold $min\_supp$; a minimum confidence threshold $min\_conf$; maximum itemset size $m_{\text{max}}$; maximum number of episodes $E_{\text{max}}$; maximum steps per episode $S_{\text{max}}$; replay buffer size $B$; mini-batch size $B_m$; discount factor $\gamma$; initial exploration rate $\varepsilon$; target network update frequency $F$; number of top actions $k$ for rule extraction. \\
\textbf{Output:} A set of association rules $\mathcal{R}$ where each rule $r \in \mathcal{R}$ is of the form $A \rightarrow B$ with $A, B \subseteq \{1, 2, \dots, M\}$, satisfying the support and confidence thresholds.

\vspace{1mm}\hrule\vspace{1mm}

\begin{algorithmic}[1]
\STATE Initialize Q-network $Q(s, a; \theta)$ and target network $Q(s, a; \theta^-)$ with random parameters $\theta$, $\theta^- = \theta$. \hfill\textit{$\mathcal{O}(|\theta|)$}
\STATE Initialize replay buffer $D$ of size $B$, $\mathcal{R} = \emptyset$, $\varepsilon = 1.0$. \hfill\textit{$\mathcal{O}(1)$}
\FOR{$e = 1$ to $E_{\text{max}}$}
    \STATE Initialize state $s$ as empty itemset (i.e. $v = [0, 0, \dots, 0]$ plus any extra feature states), step = 0. \hfill\textit{$\mathcal{O}(M)$}
    \WHILE{step $< S_{\text{max}}$ and not terminal}
        \IF{random value $< \varepsilon$}
            \STATE Select random action $a \in \mathcal{A}_s$. \hfill\textit{$\mathcal{O}(M)$}
        \ELSE
            \STATE Select $a = \arg\max_a Q(s, a; \theta)$. \hfill\textit{$\mathcal{O}(M)$}
        \ENDIF
        \STATE Execute $a$, obtain new itemset $X$, compute reward $r$ (based on support, confidence, lift of all valid rules generated by $X$, with penalty for $\text{supp}(X) < min\_supp$), next state $s'$. \hfill\textit{$\mathcal{O}(N \cdot m_{\text{max}} \cdot 2^{m_{\text{max}}})$}
        \STATE Store transition $(s, a, r, s')$ in $D$. \hfill\textit{$\mathcal{O}(1)$}
        \IF{$|D| \geq B_m$}
            \STATE Sample mini-batch of $B_m$ transitions $\{(s_i, a_i, r_i, s'_i)\}$ from $D$. \hfill\textit{$\mathcal{O}(B_m)$}
            \FOR{each transition $(s_i, a_i, r_i, s'_i)$ in mini-batch}
                \IF{$s'_i$ is terminal}
                    \STATE Set target $y_i = r_i$. \hfill\textit{$\mathcal{O}(1)$}
                \ELSE
                    \STATE Set target $y_i = r_i + \gamma \max_{a'} Q(s'_i, a'; \theta^-)$. \hfill\textit{$\mathcal{O}(M)$}
                \ENDIF
            \ENDFOR
            \STATE Update $\theta$ by minimizing loss: $\frac{1}{B_m} \sum_i (y_i - Q(s_i, a_i; \theta))^2$. \hfill\textit{$\mathcal{O}(B_m \cdot |\theta|)$}
        \ENDIF
        \STATE $s \gets s'$, step $\gets$ step + 1. \hfill\textit{$\mathcal{O}(1)$}
    \ENDWHILE
    \STATE Decay $\varepsilon$ (e.g. $\varepsilon \gets \max(\varepsilon_{\text{min}}, \varepsilon \cdot \text{decay\_rate})$). \hfill\textit{$\mathcal{O}(1)$}
    \IF{$e \mod F = 0$}
        \STATE Update $\theta^- \gets \theta$. \hfill\textit{$\mathcal{O}(|\theta|)$}
    \ENDIF
\ENDFOR
\STATE $\mathcal{R} \gets$ EXTRACT-RULES($Q$, $\mathcal{T}$, $min\_supp$, $min\_conf$, $m_{\text{max}}$, $S_{\text{max}}$, $k$). \hfill\textit{$\mathcal{O}(k \cdot N \cdot 2^{m_{\text{max}}} \cdot S_{\text{max}})$}
\RETURN $\mathcal{R}$. \hfill\textit{$\mathcal{O}(1)$}
\vspace{0.3cm}
\STATEx
\STATE \textbf{Function} EXTRACT-RULES($Q$, $\mathcal{T}$, $min\_supp$, $min\_conf$, $m_{\text{max}}$, $S_{\text{max}}$, $k$):
\begin{ALC@g}
    \STATE Initialize $\mathcal{R} \gets \emptyset$, trajectories = $\emptyset$, step = 0.
    \STATE Initialize state $s$ as empty itemset (i.e. $v = [0, 0, \dots, 0]$ plus any extra feature states). \hfill\textit{$\mathcal{O}(M)$}
    \STATE Add $(s, \emptyset)$ to trajectories. \hfill\textit{$\mathcal{O}(1)$}
    \WHILE{step $< S_{\text{max}}$ and trajectories $\neq \emptyset$}
        \STATE Initialize new\_trajectories = $\emptyset$.
        \FOR{each $(s, X)$ in trajectories}
            \STATE Select top-$k$ actions $\{a_1, \dots, a_k\}$ with highest $Q(s, a; \theta)$. \hfill\textit{$\mathcal{O}(M \cdot \log k)$}
            \FOR{each $a_i$ in $\{a_1, \dots, a_k\}$}
                \STATE Execute $a_i$, obtain itemset $X'$, next state $s'$. \hfill\textit{$\mathcal{O}(N \cdot m_{\text{max}})$}
                \IF{$|X'| \leq m_{\text{max}}$}
                    \STATE Compute $supp = \frac{1}{N} \sum_{j=1}^N \mathbb{I}(\mathbf{t}_j[l] = 1 \ \forall l \in X')$.
                    \IF{$supp \geq min\_supp$}
                        \FOR{each $(A, B)$ split of $X'$ with $B \neq \emptyset$}
                            \STATE Compute $p_A = \frac{1}{N} \sum_{j=1}^N \mathbb{I}(\mathbf{t}_j[l] = 1 \ \forall l \in A)$.
                            \STATE Compute $conf = \frac{supp}{p_A}$.
                            \IF{$conf \geq min\_conf$}
                                \STATE Add $(A \rightarrow B)$ to $\mathcal{R}$. \hfill\textit{$\mathcal{O}(N \cdot m_{\text{max}})$}
                            \ENDIF
                        \ENDFOR
                        \STATE Add $(s', X')$ to new\_trajectories. \hfill\textit{$\mathcal{O}(1)$}
                    \ENDIF
                \ENDIF
            \ENDFOR
        \ENDFOR
        \STATE trajectories $\gets$ new\_trajectories, step $\gets$ step + 1. \hfill\textit{$\mathcal{O}(1)$}
    \ENDWHILE
    \RETURN $\mathcal{R}$. \hfill\textit{$\mathcal{O}(1)$}
\end{ALC@g}
\end{algorithmic}
\end{algorithm}

\paragraph{Complexities} 
The computational complexity of the RLAR algorithm is primarily driven by the training phase (Lines 2–30) and the rule extraction phase (Lines 33–60). In the training phase, the algorithm iterates over $E_{\text{max}}$ episodes, with each episode comprising up to $S_{\text{max}}$ steps. For each step, action selection (Lines 6–9) involves either a random action ($\mathcal{O}(1)$) or evaluating Q-values for $M$ actions using the neural network ($\mathcal{O}(M)$). The reward computation and state transition (Line 11) require calculating the support of an itemset by scanning $N$ transactions and evaluating up to $m_{\text{max}}$ items, costing $\mathcal{O}(N \cdot m_{\text{max}})$, with additional rule generation for confidence and lift adding $\mathcal{O}(N \cdot 2^{m_{\text{max}}} \cdot m_{\text{max}})$ due to iterating over possible antecedent-consequent splits. The replay step (Lines 13–23) samples a mini-batch of $B_m$ transitions from a buffer of size $B$, computing targets and updating the network with $|\theta|$ parameters, costing $\mathcal{O}(B_m \cdot |\theta|)$. Adding up these, over $E_{\text{max}} \cdot S_{\text{max}}$ steps, the training complexity is $\mathcal{O}(E_{\text{max}} \cdot S_{\text{max}} \cdot (M + N \cdot m_{\text{max}} \cdot 2^{m_{\text{max}}} + B_m \cdot |\theta|))$. In the rule extraction phase (Lines 33–60), the algorithm explores up to $S_{\text{max}}$ steps, selecting the top-$k$ actions (Line 40) with a cost of $\mathcal{O}(M \cdot \log k)$ due to sorting. For each action, support and rule computations (Lines 41–50) cost $\mathcal{O}(N \cdot m_{\text{max}} \cdot 2^{m_{\text{max}}})$, leading to a total complexity of $\mathcal{O}(k \cdot S_{\text{max}} \cdot (M \cdot \log k + N \cdot m_{\text{max}} \cdot 2^{m_{\text{max}}}))$. The overall computational cost is $\mathcal{O}(E_{\text{max}} \cdot S_{\text{max}} \cdot (M + N \cdot m_{\text{max}} \cdot 2^{m_{\text{max}}} + B_m \cdot |\theta|) + k \cdot S_{\text{max}} \cdot (M \cdot \log k + N \cdot m_{\text{max}} \cdot 2^{m_{\text{max}}}))$. The exponential term $2^{m_{\text{max}}}$ appears twice in the formula, as we scan through the itemsets, constrained by $2^{m_{\text{max}}}$, and evaluate all its generated rules, in both the agent training (Line 11) and rule extraction (Line 43-50) stages. The overall cost is dominated by $\mathcal{O}(E_{\text{max}} \cdot S_{\text{max}} \cdot (M + N \cdot m_{\text{max}} \cdot 2^{m_{\text{max}}} + B_m \cdot |\theta|))$.

The memory complexity of RLAR arises from storing the dataset, neural networks, replay buffer, and extracted rules. The transaction matrix $\mathcal{T}$ requires $\mathcal{O}(N \cdot M)$ to store $N$ transactions with $M$ items. The Q-network and target network each store $|\theta|$ parameters, e.g. $|\theta| \approx (M + 3) \cdot 128 + 128 \cdot 128 + 128 \cdot M$ for an input layer of size $M + 3$ (size of state features), three hidden layers of 128 neurons, and an output layer of size $M$, resulting in $\mathcal{O}(M \cdot 128)$. The replay buffer stores e.g. $B = 10000$ transitions, each containing two states ($M + 3$), an action (scalar), reward (scalar), and terminal flag, costing $\mathcal{O}(B \cdot M)$. During rule extraction, trajectories are bounded by $k \cdot S_{\text{max}}$, each holding a state and itemset ($\mathcal{O}(M)$), while the rule set $\mathcal{R}$ can grow up to $\mathcal{O}(2^{m_{\text{max}}} \cdot m_{\text{max}})$ for itemsets of size up to $m_{\text{max}}$. The total memory complexity is $\mathcal{O}(N \cdot M + B \cdot M + |\theta| + k \cdot S_{\text{max}} \cdot M + 2^{m_{\text{max}}} \cdot m_{\text{max}})$, dominated by the transaction matrix and replay buffer for large $N$ and $B$. With $B = 10000$, $m_{\text{max}} = 10$, and typical $N$ and $M$, memory usage remains manageable, scaling linearly with dataset size and buffer capacity, though the rule set can become significant with large $m_{\text{max}}$.

\subsection{RLAR experiment}

\subsubsection{Experimental setup}

The RLAR method was tested on the same \textit{Synthetic 1} dataset which consists of 1000 transactions involving 10 items, where each transaction is represented as a binary vector indicating item presence or absence, generated using a multivariate normal distribution with an RBF kernel and positive latent values. The dataset includes a feature matrix $X \in \mathbb{R}^{10 \times 10}$ sampled from a standard normal distribution, though the feature matrix was not utilized in RLAR. The experiment employed a DQN framework with following parameters: maximum episodes $E_{\text{max}} = 1000$, maximum steps per episode $S_{\text{max}} = 100$, a replay buffer size $B = 10000$, minibatch size $B_m = 32$, discount factor $\gamma = 0.99$, initial exploration rate $\epsilon = 1.0$ exponentially decaying to $\epsilon_{\text{min}} = 0.1$ over 1000 steps, target network is updated every 1000 steps, and a maximum itemset size $m_{\text{max}} = 10$ is used. Performance metrics, including runtime, memory usage, number of frequent itemsets, and number of rules, were evaluated across minimum support thresholds of 0.1, 0.2, 0.3, 0.4, and 0.5, with a fixed minimum confidence threshold of 0.5. The experiment was executed on the same laptop as before.

\subsubsection{Results}

The RLAR experiment on the \textit{Synthetic 1} dataset showcases the efficacy of the DQN approach in extracting association rules, as detailed in Table.\ref{tab:rlar_metrics}. At a minimum support of 0.1, RLAR required 5337.17 seconds with a memory usage of 433.24 MB, producing 79 frequent itemsets and 1106 rules. As the support threshold increased, runtime exhibited variability, with a notable decrease to 716.99 seconds at 0.5, while memory usage fluctuated, reaching 96.75 MB at 0.5, indicating adaptive computational resource allocation with stricter constraints. The maximum number of frequent itemsets (148) and rules (2436) occurred at a support of 0.4, highlighting the method's capacity to identify a diverse set of patterns. The cumulative reward plots, corresponding to minimum support values of 0.1, 0.2, 0.3, 0.4, and 0.5 (as shown in Fig.\ref{fig:five_cumulative_rewards_plot}), illustrate the agent's learning trajectory, featuring an initial sharp reward increase (from approximately 115 to 130 within the first 400 episodes in Fig.\ref{fig:five_cumulative_rewards_plot}(a)) followed by a stabilization phase with minor oscillations, led by the epsilon-greedy policy shifting from broad exploration to focused exploitation.

The cumulative reward trends provide critical insights into the RLAR agent's learning behavior across different support thresholds. For a minimum support of 0.1, the reward rises steeply within the initial 200 episodes, before showing fluctuations and a gradual stabilization after 400 episodes. Accident decline in the cumulative reward is possibly due to excessive exploration or convergence toward less optimal strategies. At 0.4, the cumulative reward also stabilizes after 400 episodes, reflecting sustained rule quality, whereas at 0.5, variation of the episode-wise rewards is large, suggesting a reduction in discoverable itemsets due to the elevated threshold. These dynamics are supported by the DQN's experience replay mechanism and periodic target network updates, which promote learning consistency but may introduce variability from initial random action selections. For example, in later the rule analysis (Table.\ref{tab:rlar_rules_item2}), the top rule (ranked by confidence) with item\_2 on RHS, i.e. $[0, 3, 4, 6, 9] \to [2]$, has a confidence of 0.9838 and support of 0.485 (backed by 485 transactions among all the synthesized 1000 transaction records), confirms the agent's proficiency in discovering and prioritizing high-quality rules, despite the substantial computational effort required.

\begin{figure}[htbp]
    \centering

    \begin{minipage}[b]{0.32\textwidth}
        \centering
        \includegraphics[width=\linewidth]{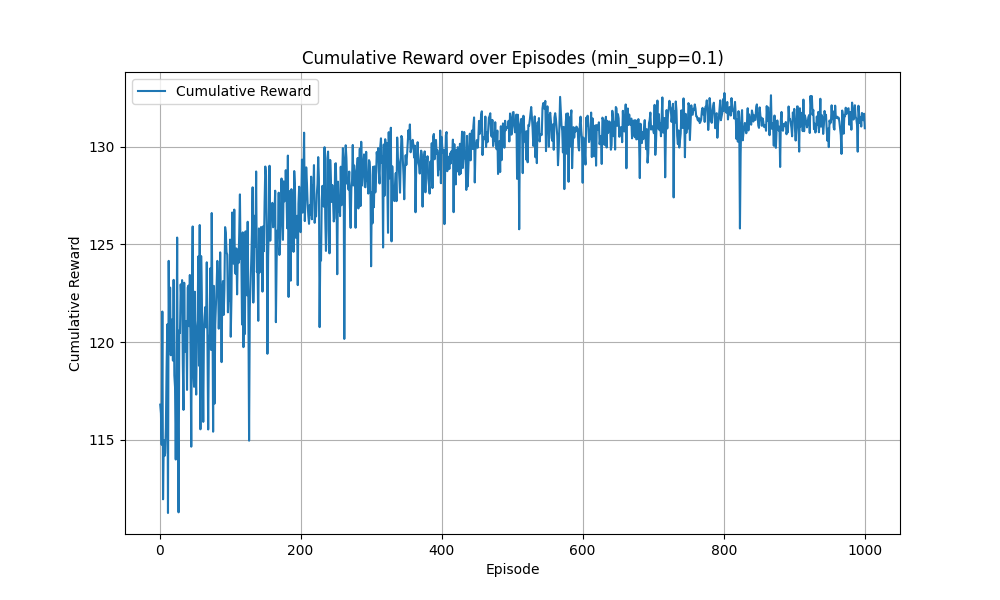}
        \caption*{(a) Min support: 0.1}
    \end{minipage}
    \hfill
    \begin{minipage}[b]{0.32\textwidth}
        \centering
        \includegraphics[width=\linewidth]{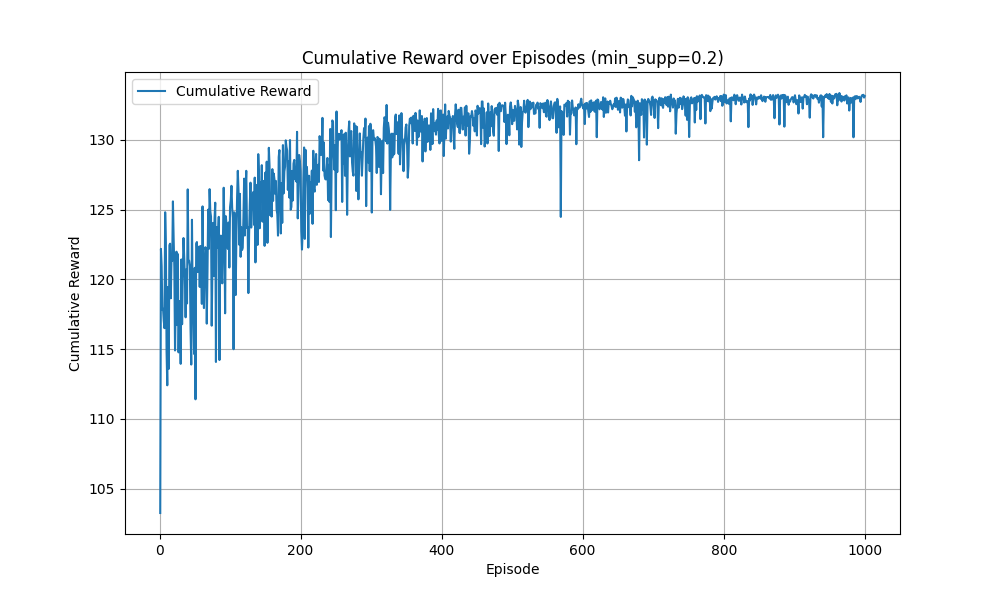}
        \caption*{(b) Min support: 0.2}
    \end{minipage}
    \hfill
    \begin{minipage}[b]{0.32\textwidth}
        \centering
        \includegraphics[width=\linewidth]{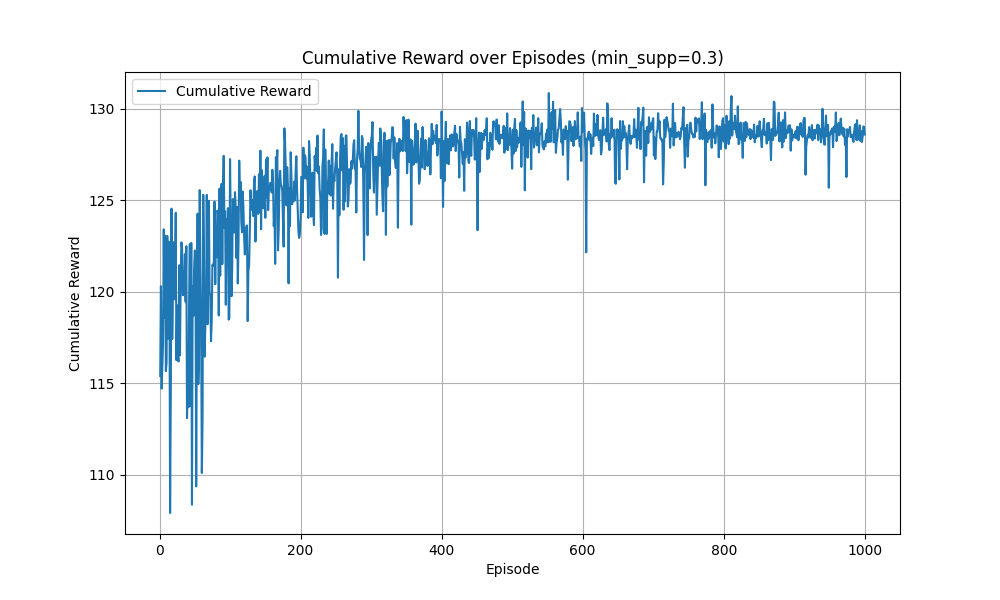}
        \caption*{(c) Min support: 0.3}
    \end{minipage}

    \vspace{0.5em} 

    \begin{minipage}[b]{0.32\textwidth}
        \centering
        \includegraphics[width=\linewidth]{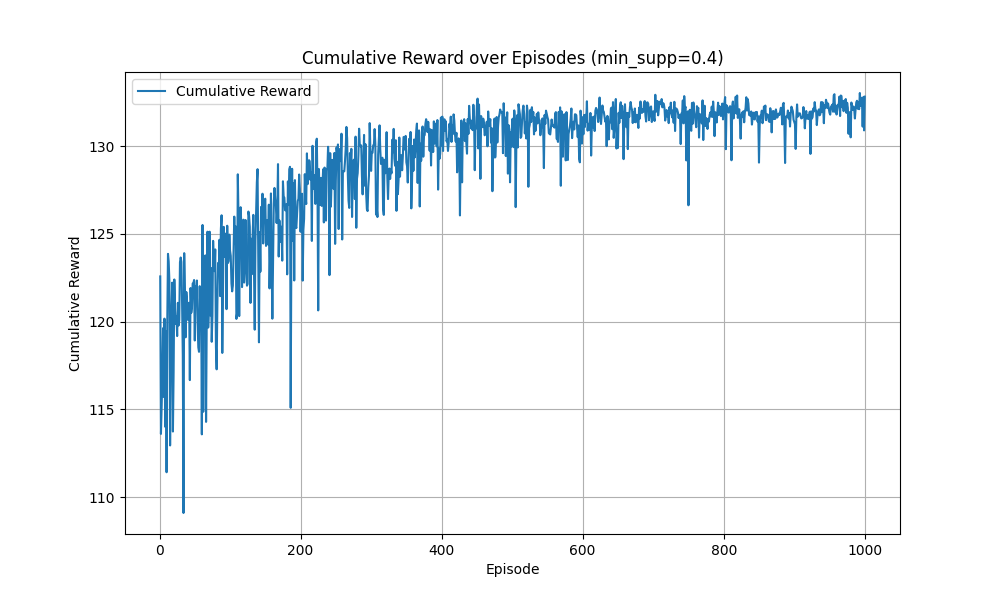}
        \caption*{(d) Min support: 0.4}
    \end{minipage}
    \begin{minipage}[b]{0.32\textwidth}
        \centering
        \includegraphics[width=\linewidth]{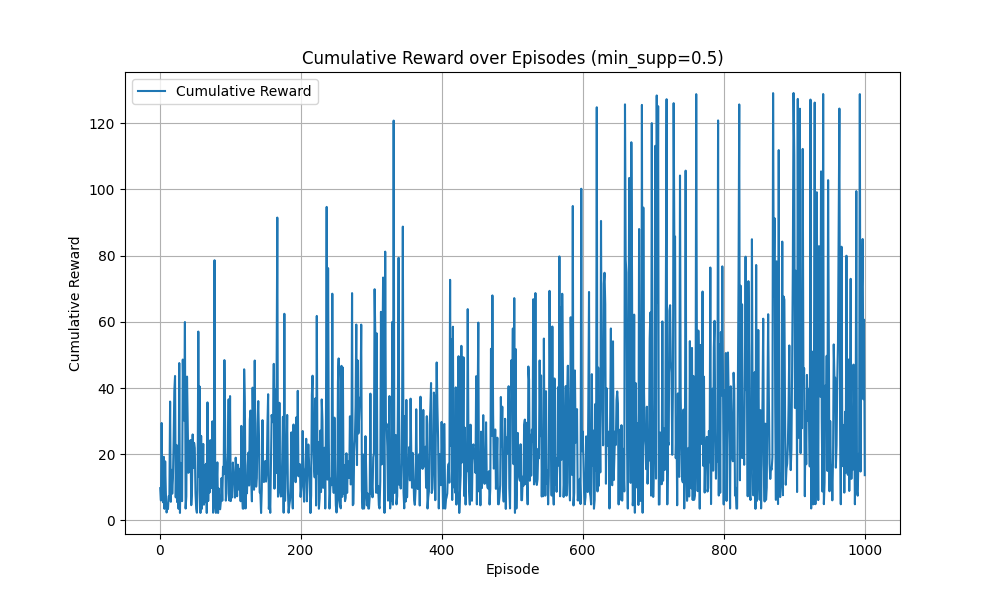}
        \caption*{(e) Min support: 0.5}
    \end{minipage}

    \caption{Cumulative reward \textit{vs} episode index for the 5 min support scenarios.}
    \label{fig:five_cumulative_rewards_plot}
\end{figure}

\begin{table}[H]
\centering
\scriptsize
\begin{threeparttable}
\caption{Performance metrics of RLAR on \textit{Synthetic 1} dataset}
\label{tab:rlar_metrics}
\begin{tabular}{p{1.6cm} p{1.6cm} p{1.8cm} p{2.3cm} p{1cm}}
\toprule
Min Support & Runtime (s) & Memory (MB) & Frequent itemsets & Rules \\
\midrule
0.1 & 5337.17 & 433.24 & 79 & 1106 \\
0.2 & 5803.99 & 1.25 & 97 & 1454 \\
0.3 & 3451.94 & 0.5 & 98 & 1394 \\
0.4 & 4046.33 & 1.25 & 148 & 2436 \\
0.5 & 716.99 & 96.75 & 145 & 2382 \\
\bottomrule
\end{tabular}
\begin{tablenotes}
\item[1] Total number of transactions: 1000.
\item[2] Minimum confidence used: 0.5.
\item[3] Runtime and memory usage include the DQN training and rule extraction phases.
\end{tablenotes}
\end{threeparttable}
\end{table}

\begin{table}[H]
\centering
\scriptsize
\begin{threeparttable}
\caption{Top 10 association rules mined by \textbf{RLAR} (ranked in descending order by \textit{\color{red}{support}})}
\label{tab:RLAR_mined_rules_rankedBySupport_minSupport01_Synthetic1}
\begin{tabular}{r >{\raggedright\arraybackslash}p{1.2cm} >{\raggedright\arraybackslash}p{1.2cm} p{0.6cm} p{0.6cm} p{0.6cm}}
\toprule
\# & Antecedent & Consequent & \color{red}{Supp} & Conf & Lift \\
\midrule
5   & [item\_5] & [item\_7] & 0.662 & 0.9118 & 1.2700 \\
6   & [item\_7] & [item\_5] & 0.662 & 0.9220 & 1.2700 \\
11  & [item\_7] & [item\_8] & 0.662 & 0.9220 & 1.2877 \\
12  & [item\_8] & [item\_7] & 0.662 & 0.9246 & 1.2877 \\
233 & [item\_0] & [item\_5] & 0.660 & 0.9141 & 1.2591 \\
234 & [item\_5] & [item\_0] & 0.660 & 0.9091 & 1.2591 \\
431 & [item\_0] & [item\_8] & 0.656 & 0.9086 & 1.2690 \\
432 & [item\_8] & [item\_0] & 0.656 & 0.9162 & 1.2690 \\
3   & [item\_5] & [item\_8] & 0.643 & 0.8857 & 1.2370 \\
4   & [item\_8] & [item\_5] & 0.643 & 0.8980 & 1.2370 \\
\bottomrule
\end{tabular}
\begin{tablenotes}
\item[1] Total number of transactions: 1000.
\item[2] Minimum support used in mining: 0.1.
\item[3] Minimum confidence used: 0.5.
\item[4] Items correspond to synthetic dataset indices.
\end{tablenotes}
\end{threeparttable}
\end{table}

\begin{table}[H]
\centering
\scriptsize
\begin{threeparttable}
\caption{Top 10 association rules mined by \textbf{RLAR} (ranked in descending order by \textit{\color{red}{confidence}})}
\label{tab:RLAR_mined_rules_rankedByConfidence_minSupport01_Synthetic1}
\begin{tabular}{r >{\raggedright\arraybackslash}p{4cm} >{\raggedright\arraybackslash}p{1.2cm} p{0.6cm} p{0.6cm} p{0.6cm}}
\toprule
\# & Antecedent & Consequent & Supp & \color{red}{Conf} & Lift \\
\midrule
797  & [item\_1, item\_4, item\_5, item\_8] & [item\_7] & 0.529 & 0.9981 & 1.3901 \\
827  & [item\_1, item\_3, item\_5, item\_8] & [item\_7] & 0.546 & 0.9964 & 1.3877 \\
491  & [item\_3, item\_4, item\_5, item\_8] & [item\_7] & 0.523 & 0.9962 & 1.3875 \\
708  & [item\_3, item\_5, item\_8, item\_9] & [item\_7] & 0.536 & 0.9944 & 1.3850 \\
1013 & [item\_0, item\_3, item\_4, item\_8] & [item\_7] & 0.522 & 0.9943 & 1.3848 \\
1073 & [item\_1, item\_3, item\_4, item\_8] & [item\_7] & 0.517 & 0.9923 & 1.3821 \\
521  & [item\_3, item\_4, item\_5, item\_9] & [item\_7] & 0.505 & 0.9921 & 1.3818 \\
1104 & [item\_0, item\_3, item\_8, item\_9] & [item\_7] & 0.537 & 0.9908 & 1.3799 \\
302  & [item\_3, item\_4, item\_8]         & [item\_7] & 0.536 & 0.9908 & 1.3799 \\
858  & [item\_5, item\_6, item\_8, item\_9] & [item\_7] & 0.533 & 0.9907 & 1.3798 \\
\bottomrule
\end{tabular}
\begin{tablenotes}
\item[1] Total number of transactions: 1000.
\item[2] Minimum support used in mining: 0.1.
\item[3] Minimum confidence used: 0.5.
\item[4] Items correspond to synthetic dataset indices.
\end{tablenotes}
\end{threeparttable}
\end{table}

\begin{table}[H]
\centering
\scriptsize
\begin{threeparttable}
\caption{Top 10 association rules mined by \textbf{RLAR} (ranked in descending order by \textit{\color{red}{lift}})}
\label{tab:RLAR_mined_rules_rankedByLift_minSupport01_Synthetic1}
\begin{tabular}{r >{\raggedright\arraybackslash}p{3cm} >{\raggedright\arraybackslash}p{3cm} p{0.6cm} p{0.6cm} p{0.6cm}}
\toprule
\# & Antecedent & Consequent & Supp & Conf & \color{red}{Lift} \\
\midrule
604  & [item\_5, item\_9]               & [item\_3, item\_4, item\_8] & 0.501 & 0.8392 & 1.5512 \\
607  & [item\_3, item\_4, item\_8]      & [item\_5, item\_9]          & 0.501 & 0.9261 & 1.5512 \\
596  & [item\_3, item\_4]               & [item\_5, item\_8, item\_9] & 0.501 & 0.8805 & 1.5447 \\
615  & [item\_5, item\_8, item\_9]      & [item\_3, item\_4]          & 0.501 & 0.8789 & 1.5447 \\
511  & [item\_3, item\_4, item\_7]      & [item\_5, item\_9]          & 0.505 & 0.9165 & 1.5352 \\
508  & [item\_5, item\_9]               & [item\_3, item\_4, item\_7] & 0.505 & 0.8459 & 1.5352 \\
500  & [item\_3, item\_4]               & [item\_5, item\_7, item\_9] & 0.505 & 0.8875 & 1.5329 \\
519  & [item\_5, item\_7, item\_9]      & [item\_3, item\_4]          & 0.505 & 0.8722 & 1.5329 \\
1032 & [item\_0, item\_3, item\_4]      & [item\_8, item\_9]          & 0.501 & 0.9347 & 1.5323 \\
1031 & [item\_8, item\_9]               & [item\_0, item\_3, item\_4] & 0.501 & 0.8213 & 1.5323 \\
\bottomrule
\end{tabular}
\begin{tablenotes}
\item[1] Total number of transactions: 1000.
\item[2] Minimum support used in mining: 0.1.
\item[3] Minimum confidence used: 0.5.
\item[4] Items correspond to synthetic dataset indices.
\end{tablenotes}
\end{threeparttable}
\end{table}

\begin{table}[H]
\centering
\scriptsize
\begin{threeparttable}
\caption{Top 10 association rules mined by \textbf{RLAR} with item 2 on RHS (min support = 0.1, rules ranked by \color{red}{confidence})}
\label{tab:rlar_rules_item2}
\begin{tabular}{r >{\raggedright\arraybackslash}p{2.8cm} p{1.2cm} p{0.8cm} p{0.9cm} >{\raggedright\arraybackslash}p{1.2cm} >{\raggedright\arraybackslash}p{5.5cm}}
\toprule
\# & Rule & Confidence & Support & Support Count & Example Trans. Indices & Example Transactions \\
\midrule
1 & [0, 3, 4, 6, 9] $\to$ [2] & 0.9838 & 0.485 & 485 & [1, 2, 3] & [[0, 1, 2, 3, 4, 5, 6, 7, 8, 9], [0, 1, 2, 3, 4, 5, 6, 7, 8, 9], [0, 1, 2, 3, 4, 5, 6, 7, 8, 9]] \\
2 & [0, 3, 4, 6, 8, 9] $\to$ [2] & 0.9836 & 0.480 & 480 & [1, 2, 3] & [[0, 1, 2, 3, 4, 5, 6, 7, 8, 9], [0, 1, 2, 3, 4, 5, 6, 7, 8, 9], [0, 1, 2, 3, 4, 5, 6, 7, 8, 9]] \\
3 & [0, 1, 4, 6, 8, 9] $\to$ [2] & 0.9835 & 0.476 & 476 & [1, 2, 3] & [[0, 1, 2, 3, 4, 5, 6, 7, 8, 9], [0, 1, 2, 3, 4, 5, 6, 7, 8, 9], [0, 1, 2, 3, 4, 5, 6, 7, 8, 9]] \\
4 & [4, 6, 7, 8, 9] $\to$ [2] & 0.9821 & 0.495 & 495 & [1, 2, 3] & [[0, 1, 2, 3, 4, 5, 6, 7, 8, 9], [0, 1, 2, 3, 4, 5, 6, 7, 8, 9], [0, 1, 2, 3, 4, 5, 6, 7, 8, 9]] \\
5 & [4, 5, 6, 8, 9] $\to$ [2] & 0.9820 & 0.492 & 492 & [1, 2, 3] & [[0, 1, 2, 3, 4, 5, 6, 7, 8, 9], [0, 1, 2, 3, 4, 5, 6, 7, 8, 9], [0, 1, 2, 3, 4, 5, 6, 7, 8, 9]] \\
6 & [0, 4, 6, 8, 9] $\to$ [2] & 0.9820 & 0.490 & 490 & [1, 2, 3] & [[0, 1, 2, 3, 4, 5, 6, 7, 8, 9], [0, 1, 2, 3, 4, 5, 6, 7, 8, 9], [0, 1, 2, 3, 4, 5, 6, 7, 8, 9]] \\
7 & [0, 4, 5, 6, 8, 9] $\to$ [2] & 0.9818 & 0.485 & 485 & [1, 2, 3] & [[0, 1, 2, 3, 4, 5, 6, 7, 8, 9], [0, 1, 2, 3, 4, 5, 6, 7, 8, 9], [0, 1, 2, 3, 4, 5, 6, 7, 8, 9]] \\
8 & [0, 4, 6, 7, 8, 9] $\to$ [2] & 0.9818 & 0.485 & 485 & [1, 2, 3] & [[0, 1, 2, 3, 4, 5, 6, 7, 8, 9], [0, 1, 2, 3, 4, 5, 6, 7, 8, 9], [0, 1, 2, 3, 4, 5, 6, 7, 8, 9]] \\
9 & [1, 4, 5, 6, 9] $\to$ [2] & 0.9817 & 0.483 & 483 & [1, 2, 3] & [[0, 1, 2, 3, 4, 5, 6, 7, 8, 9], [0, 1, 2, 3, 4, 5, 6, 7, 8, 9], [0, 1, 2, 3, 4, 5, 6, 7, 8, 9]] \\
10 & [4, 5, 6, 9] $\to$ [2] & 0.9803 & 0.498 & 498 & [1, 2, 3] & [[0, 1, 2, 3, 4, 5, 6, 7, 8, 9], [0, 1, 2, 3, 4, 5, 6, 7, 8, 9], [0, 1, 2, 3, 4, 5, 6, 7, 8, 9]] \\
\bottomrule
\end{tabular}
\begin{tablenotes}
\item[1] Total number of transactions: 1000.
\item[2] Minimum support used in mining: 0.1.
\item[3] Minimum confidence used: 0.5.
\item[4] Items correspond to synthetic dataset indices.
\end{tablenotes}
\end{threeparttable}
\end{table}

\subsubsection{Comparison with other methods}

A comparative evaluation of RLAR against GPAR, BARM, MAB-ARM, and traditional algorithms (Apriori, FP-Growth, Eclat) on the \textit{Synthetic 1} dataset reveals distinct performance profiles. Table.\ref{tab:comparison_runtime} (runtime) and Table.\ref{tab:comparison_memory} (memory usage), alongside Tables.\ref{tab:comparison_itemsets} (no. frequent itemsets) and.\ref{tab:comparison_rules} (no. rules), provide a comprehensive overview. RLAR's runtime of 5337.17 seconds at a minimum support of 0.1 significantly surpasses GPAR variants (0.133–10.179 seconds), BARM (1.1821 seconds), MAB-ARM (1.6599 seconds), Apriori (0.483 seconds), FP-Growth (2.967 seconds), and Eclat (0.556 seconds), reflecting the extensive iterative process of reinforcement learning across 1000 episodes. Memory usage for RLAR peaks at 433.24 MB at 0.1, exceeding all other methods, including Eclat's maximum of 18.149 MB, due to the DQN model's storage requirements, replay buffer, and rule extraction overhead \footnote{As pointed out in MAB-ARM memory analysis, variations in memory usage up to 100 MB is statistically insignificant to be verified.}. RLAR generates fewer frequent itemsets (79 at 0.1) and rules (1106 at 0.1) compared to GPAR RBF/shifted RBF (1013 itemsets, 54036 rules), Apriori/FP-Growth/Eclat (1023–1033 itemsets, 57002 rules), and MAB-ARM (1000 itemsets, 50372 rules), but comparable to GPAR neural network (140 itemsets, 322 rules), NTK (305 itemsets, 1485 rules), and BARM (876 itemsets, 12029 rules). Interestingly, in RLAR, the numbers of frequent itemsets identified and rules mined increase as support increases.

This unusual trend, where the number of frequent itemsets rises from 79 at a support of 0.1 to a peak of 148 at 0.4, and the number of rules increases from 1106 to 2436 over the same range, can be attributed to several factors. First, we trained 5 agents for searching frequent itemsets and mining high reward (blending confidence and lift metrics) rules, 1 agent for each min support level. The DQN's adaptive learning policy may effectively refine its focus on higher-support itemsets as the threshold increases, potentially identifying more robust patterns that meet the stricter criteria, a behavior facilitated by the experience replay mechanism and target network updates. Second, the top-$k$ action selection strategy (with $k = 3$ in our experiment) during rule extraction likely promotes the exploration of diverse item combinations, allowing RLAR to uncover additional valid rules as the support threshold narrows the search space to more prevalent associations. Third, the synthetic nature of the \textit{Synthetic 1} dataset, with its 1000 transactions and 10 items, may contain a distribution where higher-support itemsets correlate with increased rule generation, possibly due to the presence of strongly associated item clusters that become more apparent at intermediate support levels (e.g. 0.4). This contrasts with traditional methods like Apriori, where the number of itemsets and rules typically decreases with increasing support (i.e. more stricter filter). Finally, this trend also evidences the DQN's ability to leverage its 1000-episode training ($E_{\text{max}} = 1000$) to adaptively prioritize rules that emerge under varying support conditions, a capability that requires further investigation with larger datasets.

Rule quality analysis further differentiates RLAR from its peers. Table.\ref{tab:rlar_rules_item2} displays the top 10 RLAR rules with item 2 as the consequent, ordered by confidence at a minimum support of 0.1. The highest-ranking rule, $[0, 3, 4, 6, 9] \to [2]$ (with confidence 0.9838, support 0.485, count 485 out of 1000 transaction records), matches the empirical frequency (count/1000 = 0.485), avoiding the support underestimation seen in probabilistic methods such as BARM (e.g. support 0.1916 \textit{vs} count/1000 = 0.496 for the top confidence rule $[1, 4, 6, 9] \to [2]$ in Table.\ref{tab:BARM_analysis_minSupport01_Synthetic1}) and GPAR variants (e.g. support 0.3600 \textit{vs} count/1000 = 0.488 for $[0, 1, 3, 4, 8] \to [2, 7]$ for Rule $[0, 1, 3, 4, 8] \to [2, 7]$ in Table.\ref{tab:RBF_GPAR_analysis_minSupport01_Synthetic1} (RBF) and Table.\ref{tab:shifted_RBF_GPAR_analysis_minSupport01_Synthetic1} (shifted RBF)). This empirical accuracy aligns with Apriori (Table.\ref{tab:Apriori_analysis_minSupport01_Synthetic1}), FP-Growth (Table.\ref{tab:FP-Growth_analysis_minSupport01_Synthetic1}), and MAB-ARM (Table.\ref{tab:MABARM_analysis_minSupport01_Synthetic1}), which rely on observed frequencies (Eclat and BARM, as observed in Table.\ref{tab:Eclat_analysis_minSupport01_Synthetic1} and Table.\ref{tab:BARM_analysis_minSupport01_Synthetic1} respectively, however, over-estimate the supports of their mined rules. Also plotted in Fig.\ref{fig:support_vs_count}). However, RLAR's lift values, such as 1.5512 for the top rule $[5, 9] \to [3, 4, 8]$ in Table.\ref{tab:RLAR_mined_rules_rankedByLift_minSupport01_Synthetic1}, are moderate compared to GPAR RBF/shifted RBF's peak lift of 12.1212 for their Rule $[0, 8, 6, 9] \to [1, 2, 3, 4, 7]$ (Table.\ref{tab:RBF_GPAR_mined_rules_rankedByLift_minSupport01_Synthetic1} and Table.\ref{tab:shifted_RBF_GPAR_mined_rules_rankedByLift_minSupport01_Synthetic1}), neural net GPAR's 6.4103 for Rule $[8, 6, 7] \to [5]$ in Table.\ref{tab:nn_GPAR_mined_rules_rankedByLift_minSupport01_Synthetic1}, NTK GPAR's 8.5470 for Rule $[8, 1, 4] \to [2, 5, 6]$ in Table.\ref{tab:NTK_GPAR_mined_rules_rankedByLift_minSupport01_Synthetic1}, and comparable to Apriori's 1.7982 for Rule $[6, 0, 1, 9, 7] \to [2, 3, 8, 5, 4]$ in Table.\ref{tab:Apriori_mined_rules_rankedByLift_minSupport01_Synthetic1}, FP-Growth's 1.7982 for Rule $[6, 0, 1, 9, 7] \to [2, 3, 8, 5, 4]$ in Table.\ref{tab:FP-Growth_mined_rules_rankedByLift_minSupport01_Synthetic1}, Eclat's 1.8152 for Rule $[2, 0, 1, 3, 8, 4] \to [6, 7, 5]$ in Table.\ref{tab:Eclat_mined_rules_rankedByLift_minSupport01_Synthetic1}, and MAB-ARM's 1.7612 for Rule $[2, 3, 4, 5] \to [0, 1, 6, 9]$ in Table.\ref{tab:MABARM_mined_rules_rankedByLift_minSupport01_Synthetic1}.
BARM's minimal lift, i.e. 1.0001 for Rule $[0, 2, 3, 6] \to [4, 7]$ in Table.\ref{tab:BARM_mined_rules_rankedByLift_minSupport01_Synthetic1}, was due to its item independence assumption.

\begin{table}[H]
\centering
\tiny
\begin{threeparttable}
\caption{Comparison: runtime (s) on \textit{Synthetic 1} dataset} 
\label{tab:comparison_runtime}
\begin{tabular}{p{0.9cm} p{0.6cm} p{0.6cm} p{0.6cm} p{0.8cm} p{0.6cm} p{0.6cm} p{0.6cm} p{0.6cm} p{0.6cm} p{0.6cm} p{0.6cm}}
\toprule
Min Support & RLAR & GPAR (RBF) & GPAR (shifted RBF) & GPAR (neural net) & GPAR (NTK) & BARM & MAB-ARM & Apriori & FP-Growth & Eclat \\
\midrule
0.1 & 5337.17 & 10.179 & 8.930 & 0.133 & 0.449 & 1.1821 & 1.6599 & 0.483 & 2.967 & 0.556 \\
0.2 & 5803.99 & 8.448 & 8.697 & 0.046 & 0.077 & 0.2314 & 1.6036 & 0.505 & 2.734 & 0.555 \\
0.3 & 3451.94 & 3.295 & 2.658 & 0.017 & 0.030 & 0.0734 & 1.6251 & 0.565 & 2.622 & 0.575 \\
0.4 & 4046.33 & 0.215 & 0.174 & 0.002 & 0.016 & 0.0481 & 1.6126 & 0.570 & 2.700 & 0.558 \\
0.5 & 716.99 & 0.013 & 0.034 & 0.005 & 0.003 & 0.0451 & 0.9390 & 0.211 & 2.019 & 0.635 \\
\bottomrule
\end{tabular}
\end{threeparttable}
\end{table}

\begin{table}[H]
\centering
\tiny
\begin{threeparttable}
\caption{Comparison: memory usage (MB) on \textit{Synthetic 1} dataset} 
\label{tab:comparison_memory}
\begin{tabular}{p{0.9cm} p{0.6cm} p{0.6cm} p{0.6cm} p{0.8cm} p{0.6cm} p{0.6cm} p{0.6cm} p{0.6cm} p{0.6cm} p{0.6cm} p{0.6cm}}
\toprule
Min Support & RLAR & GPAR (RBF) & GPAR (shifted RBF) & GPAR (neural net) & GPAR (NTK) & BARM & MAB-ARM & Apriori & FP-Growth & Eclat \\
\midrule
0.1 & 433.24 & 0.0 & 0.0 & 0.0 & 0.0 & 0.0 & 0.0 & 0.0 & 0.063 & 18.149 \\
0.2 & 1.25 & 0.0 & 0.0 & 0.0 & 0.0 & 0.0 & 13.625 & 0.0 & 0.0 & 0.0 \\
0.3 & 0.5 & 0.0 & 0.0 & 0.0 & 0.0 & 0.0 & 0.625 & 0.0 & 0.0 & 1.438 \\
0.4 & 1.25 & 0.0 & 0.0 & 0.0 & 0.0 & 0.0 & 6.375 & 0.0 & 0.0 & 7.969 \\
0.5 & 96.75 & 0.0 & 0.0 & 0.0 & 0.0 & 0.0 & 0.125 & 0.0 & 0.0 & 0.156 \\
\bottomrule
\end{tabular}
\end{threeparttable}
\end{table}

\begin{table}[H]
\centering
\tiny
\begin{threeparttable}
\caption{Comparison: no. of frequent itemsets on \textit{Synthetic 1} dataset} 
\label{tab:comparison_itemsets}
\begin{tabular}{p{0.9cm} p{0.6cm} p{0.6cm} p{0.6cm} p{0.8cm} p{0.6cm} p{0.6cm} p{0.6cm} p{0.6cm} p{0.6cm} p{0.6cm} p{0.6cm}}
\toprule
Min Support & RLAR & GPAR (RBF) & GPAR (shifted RBF) & GPAR (neural net) & GPAR (NTK) & BARM & MAB-ARM & Apriori & FP-Growth & Eclat \\
\midrule
0.1 & 79 & 1013 & 1013 & 140 & 305 & 876 & 1000 & 1023 & 1023 & 1033 \\
0.2 & 97 & 981 & 982 & 43 & 67 & 385 & 1000 & 1023 & 1023 & 1033 \\
0.3 & 98 & 533 & 536 & 9 & 11 & 165 & 1000 & 1023 & 1023 & 1033 \\
0.4 & 148 & 67 & 65 & 0 & 2 & 45 & 1000 & 1023 & 1023 & 1033 \\
0.5 & 145 & 1 & 2 & 0 & 0 & 44 & 736 & 746 & 746 & 1033 \\
\bottomrule
\end{tabular}
\begin{tablenotes}
\item[1] Total number of transactions: 1000.
\item[2] Numbers are counts of frequent itemsets identified.
\item[3] Minimum confidence used in filtering candidate rules: 0.5.
\end{tablenotes}
\end{threeparttable}
\end{table}

\begin{table}[H]
\centering
\tiny
\begin{threeparttable}
\caption{Comparison: no. of rules on \textit{Synthetic 1} dataset} 
\label{tab:comparison_rules}
\begin{tabular}{p{0.9cm} p{0.6cm} p{0.6cm} p{0.6cm} p{0.8cm} p{0.6cm} p{0.6cm} p{0.6cm} p{0.6cm} p{0.6cm} p{0.6cm} p{0.6cm}}
\toprule
Min Support & RLAR & GPAR (RBF) & GPAR (shifted RBF) & GPAR (neural net) & GPAR (NTK) & BARM & MAB-ARM & Apriori & FP-Growth & Eclat \\
\midrule
0.1 & 1106 & 54036 & 54036 & 322 & 1485 & 12029 & 50372 & 57002 & 57002 & 57002 \\
0.2 & 1454 & 51479 & 51328 & 59 & 159 & 3024 & 50372 & 57002 & 57002 & 57002 \\
0.3 & 1394 & 19251 & 16067 & 18 & 26 & 802 & 50372 & 57002 & 57002 & 57002 \\
0.4 & 2436 & 782 & 530 & 0 & 4 & 90 & 50372 & 57002 & 57002 & 57002 \\
0.5 & 2382 & 2 & 8 & 0 & 0 & 88 & 20948 & 20948 & 20948 & 57002 \\
\bottomrule
\end{tabular}
\begin{tablenotes}
\item[1] Total number of transactions: 1000.
\item[2] Numbers are counts of rules generated.
\item[3] Minimum confidence used in filtering candidate rules: 0.5.
\end{tablenotes}
\end{threeparttable}
\end{table}

\begin{figure}[H]
    \centering
    \includegraphics[width=1.0\textwidth]{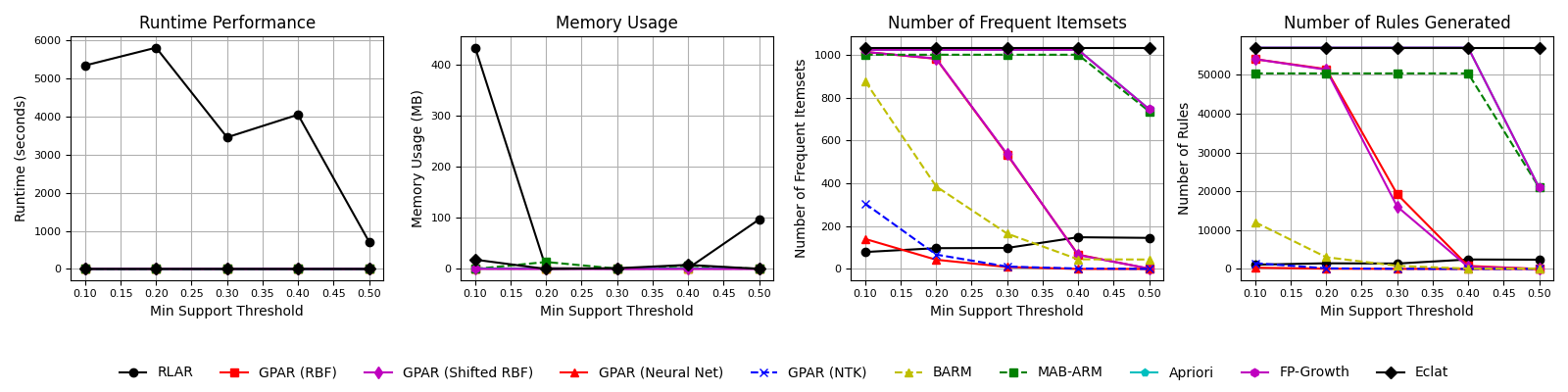}
    \caption{Performances of 10 algorithms on \textit{Synthetic 1} dataset across min support thresholds.}
    \label{fig:synthetic1_performance_comparison_10methods}
\end{figure}

\subsubsection{Discussion}

The DQN-based RLAR framework presents several strengths. In contrast to search-based methods such as MAB and MCTS, which require repeated exploration and rule evaluations, RLAR develops a generalizable policy, adaptable to target, specific rule types through a carefully crafted reward mechanism. The composite reward function, integrating support, confidence, and lift (Eq.\ref{eq:RLAR_reward_design}), is designed to prioritize the discovery of high-quality association rules. The scalability of this DQN approach suits large datasets, potentially surpassing exhaustive methods such as Apriori in flexibility (i.e. adaptive policy learning). Nonetheless, challenges such as training instability can be addressed with experience replay, target network stabilization, and precise hyperparameter tuning. The issue of reward sparsity, where high-quality rules are scarce, can be mitigated by adjusting the reward function to offer incremental rewards for frequent itemsets. The computational overhead from neural network training may be alleviated through optimized data structures and GPU acceleration.

RLAR's heavy computational cost and memory demands, particularly evident at 433.24 MB and 5337.17 seconds at a minimum support of 0.1, make it less ideal for large-scale applications compared to classic methods such as Apriori, FP-Growth, and Eclat, which offer lower runtimes (e.g. 0.483-2.967 seconds) and memory usage (e.g. 0.0-18.149 MB). However, RLAR's adaptive policy learning, as evidenced by the converged cumulative reward plots across support levels 0.1-0.4 in Fig.\ref{fig:five_cumulative_rewards_plot}, provides a strategic advantage over static methods, with potential for discovering unique patterns through refinements such as adjusting the maximum itemset size $m_{max}$, number of episode $E_{\text{max}}$, number of steps in each episode $S_{max}$, or enhancing pruning techniques \footnote{The computational and memory complexities can be exponentially reduced by constraining $m_{max}$ and linearly decreased by reducing $E_{\text{max}}$ and $S_{max}$, with theoretical evidence from complexity analysis and experimental evidence in our trials.}. Unlike probabilistic approaches such as GPAR and BARM, which may underestimate support due to smoothing effects, RLAR's empirical validation during rule extraction aligns with Apriori, FP-Growth and MAB-ARM's precision. Future improvements, e.g. integration of feature-based priors or optimization of the epsilon decay rate to balance exploration and exploitation, could position RLAR to compete with GPAR's capability to identify complex patterns while maintaining empirical accuracy. For applications requiring high-confidence rules with precise support, such as targeted marketing, RLAR remains a feasible choice given sufficient computational resources, though traditional methods are more efficient for resource-constrained analyses.

\paragraph{Top-$k$ actions} In the RLAR experiment, the adoption of a top-$k$ action selection strategy (with $k = 3$) during the rule extraction phase, as opposed to the greedy single best action ($k = 1$), enhances the identification of diverse, high-quality frequent itemsets and association rules. Unlike the greedy method, which relies exclusively on the highest Q-value predicted by the trained DQN and risks overlooking alternative rules due to potential noise from the 1000-episode training limit ($E_{\text{max}} = 1000$), the top-$k$ approach introduces controlled exploration. This enables the agent to assess multiple promising actions, capturing a wider variety of itemsets that meet the minimum support and confidence criteria. This top-$k$ method strikes a balance between exploration and computational efficiency, promoting diverse rule generation while addressing the constraints of a single-trajectory greedy approach.

\paragraph{Sparse rewards} Reward sparsity in RLAR refers to the scarcity (rare or infrequent) of high-quality rewards, i.e. those associated with discovering itemsets that generate valuable association rules (e.g. rules with high support, confidence, and lift),  within the vast combinatorial space of possible itemsets due to the limited number of itemsets that produce valuable association rules, posing a challenge to efficient DQN training. I RLAR, the majority of explored itemsets may yield low or negative rewards (e.g. -1 for itemsets with support below the minimum threshold, as specified in Eq.\ref{eq:RLAR_reward_design}. The reward design in Eq.\ref{eq:RLAR_reward_design} assigns a positive reward based on a weighted combination of confidence and lift for itemsets with support $\geq min\_supp$, a negative reward (-1) for those below, and zero otherwise. Since only a small fraction of the $2^M$ possible itemsets (where $M$ is the number of items) meet the minimum support and confidence thresholds, high-quality rewards are sparse, particularly in large or diverse datasets.), making it difficult for the agent to consistently identify rewarding actions. The combinatorial nature of association rule mining exacerbates reward sparsity. With $M=10$ items, there are $2^{10} = 1024$ possible itemsets, but only a subset (e.g. 79–148 itemsets at support thresholds 0.1–0.5, Table.\ref{tab:rlar_metrics}) meet the criteria for high-quality rules. The \textit{Synthetic 1} dataset’s specific structure, generated with a multivariate normal distribution and RBF kernel, may further limit the number of frequent itemsets, intensifying sparsity.

Sparse rewards complicate the DQN’s training process because the agent frequently encounters uninformative (zero or negative) feedback, leading to slower learning and potential training instability. This is observed in the RLAR experiment’s cumulative reward plots (Figure~\ref{fig:five_cumulative_rewards_plot}), where reward fluctuations and occasional declines (e.g. at support 0.5) suggest challenges in converging to optimal policies due to infrequent high-reward states. Reward sparsity may be alleviated by introducing incremental rewards for frequent itemsets to provide more frequent feedback, even if they do not immediately yield high-confidence rule. Additionally, hybrid reward functions incorporating domain knowledge or heuristic-guided exploration could steer the agent toward promising itemsets, reducing the impact of sparsity. Adjusting the maximum itemset size ($m_{\text{max}}$) or using optimized data structures (e.g. bit-vectors) can also limit the search space, increasing the likelihood of encountering rewarding states.

\paragraph{Strategies for optimising RLAR} To enhance the efficiency of the DQN within the RLAR framework, several optimization strategies are employed. Itemsets and transactions are represented as bit-vectors, enabling swift support calculations through bitwise operations. Early pruning of itemsets falling below the minimum support threshold leverages the anti-monotonicity property, which can minimizes the exploration of less promising candidates. Furthermore, training multiple DQNs concurrently allows simultaneous exploration of the itemset space, accelerating the detection of high-quality association rules while preserving computational practicality.

\section{Discussion} \label{sec:discussion}

\subsection{Probabilistic approaches to AR mining}

\paragraph{GPAR: overview} 
GPAR introduces a probabilistic approach to ARM by leveraging Gaussian processes (GPs) to uncover relationships between items in transactional data. GPs, typically used for continuous data, are adapted to the discrete scenario of AR mining, defined by a mean function and a covariance function (kernel) that model a distribution over functions. In GPAR, items are represented by feature vectors (e.g. [size, shape, color, price, category]) - they form a feature matrix which are used to predict the latent variables $\mathbf{z}$. For each transaction, the latent vector $\mathbf{z} \sim \mathcal{N}(\mathbf{0}, K)$, where the kernel induced covariance matrix $K$ captures item dependencies (with higher covariance between items that co-occur frequently), combining prior beliefs and evidence from data. These latent variables are then quantized, with sign indicating the presence of an item (e.g. item $k$ is present if $z_k > 0$), to match the observed binary transaction data. GPAR therefore transforms the traditionally unsupervised ARM problem into a supervised learning task by modeling the relationship between the feature vectors $\mathbf{x}_k$ and the latent variables $z_k$.

GPAR incorporates human knowledge through feature vector design and kernel selection, combining data-driven inference. By representing each item with attributes and computing co-occurrence probabilities through marginalization, GPAR can handle uncertainty and incorporate prior knowledge, providing deeper insights into item associations. Once the GP model is trained, it can be used to make probabilistic inference about rules with items within or beyond the universal itemset $\mathcal{I}$. By extending the GP framework to capture co-occurrences instead of isolated predictions, GPAR bridges the gap between probabilistic modeling and effective association rule discovery. This methodology represents a shift from individual item prediction to understanding complex, multi-item dependencies within transactional datasets. GPAR is best suited for datasets with small number of items or specialized applications (e.g. dynamic environments) where uncertainty quantification is desired or items are dynamically updated. 

\paragraph{Comparison with traditional ARM methods}
Integrating Gaussian processes with AR mining leverages their strengths to address traditional AR's limitations, particularly in handling uncertainty and incorporating prior knowledge.
Traditional AR mining techniques, such as Apriori, FP-Growth, and Eclat, rely on frequency counts and face critical limitations, including sensitivity to threshold selection, rule redundancy, and inability to capture uncertainty and rarity adequately. \cite{Gonzalez2020BRM} argued that frequentist thresholds (e.g. subjective min support threshold) often lead to significant yet infrequent associations. In traditional AR, the framework defines items and itemsets $\mathcal{I}=\{i_1, i_2, \dots, i_M\}$, with each $i_k$ as a binary variable, and rules are mined purely based on frequency counts without considering item attributes such as size, shape, color, or price. In contrast, GPAR introduces feature vectors to represent items, modeling the probability surface as a Gaussian process whose posterior accounts for both distances between input item feature vectors and information from transactions. This principled estimation of rule probabilities brings together extra info such prior knowledge about item similarities and correlations, item attributes, and transaction data to make more informed mining. GPAR's probabilistic nature enhances rule quality by modeling latent item relationships and non-linear dependencies, such as how item attributes influence co-occurrence, beyond simple frequency counts.

\paragraph{Discrete approximation}
When estimating the kernel via MLE, we use the thresholding $i_k = \mathbb{I}(z_k > 0)$ which simplifies the complex relationship between observed binary transactions and latent variables. This approximation, and the resulting approximation of $p(\mathbf{t}_j | \mathbf{z})$ via $z_k$, is a practical solution to the computational intractability of exact likelihood computation, enabling efficient parameter estimation by leveraging the Gaussian process’s structure. It is also theoretically grounded in probabilistic modeling - it aligns with common practices in binary classification with Gaussian processes, such as \textit{probit regression}, where latent variables are thresholded to produce binary outputs. For example, in binary linear classifiers such as GLM \cite{Boyd_Vandenberghe_2018}, the predicted label is often $\text{sign}(z_k)$, and here, $i_k = \mathbb{I}(z_k > 0)$ serves a similar role, bridging the gap between continuous and discrete. However, this simplified mapping may introduce bias, particularly for sparse transactions.

\paragraph{Co-occurrence modeling using Gaussian process}
Traditional GP regression and classification aim to predict individual target variables $z_k$ (the latent membership variable $z$ in our case) given input feature vectors $\mathbf{x}_k \in \mathbb{R}^d$. GP defines a distribution over functions, which allows for predictions with uncertainty estimates at new test points (i.e. new items in our case). This classical framework is useful for tasks that require independent predictions for each input. However, in the context of AR mining, single-point predictions are insufficient - ARM aims to identify statistically significant patterns of items that appear together across transactions \footnote{For example, understanding that customers who purchase 'bread' and 'butter' are also likely to purchase 'milk' involves capturing dependencies among multiple items. Predicting the appearance of 'bread', 'butter' or 'milk' alone doesn't help in decision-making.}. Our objective is therefore to identify co-occurrence patterns (measured by their joint probabilities) among items within transactions.

Classical GP regression only estimates the likelihood of individual item occurrences (e.g. 'bread' or 'milk'), but it does not capture their joint probability distribution. Thus, predicting $z_k$ alone provides limited insight, as ARM fundamentally concerns the likelihood of multiple items co-occurring (which is useful in decision-making such as goods import and layout in stores). We therefore extend GP to co-occurrence estimation through: (1) Feature-based representation of items. Each item $k \in \mathcal{I}$ is mapped to a feature vector $\mathbf{x}_k \in \mathbb{R}^d$, encapsulating attributes such as price, category, or size. These features inform the GP's covariance structure, allowing the model to generalize co-occurrence patterns even to previously unseen items. (2) Joint distribution estimation. Rather than predicting $z_k$ in isolation, GPAR estimates the joint probability $p(z_i, z_j, z_k)$, or more generally $p(\mathbf{z}_I)$ for $I \subseteq \mathcal{I}$. This is evaluated by integrating over the GP's latent space, capturing dependencies and higher-order interactions among multiple items. (3) Monte Carlo sampling for marginalization. To compute co-occurrence probabilities, GPAR performs Monte Carlo sampling (Step 8 in Algo.\ref{algo:GPAR}). For a candidate itemset $I$, the co-occurrence probability is estimated as $p(I) = \mathbb{E}[\mathbb{I}(z_s > 0)]$, where $\mathbb{I}$ is an indicator function that checks if all sampled latent variables $z_s$ are positive, corresponding to the joint activation of the itemset. The Monte Carlo estimation in GPAR is grounded in probabilistic sampling of the latent GP model. For each candidate itemset $I$, the GP models the latent variables $z_s$ as Gaussian-distributed, i.e. $z_s \sim \mathcal{N}(0, K_I)$, where $K_I$ is the submatrix of the RBF kernel computed for the features of the items in $I$. By sampling $z_s$ multiple times and checking for co-activation (all elements $z_s > 0$), GPAR estimates the co-occurrence probability $p(I) = \frac{1}{S} \sum_{s=1}^S \mathbb{I}(z_s > 0)$.

\paragraph{The probabilistic approach to uncertainty quantification} 
Uncertainty is represented in GPAR using distributions. As GP is a probabilistic model, it naturally provides uncertainty measure. For example, the covariance matrix of the GP model represents the similarity between the feature vectors of items. When the GP is trained via maximum likelihood (Eq.\ref{eq:GP_training_MLE}), this covariance matrix encodes uncertainty about the relationships between items, conditioned on the available data. Also, probability is used in the inference stage. GPAR performs rule inference (i.e. estimating the co-occurrence probability of sub-itemsets) via Monte Carlo sampling from a marginalised posterior (i.e. the multivariate Gaussian distribution after GP training), these posterior samples represent a range of possible values for the latent variables, they represent information from observed data, capture variability and enhance robustness: if co-occurrence is consistent across samples, confidence in the association is stronger.  

Uncertainty quantification is particularly important for sparse datasets in which many itemsets may have limited co-occurrence information. The GP's uncertainty estimates provide a way to account for missing data by propagating uncertainty rather than assuming zero co-occurrence. It helps generate more robust rule because it does not rely solely on observed, empirical frequencies, but instead, builds a predictive model which models the underlying distribution of co-occurrence. For example, if two items have only co-occurred twice, the GP still provides a probabilistic measure of their joint activity, rather than declaring it purely absent.

\paragraph{Kernel design} 
Representing items as feature vectors allow modelling the similarities and dissimilarities among items, taking into account these further associations between items may enhance the mining performance. However, care should be taken to craft kernels that reflect the dependency relations among substitute and complementary items. For example, using common feature representations such as [category, size, shape, color, price], papers and beers have very different feature vectors and they are naturally distant as quantified by a radial basis function or linear kernel. However, their likelihood of co-occurrence can be high. Therefore, the classic kernels may not be applicable in GPAR.

On \textit{Synthetic 1}, GPAR with neural net kernel and NTK offer fast inference (faster than Apriori, FP-Growth, and Eclat, see Table.\ref{tab:synthetic1_comparison_runtime}) and negligible memory usages (Table.\ref{tab:synthetic1_comparison_memory}). GPAR with RBF/shifted RBF kernel provide richer patterns (less than Apriori, FP-Growth, and Eclat, see Table.\ref{tab:synthetic1_comparison_frequentItemsets} and Table.\ref{tab:synthetic1_comparison_noRules}) at the cost of higher computational time. Traditional methods Apriori, FP-Growth, and Eclat offer a balance of speed and pattern discovery, making them competitive alternatives for ARM on synthetic datasets.

\paragraph{Alternative strategies to mitigate high cost of rule inference} In GPAR, rule inference is performed by computing the probability over the positive quadrant ( Eq.\ref{eq:rule_inference_via_integral_of_marginalised_joint}) using the marginalised Gaussian posterior. In high dimensions, this integral is not analytically derivable; instead, we employ Monte Carlo sampling (Eq.\ref{eq:GPAR_marginal_posterior_MC_approx} and Step 8 in Algo.\ref{algo:GPAR}), which is numerically tractable but computationally expensive, to approximately compute this integral. There are alternative strategies to calculate the high dimensional integral $p(I)$ in Eq.\ref{eq:rule_inference_via_integral_of_marginalised_joint}, including: (1) approximating with univariate marginals. We can approximate the joint probability $p(z_i > 0 \; \forall i \in I)$ by assuming independence between the latent variables and using the univariate marginal probabilities, which factorises the high-dimensional integral into products of univariate Gaussian integrals. However, similar to the item dependency-free BARM, the independence assumption degrades the performance of GPAR as well. (2) Quadrature methods such as Gauss-Hermite. Use numerical quadrature to approximate the multivariate integral. (3) Cubature methods such as spherical-radial cubature can be used to approximate the multivariate integral more efficiently than quadrature, leveraging the structure of the Gaussian distribution. (4) Sampling via particle-based variational inference (ParVI). We can use ParVI such as SMC \cite{Doucet2001}, SVGD \cite{Liu2016SVGD} (also in Appendix.\ref{app:SVGD}), EParVI \cite{EParVI2024}, SPH-ParVI \cite{SPH_ParVI2024}, MPM-ParVI \cite{MPM_ParVI2024}, etc, to generate samples from the GP posterior, then compute the probability. Efficiencies and accuracies of these methods remain to be explored though. (5) Analytical computation using CDF for small itemsets. For small itemsets ($|I| \leq 2$), computing the probability analytically using the bivariate normal CDF is more efficient than sampling, therefore, we can use a hybrid analytical-sampling scheme to generate samples from the GP posterior. (6) Importance sampling can be used to improve the efficiency of Monte Carlo sampling by focusing samples in regions of interest (where $\mathbf{z} > 0$).

Each of these alternative methods aims to reduce the computational burden while leveraging the trained GP’s covariance structure, but they vary in accuracy, complexity, and suitability for different itemset sizes. The analytical CDF approach (strategy 5) is suitable for small itemsets ($|I| \leq 2$), offering exact probabilities with minimal computational cost ($\mathcal{O}(1)$), making it ideal for early stages of GPAR where such itemsets are prevalent; however, it is intractable for larger itemsets. Importance sampling (strategy 6) and SMC (strategy 4) enhance Monte Carlo sampling by focusing on regions of interest, potentially reducing the number of samples needed, but they introduce complexity in implementation and require careful tuning to avoid high variance. Methods like quadrature (strategy 2) and cubature (strategy 3) scale poorly with itemset size due to exponential or polynomial growth in evaluations, while univariate marginals (strategy 1) oversimplifies correlations and drops off the learned covariance structure - a key advantage of using GP, leading to inaccurate approximations despite their theoretical efficiency. SVGD (strategy 4) can converge to a good approximation of the target distribution with fewer particles than Monte Carlo sampling requires samples, potentially with improved accuracy, as it (implicitly) minimizes the KL divergence to match the target, and evolves particles towards high-probability regions of $\mathcal{N}(0, K_I)$. It can potentially reduce variance compared to naive Monte Carlo sampling, while improving robustness by exploiting the covariance structure $K_I$ from the trained GP. Further, SVGD is deterministic and provides more flexibility due to its non-parametric nature, which is good for approximating complex distributions.

\paragraph{GPAR: pros and cons}
GPAR offers several benefits including: 
(1) GPAR enables modelling item dependency. Traditional association rule (AR) mining relies on frequency counts from the transaction matrix $\mathcal{T}$ to compute probabilities and identify patterns, treating items as discrete entities without considering attribute-based dependency. In contrast, GPAR directly calculates the probabilities from the learned GP model - it encodes attributes into feature vectors and employs a Gaussian process with a kernel function to capture similarities between items. The GP posterior is obtained by taking into account the distance between input feature vectors, considering human beliefs (incorporating item correlations into kernel design) and by absorbing the information from transactions, and the probability of each rule is obtained from the posterior. This principled way of estimating the rule probability enables GPAR to uncover latent patterns, such as low co-occurrence probabilities for similar or substitute items (e.g. dairy products), and high co-occurrence probabilities for complementary items (e.g. beer and diapers) that traditional methods might overlook due to low frequency. 
(2) The probabilistic GPAR framework naturally and explicitly handles uncertainty in noisy or incomplete datasets and provides uncertainty quantification (i.e. credible intervals). This is particularly useful in domains such as healthcare or finance, where data may be scarce and the quality varies. 
(3) Prior knowledge can be incorporated through wise feature design and kernel selection. The kernel (covariance) function encodes domain knowledge, such as known similarities between items based on attributes, enhancing model generalization (as the GP model generalizes based on the features, not just the transaction records). For example, in retail, items with similar sizes or colors might be more (un)likely to co-occur, which can inform store layouts. 
(4) The flexibility and robustness of GPAR enable it to capture complex patterns. GPs can model non-linear relationships, potentially capturing more nuanced associations than traditional methods, such as rare but significant patterns in niche markets, which could be missed by support thresholds as used by traditional ARM methods. Flexibility is introduced in GPAR via handling uncertainty explicitly, which potentially helps discover rare but significant patterns such as unusual pairings in specialized retail or medical symptom-disease associations. GPAR is robust due to its principled handling of uncertainty (e.g. missing data, prior knowledge representation, etc). GPAR represents items using feature vectors (e.g. [size, color, shape, price]), allowing it to uncover deeper, latent patterns and correlations based on item similarities that traditional frequency-based methods might miss. 
(5) Another significant advantage of GPAR is that it facilitates new item co-occurrence estimation. The feature-based representation overcomes the static inference limitation of traditional methods, and enables GP to handle new items seamlessly - we only need to augment the covariance matrix of GP, without training the GP model from scratch. Traditional unsupervised AR mining approaches are based on empirical counting, and are not directly extendable to new items - the absence of a predictive model limits their adaptability to new data; to generate rules involving a new item, they will have to reprocess the entire dataset.

However, there is a trade-off for using the GP’s probabilistic framework: it also induces significant computational challenges, with a dominating complexity of $\mathcal{O}(2^M (M^3 + SM^2))$ for itemset enumeration and marginal probability evaluations, requiring approximations such as Monte Carlo sampling due to the lack of closed-form solutions for high-dimensional Gaussian integrals \footnote{GPAR’s computational complexity (cubic in $M$) and probabilistic nature make it less accessible compared to traditional methods’ linear to quadratic scaling (Table.\ref{tab:gpar_comparison}). GPAR is also less scalable due to its lack of pruning as compared to Apriori. One can (marginally) improve the posterior sampling efficiency using a hybrid sampling scheme. For example, for $|I| \leq 2$, we use analytical approximations thanks to the availability of low-dimensional Gaussian CDF; for $|I| > 2$, use MCMC or SVGD to approximate the posterior and estimate the probability.}. Monte Carlo sampling the GP posterior can be slower than direct counting, especially for large itemsets. Also, accuracy of the Monte Carlo estimate relies on the number of samples $S$. A small $S$ (e.g. 100) may lead to noisy estimates, especially for large number of items (i.e. high dimensions), and rare itemsets with small occurrence probabilities. Further, GPAR requires both access to the feature matrix and transaction records, which may not be simultaneously available in real-world applications. Even when available in data holders such as supermarkets, they may be distributed in different databases, and some aggregation efforts are needed to link the product features and transaction records together. Additionally, careful feature engineering (for crafting the feature matrix $X = [\mathbf{x}_1, \mathbf{x}_2, \dots, \mathbf{x}_M]^\top$) and kernel design may be demanded to effectively incorporate domain knowledge, limiting its scalability for large datasets. Another disadvantage is the lower interpretability of GPAR results, due to its complex probabilistic outputs, compared to the straightforward ``if-then'' rules of traditional methods, which are more intuitive as they rely on frequency counts, as discussed below.

\paragraph{GPAR: interpretability}
GPAR’s lower interpretability stems from its probabilistic, latent variable-based approach, contrasting with the straightforward, frequency-driven \textit{if-then} statements of traditional ARM methods. Traditional frequency-based methods excel in interpretability due to their clear, frequency-driven statements. For example, a rule such as \{bread\} → \{milk\} (i.e. 'if bread, then milk') with 60\% \textit{confidence} and 15\% \textit{support} straightforwardly communicates actionable insights: 60\% of transactions involving bread also include milk, occurring in 15\% of all transactions. Such simplicity and clarity align closely with common statistical literacy, making these rules easily understandable and actionable in real world application such as retail, even without advanced statistical knowledge. 

In contrast, GPAR provides probabilistic outputs along with uncertainty estimates. GPAR's probabilistic rules, such as reporting a 0.75 probability with a credible interval for {bread, milk} co-occurrence, are less intuitive and require stakeholders to interpret continuous latent variables and statistical measures rather than simple frequency counts. Consequently, additional tools and training may be needed to e.g. visualize latent variable relationships for effective use. For example, a retailer might struggle to translate GPAR’s probabilistic rule into store layout decisions without additional analysis, whereas Apriori’s frequency-based rule is immediately actionable. Although GPAR is powerful in handling uncertainty and incorporating prior knowledge, it lowers interpretability and requires greater technical expertise, such as familiarity with Gaussian processes and statistical modeling techniques. This trade-off highlights the need for user-friendly interfaces or educational resources to bridge the interpretability gap, ensuring GPAR’s potential benefits in uncertainty handling and prior knowledge integration are realized in practice.

\paragraph{GPAR: implementation considerations}
When implementing GPAR, there are several practical considerations: (1) Python libraries for probabilistic models. \textit{GPyTorch} \cite{GPyTorch2018} provides a highly efficient and modular implementation of GPs, with GPU acceleration. \textit{GPy} \cite{GPy2014} and \textit{GPflow} \cite{GPflow2017} also offer convenient implementation of GPs. (2) Choice of kernel and attributes. The kernel function (e.g. RBF) and the choice of item attributes significantly affect results. Domain knowledge is required to select appropriate attributes and kernels, and feature engineering may be needed to ensure meaningful representations. (3) Scalability. Traditional ARM algorithms scale well with large datasets, with complexities linear to quadratic in the number of transactions. GPAR’s computational cost, generally cubic in $M$ for GP inference and sampling (rule inference), limits its applicability to small to medium-sized datasets. To reduce GP inference complexity, one can use sparse variational Gaussian processes (SVGP) with re-sampled inducing points.

\paragraph{Applications of GPAR}
GPAR works well for smaller datasets where we have detailed item attributes, it handles uncertainty and uses prior knowledge about item similarities. However, for large datasets with many items, GPAR can be slow as evaluating rule probabilities involves Monte Carlo approximation which is costive. Traditional methods such as Apriori are faster for big data but don’t capture uncertainty. One of the benefits of GPAR being, it can find rare but significant patterns, such as unusual item combinations in niche markets, by updating beliefs in these patterns, which standard methods might miss. Therefore, GPAR with feature vectors is particularly suitable for: (1) small to medium-sized datasets (e.g. $M \leq 15$), where computational resources are not a constraint, and probabilistic insights are valuable, such as small retail analyses. (2) specialized domains, where uncertainty quantification is critical, such as medical diagnosis (symptom-disease associations) or financial risk assessment, or where items are dynamically updated. (3) Exploratory analysis for uncovering novel patterns or testing hypotheses, especially when prior knowledge can guide the model through item attributes. However, for large-scale applications such as market basket analysis in supermarkets, traditional methods are likely more efficient. The trade-off is between the richness of probabilistic modeling and computational feasibility.

\paragraph{BARM: overview}
BARM represents a general probabilistic framework for AR mining by adopting a fully Bayesian approach to model item co-occurrence probabilities. Unlike traditional methods such as Apriori or FP-Growth, which rely on frequency counts, BARM leverages Bayesian inference to incorporate prior knowledge on item presence probabilities and a correlation structure informed by item features (there is also an item dependency-free version of BARM, see Section.\ref{sec:BARM_variant}), enhancing flexibility and robustness. This probabilistic approach offers the most principled and general framework for quantifying uncertainty, providing full posterior distributions over co-occurrence probabilities, which is particularly valuable for interpreting rare or uncertain rules, such as unusual purchasing patterns or symptom-diagnosis associations. Compared to the Gaussian process-based GPAR, BARM avoids the computationally intensive $\mathcal{O}(M^3)$ covariance matrix operations, achieving a complexity of $\mathcal{O}(S_{\text{MCMC}} \cdot N \cdot M^2 \cdot d + 2^M \cdot S \cdot M^2)$ when considering pairwise correlations, or $\mathcal{O}(S_{\text{MCMC}} \cdot N \cdot M + 2^M \cdot S \cdot M)$ otherwise. The dependency-free variant further reduces this to $\mathcal{O}(N \cdot M + 2^M \cdot S \cdot M)$ by exploiting prior-likelihood (Beta-Bernoulli) conjugacy (after removing the correlation adjustment factor), eliminating the need for MCMC sampling and feature matrices. However, BARM’s reliance on MCMC sampling in its full form can be computationally expensive, particularly for large datasets, and its performance on the Synthetic 1 dataset revealed a systematic underestimation of support probabilities compared to empirical frequencies (e.g. the observed discrepancy between a support of 0.1916 and an empirical count of 496/1000 for the top rule in Table.\ref{tab:BARM_analysis_minSupport01_Synthetic1}). This smoothing effect, inherent to probabilistic modeling, may mislead users about the true frequency of itemsets, posing a challenge for applications requiring precise frequency estimates.

Despite its strengths, such as robust uncertainty quantification through full posterior distributions, flexible incorporation of prior knowledge, and the ability to model item dependencies with feature-informed correlation structures, rare rules mining, continuous posterior updating (e.g. Bayesian online learning), etc, BARM faces challenges that desire future investigation. The computational cost of MCMC sampling, especially when considering item dependencies, limits scalability for high-dimensional or large-scale datasets, suggesting a need for more efficient approximate inference methods, such as advanced variational inference techniques to accelerate computations. The dependency-free BARM, while computationally efficient, sacrifices the ability to capture real-world item correlations, leading to near-trivial lift values (e.g. 1.0001) and weaker rule associations, as seen in the Synthetic 1 experiments. This trade-off highlights the importance of developing hybrid approaches that balance computational efficiency with dependency modeling. Additionally, the underestimation of support probabilities, a common issue in probabilistic methods, could be addressed by calibrating posterior distributions to align more closely with empirical frequencies, possibly through post-processing techniques or alternative prior specifications. Future research could also explore adaptive prior selection to mitigate bias in small datasets and enhance user-friendly interfaces to improve the interpretability of BARM’s probabilistic outputs, ensuring its practical utility in diverse applications such as market basket analysis or medical diagnostics.

\subsection{Reinforced approaches to AR mining}

The MAB-ARM (Section.\ref{sec:MAB_ARM}, MCTS (Section.\ref{subsec:MCTS-ARM}) and RLAR (Section.\ref{sec:RLAR}) frameworks represent the reinforced approaches to association rule mining, using reward-guided strategies to navigate the combinatorial space of itemsets. MAB-ARM adapts the multi-armed bandit (MAB) framework, treating itemsets as arms and employing an UCB strategy to balance exploration and exploitation, prioritizing itemsets with high co-occurrence probabilities. MCTS uses UCB for efficient exploration and dynamic reward updates for adaptive learning. RLAR uses a deep Q-network (DQN) to learn a policy that optimizes itemset selection based on a composite reward function incorporating support, confidence and lift. These methods contrast with traditional frequency-based methods (Apriori, FP-Growth, Eclat) and probabilistic approaches (GPAR and BARM), offering dynamic exploration mechanisms suited for discovering high-quality rules, particularly in scenarios requiring targeted exploration of rare patterns.

\paragraph{MAB-ARM: overview}

MAB-ARM features dynamic scheduling of exploration of the itemset space, i.e. using the UCB and pruning strategy, it can efficiently explore the itemset space \footnote{Although the total number of evaluations $T_{\text{max}}=2^M$ is used in our experiments. The need to evaluate all \(2^M\) possible itemsets can be reduced by limiting evaluations to a fixed number \(T_{\text{max}}\) in Algo.\ref{algo:MAB-ARM}.}. This results in a computational complexity of \(\mathcal{O}(T_{\text{max}} \cdot (\sum_{m=m_{\text{min}}}^{m_{\text{max}}} \mathcal{C}_m^M + N \cdot M))\), which is generally more efficient than GPAR’s \(\mathcal{O}(2^M \cdot (M^3 + S M^2))\) and comparable to BARM’s \(\mathcal{O}(S_{\text{MCMC}} \cdot N \cdot M + 2^M \cdot S \cdot M)\). On the Synthetic 1 dataset, MAB-ARM generated 1013 frequent itemsets and 57024 rules at a support threshold of 0.1 (Table.\ref{tab:MABARM_metrics_Synthetic1_secondRun}), closely matching Apriori’s output (1023 itemsets in Table.\ref{tab:synthetic1_comparison_frequentItemsets}, 57,002 rules in Table.\ref{tab:synthetic1_comparison_noRules}) while maintaining competitive runtime (2.3471 seconds \textit{vs} Apriori’s 0.483 seconds in Table.\ref{tab:synthetic1_comparison_runtime}) . Its reliance on empirical frequencies ensures accurate support estimation, avoiding the underestimation observed in GPAR and BARM (e.g. BARM’s support of 0.1916 \textit{vs} count/1000 = 0.496 for rule \([1, 4, 6, 9] \to [2]\), Table.\ref{tab:BARM_analysis_minSupport01_Synthetic1}). The improved MAB-ARM variant (Algo.\ref{algo:EMAB-ARM}), incorporating the Apriori principle for associative probability updates, further enhances efficiency by reducing redundant evaluations. This makes MAB-ARM particularly suitable for large datasets where computational efficiency and precise frequency-based insights are desired, such as market basket analysis.

\paragraph{RLAR: overview}

RLAR (Section.\ref{sec:RLAR}) leverages reinforcement learning to learn a generalizable policy for itemset selection. In our design, RLAR uses a deep Q-network (DQN) to navigate the itemset space, where states represent itemsets (encoded as bit-vectors with additional features such as support and confidence), and actions involve adding or removing items. The composite reward function, defined in Eq.\ref{eq:RLAR_reward_design}, blends support, confidence, and lift: it assigns a positive reward based on a weighted combination of confidence and lift for itemsets with support \(\geq min\_supp\), a negative reward (-1) for those below, and zero otherwise. This design enables the discovery of high-quality rules, as evidenced by the top rule \([0, 3, 4, 6, 9] \to [2]\) with confidence 0.9838 and support 0.485 on the Synthetic 1 dataset (Table.\ref{tab:rlar_rules_item2}). RLAR’s empirical support calculation aligns with observed frequencies, matching MAB-ARM and traditional methods while avoiding the smoothing effects of GPAR and BARM (Fig.\ref{fig:support_vs_count}). The DQN’s adaptability, demonstrated by increasing itemset and rule counts at higher support thresholds (Table.\ref{tab:rlar_metrics}), suggests its ability to refine focus on robust patterns. The top-\(k\) action selection strategy (e.g. \(k=3\)) during rule extraction enhances diversity in rule discovery, offering a strategic advantage over single-trajectory methods. Unlike the MCTS approach, which uses tree-based exploration with UCB, DQN offers a scalable, learning-based solution suitable for large datasets. RLAR’s scalability to large datasets, supported by bit-vector representations and pruning, positions it as a flexible alternative to exhaustive methods such Apriori, particularly for applications requiring adaptive rule discovery. By carefully designing the reward function and optimizing the training process, RLAR can provide a robust tool for ARM, complementing traditional methods.

\paragraph{Issues and challenges}
Unlike GPAR and BARM’s kernel-based correlation structures, RLAR methods such as MAB-ARM and DQN have limited ability to model item dependencies as they lack notion of item dependency. This results in moderate lift values, e.g. 1.7612 for RUle $[2, 3, 4, 5] \to [0, 1, 6, 9]$ in Table.\ref{tab:MABARM_mined_rules_rankedByLift_minSupport01_Synthetic1} and 1.5512 for Rule $[item\_5, item\_9] \to [item\_3, item\_4, item\_8]$ in Table.\ref{tab:RLAR_mined_rules_rankedByLift_minSupport01_Synthetic1}, compared to GPAR’s peak lift of 12.1212 (Table.\ref{tab:RBF_GPAR_mined_rules_rankedByLift_minSupport01_Synthetic1}). RLAR, while powerful, incurs significant computational overhead, with runtimes up to 5337.17 seconds and memory usage peaking at 433.24 MB at min support of 0.1 (Table.\ref{tab:rlar_metrics}), far exceeding MAB-ARM, GPAR, BARM, and traditional methods (Table.\ref{tab:comparison_runtime}). Its training instability at high support level (0.5), evidenced by reward fluctuations in Fig.\ref{fig:five_cumulative_rewards_plot}, poses further challenges. The unusual increase in itemsets and rules at higher support thresholds (Table.\ref{tab:rlar_metrics}) may suggest dataset-specific behavior or overfitting to the Synthetic 1 dataset’s structure, requiring further investigation.

\paragraph{Potential improvements}
For MAB-ARM, integrating feature-based priors or item similarity metrics could enhance dependency modeling, bridging the gap with GPAR and BARM while maintaining empirical accuracy. Adaptive \(T_{\text{max}}\) adjustment based on dataset size or itemset frequency could optimize exploration, reducing memory spikes and redundant evaluations. Optimized pruning strategies, such as leveraging itemset hierarchies more aggressively, could further improve efficiency. For RLAR, mitigating computational overhead involves employing GPU acceleration, optimizing data structures (e.g. bit-vectors), and reducing \(E_{\text{max}}\) or \(S_{\text{max}}\) to lower complexity \footnote{We suggest reduce the number of steps $S_{\text{max}}$ in each episode, rather than reducing the number of episodes $E_{\text{max}}$ ton ensure convergence of cumulative rewards and therefore agent training.} from \(\mathcal{O}(E_{\text{max}} \cdot S_{\text{max}} \cdot (M + N \cdot m_{\text{max}} \cdot 2^{m_{\text{max}}} + B_m \cdot |\theta|))\). Reward sparsity could be addressed through tailored reward designs and exploration strategies, e.g. incremental rewards for frequent itemsets or hybrid reward functions incorporating domain knowledge. Tuning the epsilon decay rate could enhance exploration-exploitation balance, while integrating feature-based state representations can enhance rule discovery (e.g. encouraging rule diversity), aligning RLAR closer to GPAR’s dependency-based complex pattern discovery.

\paragraph{Comparison with probabilistic approaches}
Compared to GPAR and BARM, MAB-ARM (including MCTS-ARM) offers distinct advantages in empirical accuracy and efficiency; RLAR stands out in empirical accuracy and adaptivity. MAB-ARM is fast while RLAR is slow (Table.\ref{tab:comparison_runtime}). However, both of them lack item dependency modeling, which degrades their lift values compared to GPAR’s high-lift rules. RLAR’s DQN-based policy learning generates fewer rules (1106 at 0.1, Table.\ref{tab:comparison_rules}) than GPAR with RBF kernel (54036) and MAB-ARM (50372) but achieves comparable confidence (e.g. 0.9838 for the top rule in Table.\ref{tab:RLAR_mined_rules_rankedByConfidence_minSupport01_Synthetic1} \textit{vs} MAB-ARM’s 0.9981 in Table.\ref{tab:MABARM_mined_rules_rankedByConfidence_minSupport01_Synthetic1}) and empirical support accuracy. GPAR and BARM’s dependency modelling and uncertainty quantification are advantageous. GPAR’s dependency modeling produces superior lift but at higher computational cost and with support underestimation. Future hybrid approaches combining MAB-ARM’s efficiency, RLAR’s adaptability, and GPAR/BARM’s dependency modeling and uncertainty notion could yield a robust reinforced ARM framework for diverse applications.

\paragraph{ARM with complex data structure} Traditional ARM methods relies on binary attributes for simplicity and computational efficiency. To use these algorithms one typically convert attributes into a binary format using \textit{discretization} or \textit{binarization}, which leads to an explosion of number of attributes (i.e. increase in dimensionality) and potential loss of information. Adapting these methods to accommodate complex data types, such as multi-categorical or numerical attributes, remains a challenge. Researchers have proposed methods for mining association rules directly from non-binary datasets. Such approaches preserve the richness of the original data and reduce dimensionality by capturing complex relationships without binarization. For example, techniques for numerical attributes include tailored discretization strategies or algorithms that process continuous data directly. Srikant et al. \cite{Srikant1997} introduced methods, such as \textit{Cumulate} and \textit{EstMerge}, for mining generalized association rules. These techniques extend traditional AR mining to handle hierarchical and quantitative attributes, enabling the discovery of sophisticated patterns without binarization.

\section{Conclusion} \label{sec:conclusion}

This work presents 4 novel approaches to association rule mining (ARM): Gaussian process-based association rule mining (GPAR), Bayesian association rule mining (BARM), multi-armed bandit-based association rule mining (MAB-ARM), and reinforcement learning-based association rule mining (RLAR), complementing traditional frequency-based methods such as Apriori, FP-Growth and Eclat. GPAR transforms ARM into a probabilistic, supervised learning framework by employing a Gaussian process (GP) to model item co-occurrence probabilities, utilizing feature vectors to capture latent relationships and enabling inference for new items without retraining. Experiments on synthetic datasets (e.g. \textit{Synthetic 1}) and real-world data (e.g. \textit{UK Accidents}) demonstrate GPAR’s capacity to identify complex, rare, and high-lift rules, such as those linking severe accidents to fine weather, although it also incurs significant computational overhead compared to traditional methods such as Apriori. BARM adopts a fully Bayesian approach, modeling item presence probabilities with priors and optional correlation structures, offering robust uncertainty quantification through posterior distributions, as evidenced by its performance on the \textit{Synthetic 1} dataset, where it generated a substantial number of frequent, meaningful itemsets and rules. MAB-ARM, including its Monte Carlo tree search (MCTS) companion, leverages the upper confidence bound (UCB) strategy to efficiently explore the itemset space, achieving a comparable number of frequent itemsets and rules on \textit{Synthetic 1} dataset comparing to traditional ARM methods, and with a runtime that outperforms GPAR in scalability. RLAR employs a deep Q-network (DQN) to learn a generalizable policy, excelling in adaptive rule discovery with a notable number of frequent itemsets and rules at higher support thresholds on \textit{Synthetic 1} dataset, although it requires considerable computational resources as well. Collectively, these methods advance ARM by introducing probabilistic and reinforced paradigms, with GPAR and BARM enhancing pattern richness through dependency modeling, MAB-ARM and RLAR offering efficiency and adaptability, respectively.

The novelty of these approaches lies in their distinct working mechanisms compared to traditional frequency-based methods, as well as their contributions to different aspects of AR mining and their adaptability to diverse application scenarios. GPAR’s integration of Gaussian processes with feature-based representations enables dynamic and continuous inference, a significant improvement over the static nature of traditional methods. BARM’s Bayesian framework provides a structured approach to modeling item probabilities with priors and correlation structures, generalising the probabilistic framework and offering a robust alternative to frequency-driven techniques. Both GPAR and BARM feature principled prior knowledge encoding and uncertainty quantification, going beyond the empirical nature of traditional algorithms, which may help them identify rare and useful patterns even in small datasets. MAB-ARM’s UCB-driven exploration and MCTS’s tree-based search introduce adaptive exploration strategies, while RLAR’s DQN-based learning offers a policy-driven solution, contrasting with the exhaustive searches of Apriori and others. Therefore, the primary advantages of these methods include GPAR and BARM's ability and flexibility to incorporate prior knowledge and handle uncertainty due to their probabilistic nature, MAB-ARM’s UCB-based efficient exploration, and RLAR’s adaptive policy learning. However, significant disadvantages also exist: GPAR and BARM can suffer from high computational complexity (although the conjugate version of BARM is fast), limiting scalability; MAB-ARM and RLAR lack notion of item dependency; RLAR’s extensive agent training time poses practical constraints. Interpretability also varies, with traditional methods as well as MAB-ARM and RLAR offering intuitive frequency-based rules, while GPAR and BARM’s probabilistic outputs require advanced statistical understanding. Overall, these innovations enrich the methodological landscape of ARM, enhancing its flexibility and robustness, particularly in discovering rare or complex patterns and working with small datasets, although they come with trade-offs in computational cost and interpretability.
 
\paragraph{Future work}
Future research can be directed on enhancing the scalability of GPAR and BARM through approximations, such as sparse Gaussian processes or variational inference, to reduce computational complexity (e.g. from cubic to linear or quadratic in the number of items). For MAB-ARM, integrating feature-based priors for better informed exploration and using adaptive number of iterations could improve dependency modeling and further enhance efficiency, while RLAR could benefit from GPU acceleration and optimized reward functions to address reward sparsity and training instability. Extending these methods to handle non-binary data types, such as hierarchical or numerical attributes, using techniques such as those proposed by Srikant et al. \cite{Srikant1997}, could broaden their applicability to complex datasets. Additionally, developing hybrid approaches that combine the strengths of probabilistic (GPAR, BARM) and reinforced (MAB-ARM, RLAR) methods, e.g. integrating dependency modeling with UCB exploration, may yield a robust framework for diverse applications such as dynamic retail (e.g. market basket analysis in supermarkets) and financial environments, as well as medical diagnostics where efficiency, insight and notation of risk (represented by uncertainties) are paramount.

\paragraph{Data and code availability} The data and codes used in this work can be found on this \href{https://github.com/YongchaoHuang/ARM}{Github repository}.

\bibliographystyle{plain}
\bibliography{references}

\newpage
\appendix

\section{Association rule learning: more details} \label{app:AR}

\subsection{AR: background and concepts}

Association rules (ARs \cite{Agrawal1993mining}) are a statistical method used in data mining for finding item relationships (e.g. co-occurrences) from large datasets such as transactional data (e.g. supermarket purchase records). It is used to find significant co-occurrences, and widely applied in retail, healthcare, and finance for strategic decision-making. For example, it is extensively used in market basket analysis, where retailers identify product pairings for promotions or store layouts. In healthcare, it can uncover symptom-disease associations, and in finance, it aids in fraud detection or customer segmentation. It is also used in ads recommendations. A famous case is the 'beer and diapers' anecdote \footnote{The 'Beer and Diapers' story is debatable - it serves more as a folk tale in data mining than an empirical case study.} \cite{Berry1997Book}, where association mining revealed when young men buying diapers also bought beer, highlighting how seemingly unrelated items can show strong correlations (not causation though), offering insights for business strategies.

ARs are defined as \textit{if-then statements} in the form of implication \{$A$\} → \{$B$\}. For example, a rule stating that if a customer buys bread, they are likely to buy milk, is expressed as \{bread\} → \{milk\}. This method falls under unsupervised learning, as it does not require labeled data, focusing instead on uncovering hidden patterns from raw transactional data. AR learning involves two phases: \textit{identifying frequent itemsets} and \textit{generating association rules}. It starts by finding groups of items, called \textit{frequent itemsets}, that appear together frequently in the data. For example, if many people buy bread and milk together, that can be a frequent itemset. Once frequent itemsets are identified, association rules are generated by splitting each itemset into an antecedent X and a consequent Y, forming rules like \{$A$\} → \{$B$\} (e.g. 'if bread, then milk'). The strength of these rules is evaluated using \textit{confidence} (definition comes later) which indicates how often the rule is true (e.g. how many times people buy milk when they buy bread). Rules must meet a minimum confidence threshold to be considered valid, ensuring reliability. Other metrics such as \textit{support}, implying how often the itemset appears in all transactions, and \textit{lift} which compares the rule's strength to what we'd expect if items were bought independently (i.e. usefulness of the rule), also help decide if the mined rule is useful \footnote{Beyond support and confidence, additional metrics such as \textit{conviction} and \textit{leverage} can also evaluate rule strength. Conviction measures the ratio of expected to observed frequency of X without , while leverage compares the actual support to expected under independence. These metrics help in filtering out less interesting rules.}.

Algorithms with varying mechanisms and efficiencies can be used to identify frequent itemsets. For example, the \textit{Apriori} \cite{Agrawal1994fast} algorithm uses a bottom-up approach, starting with single items and iteratively building larger itemsets while pruning those below the support threshold \footnote{For example, user can specify a '\textit{minimum-interest-level}' - a measure which helps prune a large number of redundant rules (40–60\% of all the rules pruned on real-life datasets as reported in \cite{Srikant1997}).}; the \textit{frequent pattern growth} (FP-Growth \cite{Han2000mining}) algorithm uses a tree structure (FP-tree) to find patterns more efficiently for large datasets, and the \textit{Equivalence class transformation} (Eclat \cite{Zaki2000scalable}) algorithm employs a depth-first search strategy, leveraging equivalence classes for memory efficiency.  

\paragraph{A trivial example}
Consider a synthesized supermarket sales data, shown in Table.\ref{tab:transaction_data}, which contains 5 transaction records. This transaction table only \footnote{Item features such as size, shape, color, price, etc, are not recorded - they may be stored in another database and not used in traditional AR mining.} records the presence or absence of each item. In the table, each column (milk, bread, etc.) is a binary attribute (variable) representing the appearance of a specific item across all transactions. Each row is a transaction (i.e. record of a purchase event), encoded as a binary vector. A '1' means the item was present in the transaction, and a '0' means the item was absent. The scanning of this matrix (i.e. across all transactions and over all items) to find frequent patterns among the 1s is the essence of association rule mining. The binary attribute representation is important, forming the core of how association rule mining works. 

\begin{table}[H]
\centering
\caption{An example transaction (sales) dataset}
\label{tab:transaction_data}
\tiny
\setlength{\tabcolsep}{3pt}
\renewcommand{\arraystretch}{1.1}
\begin{tabular}{p{1.5cm} p{0.6cm} p{0.6cm} p{0.6cm} p{0.6cm} p{0.6cm} p{0.8cm} p{0.6cm}}
\toprule
\textbf{Transaction ID} & \textbf{milk} & \textbf{bread} & \textbf{butter} & \textbf{eggs} & \textbf{beer} & \textbf{diapers} & \textbf{fruit} \\
\midrule
T1 & 1 & 1 & 0 & 0 & 0 & 0 & 1 \\
T2 & 0 & 0 & 1 & 1 & 0 & 0 & 1 \\
T3 & 0 & 0 & 0 & 0 & 1 & 1 & 0 \\
T4 & 1 & 1 & 1 & 1 & 0 & 0 & 1 \\
T5 & 0 & 1 & 0 & 0 & 0 & 0 & 0 \\
\bottomrule
\end{tabular}
\end{table}

As a data analyser, we want to discover the hidden patterns from this transactions to better plan product layout, supply chain and to grow the businesses. For example, if we find that customers buying diapers often buy beer, we can suggest the supermarket to place them together to boost sales \footnote{This 'diapers and beer' example is illustrative, though its empirical basis is debated.}. In practice, association rules are generated using user-specified thresholds of \textit{support} (fraction of transactions containing the itemset) and \textit{confidence} (conditional probability of the rule). For example, the rule \{diapers\} → \{beer\} would need high support and confidence to be actionable, helping inform product placement and promotions. By examining such patterns (rules) across numerous datasets such as transaction records, businesses can make informed decisions about product placements, promotions, and inventory management, enhancing customer satisfaction and profitability.

\paragraph{Exemplifying the concepts}

Here we exemplify the definitions of some concepts used in AR \cite{Agrawal1993mining, Pei2009} using the trivial example.

\begin{itemize}
    \item \textit{Item and itemset}:  an \textit{itemset} is a collection of \textit{one or more items} from $\mathcal{I}$ - the largest itemset. For example, \{milk, bread\} (a 2-itemset) or \{butter, bread, milk\} (a 3-itemset). The largest itemset $\mathcal{I}= \{i_1, i_2, \ldots, i_7\}$ = \{milk, bread, butter, eggs, beer, diapers, fruit\} defines the universe of items (i.e. the universe itemset), representing \textit{distinct} items $i_k$ that customers can buy or of consideration. $\mathcal{I}$ is a set of binary attributes (variables); each element $i_k \in \mathcal{I}$ is an \textit{item} described by a literal - it is a binary \textit{attribute} or \textit{variable}, representing individual products (or services) available for purchase. For example, $i_1 = \text{milk}$ takes value 0 or 1, meaning that in any transaction, this item is either present (1) or absent (0). This binary representation is good for computational efficiency.
    
    \item \textit{Transactions}: the transaction database, or simply \textit{transactions}, is denoted by $\mathcal{T}$, which is a collection of records of individual purchase events, detailing the items bought together by a customer during a single shopping trip. Each transaction $t \in T$ is a binary \textit{vector} (or \textit{list}, indexable) over the set $\mathcal{I}$, represented as a binary vector. For example, Transaction 1 in Table.\ref{tab:transaction_data} is \{milk, bread, fruit\}, and it can be encoded as $\mathbf{t}_1 = [1, 1, 0, 0, 0, 0, 1]$, meaning that in this transaction, this customer bought items $i_1, i_2, i_7$, but not others \footnote{Note the length difference between the literal representation of a transaction and its encoded vector representation.}. In the transaction database, each transaction is represented as an $M$-vector (where $M = |\mathcal{I}|$, i.e. the number of unique items), with each element taking binary value (1 for being present, 0 for absent). A transaction $t \in T$ satisfies \cite{Agrawal1993mining} an itemset X if $t[k] = 1$ for all items $i_k$ in X.

    \item \textit{Support} is defined as the ratio of transactions containing the itemset to the total number of transactions:
    \begin{equation}
        \text{\textit{support}}(A) = \frac{\text{\textit{no. of transactions containing }} A}{\text{\textit{total number of transactions}}}
    \end{equation}
    support measures how frequently the itemset appears in the dataset. For example, if A=\{milk,bread\}, and 2 transactions out of 5 total transactions include both bread and milk, the support is $2/5 = 40 \%$.
    A \textit{frequent itemset} is a collection of items that appear together in a significant number of transactions, determined by a minimum support threshold.
    \item \textit{Confidence} is the conditional probability that a transaction containing the antecedent also contains the consequent, calculated as the ratio of transactions containing both to those containing the antecedent:
    \begin{equation}
        \text{\textit{confidence}}(A \rightarrow B) = p(B \mid A) = \frac{\text{\textit{support}}(A \cup B)}{\text{\textit{support}}(A)}
    \end{equation}
    confidence measures the proportion of transactions containing A that also contain B. It indicates the likelihood that itemset $B$ is purchased when itemset $A$ is purchased. For example, we note from Table.\ref{tab:transaction_data} that, $\text{\textit{support}}(\text{milk} \cup \text{bread})$=2/5=0.4, $\text{\textit{support}}(\text{milk})$=2/5=0.4, therefore $\text{\textit{confidence}}(\text{milk} \rightarrow \text{bread})$ = 1.0, meaning that $100\%$ of transactions with milk also have bread (the other way round is NOT true), the confidence of the rule '$\text{milk} \rightarrow \text{bread}$' is $100\%$.
    AR algorithms such as Apriori rely on thresholds of support and confidence to prune infrequent itemsets.
    \item \textit{Lift} is used to assess the rule's usefulness, calculated as:
    \begin{equation}
        \text{\textit{lift}}(A \rightarrow B) = \frac{\text{\textit{confidence}}(A \rightarrow B)}{\text{\textit{support}}(B)} = \frac{\text{\textit{support}}(A \cup B)}{\text{\textit{support}}(A) \times \text{\textit{support}}(B)}
    \end{equation}
    lift assesses the strength of a rule over the random occurrence of itemset $B$, given itemset $A$. A rule with a lift of $1.0$ indicates that the antecedent and consequent occur independently, meaning there's no meaningful association between them and no rule can be inferred. A lift greater than 1 indicates a positive correlation, suggesting the rule is more significant than expected under independence. For example, from Table.\ref{tab:transaction_data} hints that, $\text{\textit{lift}}(A \rightarrow B)=\frac{2/5}{2/5 \times 3/5} \approx 1.65$, indicating a positive correlation and implying that the rule is $65\%$ more likely than random chance, suggesting an actionable insight that placing bread and milk near each other might increase sales.
\end{itemize}

\subsection{Traditional AR mining algorithms} \label{app:traditional_AR_mining_algos}

AR mining has traditionally relied on frequentist techniques that identify rules based on observed frequency counts in datasets. Here we give more details about the 3 commonly used algorithms in AR, i.e. Apriori, FP-Growth, Eclat. The seminal Apriori algorithm proposed by Agrawal et al. \cite{Agrawal1994fast} leverages a bottom-up approach by generating candidate itemsets and pruning those failing to meet a minimum support threshold. FP-Growth \cite{Han2000mining} improved efficiency by eliminating candidate generation through a compact FP-tree structure, while Eclat \cite{Zaki2000scalable} introduced depth-first searching and intersection of transaction lists for frequent itemset discovery. As in the main texts, $N$ denotes the number of transactions and $M$ the number of items.

\paragraph{Apriori}
The Apriori algorithm \cite{Agrawal1994fast} is a foundational method in data mining for discovering frequent itemsets and generating association rules. It operates on the Apriori principle that if an itemset is frequent, all its subsets must also be frequent. The algorithm employs a breadth-first, level-wise search strategy to identify frequent itemsets by generating candidate sets and pruning those that do not meet a minimum support threshold. This process continues iteratively, expanding itemsets by one item at a time, until no new frequent itemsets are found. Subsequently, association rules are derived from these frequent itemsets, retaining those with confidence above a specified threshold.

\begin{itemize}
    \item Initialize with frequent itemsets of size 1. Scan the dataset to find all items with support greater than the minimum support threshold, denoted as $\text{min\_sup}$. Support is calculated as:
       \[
       \text{support}(A) = \frac{\text{no. of transactions containing } A}{\text{total number of transactions}}
       \]
    Form the set of frequent 1-itemsets, $L_1$.
    
    \item Generate candidate itemsets. For each level $k$, generate candidate itemsets of size $k+1$ by joining frequent itemsets of size $k$, $L_k$, using the Apriori property. Prune candidates whose subsets are not in $L_k$.
    
    \item Count support. Scan the dataset to count the support of each candidate itemset, retaining those with support $\geq \text{min\_sup}$ to form $L_{k+1}$.
    
    \item Iterate. Repeat until no new frequent itemsets are found, i.e. $L_{k+1} = \emptyset$.
    
    \item Generate rules. For each frequent itemset $A$, generate all possible rules $A \setminus \{$B$\} \rightarrow B$, where $B \subseteq A$. Calculate confidence:
       \[
       \text{confidence}(A \setminus \{B\} \rightarrow B) = \frac{\text{support}(A)}{\text{support}(A \setminus \{B\})}
       \]
   Retain rules with confidence $\geq$ $\text{min\_sup}$.
\end{itemize}

Time complexity: in the worst case, Apriori generates all possible itemsets, leading to $\mathcal{O}(N \times M \times 2^M)$, where $N$ is the number of transactions and $M$ the number of items. This arises from scanning the dataset $N$ transactions for each candidate, and the number of candidates can be up to $2^M - 1$. However, pruning reduces this in practice, making it often $\mathcal{O}(N \times M \times C)$, where $C$ is the number of candidate itemsets after pruning.
Space complexity: requires storing candidate itemsets, which can be $\mathcal{O}(2^M)$ in the worst case, though pruning helps reduce memory usage.

\paragraph{Frequent pattern growth (FP-Growth)} 
FP-Growth \cite{Han2000mining} addresses the inefficiencies of Apriori by eliminating the need for candidate generation and compressing the dataset into a Frequent Pattern Tree (FP-Tree, a compact representation of the dataset \footnote{The FP-Tree compresses the dataset, reducing the need for multiple dataset scans. The mining process leverages the tree structure to find patterns efficiently.}) and mining it using a divide-and-conquer strategy. It is efficient for large datasets, especially when memory is a constraint. FP-Growth is noted for its ability to handle dense datasets better than Apriori, reducing runtime significantly.

\begin{itemize}
    \item Build FP-Tree. Construct the FP-Tree by scanning the dataset twice. First, find frequent 1-itemsets with support $\geq$ min\_sup. Second, build the tree by inserting transactions as paths, with nodes representing items and counts reflecting frequency. Use a header table for quick access to frequent items.
    \item Mine frequent itemsets. Use a recursive \textit{divide-and-conquer} approach:
        \begin{itemize}
            \item Start with the least frequent item in the header table.
            \item For each frequent item, construct a conditional FP-Tree based on paths leading to it, and recursively mine for frequent itemsets.
            \item Combine results to form all frequent itemsets without generating candidates.
        \end{itemize}
    \item Generate association rules. From the frequent itemsets, generate rules by splitting each itemset into antecedent and consequent, computing confidence, and retaining rules with confidence $\geq$ $min\_conf$.
\end{itemize}

Time complexity: building the FP-Tree is $\mathcal{O}(N \times M)$, as it requires two scans of the dataset. Mining the tree depends on the tree's structure, typically $\mathcal{O}(M^2)$ for dense datasets, but can be better for sparse data. Overall, its time complexity is $\mathcal{O}(N \times M + M^2)$, with the mining phase being more efficient than Apriori due to no candidate generation.
Space complexity: the space required is proportional to the size of the FP-Tree, which can be $\mathcal{O}(N \times M)$ in the worst case, but is usually more compact due to compression.

\paragraph{Equivalence class transformation (Eclat)} 

Eclat \cite{Zaki2000scalable} algorithm utilizes a depth-first search strategy and a vertical data format to find frequent itemsets. By leveraging transaction ID lists (tidsets) for each item, Eclat computes the intersection \footnote{Eclat represents the dataset in a vertical format, where each item is associated with a list of transaction IDs in which it appears. Frequent itemsets are discovered by intersecting these tidsets, and itemsets with support above the minimum threshold are considered frequent.} of these tidsets to determine the support of itemsets, which can be more efficient than the horizontal data layout used by Apriori, especially for sparse datasets. It is memory-efficient, particularly for sparse datasets, as it avoids generating all possible candidates, relying instead on equivalence classes. 

\begin{itemize}
    \item Convert to Binary Matrix. Transform the dataset into a binary matrix where each row is a transaction and each column is an item, or equivalently, maintain tidsets for each item (list of transaction IDs containing the item).
    \item Depth-first search with equivalence classes. Start with single items and their tidsets. For each frequent item, form equivalence classes by joining with other items, computing the intersection of tidsets to find larger frequent itemsets:
        \begin{itemize}
            \item If tidset(A) and tidset(B) are the transaction IDs for items A and B, the tidset for {A, B} is their intersection.
            \item Prune itemsets with support \footnote{Support of itemset X: $\text{support}(A) = \frac{|\text{tidset}(A)|}{N}$, where $\text{tidset}(A)$ is the intersection of tidsets of items in X, and $N$ is the total number of transactions.} below $\text{min\_sup}$.
        \end{itemize}
    \item Generate association rules. From frequent itemsets, generate rules by splitting into antecedent and consequent, computing confidence, and retaining rules with $\text{confidence} \geq min\_conf$.
\end{itemize}

The time complexity of Eclat depends on the size of the tidsets and the number of items. Intersecting tidsets costs $\mathcal{O}(N)$ in the worst case, and with $M$ items, the number of pairs is $\mathcal{O}(M^2)$, leading to $\mathcal{O}(M^2 \times N)$. However, for sparse datasets, where tidsets are smaller, Eclat often outperforms Apriori. Space complexity: requires storing tidsets for each item, which can be $\mathcal{O}(N \times M)$, efficient for sparse data due to smaller tidset sizes.

\paragraph{Comparison}

Apriori, FP-Growth, and Eclat each offer unique approaches to association rule mining, with varying efficiencies depending on the dataset characteristics. Apriori is straightforward but may struggle with scalability due to candidate generation. FP-Growth enhances efficiency by compressing the dataset into an FP-Tree and eliminating candidate generation, making it suitable for large, dense datasets. Eclat, with its vertical data format and depth-first search, is particularly effective for sparse datasets. A comparison of these 3 traditional AR mining algorithms is made in Table.\ref{tab:traditional_AR_mining_algorithms}, which highlights that while Apriori is straightforward, its scalability is limited by the exponential growth in candidate itemsets. FP-Growth and Eclat offer improvements, with FP-Growth excelling in dense datasets and Eclat in sparse ones.

Other algorithms, such as OPUS \footnote{OPUS is an efficient search algorithm for exploring the space of conjunctive patterns, supporting rapid rule discovery without requiring predefined, monotone or anti-monotone constraints such as minimum support.} \cite{Webb2000} and ASSOC \cite{Hijek1984}, offer different approaches to association rule mining. OPUS does not require monotone constraints like minimum support, allowing for more flexible rule discovery. ASSOC, a GUHA \cite{Hijek1984,Rauch2017} method \footnote{The ASSOC procedure is an implementation of the GUHA method, which automates hypothesis generation based on empirical data. It mines generalized association rules using fast bitstring operations, enabling the discovery of complex patterns beyond traditional association rules.}, uses bitstring operations to mine generalized association rules, providing a broader scope of analysis. These alternatives offer flexibility but are less commonly discussed in standard applications.

\begin{table}[H]
\centering
\caption{Comparison of classic AR mining algorithms}
\label{tab:traditional_AR_mining_algorithms}
\tiny
\setlength{\tabcolsep}{3pt}
\renewcommand{\arraystretch}{1.1}
\begin{tabular}{p{2.7cm} p{3.5cm} p{3.5cm} p{3.5cm}}
\toprule
\textbf{Aspect} & \textbf{Apriori} & \textbf{FP-Growth} & \textbf{Eclat} \\
\midrule
\textbf{Approach} & breadth-first, bottom-up, level-wise search, candidate generation & Tree-based, divide-and-conquer & Depth-first, tidset intersection, equivalence classes \\
\textbf{Time Complexity} & Moderate, $\mathcal{O}(N \times M \times 2^M)$ & High, no candidates, $\mathcal{O}(N \times M + M^2)$ & High, vertical data, $\mathcal{O}(M^2 \times N)$ (sparse) \\
\textbf{Space Complexity} & High, due to candidates, $\mathcal{O}(2^M)$ & Lower than Apriori, $\mathcal{O}(N \times M)$ & Lower, memory-efficient, $\mathcal{O}(N \times M)$ \\
\textbf{Efficiency for large number of items $M$} & Low, due to exponential growth & High, compact tree structure & High for sparse data \\
\textbf{Scalability} & Limited by candidate generation & Better, avoids candidates & Better for sparse datasets \\
\textbf{Use cases} & Small to medium datasets & Large, dense datasets & Sparse datasets \\
\bottomrule
\end{tabular}
\end{table}

\section{Gaussian process: more details} \label{app:GP_more_details}

\paragraph{Problem setup and notation}
Let $f: \mathbb{R}^d \to \mathbb{R}$ be a noise-free, latent function we want to infer. We observe a dataset \footnote{In this work we focus on single-output GP, i.e. the output label is scalar-valued.}:
\[
D = \{( \mathbf{x}_i, y_i ) \}_{i=1}^n
\]
where $\mathbf{x}_i \in \mathbb{R}^d$ are input (attribute) vectors (also called input indices) and $y_i \in \mathbb{R}$ are noisy observations of the form:
\[
y_i = f(\mathbf{x}_i) + \epsilon_i, \quad \epsilon_i \sim \mathcal{N}(0, \sigma_n^2)
\]
with noise variance $\sigma_n^2 \in \mathbb{R}_{>0}$. Let $X \in \mathbb{R}^{n \times d}$ be the matrix of training inputs and $\mathbf{y} \in \mathbb{R}^n$ the corresponding vector of outputs (labels). We aim to learn a GP and predict the latent function values $\mathbf{f}_* \in \mathbb{R}^{n_*}$ at test inputs $X_* \in \mathbb{R}^{n_* \times d}$.

\paragraph{GP Prior}
A Gaussian process is defined by a mean function $m(\mathbf{x})$ and a positive definite covariance function (kernel) $k(\mathbf{x}, \mathbf{x}')$, i.e.
\[
f(\mathbf{x}) \sim \mathcal{GP}(m(\mathbf{x}), k(\mathbf{x}, \mathbf{x}'))
\]
Note that, zero mean is generally assumed: if $f(\mathbf{x}) \sim \mathcal{GP}(m(\mathbf{x}), k(\mathbf{x}, \mathbf{x}'))$, then $f(\mathbf{x})  - m(\mathbf{x}) \sim \mathcal{GP}(0, k(\mathbf{x}, \mathbf{x}'))$.

A common kernel choice is the radial basis function (RBF):
\begin{equation} \label{eq:RBF_kernel}
    k(\mathbf{x}, \mathbf{x}') = \sigma_f^2 \exp\left(-\frac{\|\mathbf{x} - \mathbf{x}'\|^2}{2\ell^2}\right)
\end{equation}
with signal variance $\sigma_f^2$ and length-scale $\ell$. These are hyper-parameters learned from data during GP training (via e.g. maximizing the marginal likelihood \cite{Rasmussen2006GPbook} - see later). Some literature also ignore the variance magnitude $\sigma_f^2$. This standard RBF kernel is stationary, isotropic, and positive definite.

Let's define the key kernel matrices which will be used in our analysis:
\begin{itemize}
    \item $K = k(X, X) \in \mathbb{R}^{n \times n}$: training covariance matrix
    \item $K_* = k(X, X_*) \in \mathbb{R}^{n \times n_*}, K_*^\top = k(X_*,X) \in \mathbb{R}^{n_* \times n}$: cross-covariance between train and test
    \item $K_{**} = k(X_*, X_*) \in \mathbb{R}^{n_* \times n_*}$: test covariance matrix
    \item $I \in \mathbb{R}^{n \times n}$: identity matrix
\end{itemize}

Including observation noise (assuming the noises are i.i.d Gaussian with zero means), the GP prior over $[\mathbf{y}; \mathbf{f}_*] \in \mathbb{R}^{n + n_*}$ is \cite{Rasmussen2006GPbook,Murphy2012}:
\[
\begin{bmatrix}
\mathbf{y} \\
\mathbf{f}_*
\end{bmatrix}
\sim
\mathcal{N} \left( 
\begin{bmatrix}
\mathbf{m} \\
\mathbf{m}_*
\end{bmatrix},
\begin{bmatrix}
K + \sigma_n^2 I & K_* \\
K_*^\top & K_{**}
\end{bmatrix}
\right).
\]
where $\mathbf{m}^\top=[m(\mathbf{x}_1),m(\mathbf{x}_2),...,m(\mathbf{x}_n)]^\top$ and $\mathbf{m}_*^\top=[m(\mathbf{x}_{1*}),m(\mathbf{x}_{2*}),...,m(\mathbf{x}_{n_*})]^\top$ are the vector of means.

\paragraph{Posterior predictive distribution}
Conditioning on the observed data $(X, \mathbf{y})$, the posterior distribution of $\mathbf{f}_*$ given $X_*$ is \cite{Rasmussen2006GPbook}:
\[
\mathbf{f}_* \mid X, \mathbf{y}, X_* \sim \mathcal{N}(\bar{\mathbf{f}}_*, \operatorname{cov}(\mathbf{f}_*))
\]
where
\begin{align}
\bar{\mathbf{f}}_* &= \mathbf{m}_* + K_*^\top (K + \sigma_n^2 I)^{-1} (\mathbf{y} - \mathbf{m}) \label{eq:gp_mean}\\
\operatorname{cov}(\mathbf{f}_*) &= K_{**} - K_*^\top (K + \sigma_n^2 I)^{-1} K_* \label{eq:gp_var}
\end{align}

\paragraph{Posterior predictive for a single test point}
Let $\mathbf{x}_* \in \mathbb{R}^d$ be a single test input. Define \footnote{$\mathbf{k}_*$ can be extracted from $K_*$.}:
\[
\mathbf{k}_* = k(X, \mathbf{x}_*) = [k(\mathbf{x}_1,\mathbf{x}_*),k(\mathbf{x}_2,\mathbf{x}_*),...,k(\mathbf{x}_n,\mathbf{x}_*)] \in \mathbb{R}^n, \quad k_{**} = k(\mathbf{x}_*, \mathbf{x}_*) \in \mathbb{R}
\]
then the posterior mean and variance become \cite{Rasmussen2006GPbook,Murphy2012}:
\begin{equation} \label{eq:posterior_predictive_signle_test_point}
    \begin{aligned}
    \mu(\mathbf{x}_*) &= m_* + \mathbf{k}_*^\top (K + \sigma_n^2 I)^{-1} (\mathbf{y} - \mathbf{m}), \\
    \sigma^2(\mathbf{x}_*) &= k_{**} - \mathbf{k}_*^\top (K + \sigma_n^2 I)^{-1} \mathbf{k}_*
    \end{aligned}
\end{equation}
This expresses the predictive distribution as:
\[
f(\mathbf{x}_*) \mid X, \mathbf{y}, \mathbf{x}_* \sim \mathcal{N}(\mu(\mathbf{x}_*), \sigma^2(\mathbf{x}_*))
\]

The predictive variance in Eq.\ref{eq:posterior_predictive_signle_test_point} depends only on the input locations and kernel choice, not on the observed values \(\mathbf{y}\). The mean, however, is influenced by \(\mathbf{y}\), and reflects both the prior assumptions and the data.
Note that, the mean prediction in Eq.\ref{eq:posterior_predictive_signle_test_point} is a linear combination of the observed labels $\mathbf{y}$, which is sometimes referred to as a \textit{linear predictor} \cite{Rasmussen2006GPbook} or \textit{interpolator} \cite{Murphy2012}. If the mean function is zero, it can also be viewed as a linear combination of kernel evaluations at the test point \footnote{The fact that GP can be represented in terms of a (possibly infinite) number of basis functions in Eq.\ref{eq:GP_linear_predictor} is one manifestation of the \textit{representer theorem} \cite{Rasmussen2006GPbook}.} \cite{Rasmussen2006GPbook,Murphy2012}:
\begin{equation} \label{eq:GP_linear_predictor}
    \mu(\mathbf{x}_*) = \mathbf{k}_*^\top \boldsymbol{\alpha} = \sum_{i=1}^n \alpha_i k(\mathbf{x}_i, \mathbf{x}_*), \quad
    \text{where } \boldsymbol{\alpha} = (K + \sigma_n^2 I)^{-1} \mathbf{y}
\end{equation}

\paragraph{GP training}
Training a GP model entails two key steps: selecting appropriate functional forms for the mean and covariance functions, and tuning their associated hyper-parameters $\boldsymbol{\theta}$, e.g. the signal variance magnitude $\sigma_f^2$, length scale $\ell$ and noise variance $\sigma_n^2$. This process is often referred to as GP training or model selection, and it is carried out by maximizing the log marginal likelihood of the observed data under the GP prior. 

The marginal likelihood \footnote{The term \textit{marginal likelihood} refers to the fact that we have marginalized over the latent function $\mathbf{f}$. While the GP prior is defined in the space of functions, the marginal likelihood reflects how well the kernel (parameterized by $\boldsymbol{\theta}$) explains the observed data. It is sometimes referred to as the evidence, and forms the basis for model selection and hyperparameter tuning in a Bayesian setting. Note that, in Bayesian modelling, there isn't necessarily a notion of loss function.} is the result of integrating out the latent function values $\mathbf{f}$, yielding the marginal distribution of the noisy targets $\mathbf{y}$:
\[
p(\mathbf{y} \mid X, \boldsymbol{\theta}) = \int p(\mathbf{y} \mid \mathbf{f}) p(\mathbf{f} \mid X, \boldsymbol{\theta}) \, d\mathbf{f}
\]
where $p(\mathbf{f} \mid X, \boldsymbol{\theta}) = \mathcal{N}(\mathbf{m}, K)$ is the GP prior over function values at the training inputs, and $p(\mathbf{y} \mid \mathbf{f}) = \mathcal{N}(\mathbf{f}, \sigma_n^2 I)$ models Gaussian noise.

This integral results in a closed-form Gaussian \cite{Rasmussen2006GPbook}:
\[
\mathbf{y} \mid X, \boldsymbol{\theta} \sim \mathcal{N}(\mathbf{m}, K + \sigma_n^2 I)
\]
Therefore, the log marginal likelihood becomes \cite{Rasmussen2004GP,Rasmussen2006GPbook}:
\begin{equation} \label{eq:GP_log_likelihood}
\log p(\mathbf{y} \mid X, \boldsymbol{\theta}) = 
    -\frac{1}{2} (\mathbf{y}-\mathbf{m})^\top K_y^{-1} (\mathbf{y}-\mathbf{m})
    - \frac{1}{2} \log |K_y| 
    - \frac{n}{2} \log 2\pi
\end{equation}
where:
\begin{itemize}
    \item $K = k(X, X) \in \mathbb{R}^{n \times n}$ is the noise-free kernel matrix;
    \item $K_y = K + \sigma_n^2 I \in \mathbb{R}^{n \times n}$ is the covariance matrix of the noisy observations;
    \item $\boldsymbol{\theta}$ denotes all kernel hyper-parameters (including noise variance $\sigma_n^2$).
\end{itemize}
Eq.\ref{eq:GP_log_likelihood} can also be obtained directly by observing $\mathbf{y} \sim \mathcal{N}(\mathbf{m}, K + \sigma_n^2 I)$. We see that, the log marginal likelihood in Eq.\ref{eq:GP_log_likelihood} consists of 3 terms:
\begin{itemize}
  \item a data fit term: $-\frac{1}{2} (\mathbf{y}-\mathbf{m})^\top K_y^{-1} (\mathbf{y}-\mathbf{m})$
  \item a complexity penalty term: $-\frac{1}{2} \log |K_y|$
  \item a normalization constant (not quite interesting to the optimisation problem): $-\frac{n}{2} \log 2\pi$
\end{itemize}
The marginal likelihood reflects a trade-off between model fit and complexity. A model that fits the data well but is too complex (e.g. overly flexible kernels with high variance or too-short length scales) is penalized via the log determinant term.

The hyper-parameters $\boldsymbol{\theta}$, including those defining the mean function (each denoted by $\theta_m$), the kernel (e.g. length-scale $\ell$, signal variance magnitude $\sigma_f^2$, each denoted by $\theta_k$) and noise variance $\sigma_n^2$, can be learned by maximizing the log marginal likelihood in Eq.\ref{eq:GP_log_likelihood}:
\[
\boldsymbol{\theta}^* = \arg\max_{\boldsymbol{\theta}} \log p(\mathbf{y} \mid X, \boldsymbol{\theta})
\]
This optimisation is typically performed using gradient-based methods, since the derivatives of the log marginal likelihood with respect to the hyper-parameters can be computed analytically (see Chapter 5 of \cite{Rasmussen2006GPbook}). The gradient of the log marginal likelihood \textit{w.r.t.} $\theta_m$ and $\theta_k$, for example, can be derived as \cite{Rasmussen2004GP}:
\begin{align*}
\frac{\partial \log p(\mathbf{y} \mid X, \boldsymbol{\theta})}{\partial \theta_m} &= -(\mathbf{y} - \boldsymbol{\mu})^\top \boldsymbol{\Sigma}^{-1} \frac{\partial \mathbf{m}}{\partial \theta_m} \\
\frac{\partial \log p(\mathbf{y} \mid X, \boldsymbol{\theta})}{\partial \theta_k} &= \frac{1}{2} \mathrm{trace} \left( \boldsymbol{\Sigma}^{-1} \frac{\partial \boldsymbol{\Sigma}}{\partial \theta_k} \right) + \frac{1}{2} (\mathbf{y} - \boldsymbol{\mu})^\top \frac{\partial \boldsymbol{\Sigma}}{\partial \theta_k} \boldsymbol{\Sigma}^{-1} \frac{\partial \boldsymbol{\Sigma}}{\partial \theta_k} (\mathbf{y} - \boldsymbol{\mu})
\end{align*}
Maximizing the marginal likelihood aligns the GP model with the observed data in a principled way: we penalize overly complex models to avoid overfitting without the need of a separate validation set.

\section{GPAR: sampling marginal posterior to estimate support} \label{app:marginal_posterior_sampling}

Here we extend the discussion in Section.\ref{sec:GPAR}, focusing on sampling the marginal posterior for estimating the support of an itemset. In GPAR, the latent variables $\mathbf{z} = [z_1, z_2, \dots, z_M]^\top$ represent the presence or absence of $M$ items, the GP posterior distribution $\mathbf{z} \sim \mathcal{N}(\mathbf{0},K)$, where $K$ is a kernel matrix capturing dependencies between items based on their features, is again a multivariate normal (MVN), and any subset of these variables corresponding to an itemset is also multivariate Gaussian. For example, for a specific itemset $I$, the subset of latent variables $\mathbf{z}_I$ is also Gaussian: $\mathbf{z}_I \sim \mathcal{N}(\mathbf{0}, K_I)$, with $K_I$ being the submatrix of $K$ corresponding to the items in $I$. $\mathcal{N}(\mathbf{0}, K_I)$ is the marginal density of the GP posterior - thanks to the Gaussian distribution \footnote{Gaussian distribution has some nice properties such as marginal, conditional densities are still Gaussian, and if variables are independent, the MVN density can be decomposed into products of univariate Gaussian densities, etc.}, marginalisation is simply extracting the corresponding components from the mean vector and covariance matrix to form a sub Gaussian.

To evaluate the quality of a rule, we first calculate the support of the itemset, i.e. the co-occurrence probability, and then proceed to calculating the confidence of the rule if the support meets a minimum threshold. The support of an itemset $I$ is the probability that all items in $I$ are present:
\begin{equation} \label{eq:sub_itemset_support}
    support(I) = p(\text{all } z_k > 0 \text{ for } k \in I) = \int_{0}^{\infty} \cdots \int_{0}^{\infty} \mathcal{N}(\mathbf{z}_I | \mathbf{0}, K_I) \, d\mathbf{z}_I
\end{equation}
which is the probability that all $z_k$ values for $k \in I$ are positive (we count all $z_k$ with positive sign as presence), which requires integrating the multivariate Gaussian density over a hyper-rectangle (i.e. the positive orthant for those variables).

However, calculating the integral in Eq.\ref{eq:sub_itemset_support}, or the co-occurrence probability of the itemset $I$, cannot be done analytically in general - unlike 1D Gaussian distribution, the \textit{cumulative density function} (CDF) of a multivariate normal does not have a closed form \footnote{Extending the univariate CDF to the multivariate case is impractical, as the multivariate Gaussian CDF over a hyper-rectangle (e.g. $(0, \infty)^{|I|}$) lacks a closed-form solution when variables are dependent. Without extra efforts, numerical integration for $>4$ dimensions is in general hard \cite{Genz1992numerical}.}. Let's probe into this. For a single item, say $z_k$, the probability $p(z_k > 0)$ is straightforward because it involves a univariate Gaussian. Since $z_k \sim \mathcal{N}(0, K_{kk})$, where $K_{kk}$ is the variance of $z_k$, we can use the Gaussian CDF:
\[
p(z_k > 0) = 1 - \Phi\left( \frac{0}{\sqrt{K_{kk}}} \right) = 1 - \Phi(0) = 1 - 0.5 = 0.5
\]
where $\Phi(\cdot)$ is the standard normal CDF. This has a closed-form solution because it’s a one-dimensional integral, and the Gaussian CDF is well-defined and computable.

Consider an itemset with more than one item, say $I = \{k, l\}$. We need $p(z_k > 0, z_l > 0)$, which involves the joint distribution $(z_k, z_l) \sim \mathcal{N}(\mathbf{0}, K_I)$, where:
\[
K_I = \begin{bmatrix} K_{kk} & K_{kl} \\ K_{lk} & K_{ll} \end{bmatrix}
\]
Here, $K_{kl} = K_{lk}$ is the covariance between $z_k$ and $z_l$, reflecting their dependency. The probability becomes:
\[
p(z_k > 0, z_l > 0) = \int_{0}^{\infty} \int_{0}^{\infty} \mathcal{N}([z_k, z_l] | \mathbf{0}, K_I) \, dz_k dz_l
\]
which is the CDF of a bivariate Gaussian over the positive quadrant. Unlike the univariate case, there’s no simple closed-form expression for this integral when $K_{kl} \neq 0$ (i.e. when the variables are correlated). If $z_k$ and $z_l$ were independent ($K_{kl} = 0$), the probability would factorize:
\[
p(z_k > 0, z_l > 0) = p(z_k > 0) \cdot p(z_l > 0) = 0.5 \cdot 0.5 = 0.25.
\]
However, in GPAR, the posterior kernel $K$ typically introduces correlations between items (maybe sparse but not diagonal), so $K_I$ is not diagonal \footnote{The key issue is the dependencies encoded in $K_I$. In GPAR, the kernel (e.g. a radial basis function or linear kernel) ensures that items with similar features have correlated latent variables. This correlation prevents the joint probability from being a simple product of individual probabilities, unlike the independent case. As the size of the itemset $|I|$ increases, the dimensionality of the integral grows, and the dependencies make it impossible to solve analytically.}, and the above factorization doesn’t generally hold. 

Since an analytical solution isn’t available, we have to compute $p(\text{all } z_k > 0 \text{ for } k \in I)$ via numerical routines such as Monte Carlo (MC) sampling (fortunately there are standard ways to sample from an MVN \cite{Rasmussen2006GPbook}, although the cost being cubic to the dimension due to Cholesky decomposition). The MC sampling procedure involves:
(1) drawing samples $\mathbf{z}_s \sim \mathcal{N}(\mathbf{0}, K_I)$;
(2) checking, for each sample, whether all components $z_{s,k} > 0$ for $k \in I$.
(3) estimating the probability as the fraction of samples satisfying this condition.
This method is flexible, scalable, and can handle correlations naturally by sampling from the full multivariate distribution. While it’s approximate and computationally intensive, it avoids the intractable integration problem.

\section{GPAR: kernel designs} \label{app:kernel_designs}

This appendix extends the discussion of custom kernel design prescribed in Section.\ref{sec:custom_kernel_design}. Kernels originate from functional analysis and kernel theory. In the context of GP, a kernel $k(x,y)$ turns a distance or metrics into similarity (covariance). We first define a metric (or distance) space \cite{Alpay2023}:

\begin{definition}
Let $X \subseteq \mathbb{R}^d$. A function $d : X \times X \to [0, \infty)$ is called a \emph{metric} (or \emph{a distance function}) if for all $\mathbf{x}_i, \mathbf{x}_j, \mathbf{x}_k \in X$, the following hold:
\begin{align}
\text{(1)}\quad & d(\mathbf{x}_i, \mathbf{x}_j) = 0 \quad \Longleftrightarrow \quad \mathbf{x}_i = \mathbf{x}_j \nonumber \\
\text{(2)}\quad & d(\mathbf{x}_i, \mathbf{x}_j) = d(\mathbf{x}_j, \mathbf{x}_i) \nonumber \\
\text{(3)}\quad & d(\mathbf{x}_i, \mathbf{x}_j) \leq d(\mathbf{x}_i, \mathbf{x}_k) + d(\mathbf{x}_k, \mathbf{x}_j). \nonumber
\end{align}
\end{definition}

For example, the \textit{Euclidean distance}: $d(\mathbf{x}_i, \mathbf{x}_j) = \|\mathbf{x}_i - \mathbf{x}_j\|_2$, \textit{Manhattan distance}: $d(\mathbf{x}_i, \mathbf{x}_j) = \|\mathbf{x}_i - \mathbf{x}_j\|_1$, \textit{Mahalanobis distance}, etc.

There are positive definite and negative definite kernels (negative definite kernels are also called conditionally negative definite kernels). Note that, negative definite kernels are not the opposite of positive definite kernels \cite{Alpay2023}. In the following, we exchangeably use the terminology \textit{function} and \textit{kernel}, although function here refers in general to a smaller class of kernels. In many literature, the term 'positive definite (PD) kernel' is equivalent to 'positive semi-definite (PSD) kernel' in matrix algebra (i.e. the resulting kernel matrix is PSD).

\begin{definition}[Positive semi-definite (PSD) kernel\footnote{More generally, one allows $\mathbf{a} \in \mathbb{C}^n$, and the PSD condition becomes: $\sum_{i,j} \overline{c_i} k(x_i, x_j) c_j \geq 0$ \cite{Alpay2023}.} \cite{Alpay2023}]
A function \( k(\mathbf{x}, \mathbf{x}') \) is \textbf{PSD} if for any finite set of points \( \{\mathbf{x}_1, \ldots, \mathbf{x}_n\} \), the kernel matrix \( K \) with entries \( K_{ij} = k(\mathbf{x}_i, \mathbf{x}_j) \) is symmetric and satisfies:
\[
\sum_{i,j} a_i a_j k(\mathbf{x}_i, \mathbf{x}_j) \geq 0 \quad \text{for all } \mathbf{a} \in \mathbb{R}^n.
\]
\end{definition}
Note that PSD (or PD) allows the possibility that $\sum_{i,j} a_i a_j k(\mathbf{x}_i, \mathbf{x}_j) = 0$ for some non-zero $\{a_i\}$. Strictly positive definite refers to the condition that $\sum_{i,j} a_i a_j k(\mathbf{x}_i, \mathbf{x}_j) > 0 \quad \text{for all } a_i \in \mathbb{R}^n \text{ and } a_i \neq 0$.

\begin{definition}[Conditionally negative definite (CND) function \footnote{\textit{'Conditional'} refers to the 'zero-sum' condition $\sum_i a_i = 0$. Strictly speaking, one needs to distinguish 'positive definite' and 'conditional positive definite', 'negative definite' and 'conditional negative definite', see \cite{Scholkopf2002KernelLearning}.} \cite{Alpay2023}]
\label{def:CND}
A symmetric function \( \phi(\mathbf{x}, \mathbf{x}') \) is \textbf{CND} if:
\[
\sum_{i,j} a_i a_j \phi(\mathbf{x}_i, \mathbf{x}_j) \leq 0 \quad \text{for all } \mathbf{a} \in \mathbb{R}^n \text{ such that } \sum_i a_i = 0.
\]
\end{definition}

\begin{example} \label{example:Euclidean_dist_CND}
Prove the squared Euclidean distance $\phi(\mathbf{x}_i, \mathbf{x}_j) = \|\mathbf{x}_i - \mathbf{x}_j\|^2$ is CND. 
\begin{proof}
   As per Definition.\ref{def:CND}, a function $d(\mathbf{x}, \mathbf{y})$ is CND if $\sum_{i,j=1}^n a_i a_j \phi(\mathbf{x}_i, \mathbf{x}_j) \leq 0$ holds for any set of points $\{\mathbf{x}_1, \mathbf{x}_2, \dots, \mathbf{x}_n\}$ and any coefficients $\mathbf{a} = [a_1, a_2, \dots, a_n]^\top \in \mathbb{R}^n$ satisfying $\sum_{i=1}^n a_i = 0$. In order to show $\phi(\mathbf{x}_i, \mathbf{x}_j) = \|\mathbf{x}_i - \mathbf{x}_j\|^2$ is CND, we need to compute:
        \[ \sum_{i,j=1}^n a_i a_j \|\mathbf{x}_i - \mathbf{x}_j\|^2 \]
    where $\sum_{i=1}^n a_i = 0$. \\
    Expanding the distance: 
    \[ \phi(\mathbf{x}_i, \mathbf{x}_j) =
    \|\mathbf{x}_i - \mathbf{x}_j\|^2 = \mathbf{x}_i^\top \mathbf{x}_i + \mathbf{x}_j^\top \mathbf{x}_j - 2 \mathbf{x}_i^\top \mathbf{x}_j \]
    The first term:
    \[ \sum_{i,j=1}^n a_i a_j \mathbf{x}_i^\top \mathbf{x}_i = \sum_{i=1}^n a_i \mathbf{x}_i^\top \mathbf{x}_i \sum_{j=1}^n a_j = 0 \]
    as $\sum_{i=1}^n a_i = 0$.
    Similarly, the second term:
    \[ \sum_{i,j=1}^n a_i a_j \mathbf{x}_j^\top \mathbf{x}_j = \sum_{j=1}^n a_j \mathbf{x}_j^\top \mathbf{x}_j \sum_{i=1}^n a_i = 0 \]
    The third term:
    \[ \sum_{i,j=1}^n a_i a_j (-2 \mathbf{x}_i^\top \mathbf{x}_j) = -2 \sum_{i,j=1}^n a_i a_j \mathbf{x}_i^\top \mathbf{x}_j = -2 \left( \sum_{i=1}^n a_i \mathbf{x}_i \right)^\top \left( \sum_{j=1}^n a_j \mathbf{x}_j \right) = -2 \left\| \sum_{i=1}^n a_i \mathbf{x}_i \right\|^2 \leq 0 \]
    Therefore,
    \begin{equation*}
    \begin{aligned}
    \sum_{i,j=1}^n a_i a_j \|\mathbf{x}_i - \mathbf{x}_j\|^2 
    &= \sum_{i,j=1}^n a_i a_j \left( \mathbf{x}_i^\top \mathbf{x}_i + \mathbf{x}_j^\top \mathbf{x}_j - 2 \mathbf{x}_i^\top \mathbf{x}_j \right) \\
    &= 0 + 0 - 2 \left\| \sum_{i=1}^n c_i \mathbf{x}_i \right\|^2 = -2 \left\| \sum_{i=1}^n c_i \mathbf{x}_i \right\|^2 \leq 0
    \end{aligned}
    \end{equation*}
\end{proof}
\end{example}

The connection between PSD and CND is noted:

\begin{theorem}[Connection between PSD and CND \cite{Schoenberg1938metric,Berg1984harmonic,Scholkopf2002KernelLearning,Alpay2023}]
\label{theorem:connection_PSD_and_CND}
Let $X$ be a set and let $d(\mathbf{x}, \mathbf{x}')$ be a complex-valued function defined on $X \times X$. 
Then, $d(\mathbf{x}, \mathbf{x}')$ is \textbf{negative definite} if and only if one of the following equivalent conditions holds for ALL $t > 0$:
\begin{enumerate}
    \item $e^{-t d(\mathbf{x}, \mathbf{x}')}$ is \textbf{positive definite}.
    \item $\frac{1}{1 + t d(\mathbf{x},\mathbf{x}')}$ is \textbf{positive definite}.
\end{enumerate}
\end{theorem}
\begin{proof}
see e.g. \cite{Alpay2023}.
\end{proof}

We can apply this theorem to an exponentiated radial function $k(\|\mathbf{x} - \mathbf{x}'\|) = \exp\left( - \phi(\|\mathbf{x} - \mathbf{x}'\|) \right)$ which is symmetric, leading to the conclusion that, $k$ is a PSD kernel if and only if $\phi$ is a CND function:

\begin{corollary}[PSD condition for exponentiated radial functions]
\label{corollary:PSD_and_radial_funcs}
Let $\phi : \mathbb{R}^d \to \mathbb{R}$ be a symmetric, radial function such that $\phi(\mathbf{x}, \mathbf{x}') = \phi(\|\mathbf{x} - \mathbf{x}'\|)$. Then the kernel
\[
k(\mathbf{x}, \mathbf{x}') = \exp\left( -\phi(\|\mathbf{x} - \mathbf{x}'\|) \right)
\]
is positive semi-definite (PSD) and hence defines a valid kernel for all $\mathbf{x}, \mathbf{x}' \in \mathbb{R}^d$ \emph{if and only if} $\phi$ is conditionally negative definite (CND). More generally, the equivalence holds:
\[
\phi \text{ is CND} \quad \Longleftrightarrow \quad \exp(-t \phi) \text{ is PSD for all } t > 0
\]
\end{corollary}

\begin{proof}
This follows directly from Theorem.\ref{theorem:connection_PSD_and_CND}, and is a classical result in harmonic analysis (see, e.g. \cite{Schoenberg1938metric, Berg1984harmonic, Alpay2023}).
\end{proof}

\noindent
This corollary connects distance-like functions to PSD kernels via exponentiation: exponentiating a CND function gives a PSD function. Conversely, if $\phi$ is not CND, then $k(r) = \exp(-\phi(r))$ may fail to be PSD, and the resulting kernel matrix may not be valid. Therefore, when designing kernels using exponentiated functions, we must check whether the base function $\phi$ is CND. Note that, radiality alone is not sufficient for a function to be CND, i.e. a radial function $\phi(\mathbf{x})=\phi(\|\mathbf{x}\|)$ is not automatically conditionally negative definite; additional structural properties must hold.

\begin{example} \label{example:RBF_kernel}
The RBF kernel $k(\mathbf{x}, \mathbf{x}') = \exp\left(-\frac{\|\mathbf{x} - \mathbf{x}'\|^2}{2\ell^2}\right)$ follows Corollary.\ref{corollary:PSD_and_radial_funcs}. Example.\ref{example:Euclidean_dist_CND} confirms that $\|\mathbf{x}_i - \mathbf{x}_j\|^2$ is CND, as the sum is non-positive under the constraint $\sum a_i = 0$. 
Therefore, as per Corollary.\ref{corollary:PSD_and_radial_funcs}, the RBF kernel is PSD.
\end{example}

A \textit{completely monotone function} is a smooth function that alternates signs in its derivatives in a specific way \footnote{One has distinguish \textit{absolutely monotonic} and \textit{completely monotonic} functions. Both have derivatives of all orders and express strong monotonicity properties. However, an absolutely monotonic function has all derivatives non-negative, implying the function and its derivatives are all monotonically increasing; a completely monotonic function has derivatives that alternate in sign, implying the function and its derivatives alternate between being monotonically decreasing and increasing.}:

\begin{definition}[Completely monotone function \cite{Merkle2012,Sendov2015new}] (total or complete monotonicity)
A function $f$ defined on $(0, \infty)$ is called \emph{completely monotone} or \emph{totally monotone} if it is infinitely differentiable (i.e. has derivatives of all orders) and satisfies:
\[
(-1)^n f^{(n)}(t) > 0 \quad \text{for all } t > 0 \text{ and } n \in \mathbb{N}_0
\]
\end{definition}
That said, $f(t) >0$ (convex), $f'(t) < 0$ (concave first derivative), $f''(t) > 0$, $f^{(3)}(t) < 0$, and so on. This implies that each completely monotone function on $(0, \infty)$ is always positive, always decreasing, concave up and down in alternating directions at each derivative. It encodes long-range decay and non-oscillatory behavior. 

Complete monotonicity is deeply connected to Laplace transforms, probability theory, and positive definite kernels. In real analysis, \textit{Bernstein's theorem} \cite{Bernstein1928} states that every real-valued function on the half-line $(0, \infty)$ that is totally monotone is a mixture of exponential functions:

\begin{theorem}[Bernstein \cite{Merkle2012,Sendov2015new}]
\label{theorem:Bernsteins}
A function $f$ is completely monotone on $(0, \infty)$ if and only if it is the Laplace transform of a unique, non-negative Radon measure $\mu$ on $[0, \infty)$, i.e.
\[
f(t) = \int_0^\infty e^{-ts} \, d\mu(s), \quad \mu \geq 0
\]
for all $t>0$.
\end{theorem}
\begin{proof}
see e.g. \cite{Merkle2012}.
\end{proof}

\begin{figure} [H]
    \centering
    \includegraphics[width=0.5\linewidth]{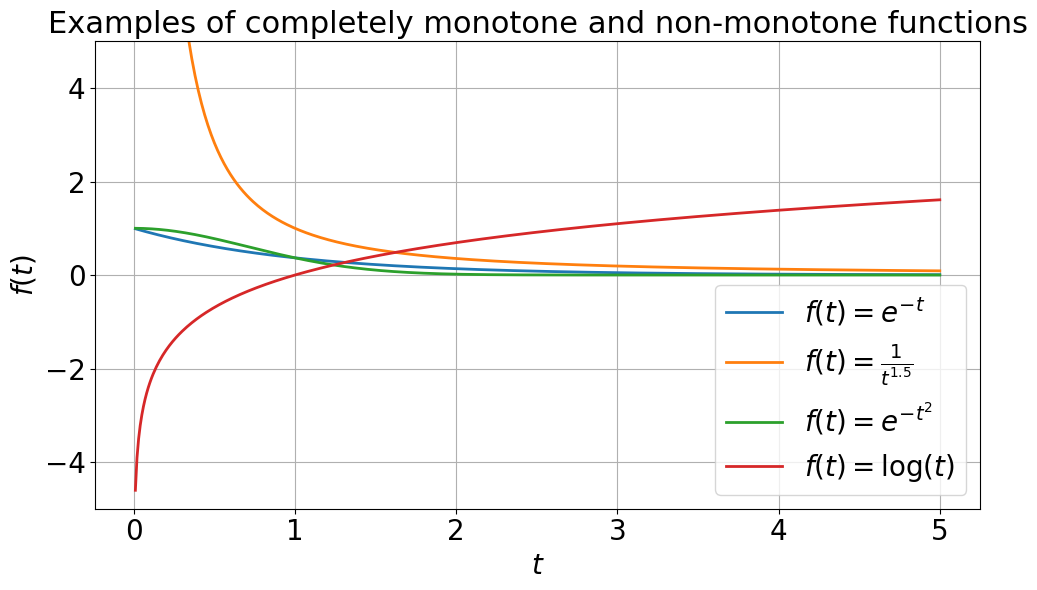}
    \caption{Four examples of completely monotone and not completely monotone functions.}
    \label{fig:completely_monotone_and_non_funcs}
\end{figure} 

We present 4 examples of completely monotone and not completely monotone functions in Fig.\ref{fig:completely_monotone_and_non_funcs}, in which $e^{-t}$ and $\frac{1}{t^\alpha}$ (for $\alpha > 0$) are completely monotone, $\log t$ and $\exp(-t^2)$ are not completely monotone on $(0, \infty)$.

\begin{example}
 $f(t) = e^{-t}$ is completely monotone because it all derivatives alternate sign:
\[
  f(t) = e^{-t},\quad f'(t) = -e^{-t},\quad f''(t) = e^{-t},\quad f^{(3)}(t) = -e^{-t},\ \dots
\] 
which gives 
\[
  (-1)^n f^{(n)}(t) = e^{-t} > 0 \quad \forall t > 0, \; n \in \mathbb{N}
\]
In fact, it is the Laplace transform of the Dirac delta at 1:
\[
\int_0^\infty e^{-ts} \delta(s - 1)\, ds = e^{-t}
\]
\end{example}
\begin{example}
$f(t) = \frac{1}{t^\alpha}$ for $\alpha > 0$ is also completely monotone on $(0, \infty)$, for example $\alpha = 1$,
\[
f(t) = \frac{1}{t},\quad f'(t) = -\frac{1}{t^2},\quad f''(t) = \frac{2}{t^3},\ \dots
\]
So $(-1)^n f^{(n)}(t) > 0$ for all $n$ and $t > 0$. In fact, it can be represented as the Laplace transform of a Gamma-related measure:
\[
\frac{1}{t^\alpha} = \frac{1}{\Gamma(\alpha)} \int_0^\infty s^{\alpha - 1} e^{-ts} ds
\]
\end{example}
\begin{example}
$f(t) = \log t$ is not completely monotone even though as its derivatives:
\[
f'(t) = \frac{1}{t} > 0,\quad f''(t) = -\frac{1}{t^2} < 0,\quad f^{(3)}(t) = \frac{2}{t^3} > 0
\]
alternate in sign, but its function values are not always positive: 
\[
f(t) = \log t \not\geq 0 \quad \text{as } t \to 0^+
\]
and it grows without bound, so it cannot be written as a Laplace transform of a finite, non-negative Radon measure. 
\end{example}
\begin{example}
$\exp(-t^2)$ is not completely monotone as it violates the derivative condition. Its first derivative:
\[
f'(t) = \frac{d}{dt} e^{-t^2} = -2t e^{-t^2}
\Rightarrow f'(t) < 0 \text{ for } t > 0
\]
satisfies the sign condition for $n = 1$. Its second derivative:
\[
f''(t) = \frac{d}{dt}(-2t e^{-t^2}) = -2e^{-t^2} + 4t^2 e^{-t^2} = (4t^2 - 2)e^{-t^2}
\]
violates the requirement that $f''(t) \geq 0$ for all $t > 0$ because for $t < \frac{1}{\sqrt{2}}$, $f''(t) < 0$, for $t > \frac{1}{\sqrt{2}}$, $f''(t) > 0$. 
\end{example}

In kernel methods, harmonic analysis, and Gaussian processes, \textit{Bochner’s theorem} and \textit{Schoenberg’s theorem} are two cornerstones in the theory of positive definite functions. They both describe when a function defines a positive definite kernel \footnote{Schoenberg’s theorem is often viewed as a radial analogue of Bochner’s theorem - a special case of Bochner’s theorem under spherical symmetry.}: Bochner’s theorem applies to stationary kernels $k(\boldsymbol{\tau}) = k(\mathbf{x} - \mathbf{x}')$, while Schoenberg’s theorem handles radial kernels $k(\mathbf{x}, \mathbf{x}') = f(\|\mathbf{x} - \mathbf{x}'\|)$. Bochner’s theorem builds the equivalence between stationary positive definite kernels and Fourier transforms of finite positive measures, while Schoenberg’s theorem explains radial positive definite kernels as Laplace transforms of finite positive measures.  

\textit{Bochner's theorem} \cite{Bochner1955harmonic} states that, a continuous, positive definite, stationary kernel can be represented as the inverse Fourier transform of a positive finite measure, i.e. the covariance function and the spectral density are Fourier duals of each other \cite{Rasmussen2006GPbook}:

\begin{theorem}[Bochner \cite{Bochner1955harmonic,Rasmussen2006GPbook}]
\label{theorem:Bochners}
Let $k : \mathbb{R}^d \to \mathbb{C}$ be a continuous function. Then $k$ is the covariance function of a weakly stationary, mean-square continuous, complex-valued stochastic process on $\mathbb{R}^d$ if and only if there exists a positive finite measure $\mu$ on $\mathbb{R}^d$ such that
\[
k(\boldsymbol{\tau}) = \int_{\mathbb{R}^d} e^{2\pi i \mathbf{s} \cdot \boldsymbol{\tau}} \, d\mu(\mathbf{s}),
\]
where $\boldsymbol{\tau} = \mathbf{x} - \mathbf{x}'$. If $\mu$ admits a density $S(\mathbf{s})$, called the \emph{spectral density}, then
\[
k(\boldsymbol{\tau}) = \int_{\mathbb{R}^d} S(\mathbf{s}) e^{2\pi i \mathbf{s} \cdot \boldsymbol{\tau}} \, d\mathbf{s}, \qquad 
S(\mathbf{s}) = \int_{\mathbb{R}^d} k(\tau) e^{-2\pi i \mathbf{s} \cdot \boldsymbol{\tau}} \, d\boldsymbol{\tau}.
\]
\end{theorem}
\begin{proof}
see \cite{Bochner1955harmonic}.
\end{proof}

\textit{Schoenberg’s theorem} complements Bochner’s theorem for radial functions:

\begin{theorem}[Schoenberg \cite{Schoenberg1938metric,Ressel1976,Khoshnevisan2005}]
\label{theorem:Schoenbergs}
Let $k : \mathbb{R}_+ \to \mathbb{R}_+$ be a continuous function. Then the following are equivalent:

\begin{enumerate}
    \item The function $\mathbf{x} \mapsto k(\|\mathbf{x}\|_n)$ is positive semi-definite on $\mathbf{x} \in \mathbb{R}^n$ for all $n \in \mathbb{N}$.
    
    \item The function $t \mapsto k(\sqrt{t})$ is the Laplace transform of a finite measure $\nu$ on $t \in \mathbb{R}_+$, i.e.
    \[
    k(\sqrt{t}) = \int_0^\infty e^{-t s} d\nu(s)
    \]
\end{enumerate}
\end{theorem}
\begin{proof}
see \cite{Schoenberg1938metric}.
\end{proof}

Combining \textit{Schoenberg’s theorem} and \textit{Bernstein’s theorem}, we arrive at the conclusion that, a function $k(\mathbf{x}, \mathbf{x}') = f(\|\mathbf{x} - \mathbf{x}'\|)$ defines a \textit{positive definite kernel} on $\mathbb{R}^n$ for all $n$ if and only if $f(\sqrt{t})$ is a \textit{completely monotone function} on $(0, \infty)$, that is, if $k(\|\mathbf{x} - \mathbf{x}'\|)$ is PSD in all dimensions, then $k(\sqrt{t})$ must be \textit{completely monotone}. 

\textit{Mercer’s theorem} allows expressing the covariance function, under certain conditions, in terms of its eigenfunctions and eigenvalues:

\begin{theorem}[Mercer \cite{Konig1986eigenvalue,Rasmussen2006GPbook}]
\label{theorem:mercer}
Let $(X, \mu)$ be a finite measure space, and let $k \in L_\infty(X^2, \mu^2)$ be a kernel such that the associated integral operator
\[
T_k : L_2(X, \mu) \to L_2(X, \mu), \qquad (T_k f)(\mathbf{x}) = \int_X k(\mathbf{x}, \mathbf{x}') f(\mathbf{x}') \, d\mu(\mathbf{x}')
\]
is positive definite. Let $\{ \phi_i \}_{i=1}^\infty \subset L_2(X, \mu)$ denote the orthonormal eigenfunctions of $T_k$, associated with eigenvalues $\{ \lambda_i \}_{i=1}^\infty$, where each $\lambda_i > 0$.

Then:
\begin{enumerate}
    \item The eigenvalues $\{\lambda_i\}$ are absolutely summable, i.e. $\sum_{i=1}^\infty \lambda_i < \infty$.
    \item The kernel admits the expansion:
    \[
    k(\mathbf{x}, \mathbf{x}') = \sum_{i=1}^\infty \lambda_i \phi_i(\mathbf{x}) \overline{\phi_i(\mathbf{x}')}
    \]
    where the series converges absolutely and uniformly $\mu^2$-almost everywhere.
\end{enumerate}
\end{theorem}
\begin{proof}
    See e.g. \cite{Konig1986eigenvalue}.
\end{proof}

For a stationary covariance function $k$ on \( \mathbb{R}^d \), the Mercer expansion is replaced by the Bochner representation:

\begin{example}[Bochner \textit{vs} Mercer expansions \cite{Rasmussen2006GPbook}]
The statement of Mercer's theorem applies to kernels defined with respect to a finite measure \( \mu \). However, if we replace this with Lebesgue measure and consider a stationary covariance function, then by Bochner’s theorem we obtain \cite{Rasmussen2006GPbook}:

\begin{equation*}
k(\mathbf{x} - \mathbf{x}') 
= \int_{\mathbb{R}^d} e^{2\pi i \mathbf{s} \cdot (\mathbf{x} - \mathbf{x}')} \, d\mu(\mathbf{s}) 
= \int_{\mathbb{R}^d} e^{2\pi i \mathbf{s} \cdot \mathbf{x}} \left( e^{2\pi i \mathbf{s} \cdot \mathbf{x}'} \right)^* \, d\mu(\mathbf{s})
\end{equation*}
in which the complex exponentials $e^{2\pi i \mathbf{s} \cdot \mathbf{x}}$ are the eigenfunctions of a stationary kernel with respect to Lebesgue measure. This expression closely resembles the Mercer expansion:
$
k(\mathbf{x}, \mathbf{x}') = \sum_{i=1}^\infty \lambda_i \phi_i(\mathbf{x}) \overline{\phi_i(\mathbf{x}')}
$
except that the summation over discrete eigenfunctions has been replaced by an integral over the continuous spectral domain $\mathbb{R}^d$.
\end{example}

Table. \ref{tab:kernel_theorems} summarizes the above theorems - they serve as testing rule for a proposed kernel, or guidance for designing new kernels. PSD kernels can be constructed from known PSD components or use kernel composition techniques (e.g. via convolution, spectral mixtures, or positive combinations) \cite{Rasmussen2006GPbook}.

\begin{table}[H]
\centering
\caption{Summary of the theorems}
\label{tab:kernel_theorems}
\scriptsize
\setlength{\tabcolsep}{3pt}
\renewcommand{\arraystretch}{1.1}
\begin{tabular}{p{2.6cm} p{9.5cm}}
\toprule
\textbf{Theorem} & \textbf{Essence} \\
\midrule
\textbf{Theorem.\ref{theorem:connection_PSD_and_CND}} & PSD and CND connectivity via exponentiation and reciprocal transformation \\
\textbf{Corollary.\ref{corollary:PSD_and_radial_funcs}} & \( \exp(-\phi) \) is PSD \( \Longleftrightarrow \) \( \phi \) is CND \\
\textbf{Bernstein's.\ref{theorem:Bernsteins}} & Completely monotone \( \Longleftrightarrow \) Laplace transforms of non-negative measures \\
\textbf{Bochner's.\ref{theorem:Bochners}} & Stationary PSD kernels \( \Longleftrightarrow \) Fourier transforms of positive measures \\
\textbf{Schoenberg's.\ref{theorem:Schoenbergs}} & Radial PSD kernels \( \Longleftrightarrow \) Laplace transforms of positive measures \\
\textbf{Mercer's.\ref{theorem:mercer}} & PSD kernels on compact domains \( \Longrightarrow \) eigenfunction–eigenvalue expansion (discrete spectral decomposition) \\
\bottomrule
\end{tabular}
\end{table}

\subsection{The inverse multiquadratic (IMQ) kernel} \label{app:IMQ_kernel}

One of the kernels which show similar behaviour (i.e. similarity decreases as distance increases) as RBF is the IMQ kernel, which is defined as \footnote{In \cite{Scholkopf2002KernelLearning} pp.69, the IMQ kernel is defined with $\beta=1/2$ and normally instantiated with $c=1$.} \cite{Scholkopf2002KernelLearning,Micchelli1986interpolation}:
\begin{equation} \label{eq:IMQ_kernel}
k(\mathbf{x}_i, \mathbf{x}_j) = \left( \|\mathbf{x}_i - \mathbf{x}_j\|^2 + c^2 \right)^{-\beta}, \quad c > 0, \, \beta > 0
\end{equation}
in which $\|\mathbf{x}_i - \mathbf{x}_j\|^2$ is the Euclidean distance, $c$ is a scale parameter controlling the range of interaction, $\beta$ is a shape parameter (decay rate) controlling how fast similarity decays. 

The IMQ kernel is completely monotonic on $(0,\infty)$, and $\beta > 0$ is necessary for it to be positive definite \cite{Micchelli1986interpolation}. It is symmetric and its value depends only on the relative distance between inputs, it's thus stationary and isotropic \footnote{A kernel is \textit{isotropic} if it depends only on the Euclidean distance between the points, i.e. $\|\mathbf{x}_i-\mathbf{x}_j\|^2$, and not on the direction of the difference. This implies rotational invariance, meaning that rotating both inputs by the same rotation matrix does not affect the kernel’s value. To verify rotational invariance of this radial function, consider a rotation matrix $R$ with the property $R^\top R = I$. The kernel for the rotated inputs is:
\[
k(R\mathbf{x}_i, R\mathbf{x}_j) = \left( \|R\mathbf{x}_i - R\mathbf{x}_j\|^2 + c^2 \right)^{-\beta}
\]
Since $R\mathbf{x}_i - R\mathbf{x}_j = R(\mathbf{x}_i - \mathbf{x}_j)$, the squared norm can be computed as:
\[
\|R(\mathbf{x}_i - \mathbf{x}_j)\|^2 = (R(\mathbf{x}_i - \mathbf{x}_j))^\top R(\mathbf{x}_i - \mathbf{x}_j) = (\mathbf{x}_i - \mathbf{x}_j)^\top R^\top R (\mathbf{x}_i - \mathbf{x}_j) = (\mathbf{x}_i - \mathbf{x}_j)^\top (\mathbf{x}_i - \mathbf{x}_j) = \|\mathbf{x}_i - \mathbf{x}_j\|^2
\]
The squared distance is unchanged under rotation, therefore:
\[
k(R\mathbf{x}_i, R\mathbf{x}_j) = \left( \|\mathbf{x}_i - \mathbf{x}_j\|^2 + c^2 \right)^{-\beta} = k(\mathbf{x}_i, \mathbf{x}_j)
\]
This shows the kernel is rotationally invariant, confirming that it is isotropic.}. The IMQ kernel is strictly positive definite (PD) on $\mathbb{R}^d$ if $\beta < \frac{d + 1}{2}$ \cite{Micchelli1986interpolation}, based on the Bochner’s theorem which links positive definiteness to the non-negativity of the spectral density of the kernel (proof see \cite{Micchelli1986interpolation}). The Fourier transform of $k(r) = (r^2 + c^2)^{-\beta}$ is known and is non-negative when $\beta < \frac{d+1}{2}$, guaranteeing that the kernel matrix is PSD for any finite set of inputs. For example, in 1D, the IMQ kernel is PD if $\beta < 1$; in 3D, if $\beta < 2$, etc.

\subsection{Exponentiated IMQ (EIMQ) kernel}  \label{app:EIMQ_kernel}

Unlike RBF or IMQ kernels, in AR mining we expect an 'optimal distance' behaviour, i.e. the similarity peaks at certain distance, which compromises the co-occurring patterns of substitute and complementary items. Here we focus on how we can achieve an inverse-distance behaviour, i.e. the kernel-output similarity grows as distance increases. To have a valid PSD kernel which preserves this reciprocal distance-like behavior, we combine the idea of (negative) exponentiation used in RBF etc and the idea of reciprocal distance as desired by AR mining, and introduce the exponentiated inverse multiquadratic (EIMQ) kernel:
\begin{equation} \label{eq:EIMQ_kernel}
k(\mathbf{x}_i, \mathbf{x}_j) = \exp \left(- \frac{1}{(\|\mathbf{x}_i - \mathbf{x}_j\|^2 + c^2 )^{\beta}} \right), \quad c > 0, \, \beta > 0
\end{equation}
and two instances shown in Fig.\ref{fig:EIMQ_kernel}. It is empirically observed that, some eigenvalues from these two instances are negative, implying the EIMQ kernel is not PSD.

\begin{figure}[H]
    \centering
    \begin{subfigure}[b]{0.5\linewidth}
        \centering
        \includegraphics[width=\linewidth]{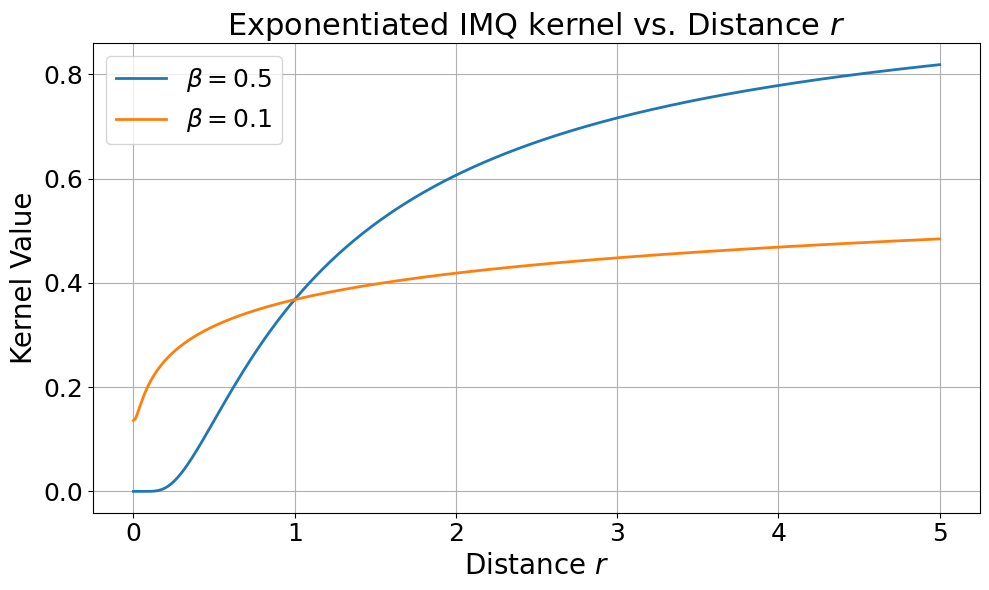}
        \caption{Example EIMQ kernels with $\beta=0.1$ and $\beta=0.5$}
        \label{fig:EIMQ}
    \end{subfigure}
    \hfill
    \begin{subfigure}[b]{0.45\linewidth}
        \centering
        \includegraphics[width=\linewidth]{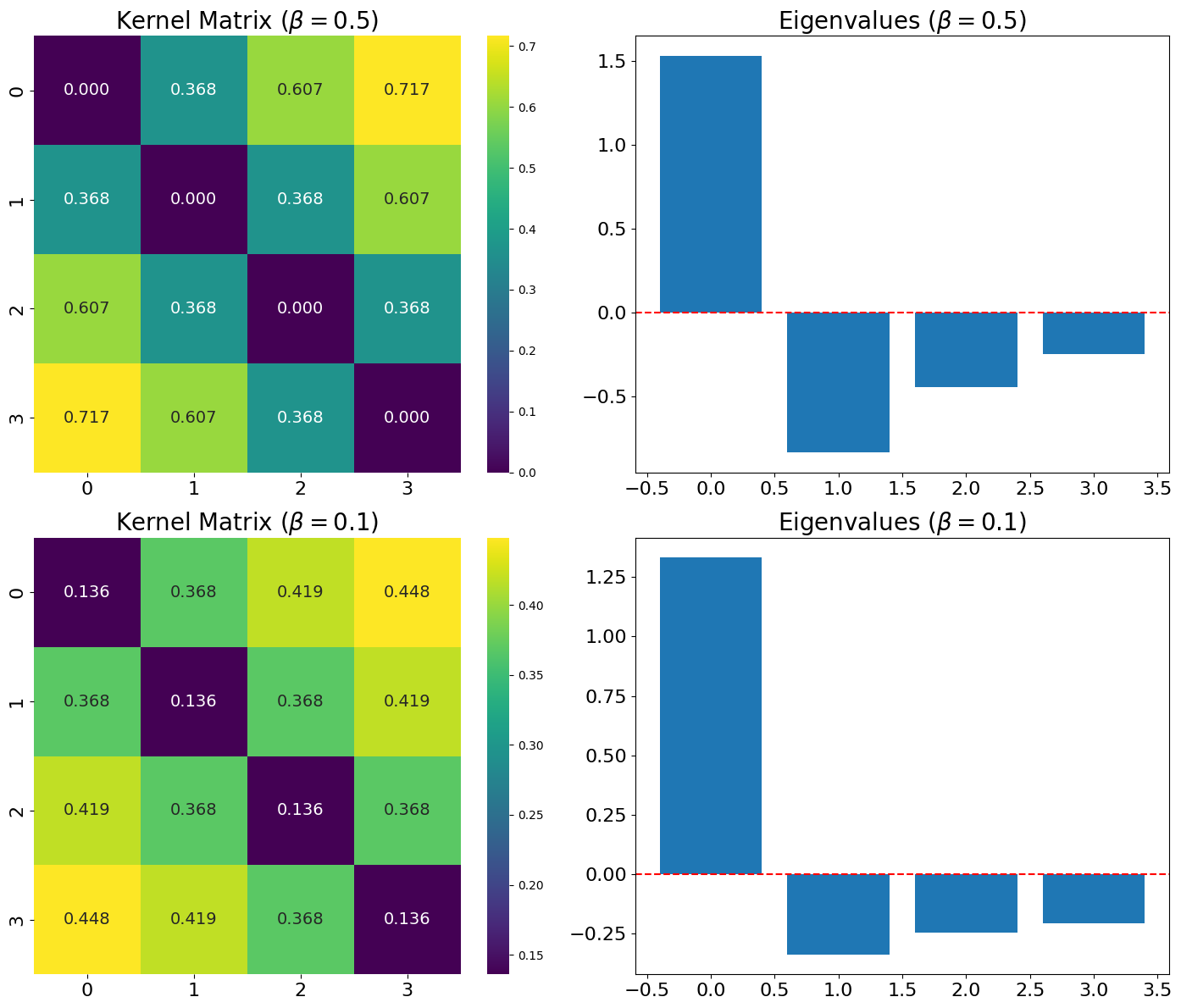}
        \caption{Example covariance matrices and eigenvalues}
        \label{fig:EIMQ_matrix_eigenvalues}
    \end{subfigure}
    \caption{Two examples of the EIMQ kernel, both with $c^2$=0.001. The 4 data points are on 1D real line: [$x_1$ = 0, $x_2$ = 1, $x_3$ = 2, $x_4$ = 3].}
    \label{fig:EIMQ_kernel}
\end{figure}

The EIMQ kernel has a reciprocal of squared distance $r^2 = \|\mathbf{x}_i - \mathbf{x}_j\|^2$, rather than a linear dependence on squared distance $r$. This reciprocal form causes the kernel to decay more slowly for large distances, and sharply increase for small distances. The EIMQ kernel is stationary or  translation-invariant \footnote{The constant $c^2$ doesn't break translation-invariance. Translation invariance or stationarity means that shifting both inputs by the same vector does not change the kernel’s value, i.e. the kernel value depends only on distance $r$, not the absolute locations $\mathbf{x}_i$ and $\mathbf{x}_j$. This property implies the kernel can be written as $k(\mathbf{x},\mathbf{y})=f(\mathbf{x}-\mathbf{y})$ for some function $f$. A kernel like $k(\mathbf{x}_i, \mathbf{x}_j) = \exp(-\|\mathbf{x}_i\|^2 - \|\mathbf{x}_j\|^2)$ or $k(\mathbf{x}_i, \mathbf{x}_j) = \mathbf{x}_i^\top \mathbf{x}_j$ would not be translation-invariant, as they depend on the individual positions, not just the difference.}, i.e. $k(\mathbf{x}+\mathbf{a},\mathbf{y}+\mathbf{a})=k(\mathbf{x},\mathbf{y})$, and isotropic. 

Let's examine Bochner’s theorem for stationary kernels. Bochner’s theorem .\ref{theorem:Bochners} says that a stationary kernel $k(\mathbf{x}_i - \mathbf{x}_j)$ is PSD \textit{iff} its Fourier transform is non-negative. Unfortunately, $\exp (- \frac{1}{(r^2 + c^2)})$ is not a Gaussian in $r$, and its Fourier transform is not necessarily non-negative. In fact, such reciprocal functions can induce negative components in the spectrum, violating PSD, as observed from the numerical counterexamples in Fig.\ref{fig:EIMQ_matrix_eigenvalues}, in which we empirically test positive semi-definiteness of the EIMQ kernel matrix on . 

\paragraph{Compare RBF, IMQ and EIMQ kernels} 

A comparison of the RBF, IMQ, and EIMQ kernel instances is made in Fig.\ref{fig:RBF_IMQ_EIMQ_compare}, from which we observe that, while the RBF kernel decays rapidly (Gaussian-like), focusing on local similarity, the IMQ kernel decays slower than the Gaussian (RBF), meaning it has longer-range influence (i.e. maintaining high similarity over long ranges). At short distances $r \to 0$, IMQ behaves like $(c^2)^{-\beta}$, giving a sharp repulsion if $\beta$ is large; at long distances $r \to \infty$, it decays as $r^{-2\beta}$, which is slower than exponential decay (like in the RBF). IMQ kernel is infinitely differentiable and very smooth, making it suitable for applications requiring smooth interpolation, it is used in particle-based variational inference (ParVI) and Stein variational gradient descent (SVGD \cite{Liu2016SVGD}), as it yields heavy tails, which encourages particles to spread out (avoid mode collapse), and it gives strong repulsive force (the gradient of IMQ) when particles are close \footnote{In SVGD \cite{Liu2016SVGD}, the update for particle $\mathbf{x}_i$ involves:
\[
\phi(\mathbf{x}_i) = \frac{1}{n} \sum_{j=1}^n \left[ k(\mathbf{x}_j, \mathbf{x}_i) \nabla_{\mathbf{x}_j} \log p(\mathbf{x}_j) + \nabla_{\mathbf{x}_j} k(\mathbf{x}_j, \mathbf{x}_i) \right]
\]
The IMQ kernel's gradient term gives strong repulsion:
\[
\nabla_{\mathbf{x}_j} k(\mathbf{x}_j, \mathbf{x}_i) = -2\beta (\mathbf{x}_j - \mathbf{x}_i) \left( \|\mathbf{x}_j - \mathbf{x}_i\|^2 + c^2 \right)^{-\beta - 1}
\]
This sharpens near $\mathbf{x}_i \approx \mathbf{x}_j$, preventing mode collapse in particle approximations.}. Still, similar to the RBF kernel, IMQ advocates for the high co-occurrence of similar items, which may not be ideal for mining complementary items. EIMQ kernel gives higher covariance for distant points, which is desired for complementary items in AR mining; however, the EIMQ kernel is not PSD and thus not valid for GP use.

\begin{figure}[H]
    \centering
    \includegraphics[width=0.40\linewidth]{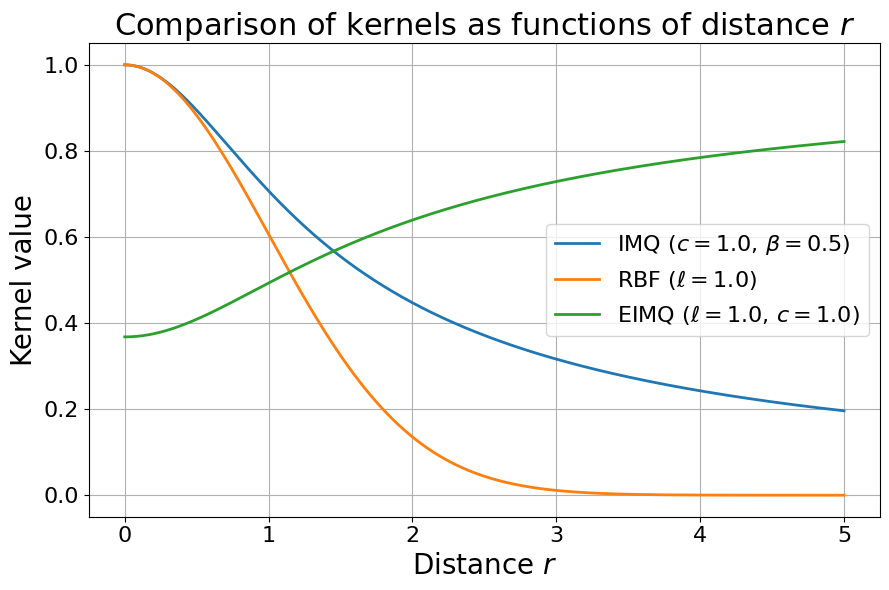}
    \caption{Comparison of RBF, IMQ, EIMQ examples.}
    \label{fig:RBF_IMQ_EIMQ_compare}
\end{figure}

\section{Neural tangent kernel: an analytic example} \label{app:NTK_3layer_NN}

Following our discussion of NTK in Section.\ref{sec:custom_kernel_design}, here we present a fully connected neural network with one hidden layer and ReLU activation, and derive the analytic form the NTK based on this architecture \footnote{Generalizing the NTK to convolutional neural nets, for example, yields the kernel convolutional neural tangent kernel (CNTK \cite{Yang2020NTK,Arora2019NTK}).}.

\paragraph{Network setup} Consider a simple 3-layer neural network with:
\begin{itemize}
    \item Input: $\mathbf{x} \in \mathbb{R}^d$
    \item Hidden layer: $m$ neurons with \textit{ReLU} activation, $\sigma(z) = \max(0, z)$
    \item Output: a single neuron with linear activation
    \item Parameters:
    \begin{itemize}
        \item Input-to-hidden weights: $\mathbf{w}_j \in \mathbb{R}^d$, initialized as $\mathbf{w}_j \sim \mathcal{N}(0, I_d)$ (Gaussian with zero mean and identity covariance)
        \item Hidden-to-output weights: $v_j \in \mathbb{R}$, initialized as $v_j \sim \mathcal{N}(0, 1)$
        \item Biases set to zero for simplicity
    \end{itemize}
\end{itemize}
We use this simple yet powerful perceptron as an example because an infinitely wide neural network with a single hidden layer is an universal approximator \cite{Hornik1993universal,Rasmussen2006GPbook}.
The output of the network is:
\[ f(\mathbf{x}) = \sum_{j=1}^m v_j \cdot \sigma(\mathbf{w}_j^\top \mathbf{x}) \]

The NTK is defined as the inner product of the gradients of the network output with respect to all parameters $\theta = \{\mathbf{w}_1, v_1, \dots, \mathbf{w}_m, v_m\}$:
\[ k(\mathbf{x}, \mathbf{x}') = \left\langle \frac{\partial f(\mathbf{x})}{\partial \theta}, \frac{\partial f(\mathbf{x}')}{\partial \theta} \right\rangle \]
In the infinite-width limit ($m \to \infty$), the NTK becomes its expectation over the random initialization:
\[ k^{NTK}(\mathbf{x}, \mathbf{x}') = \mathbb{E}\left[ k(\mathbf{x}, \mathbf{x}') \right] \]

The gradients depend on the parameters:
\begin{itemize}
    \item Gradient \textit{w.r.t.} $v_j$:
        \[ \frac{\partial f(\mathbf{x})}{\partial v_j} = \sigma(\mathbf{w}_j^\top \mathbf{x}) \]
    \item Gradient \textit{w.r.t.} $\mathbf{w}_j$:
          \[ \frac{\partial f(\mathbf{x})}{\partial \mathbf{w}_j} = v_j \cdot \mathbb{I}(\mathbf{w}_j^\top \mathbf{x} > 0) \cdot \mathbf{x} \]
          where $\mathbb{I}$ is the indicator function (1 if true, 0 otherwise). This indicator function results from the gradient of the RELU activation $\sigma(z) = \max(0, z)$.
\end{itemize}

For a finite network, the NTK is:
\[ k(\mathbf{x}, \mathbf{x}') = \sum_{j=1}^m \left( \sigma(\mathbf{w}_j^\top \mathbf{x}) \sigma(\mathbf{w}_j^\top \mathbf{x}') + v_j^2 \cdot \mathbb{I}(\mathbf{w}_j^\top \mathbf{x} > 0) \mathbb{I}(\mathbf{w}_j^\top \mathbf{x}' > 0) \cdot \mathbf{x}^\top \mathbf{x}' \right) \]

As $m \to \infty$, we compute the expectation over random initialization:
\[ k^{NTK}(\mathbf{x}, \mathbf{x}') = m \cdot \left( \mathbb{E}_{\mathbf{w}}\left[ \sigma(\mathbf{w}^\top \mathbf{x}) \sigma(\mathbf{w}^\top \mathbf{x}') \right] + \mathbb{E}_{\mathbf{w}, v}\left[ v^2 \cdot \mathbb{I}(\mathbf{w}^\top \mathbf{x} > 0) \mathbb{I}(\mathbf{w}^\top \mathbf{x}' > 0) \right] \cdot \mathbf{x}^\top \mathbf{x}' \right) \]
Since $v \sim \mathcal{N}(0, 1)$, $\mathbb{E}[v^2] = 1$, so the second term simplifies:
\[ \mathbb{E}_{\mathbf{w}, v}\left[ v^2 \cdot \mathbb{I}(\mathbf{w}^\top \mathbf{x} > 0) \mathbb{I}(\mathbf{w}^\top \mathbf{x}' > 0) \right] = \mathbb{E}_{\mathbf{w}}\left[ \mathbb{I}(\mathbf{w}^\top \mathbf{x} > 0) \mathbb{I}(\mathbf{w}^\top \mathbf{x}' > 0) \right] \]

Thus:
\begin{equation} \label{eq:NTK_3_layer}
    k^{NTK}(\mathbf{x}, \mathbf{x}') = m \cdot \left\{ \mathbb{E}_{\mathbf{w}}\left[ \sigma(\mathbf{w}^\top \mathbf{x}) \sigma(\mathbf{w}^\top \mathbf{x}') \right] + \mathbb{E}_{\mathbf{w}}\left[ \mathbb{I}(\mathbf{w}^\top \mathbf{x} > 0) \mathbb{I}(\mathbf{w}^\top \mathbf{x}' > 0) \right] \cdot \mathbf{x}^\top \mathbf{x}' \right\}
\end{equation}

In the following, we evaluate the two parts separately.

\subsection{The two correlated Gaussian variables $\mathbf{w}^\top \mathbf{x}$ and $\mathbf{w}^\top \mathbf{x}'$}

Since $\mathbf{w} \sim \mathcal{N}(0, I_d)$, $\mathbf{w}^\top \mathbf{x}$ and $\mathbf{w}^\top \mathbf{x}'$ are jointly Gaussian with zero mean and covariance $\mathbf{x}^\top \mathbf{x}'$. We can see this from the below: $\mathbf{w}$ is a $d$-dimensional random vector distributed as $\mathcal{N}(0, I_d)$, i.e. $\mathbf{w}$ follows a multivariate Gaussian (normal) distribution with mean $\mathbb{E}[\mathbf{w}] = \mathbf{0}$, a zero vector of length $d$, and covariance matrix $I_d$, the $d \times d$ identity matrix. Each component $w_i$ of $\mathbf{w}$ is independent, with mean 0 and variance 1 (since $\text{Var}(w_i) = \mathbb{E}[w_i^2] = 1$). $\mathbf{x}$ and $\mathbf{x}'$ are fixed (non-random) vectors in $\mathbb{R}^d$. $\mathbf{w}^\top \mathbf{x}$ and $\mathbf{w}^\top \mathbf{x}'$ are scalar random variables, representing the dot products of $\mathbf{w}$ with $\mathbf{x}$ and $\mathbf{x}'$, respectively.

A fundamental property of Gaussian random vectors is that linear transformations of Gaussians remain Gaussian. Since $\mathbf{w} \sim \mathcal{N}(0, I_d)$, any linear combination of its components is Gaussian. $\mathbf{w}^\top \mathbf{x} = w_1 x_1 + w_2 x_2 + \cdots + w_d x_d$ is a linear combination of the $w_i$’s, with coefficients $x_i$, same for $\mathbf{w}^\top \mathbf{x}' = w_1 x'_1 + w_2 x'_2 + \cdots + w_d x'_d$. As each $w_i \sim \mathcal{N}(0, 1)$ and the transformation is linear, both $\mathbf{w}^\top \mathbf{x}$ and $\mathbf{w}^\top \mathbf{x}'$ are Gaussian random variables. Moreover, since they are both derived from the same Gaussian vector $\mathbf{w}$, they are jointly Gaussian. This means the pair $[\mathbf{w}^\top \mathbf{x}, \mathbf{w}^\top \mathbf{x}']$ follows a bivariate Gaussian distribution. Let's examine its mean vector and covariance matrix.

For $\mathbf{w}^\top \mathbf{x}$, we have:
  \[
  \mathbb{E}[\mathbf{w}^\top \mathbf{x}] = \mathbb{E}\left[ \sum_{i=1}^d w_i x_i \right]
  \]
  Since expectation is linear, we can bring it inside the sum:
  \[
  \mathbb{E}[\mathbf{w}^\top \mathbf{x}] = \sum_{i=1}^d \mathbb{E}[w_i] x_i
  \]
  Given $\mathbb{E}[w_i] = 0$ (because $\mathbf{w} \sim \mathcal{N}(0, I_d)$), for all $i$ we have:
  \[
  \mathbb{E}[\mathbf{w}^\top \mathbf{x}] = \sum_{i=1}^d 0 \cdot x_i = 0
  \]

Similarly, for $\mathbf{w}^\top \mathbf{x}'$ we have:
  \[
  \mathbb{E}[\mathbf{w}^\top \mathbf{x}'] = \sum_{i=1}^d \mathbb{E}[w_i] x'_i = \sum_{i=1}^d 0 \cdot x'_i = 0
  \]

Therefore, both $\mathbf{w}^\top \mathbf{x}$ and $\mathbf{w}^\top \mathbf{x}'$ have mean 0. This makes sense intuitively: since $\mathbf{w}$ is centered at zero, projecting it onto any fixed direction (via dot product) keeps the result centered at zero. And any finite, linear combination of standard Gaussian random variables still has zero mean.

The covariance between $\mathbf{w}^\top \mathbf{x}$ and $\mathbf{w}^\top \mathbf{x}'$ hints how they vary together. By definition:
\[
\text{Cov}(\mathbf{w}^\top \mathbf{x}, \mathbf{w}^\top \mathbf{x}') = \mathbb{E}[(\mathbf{w}^\top \mathbf{x})(\mathbf{w}^\top \mathbf{x}')] - \mathbb{E}[\mathbf{w}^\top \mathbf{x}] \mathbb{E}[\mathbf{w}^\top \mathbf{x}']
\]
Since we’ve established that $\mathbb{E}[\mathbf{w}^\top \mathbf{x}] = 0$ and $\mathbb{E}[\mathbf{w}^\top \mathbf{x}'] = 0$, this simplifies to:
\[
\text{Cov}(\mathbf{w}^\top \mathbf{x}, \mathbf{w}^\top \mathbf{x}') = \mathbb{E}[(\mathbf{w}^\top \mathbf{x})(\mathbf{w}^\top \mathbf{x}')]
\]

Let’s compute that expectation:
\[
\mathbf{w}^\top \mathbf{x} = \sum_{i=1}^d w_i x_i, \quad \mathbf{w}^\top \mathbf{x}' = \sum_{j=1}^d w_j x'_j
\]
So:
\[
(\mathbf{w}^\top \mathbf{x})(\mathbf{w}^\top \mathbf{x}') = \left( \sum_{i=1}^d w_i x_i \right) \left( \sum_{j=1}^d w_j x'_j \right) = \sum_{i=1}^d \sum_{j=1}^d w_i w_j x_i x'_j
\]
Taking the expectation:
\[
\mathbb{E}[(\mathbf{w}^\top \mathbf{x})(\mathbf{w}^\top \mathbf{x}')] = \mathbb{E}\left[ \sum_{i=1}^d \sum_{j=1}^d w_i w_j x_i x'_j \right] = \sum_{i=1}^d \sum_{j=1}^d x_i x'_j \mathbb{E}[w_i w_j]
\]
Since $\mathbf{w} \sim \mathcal{N}(0, I_d)$, the components $w_i$ are independent and standard normal. If $i = j$, $\mathbb{E}[w_i w_j] = \mathbb{E}[w_i^2] = 1$; if $i \neq j$, $\mathbb{E}[w_i w_j] = \mathbb{E}[w_i] \mathbb{E}[w_j] = 0 \cdot 0 = 0$. That is, all off-diagonal elements in $I_d$ is zero. Thus, all cross-terms where $i \neq j$ vanish, and we’re left with:
\[
\sum_{i=1}^d \sum_{j=1}^d x_i x'_j \mathbb{E}[w_i w_j] = \sum_{i=1}^d x_i x'_i \cdot 1 = \mathbf{x}^\top \mathbf{x}'
\]
Therefore:
\[
\text{Cov}(\mathbf{w}^\top \mathbf{x}, \mathbf{w}^\top \mathbf{x}') = \mathbf{x}^\top \mathbf{x}'
\]
Therefore, the covariance between $\mathbf{w}^\top \mathbf{x}$ and $\mathbf{w}^\top \mathbf{x}'$ is the dot product $\mathbf{x}^\top \mathbf{x}'$, which measures the alignment between the vectors $\mathbf{x}$ and $\mathbf{x}'$.

Since $\mathbf{w}^\top \mathbf{x}$ and $\mathbf{w}^\top \mathbf{x}'$ are jointly Gaussian (as linear transformations of a Gaussian vector), we can fully specify their joint distribution using their means and covariance matrix:

\begin{itemize}
    \item means: $\mathbb{E}(\mathbf{w}^\top \mathbf{x}) = \mathbb{E}(\mathbf{w}^\top \mathbf{x}') = 0$
    \item Covariance matrix:
      \begin{itemize}
          \item variances: $\text{Var}(\mathbf{w}^\top \mathbf{x}) = \mathbb{E}[(\mathbf{w}^\top \mathbf{x})^2] = \sum_{i=1}^d x_i^2 \mathbb{E}[w_i^2] = \sum_{i=1}^d x_i^2 = \|\mathbf{x}\|^2$, $\text{Var}(\mathbf{w}^\top \mathbf{x}') = \|\mathbf{x}'\|^2$
          \item covariance: $\text{Cov}(\mathbf{w}^\top \mathbf{x}, \mathbf{w}^\top \mathbf{x}') = \mathbf{x}^\top \mathbf{x}'$
      \end{itemize}
\end{itemize}
Therefore, we have the joint binary Gaussian distribution:
\begin{equation} \label{eq:NTK_3_layers_1}
    \begin{bmatrix} \mathbf{w}^\top \mathbf{x} \\ \mathbf{w}^\top \mathbf{x}' \end{bmatrix} \sim \mathcal{N} \left( \begin{bmatrix} 0 \\ 0 \end{bmatrix}, \begin{bmatrix} \|\mathbf{x}\|^2 & \mathbf{x}^\top \mathbf{x}' \\ \mathbf{x}^\top \mathbf{x}' & \|\mathbf{x}'\|^2 \end{bmatrix} \right)
\end{equation}

\subsection{Evaluating $\mathbb{E}_{\mathbf{w}}\left[ \mathbb{I}(\mathbf{w}^\top \mathbf{x} > 0) \mathbb{I}(\mathbf{w}^\top \mathbf{x}' > 0) \right]$} 

\begin{tcolorbox}[label={result1}, colback=gray!5!white, colframe=black!75!white, 
                  title={Result 1 (intermediate, NTK of 3-layer MLP)}, breakable, enhanced jigsaw,
                  sharp corners=south, boxrule=0.8pt]

\begin{align*}
\mathbb{E}_{\mathbf{w}}\left[ \mathbb{I}(\mathbf{w}^\top \mathbf{x} > 0)\, \mathbb{I}(\mathbf{w}^\top \mathbf{x}' > 0) \right]
&= p(\mathbf{w}^\top \mathbf{x} > 0,\, \mathbf{w}^\top \mathbf{x}' > 0) \\
&= \frac{1}{4} + \frac{\arcsin(\rho)}{2\pi}
= \frac{1}{4} + \frac{\arcsin\left( \frac{\mathbf{x}^\top \mathbf{x}'}{\|\mathbf{x}\|\, \|\mathbf{x}'\|} \right)}{2\pi} \\
\text{with } \rho = \frac{\mathbf{x}^\top \mathbf{x}'}{\|\mathbf{x}\| \|\mathbf{x}'\|} = \cos\theta.
\end{align*}
\end{tcolorbox}
We prove this in the below.

For notational convenience, let's denote $z_1=\mathbf{w}^\top \mathbf{x}$ and $z_2=\mathbf{w}^\top \mathbf{x}'$, $\mu = \begin{bmatrix} 0 \\ 0 \end{bmatrix}$ and $\Sigma_{2 \times 2} = \begin{bmatrix} \|\mathbf{x}\|^2 & \mathbf{x}^\top \mathbf{x}' \\ \mathbf{x}^\top \mathbf{x}' & \|\mathbf{x}'\|^2 \end{bmatrix}$. The \textit{pdf} of the bivariate normal distribution is:
\[
p(z_1, z_2) = \frac{1}{2\pi \sqrt{|\Sigma|}} \exp\left( -\frac{1}{2}
\begin{bmatrix} z_1 & z_2 \end{bmatrix}
\Sigma^{-1}
\begin{bmatrix} z_1 \\ z_2 \end{bmatrix}
\right)
\]
The determinant
$
|\Sigma| = \|\mathbf{x}\|^2 \|\mathbf{x}'\|^2 - (\mathbf{x}^\top \mathbf{x}')^2
$
and inverse
$
\Sigma^{-1} =
\frac{1}{|\Sigma|} \begin{bmatrix}
\|\mathbf{x}'\|^2 & -\mathbf{x}^\top \mathbf{x}' \\
-\mathbf{x}^\top \mathbf{x}' & \|\mathbf{x}\|^2
\end{bmatrix}
$
Therefore, the \textit{pdf} becomes \footnote{The general form of the pdf of a bivariate Gaussian random vector  
\[
\mathbf{z} = \begin{bmatrix} z_1 \\ z_2 \end{bmatrix} \sim \mathcal{N}(\boldsymbol{\mu}, \Sigma)
\]
with mean vector $\boldsymbol{\mu} = \begin{bmatrix} \mu_1 \\ \mu_2 \end{bmatrix}$ and covariance matrix
\[
\Sigma =
\begin{bmatrix}
\sigma_1^2 & \rho \sigma_1 \sigma_2 \\
\rho \sigma_1 \sigma_2 & \sigma_2^2
\end{bmatrix}
\]
is given by:
\[
p(z_1, z_2) = \frac{1}{2\pi \sigma_1 \sigma_2 \sqrt{1 - \rho^2}} 
\exp\left( -\frac{1}{2(1 - \rho^2)} \left[
\left(\frac{z_1 - \mu_1}{\sigma_1}\right)^2 - 2\rho \left(\frac{z_1 - \mu_1}{\sigma_1}\right)\left(\frac{z_2 - \mu_2}{\sigma_2}\right)
+ \left(\frac{z_2 - \mu_2}{\sigma_2}\right)^2
\right] \right)
\]
in which $\mu_1, \mu_2$ are the means of $z_1$ and $z_2$, $\sigma_1^2, \sigma_2^2$ are the variances of $z_1$, $z_2$, $\rho \in [-1, 1]$ is the correlation coefficient between $z_1$ and $z_2$.
Special cases: if $\rho = 0$, the variables are uncorrelated, and the pdf factorizes, i.e. $p(z_1, z_2) = p(z_1) \cdot p(z_2)$. If $\rho = \pm 1$, the distribution is degenerate (the variables are perfectly linearly dependent).}:
\begin{equation} \label{eq:NTK_3_layers_2}
    p(z_1, z_2) =
    \frac{1}{2\pi \sqrt{ \|\mathbf{x}\|^2 \|\mathbf{x}'\|^2 - (\mathbf{x}^\top \mathbf{x}')^2 }}
    \exp\left( -\frac{1}{2}
    \frac{
    \|\mathbf{x}'\|^2 z_1^2 - 2 (\mathbf{x}^\top \mathbf{x}') z_1 z_2 + \|\mathbf{x}\|^2 z_2^2
    }{
    \|\mathbf{x}\|^2 \|\mathbf{x}'\|^2 - (\mathbf{x}^\top \mathbf{x}')^2
    }
    \right)
\end{equation}

Now we aim to evaluate $p(\mathbf{w}^\top \mathbf{x} > 0, \mathbf{w}^\top \mathbf{x}' > 0) > 0$, given two general bivariate Gaussian variables $z_1=\mathbf{w}^\top \mathbf{x}$ and $z_2=\mathbf{w}^\top \mathbf{x}'$. This can be done by evaluating a standard bivariate Gaussian. We define the standard normal variables: $\tilde{z}_1 = \frac{z_1}{\|\mathbf{x}\|}, \quad \tilde{z}_2 = \frac{z_2}{\|\mathbf{x}'\|}$, where $\tilde{z}_1$ and $\tilde{z}_2$ is standard normal (mean zero, unit variance). This helps transform the joint distribution into a standard bivariate normal:
\begin{equation} \label{eq:standard_bivariate_normal}
    \begin{bmatrix}
    \tilde{z}_1 \\
    \tilde{z}_2
    \end{bmatrix}
    \sim \mathcal{N}\left( \mathbf{0},
    \begin{bmatrix}
    1 & \rho \\
    \rho & 1
    \end{bmatrix}
    \right),
    \quad \text{with } \rho = \text{Corr}(\tilde{z}_1, \tilde{z}_2) = \frac{\text{Cov}(\tilde{z}_1, \tilde{z}_2)}{\sqrt{\text{Var}(\tilde{z}_1)\text{Var}(\tilde{z}_2)}} = \frac{\mathbf{x}^\top \mathbf{x}'}{\|\mathbf{x}\| \|\mathbf{x}'\|} = \cos\theta
\end{equation}
So the joint probability becomes:
\begin{equation} \label{eq:standard_bivariate_Gaussian_quadrant_integral1}
    p(\mathbf{w}^\top \mathbf{x} > 0, \mathbf{w}^\top \mathbf{x}' > 0) = p(z_1>0,z_2>0) = p(\tilde{z}_1 > 0, \tilde{z}_2 > 0)
\end{equation}
The geometric intuition for $p(\tilde{z}_1 > 0, \tilde{z}_2 > 0)$ being, the event $\tilde{z}_1 > 0, \tilde{z}_2 > 0$ means the vector $(\tilde{z}_1, \tilde{z}_2)$ lies in the positive quadrant of $\mathbb{R}^2$. Without correlation ($\rho = 0$), each quadrant has equal probability of $\frac{1}{4}$. A positive correlation ($\rho > 0$) increases the probability of joint positivity, and a negative correlation ($\rho < 0$) reduces this joint probability below $\frac{1}{4}$.

For two correlated standard normal variables $\tilde{z}_1$ and $\tilde{z}_2$, the bivariate normal density is:
\[
f(\tilde{z}_1, \tilde{z}_2) = \frac{1}{2\pi\sqrt{1-\rho^2}}\exp\left(-\frac{\tilde{z}_1^2 - 2\rho \tilde{z}_1 \tilde{z}_2 + \tilde{z}_2^2}{2(1-\rho^2)}\right).
\]
The integral for the probability $p(\tilde{z}_1>0, \tilde{z}_2>0)$ is:
\[
p(\tilde{z}_1>0, \tilde{z}_2>0) = \int_0^\infty \int_0^\infty f(\tilde{z}_1,\tilde{z}_2)\, d\tilde{z}_1 d\tilde{z}_2
\]
Evaluating this integral exactly gives \footnote{We aim to evaluate the a quadrant-restricted double integral of a quadratic exponential form $\int_0^\infty \int_0^\infty \exp(ax^2 + bxy + cy^2) dxdy$. Specifically, we have the following integral \cite{Structural2017}:
\[
L(h, k, \rho) \equiv \frac{1}{2\pi}(1 - \rho^2)^{-1/2} \int_h^\infty \int_k^\infty 
\exp\left[ \frac{-\frac{1}{2}(x^2 + y^2 - 2\rho xy)}{1 - \rho^2} \right] dxdy
\], which can be reduced to \cite{Sheppard1900Quadrature}:
\[
L(h, k, \rho) = \frac{1}{2\pi} \int_{\arccos \rho}^{\pi} 
\exp\left[ -\frac{\frac{1}{2}(h^2 + k^2 - 2hk \cos \theta)}{\sin^2 \theta} \right] d\theta, \quad h, k \geq 0
\]
where $\rho=cos \theta$. Taking $h=k=0$, we have $L(0, 0, \rho)=\frac{1}{2} - \frac{arccos(\rho)}{2 \pi}$. Using the relation between inverse trigonometric functions $arccos (\rho) = \frac{\pi}{2} - arcsin (\rho)$, we arrive at: $L(0, 0, \rho)=\frac{1}{4}+\frac{1}{2\pi} arcsin (\rho)$.
} the closed-form solution:
\begin{equation} \label{eq:standard_bivariate_Gaussian_quadrant_integral2}
    \int_0^\infty \int_0^\infty f(\tilde{z}_1,\tilde{z}_2)\, d\tilde{z}_1 d\tilde{z}_2 = \frac{1}{4} + \frac{\arcsin(\rho)}{2\pi}
\end{equation}
For example, $\rho = 0$ (independent $\tilde{z}_1,\tilde{z}_2$):  
\[
p(\tilde{z}_1>0,\tilde{z}_2>0) = \frac{1}{4} + \frac{\arcsin(0)}{2\pi} = \frac{1}{4}
\]
$\rho = 1$ (perfect positive correlation between $\tilde{z}_1,\tilde{z}_2$):  
\[
p(\tilde{z}_1>0,\tilde{z}_2>0) = \frac{1}{4} + \frac{\arcsin(1)}{2\pi} = \frac{1}{4} + \frac{\pi/2}{2\pi} = \frac{1}{4} + \frac{1}{4} = \frac{1}{2}
\]
$\rho = -1$ (perfect negative correlation between $\tilde{z}_1,\tilde{z}_2$):  
\[
p(\tilde{z}_1>0,\tilde{z}_2>0) = \frac{1}{4} + \frac{\arcsin(-1)}{2\pi} = \frac{1}{4} - \frac{\pi/2}{2\pi} = \frac{1}{4} - \frac{1}{4} = 0
\]

Now we go back to the NTK Eq.\ref{eq:NTK_3_layer}:
\[
    k^{NTK}(\mathbf{x}, \mathbf{x}') = m \cdot \left( \mathbb{E}_{\mathbf{w}}\left[ \sigma(\mathbf{w}^\top \mathbf{x}) \sigma(\mathbf{w}^\top \mathbf{x}') \right] + \mathbb{E}_{\mathbf{w}}\left[ \mathbb{I}(\mathbf{w}^\top \mathbf{x} > 0) \mathbb{I}(\mathbf{w}^\top \mathbf{x}' > 0) \right] \cdot \mathbf{x}^\top \mathbf{x}' \right)
\]
Expectation of the indicator terms reduces to:
\[ \mathbb{E}_{\mathbf{w}}\left[ \mathbb{I}(\mathbf{w}^\top \mathbf{x} > 0) \mathbb{I}(\mathbf{w}^\top \mathbf{x}' > 0) \right] = p(\mathbf{w}^\top \mathbf{x} > 0, \mathbf{w}^\top \mathbf{x}' > 0) \]
For two correlated Gaussians $\mathbf{w}^\top \mathbf{x}$ and $\mathbf{w}^\top \mathbf{x}'$, as in Eq.\ref{eq:NTK_3_layers_1}, given its \textit{pdf} Eq.\ref{eq:NTK_3_layers_2}, we have already known the integral from Eq.\ref{eq:standard_bivariate_Gaussian_quadrant_integral1} and Eq.\ref{eq:standard_bivariate_Gaussian_quadrant_integral2} that:
\begin{equation} \label{eq:standard_bivariate_Gaussian_quadrant_integral3}
    p(\mathbf{w}^\top \mathbf{x} > 0, \mathbf{w}^\top \mathbf{x}' > 0) = \frac{1}{4} + \frac{\arcsin(\rho)}{2\pi}
    = \frac{1}{4} + \frac{\arcsin(\frac{\mathbf{x}^\top \mathbf{x}'}{\|\mathbf{x}\| \|\mathbf{x}'\|})}{2\pi} 
\end{equation}

Next, we proceed to calculate the first term $\mathbb{E}_{\mathbf{w}}\left[ \sigma(\mathbf{w}^\top \mathbf{x}) \sigma(\mathbf{w}^\top \mathbf{x}') \right]$.

\subsection{Evaluating $\mathbb{E}_{\mathbf{w}}\left[ \sigma(\mathbf{w}^\top \mathbf{x}) \sigma(\mathbf{w}^\top \mathbf{x}') \right]$}

\begin{tcolorbox}[label={result2}, colback=gray!5!white, colframe=black!75!white, 
                  title={Result 2 (intermediate, NTK of 3-layer MLP)}, breakable, enhanced jigsaw,
                  sharp corners=south, boxrule=0.8pt]
\begin{equation*}
\mathbb{E}[\sigma(\mathbf{w}^\top \mathbf{x}) \sigma(\mathbf{w}^\top \mathbf{x}')] = \frac{\|\mathbf{x}\| \|\mathbf{x}'\|}{2\pi} \left( \sqrt{1 - \rho^2} + \rho \arccos(-\rho) \right)
\end{equation*}
\end{tcolorbox}
We prove this in the below.

We aim to compute the expectation $\mathbb{E}_{\mathbf{w}}\left[ \sigma(\mathbf{w}^\top \mathbf{x}) \sigma(\mathbf{w}^\top \mathbf{x}') \right]$ for ReLU activation, where $\sigma(z) = \max(0, z)$, and $\mathbf{w}^\top \mathbf{x}$ and $\mathbf{w}^\top \mathbf{x}'$ are correlated Gaussian random variables.$\mathbf{w} \sim \mathcal{N}(0, I_d)$ is a $d$-dimensional Gaussian random vector with mean 0 and identity covariance matrix $I_d$, meaning each component $w_i \sim \mathcal{N}(0, 1)$ and the components are independent. $\mathbf{x}, \mathbf{x}' \in \mathbb{R}^d$ are fixed (non-random) vectors. $\mathbf{w}^\top \mathbf{x}$ and $\mathbf{w}^\top \mathbf{x}'$ are scalar random variables, representing dot products. $\sigma(z) = \max(0, z)$ is the ReLU activation function, which outputs $z$ if $z > 0$ and 0 otherwise.

With $z_1 = \mathbf{w}^\top \mathbf{x}$ and $z_2 = \mathbf{w}^\top \mathbf{x}'$, $\tilde{z}_1 = \frac{z_1}{\|\mathbf{x}\|} = \frac{\mathbf{w}^\top \mathbf{x}}{\|\mathbf{x}\|}$ and $\tilde{z}_2 = \frac{z_2}{\|\mathbf{x}'\|} = \frac{\mathbf{w}^\top \mathbf{x}'}{\|\mathbf{x}'\|}$, we express the expectation in terms of $\tilde{z}_1$ and $\tilde{z}_2$:
\begin{align*}
    &\sigma(z_1) = \max(0, z_1) = \max(0, \|\mathbf{x}\| \tilde{z}_1) = \|\mathbf{x}\| \max(0, \tilde{z}_1) = \|\mathbf{x}\| \sigma(\tilde{z}_1) \\
    &\sigma(z_2) = \|\mathbf{x}'\| \sigma(\tilde{z}_2)
\end{align*}
Therefore:
\begin{equation} \label{eq:unnormalised_expectation}
    \mathbb{E}[\sigma(z_1) \sigma(z_2)] = \mathbb{E}[(\|\mathbf{x}\| \sigma(\tilde{z}_1)) (\|\mathbf{x}'\| \sigma(\tilde{z}_2))] = \|\mathbf{x}\| \|\mathbf{x}'\| \mathbb{E}[\sigma(\tilde{z}_1) \sigma(\tilde{z}_2)]
\end{equation}
Our task reduces to computing $\mathbb{E}[\sigma(\tilde{z}_1) \sigma(\tilde{z}_2)]$, where $\tilde{z}_1$ and $\tilde{z}_2$ are standard Gaussians with correlation $\rho$.

Define $U = \tilde{z}_1$, $V = \tilde{z}_2$, where $[U, V] \sim \mathcal{N}\left( \mathbf{0}, \begin{bmatrix} 1 & \rho \\ \rho & 1 \end{bmatrix} \right)$, with $\rho=Corr(U,V)= \frac{Cov(U,V)}{\|U\| \|V\|}$.
We want to calculate:
\begin{equation} \label{eq:RELU_expectation_double_integral1}
  \mathbb{E}[\sigma(U) \sigma(V)] = \mathbb{E}[\max(0, U) \max(0, V)]
\end{equation}
This expectation is non-zero only when both $U > 0$ and $V > 0$, so we compute the integral over the first quadrant:
\begin{equation}  \label{eq:RELU_expectation_double_integral2}
    \mathbb{E}[\sigma(U) \sigma(V)] = \int_{0}^\infty \int_{0}^\infty u v \cdot p(u, v) \, du \, dv
\end{equation}
where $p(u, v)$ is the probability density function of the bivariate Gaussian:
\begin{equation} \label{eq:pdf_standard_bivariate_Gaussian}
p(u, v) = \frac{1}{2\pi \sqrt{1 - \rho^2}} \exp\left( -\frac{u^2 - 2\rho u v + v^2}{2(1 - \rho^2)} \right)
\end{equation}
In order to evaluate the double integral in Eq.\ref{eq:RELU_expectation_double_integral2}, we switch to polar coordinates. Let $u = r \cos \alpha$, $v = r \sin \alpha$, where $r \geq 0$, $\alpha \in [0, \pi/2]$ (first quadrant). The Jacobian of the transformation is $r$, so $du dv = r dr d\alpha$. The limits are: $r$ from 0 to $\infty$, $\alpha$ from 0 to $\pi/2$.          

Rewrite the exponent: $u^2 + v^2 = r^2 \cos^2 \alpha + r^2 \sin^2 \alpha = r^2$, $uv = (r \cos \alpha)(r \sin \alpha) = r^2 \cos \alpha \sin \alpha = \frac{r^2}{2} \sin 2\alpha$, $u^2 - 2\rho u v + v^2 = r^2 \cos^2 \alpha - 2\rho (r^2 \cos \alpha \sin \alpha) + r^2 \sin^2 \alpha = r^2 (1 - 2\rho \cos \alpha \sin \alpha) = r^2 (1 - \rho \sin 2\alpha)$.
Then the density Eq.\ref{eq:pdf_standard_bivariate_Gaussian} becomes:
\[
p(r \cos \alpha, r \sin \alpha) = \frac{1}{2\pi \sqrt{1 - \rho^2}} \exp\left( -\frac{r^2 (1 - \rho \sin 2\alpha)}{2(1 - \rho^2)} \right)
\]
The expectation Eq.\ref{eq:RELU_expectation_double_integral2} transforms to:
\[
\mathbb{E}[\sigma(U) \sigma(V)] = \int_{0}^{\pi/2} \int_{0}^\infty \left( \frac{r^2}{2} \sin 2\alpha \right) \cdot \frac{1}{2\pi \sqrt{1 - \rho^2}} \exp\left( -\frac{r^2 (1 - \rho \sin 2\alpha)}{2(1 - \rho^2)} \right) \cdot r \, dr \, d\alpha
\]
We separate the integrals over $r$ and $\alpha$:
\begin{equation} \label{eq:RELU_expectation_double_integral3}
  \mathbb{E}[\sigma(U) \sigma(V)] = \frac{1}{2\pi \sqrt{1 - \rho^2}} \int_{0}^{\pi/2} \frac{\sin 2\alpha}{2} d\alpha \cdot 
  \int_{0}^\infty r^3 \cdot  \exp\left( -\frac{r^2 (1 - \rho \sin 2\alpha)}{2(1 - \rho^2)} \right) dr
\end{equation}

\paragraph{Radial integral}
Let $a = \frac{1 - \rho \sin 2\alpha}{2(1 - \rho^2)}$, so the exponent is $-r^2 a$, and the term involves $r^3$:
  \[
  \int_{0}^\infty r^3 \exp(-a r^2) \, dr
  \]
Use the substitution $s = r^2$, $r = \sqrt{s}$, $dr = \frac{1}{2} s^{-1/2} ds$, so $r^3 = s^{3/2}$:
  \[
  \int_{0}^\infty r^3 \exp(-a r^2) \, dr = \int_{0}^\infty s^{3/2} \exp(-a s) \cdot \frac{1}{2} s^{-1/2} ds = \frac{1}{2} \int_{0}^\infty s \exp(-a s) \, ds
  \]
which is a \textit{Gamma} integral:
  \[
  \int_{0}^\infty s \exp(-a s) ds = \frac{1}{a^2}
  \]
So:
  \[
  \int_{0}^\infty r^3 \exp(-a r^2) \, dr = \frac{1}{2 a^2}, \quad \text{ with } a = \frac{1 - \rho \sin 2\alpha}{2(1 - \rho^2)}
  \]

\paragraph{Angular integral} 
The expectation Eq.\ref{eq:RELU_expectation_double_integral3} now becomes:
\begin{equation}
\begin{aligned}
    \mathbb{E}[\sigma(U) \sigma(V)] &= \frac{1}{2\pi \sqrt{1 - \rho^2}} \int_{0}^{\pi/2} \frac{\sin 2\alpha}{2} d\alpha \cdot 
    \int_{0}^\infty r^3 \cdot  \exp\left( -\frac{r^2 (1 - \rho \sin 2\alpha)}{2(1 - \rho^2)} \right) dr \\
    &= \frac{1}{2\pi \sqrt{1 - \rho^2}} \int_{0}^{\pi/2} \frac{\sin 2\alpha}{2} \cdot \left( \frac{1}{2 a^2} \right) \, d\alpha
\end{aligned}
\end{equation}
Substitute $a$:
\[
a = \frac{1 - \rho \sin 2\alpha}{2(1 - \rho^2)}, \quad a^2 = \left( \frac{1 - \rho \sin 2\alpha}{2(1 - \rho^2)} \right)^2
,
\frac{1}{a^2} = \frac{4 (1 - \rho^2)^2}{(1 - \rho \sin 2\alpha)^2}
\]
So the integral is:
\begin{align*}
\mathbb{E}[\sigma(U) \sigma(V)] 
&= \frac{1}{2\pi \sqrt{1 - \rho^2}} \cdot \frac{1}{2} \cdot 4 (1 - \rho^2)^2 \int_{0}^{\pi/2} \frac{\sin 2\alpha}{(1 - \rho \sin 2\alpha)^2} \, d\alpha \cdot \frac{1}{2} \\
&= \frac{1}{2\pi \sqrt{1 - \rho^2}} \cdot (1 - \rho^2)^2 \int_{0}^{\pi/2} \frac{\sin 2\alpha}{(1 - \rho \sin 2\alpha)^2} \, d\alpha
\end{align*}

For the trigonometric integral, we directly present the following result:
\[
\int_0^{\pi/2} \frac{\sin 2\alpha}{(1 - \rho \sin 2\alpha)^2} d\alpha = \frac{2\sqrt{1 - \rho^2} + \rho (\pi + 2 \arcsin \rho)}{2(1 - \rho^2)^{3/2}}
, \quad |\rho| < 1
\],
the derivation of which is presented in Appendix.\ref{app:inverse_trigonometric_integral}. Using this, we arrive at:
\[
\mathbb{E}[\sigma(U) \sigma(V)] = \frac{1}{2\pi} \left( \sqrt{1 - \rho^2} + \rho (\frac{\pi}{2} + \arcsin \rho) \right)
\]
Or equivalently \footnote{Using the complementary angle identity $\arcsin x + \arccos x = \frac{\pi}{2}$, we can write $\arccos(-\rho) = \frac{\pi}{2} - \arcsin(-\rho)$. As $\arcsin(-x) = -\arcsin(x)$, we have $\arccos(-\rho) = \frac{\pi}{2} - (-\arcsin \rho) = \frac{\pi}{2} + \arcsin \rho$.}:
\[
\mathbb{E}[\sigma(U) \sigma(V)] = \frac{1}{2\pi} \left( \sqrt{1 - \rho^2} + \rho \arccos(-\rho) \right)
\]
which is $\mathbb{E}[\sigma(\tilde{z}_1) \sigma(\tilde{z}_2)]$. To obtain $\mathbb{E}[\sigma(z_1) \sigma(z_2)]$, we multiply the above by the scaling factors, as per Eq.\ref{eq:unnormalised_expectation}:
\[
\mathbb{E}[\sigma(z_1) \sigma(z_2)] = \|\mathbf{x}\| \|\mathbf{x}'\| \cdot \mathbb{E}[\sigma(\tilde{z}_1) \sigma(\tilde{z}_2)] = \frac{\|\mathbf{x}\| \|\mathbf{x}'\|}{2\pi} \left( \sqrt{1 - \rho^2} + \rho \arccos(-\rho) \right)
\]
which is \textit{Result 2}.

\paragraph{Final NTK covariance function for the 3-layer network}

Combining .\ref{result1} and .\ref{result2}, the NTK in Eq.\ref{eq:NTK_3_layer} becomes:
\begin{equation}  \label{eq:NN_3layer_NTK}
    k^{NTK}(\mathbf{x}, \mathbf{x}') = m \cdot \left\{ \frac{\|\mathbf{x}\| \|\mathbf{x}'\|}{2\pi} \left( \sqrt{1 - \rho^2} + \rho \arccos(-\rho) \right) + \left( \frac{1}{4} + \frac{\arcsin(\rho)}{2\pi} \right) \mathbf{x}^\top \mathbf{x}' \right\}
\end{equation}
where $\rho = \frac{\mathbf{x}^\top \mathbf{x}'}{\|\mathbf{x}\| \|\mathbf{x}'\|}$.

This symmetric form captures how the network behaves like a kernel method in the infinite-width limit, with the ReLU activation introducing non-linearity into the kernel.

\paragraph{Intuition of NTK} In essence, the NTK’s inner product of weight sensitivity vectors, i.e. $k(\mathbf{x},\mathbf{x}')=\langle\nabla_w f(\mathbf{x};w),\nabla_w f(\mathbf{x}';w)\rangle$, builds a data‑dependent kernel that quantifies how similarly two inputs would respond to an infinitesimal weight update. As gradients capture the local 'geometry' of the model’s function space independently of any labels, $k(\mathbf{x},\mathbf{x}')$ encodes a data‐dependent, label‐agnostic notion of similarity: examples with aligned sensitivities move the network output in concert under gradient descent. During training, this pre‐existing geometric structure governs how errors at one point influence predictions elsewhere, thus determining both the speed and the pattern of learning even before any ground‐truth labels are observed.

\section{Derivation of $\int_0^{\pi/2} \frac{\sin 2\alpha}{(1 - \rho \sin 2\alpha)^2} d\alpha$} \label{app:inverse_trigonometric_integral}

\subsection{Prerequisite 1: integral of reciprocal quadratic function}
We derive the following integral which involves quadratic expressions in the denominator:

\begin{tcolorbox}[colback=gray!5!white, colframe=black!75!white, 
                  title={Result 3 (general)}, breakable, enhanced jigsaw,
                  sharp corners=south, boxrule=0.7pt, width=\textwidth]
\small 
\begin{equation*}
\int \frac{1}{A t^2 + B t + C} dt =
\frac{2}{\sqrt{4AC - B^2}} \arctan\left( \frac{2A t + B}{\sqrt{4AC - B^2}} \right) + C,
\quad \text{if } 4AC - B^2 > 0
\end{equation*}
\end{tcolorbox}

\begin{proof}
We want to evaluate the integral:
$$\int \frac{1}{At^2 + Bt + C} dt, \quad \text{where } 4AC - B^2 > 0$$

First, we complete the square in the denominator. We factor out $A$ from the first two terms:
$$A\left(t^2 + \frac{B}{A}t\right) + C$$
Then complete the square for the expression inside the parentheses:
$$A\left(\left(t + \frac{B}{2A}\right)^2 - \frac{B^2}{4A^2}\right) + C$$

The integral becomes:
$$\int \frac{1}{A\left(t + \frac{B}{2A}\right)^2 + \frac{4AC - B^2}{4A}} dt = \frac{1}{A} \int \frac{1}{\left(t + \frac{B}{2A}\right)^2 + \frac{4AC - B^2}{4A^2}} dt$$

Let $u = t + \frac{B}{2A}$, so $du = dt$. Let $a^2 = \frac{4AC - B^2}{4A^2}$, so $a = \frac{\sqrt{4AC - B^2}}{2|A|}$. Assuming $A > 0$, $a = \frac{\sqrt{4AC - B^2}}{2A}$. The integral becomes:
$$\frac{1}{A} \int \frac{1}{u^2 + a^2} du$$

We know the standard integral:
$$
\int \frac{1}{x^2 + a^2} dx = \frac{1}{a} \arctan\left(\frac{x}{a}\right) + C
$$
which gives:
$$\frac{1}{A} \int \frac{1}{u^2 + a^2} du = \frac{1}{A} \cdot \frac{1}{a} \arctan\left(\frac{u}{a}\right) + C$$

Substitute back $u$ and $a$:
\begin{align*}
\frac{1}{A} \cdot \frac{2A}{\sqrt{4AC - B^2}} \arctan\left(\frac{t + \frac{B}{2A}}{\frac{\sqrt{4AC - B^2}}{2A}}\right) + C
&=\frac{2}{\sqrt{4AC - B^2}} \arctan\left(\frac{\frac{2At + B}{2A}}{\frac{\sqrt{4AC - B^2}}{2A}}\right) + C \\
&=\frac{2}{\sqrt{4AC - B^2}} \arctan\left(\frac{2At + B}{\sqrt{4AC - B^2}}\right) + C
\end{align*}    
\end{proof}

\subsection{Prerequisite 2: integral of reciprocal trigonometric function}
We derive the following result:

\begin{tcolorbox}[colback=gray!5!white, colframe=black!75!white, 
                  title={Result 4 (general)}, breakable, enhanced jigsaw,
                  sharp corners=south, boxrule=0.8pt]

\begin{equation*}
\int_0^{\pi} \frac{1}{1 - \rho \sin \theta} d\theta = \frac{\pi + 2 \arcsin \rho}{\sqrt{1 - \rho^2}} \quad |\rho| < 1
\end{equation*}
\end{tcolorbox}

\begin{proof}
We use the substitution: $t = \tan \frac{x}{2}$, then $x=2$ arctan $t$, and 
\[
dx = \frac{2}{1 + t^2} dt
\]
Limits of integration:
\begin{align*}
x = 0 &\Rightarrow t = 0 \\
x = \pi &\Rightarrow t = \tan\left( \frac{\pi}{2} \right) \to \infty
\end{align*}

We know the trigonometric identity:
\[
\sin x = \frac{2 \tan \frac{x}{2}}{1 + \tan^2 \frac{x}{2}}
= \frac{2t}{1 + t^2}
\]

Substitute $\sin x = \frac{2t}{1 + t^2}$ and $dx = \frac{2}{1 + t^2} dt$:
\begin{align*}
\int_0^{\pi} \frac{1}{1 - \rho \sin x} \, dx
&= \int_0^{\infty} \frac{1}{1 - \rho \cdot \frac{2t}{1 + t^2}} \cdot \frac{2}{1 + t^2} dt \\
&= \int_0^{\infty} \frac{2}{1 + t^2 - 2\rho t} \, dt \\
&= 2 \int_0^{\infty} \frac{1}{t^2 - 2\rho t + 1} \, dt
\end{align*}

Completing the square:
$
t^2 - 2\rho t + 1 = (t - \rho)^2 + 1 - \rho^2
$, let $u = t - \rho$, so $du = dt$. We have:
\begin{align*}
2 \int_0^{\infty} \frac{1}{(t - \rho)^2 + 1 - \rho^2} dt
&= 2 \int_{-\rho}^{\infty} \frac{1}{u^2 + (1 - \rho^2)} \, du \\
&= \frac{2}{1 - \rho^2} \int_{-\rho}^{\infty} \frac{1}{\left( \frac{u}{\sqrt{1 - \rho^2}} \right)^2 + 1} \, du \\
&= \frac{2}{\sqrt{1 - \rho^2}} \left[ \arctan\left( \frac{u}{\sqrt{1 - \rho^2}} \right) \right]_{-\rho}^{\infty}
\end{align*}

\[
= \frac{2}{\sqrt{1 - \rho^2}} \left( \frac{\pi}{2} - \arctan\left( \frac{-\rho}{\sqrt{1 - \rho^2}} \right) \right)
= \frac{2}{\sqrt{1 - \rho^2}} \left( \frac{\pi}{2} + \arctan\left( \frac{\rho}{\sqrt{1 - \rho^2}} \right) \right)
\]

Using $\arcsin \rho = \arctan \left( \frac{\rho}{\sqrt{1 - \rho^2}} \right)$, the result becomes:

\begin{equation*}
\int_0^{\pi} \frac{1}{1 - \rho \sin x} \, dx
= \frac{2}{\sqrt{1 - \rho^2}} \left( \frac{\pi}{2} + \arcsin \rho \right) 
= \frac{\pi + 2 \arcsin \rho}{\sqrt{1 - \rho^2}}
\end{equation*}
\end{proof}

\subsection{Final: $\int_0^{\pi/2} \frac{\sin 2\alpha}{(1 - \rho \sin 2\alpha)^2} d\alpha$}
The method of \textit{differentiation under the integral sign} hints: 
\[
\frac{d}{d\rho} \int_0^{\pi} \frac{1}{1 - \rho \sin \theta} d\theta = \int_0^{\pi} \frac{\sin \theta}{(1 - \rho \sin \theta)^2} d\theta
\]

We have already in \textit{Result 4} the integral on the LHS $\int_0^{\pi} \frac{1}{1 - \rho \sin \theta} d\theta = \frac{\pi + 2 \arcsin \rho}{\sqrt{1 - \rho^2}}$ $(|\rho| < 1)$, therefore, we only need to differentiate it \textit{w.r.t.} $\rho$:
\[
\int_0^{\pi} \frac{\sin \theta}{(1 - \rho \sin \theta)^2} d\theta
=
\frac{d}{d\rho} \left( \frac{\pi + 2 \arcsin \rho}{\sqrt{1 - \rho^2}} \right) 
= \frac{2\sqrt{1 - \rho^2} + \rho (\pi + 2 \arcsin \rho)}{(1 - \rho^2)^{3/2}}, 
\quad |\rho| < 1
\]

To evaluate the target integral $\int_0^{\pi/2} \frac{\sin 2\alpha}{(1 - \rho \sin 2\alpha)^2} d\alpha$, we make the change of variable:
$$
\theta = 2\alpha \Rightarrow d\theta = 2d\alpha,\quad \alpha \in [0, \pi/2] \Rightarrow \theta \in [0, \pi]
$$
We have therefore:
\[
\int_0^{\pi/2} \frac{\sin 2\alpha}{(1 - \rho \sin 2\alpha)^2} d\alpha = \frac{1}{2} \int_0^\pi \frac{\sin \theta}{(1 - \rho \sin \theta)^2} d\theta
= \frac{2\sqrt{1 - \rho^2} + \rho (\pi + 2 \arcsin \rho)}{2(1 - \rho^2)^{3/2}}
, \quad |\rho| < 1
\]

\section{Infinite‑width neural network as NTK regression} \label{app:NN_equivalent_to_NTK}

We use the following notations:
training data $\{(x_i, y_i)\}_{i=1}^n \subset \mathbb{R}^d \times \mathbb{R}$, $K \in \mathbb{R}^{n \times n}$ is the covariance matrix produced by the NTK $\Theta^{(L)}(x_i, x_j)$ evaluated on the training data, i.e. $K_{i, j} = \Theta^{(L)}(x_i, x_j)$. For a testing point $x \in \mathbb{R}^d$, we let $k(x, X) \in \mathbb{R}^n$ be the kernel evaluated between the testing point and the $n$ training points, i.e. $[k(x, X)]_i = \Theta^{(L)}(x, x_i)$. The neural network prediction over the training set at time $t$ is denoted by $\mathbf{u}(t)\;=\;\bigl[f(x_1;\theta(t)),\dots,f(x_n;\theta(t))\bigr]^\top$.

In the limit where every hidden layer’s width tends to infinity (with weights initialized \textit{i.i.d.} and properly scaled), or when the network is sufficiently wide, the network’s neural tangent kernel (NTK) converges to a deterministic, positive–definite kernel $K_\infty$ and remains fixed throughout training.  Under squared‐loss gradient flow, the network’s function $f_t$ then evolves exactly according to the kernel‐gradient dynamics in the reproducing kernel Hilbert space (RKHS) induced by $K_\infty$.  As a result, both the trajectory $f_t$ and its limit $f_\infty$ coincide with those of kernel regression using $K_\infty$. These facts are stated as below.

\begin{lemma}[NTK regression dynamics under gradient flow; Jacot \emph{et al.} \cite{Jacot2020NTK} 2018, \& Arora \emph{et al.} \cite{Arora2019NTK} 2019] \label{lem:NTK_kernel_flow}
Let $\{(x_i,y_i)\}_{i=1}^n\subset\mathbb{R}^d\times\mathbb{R}$ be training data, and consider a fully‐connected neural network $f(x;\theta)$ of width $m$ per hidden layer with weights initialized so that $\theta\sim\mathcal{N}(0,I)$, trained by infinitesimal‐step gradient descent on the squared loss
\[
L(\theta)
=\frac12\sum_{i=1}^n\bigl(f(x_i;\theta)-y_i\bigr)^2.
\]
Define the network output vector at time $t$:
\[
\mathbf{u}(t)\;=\;\bigl[f(x_1;\theta(t)),\dots,f(x_n;\theta(t))\bigr]^\top
\]
and the empirical NTK matrix at time $t$:
\[
K_m(t)\;=\;\bigl[\langle\nabla_\theta f(x_i;\theta(t)),\,\nabla_\theta f(x_j;\theta(t))\rangle\big]_{i,j=1}^n
\]
with initializing limit \footnote{While the NTK is random at initialization and varies during training, in the infinite-width limit it converges to an explicit limiting kernel and it stays constant during training \cite{Jacot2020NTK}.} $K_m(0)\xrightarrow[m\to\infty]{}K_\infty$. \\
Then under gradient flow, the network function $\mathbf{u}(t)$ follows a linear differential equation in the infinite-width limit:
\[
\frac{d\mathbf{u}(t)}{dt}
\;=\;
-\,K_m(t)\,\bigl(\mathbf{u}(t)-\mathbf{y}\bigr)
\]
where $\mathbf{y}=[y_1,\dots,y_n]^\top$.  Moreover, if $K_m(t)\equiv K_\infty$ remains constant (the so‐called 'NTK‐freeze' regime), these dynamics coincide exactly with kernel‐gradient flow in the RKHS of $K_\infty$.
\end{lemma}
\begin{proof}
    see proof of \textit{Theorem 1 \& 2} in \cite{Jacot2020NTK}, \textit{Lemma 3.1} and \textit{Theorem 3.1} in \cite{Arora2019NTK}.
\end{proof}

\begin{theorem}[NTK ’freeze’ and equivalence; Jacot \emph{et al.} \cite{Jacot2020NTK} 2018 \& Arora \emph{et al.} \cite{Arora2019NTK} 2019 \& Yang \cite{Yang2020NTK}] \label{thm:NTKntk_equiv}
  Let $K = K(X,X)$, and
  \[
    K_m(0)\;\xrightarrow[m\to\infty]{}\;K_\infty\;\succ\;0,
    \qquad
    f(x;\theta(0))\approx 0
  \]
  under the usual NTK scaling of a width‑\(m\) network initialized i.i.d.\ Gaussian.  Then, with high probability:
  \begin{enumerate}[label=\emph{(\arabic*)},  
                    leftmargin=*,        
                    labelsep=0.5em,      
                    itemsep=0.75\baselineskip]
    \item Kernel freeze.  
      \(K_m(t)\to K_\infty\) uniformly over \(t\ge0\).
    \item Gradient‑flow equivalence.  
      For every \(t\ge0\) and test point \(x\),
      \[
        f(x;\theta(t))
        =
        (K_\infty(x,X))^{\top} \bigl(I - e^{-t\,K_\infty}\bigr)\,K_\infty^{-1} \mathbf{y}.
      \]
    \item Convergence to kernel regression.  
      As \(t\to\infty\), the prediction of the neural net on this testing point $x$ matches the kernel regression output using this NTK:
      \[
        f_\infty(x)
        =
        \lim_{t\to\infty}f(x;\theta(t))
        = (K_\infty(x,X))^{\top} K_\infty^{-1} \mathbf{y}
      \]
      which is exactly the minimum‑norm interpolant of \(Y\) in the RKHS of \(K_\infty\).
  \end{enumerate}
\end{theorem}
\begin{proof}
    see proofs of \textit{Theorem 2} in \cite{Jacot2020NTK}, \textit{Theorem 3.2} in \cite{Arora2019NTK}, also relevant discussions: Section.5 in \cite{Jacot2020NTK}, Section.3 in \cite{Arora2019NTK}, Section.2.3 in \cite{Yang2020NTK}.
\end{proof}

Jacot et al. \cite{Jacot2020NTK} establishes that the network function $f_{\theta}$ follows the kernel gradient of the functional cost with respect to the NTK. \textit{Theorem 1} in \cite{Jacot2020NTK} states that in the infinite-width limit, the NTK $\Theta^{(L)}$ converges in probability to a deterministic limiting kernel $\Theta_{\infty}^{(L)}$ at \textit{initialization}; \textit{Theorem 2} states that this NTK stays asymptotically constant during \textit{training}, under certain conditions. This enables studying the training of neural nets in function space instead of parameter space, and convergence of the training can then be related to the positive-definiteness of the limiting NTK \cite{Jacot2020NTK}. For a least-squares regression cost, \cite{Jacot2020NTK} shows that, the network function $f_{t}$ during kernel gradient descent training with a kernel $K$ follows a linear differential equation dynamics, and the solution to which is in exponential decay form (Section.5 in \cite{Jacot2020NTK}).

Arora et al. \cite{Arora2019NTK} give the first non-asymptotic proof showing the equivalence between a fully-trained, sufficiently wide net and the kernel regression predictor using NTK. \textit{Lemma 3.1} in \cite{Arora2019NTK} describes the evolution of the network outputs $u(t)$ on training data during gradient descent: $\frac{d\mathbf{u}(t)}{dt}=-K(t)\cdot(\mathbf{u}(t)-\mathbf{y})$, where $[K(t)]_{i,j} = \langle\frac{\partial f(\theta(t),x_{i})}{\partial\theta},\frac{\partial f(\theta(t),x_{j})}{\partial\theta}\rangle$. In the infinite-width limit, $K(t)$ remains constant and equal to $K(0)$, which converges to a deterministic kernel matrix $K^{}$ (the NTK evaluated on training data). The dynamics then becomes $\frac{d\mathbf{u}(t)}{dt}=-K^{}\cdot(\mathbf{u}(t)-\mathbf{y})$. Solution to this linear ODE, assuming $u(0)=0$, leads to the kernel regression predictor at $t\rightarrow\infty$: $f^{}(x)=K_\infty(x,X)\cdot K^{-1} \mathbf{y}$. Our statement extends this by showing the function $f(x;\theta(t))$ at any time $t$ and for any test point $x$: $f(x;\theta(t)) = (K_\infty(x,X))^{\top} \bigl(I - e^{-t\,K_\infty}\bigr)\,K_\infty^{-1} \mathbf{y}$, which describes the evolution of a neural network's output function $f(x;\theta(t))$ during training via gradient flow, particularly in the infinite-width limit where the NTK becomes constant. \textit{Theorem 3.2} (\textit{equivalence between trained net and kernel regression}) in \cite{Arora2019NTK} rigorously proves that a fully-trained sufficiently wide ReLU neural network is equivalent to the kernel regression predictor using the NTK, i.e. $|f_{nn}(x)-f_{ntk}(x)|\le\epsilon$, where $f_{ntk}(x)=(K(x,X))^{\top} K^{-1} \mathbf{y}$ is the kernel regression predictor.

Yang \cite{Yang2020NTK} also discusses the NTK, $K_{\theta}(x,x^{\prime}) = \langle\nabla_{\theta}f(x;\theta), \nabla_{\theta}f(x^{\prime};\theta)\rangle$, and its convergence to a scaling limit $K_{\infty}$ as widths go to infinity. It proves (\textit{Corollary 2.4} (NTK convergence)) that, for a network of fixed standard architecture (without batchnorm) and under certain conditions, almost surely $K_{\theta_{t}}(\mathcal{X},\mathcal{X})\rightarrow K_{\infty}(\mathcal{X},\mathcal{X})$ for all $0\le t\le T$ in the large width limit, where $\mathcal{X}$ is a finite input set. It also mentions (Section.2.3) the gradient flow equation (i.e. the training dynamics) derived by Jacot et al. \cite{Jacot2020NTK}: $\frac{\partial f_{t}}{\partial t}=-\frac{1}{|\mathcal{X}|}K_{\theta_{t}}(\mathcal{X},\mathcal{X})(f_{t}-f^{*})$ where where $f^{*}$ is the 'ground truth' function, implying linear differential equation dynamics for $f$.

The statement for $f(x;\theta(t))$ is a direct consequence of applying kernel gradient descent with the limiting NTK ($K_\infty$) to a least-squares loss. The term $(I - e^{-t\,K_\infty})$ arises from solving the linear differential equation that governs the function's evolution in this regime, representing the dynamic path from an initial state towards the kernel regression solution $K_\infty(x,X)\,K_\infty^{-1} \mathbf{y}$.

\section{The 3-layer neural network kernel: derivation} \label{app:NN_kernel_derivation}

We briefly derive the 3-layer neural network kernel described in Section.\ref{sec:custom_kernel_design}, following \cite{Rasmussen2006GPbook}. This neural network has a single hidden layer, taking an input vector $\mathbf{x} \in \mathbb{R}^d$ and producing a scalar output $f(\mathbf{x})$. The mapping from input to output is:
\begin{equation} \label{3layer_NN_kernel_input_to_output}
    f(\mathbf{x}) = b + \sum_{j=1}^{m} v_j \sigma(\mathbf{x}; \mathbf{w}_j)
\end{equation}
in which $m$ is again the number of hidden units, $b$ is a bias term. $v_j$ are the weights connecting the $j$-th hidden unit to the output unit. $\sigma(\mathbf{x}; \mathbf{w}_j)$ is the activation function (or \textit{transfer function}) of the $j$-th hidden unit; it is assumed to be bounded. This function depends on the input $\mathbf{x}$ and the weights $\{\mathbf{w}_j\}_{j=1}^m$ connecting the input layer to the $j$-th hidden unit.

To form a Gaussian process, we define prior distributions over the parameters of the neural network. Following \cite{Neal1996}, we assume the bias $b$ and the hidden-to-output weights $v_j$ have independent zero-mean Gaussian priors with variances $\sigma_b^2$ and $\sigma_v^2$, respectively, i.e. $p(b) = \mathcal{N}(b | 0, \sigma_b^2)$, $p(v_j) = \mathcal{N}(v_j | 0, \sigma_v^2)$ for $j=1, \dots, m$. The input-to-hidden weights $\mathbf{w}_j$ for each hidden unit are drawn i.i.d. from some shared distribution $p(\mathbf{w})$.

We want to find the covariance between the function's output at two different input points, $\mathbf{x}$ and $\mathbf{x}'$. This covariance will form our kernel $k(\mathbf{x}, \mathbf{x}')$. The expected value of the function output $f(\mathbf{x})$ over the distribution of weights $\mathbf{w}$ is:
$$\mathbb{E}_{\mathbf{w}}[f(\mathbf{x})] = \mathbb{E}[b] + \sum_{j=1}^{m} \mathbb{E}[v_j] \mathbb{E}_{\mathbf{w}_j}[\sigma(\mathbf{x}; \mathbf{w}_j)]$$
As $\mathbb{E}[b] = 0$ and $\mathbb{E}[v_j] = 0$, we have:
$\mathbb{E}_{\mathbf{w}}[f(\mathbf{x})] = 0$, which means the prior over functions defined by this Bayesian neural network has a zero mean.

The covariance is:
$$\mathbb{E}_{\mathbf{w}}[f(\mathbf{x})f(\mathbf{x}')] = \mathbb{E}_{\mathbf{w}}\left[ \left(b + \sum_{j=1}^{m} v_j \sigma(\mathbf{x}; \mathbf{w}_j)\right) \left(b + \sum_{k=1}^{m} v_k \sigma(\mathbf{x}'; \mathbf{w}_k)\right) \right]$$
expanding the product and noting the independence of $b$, $v_j$ and $\mathbf{w}_j$, we have:
$$\mathbb{E}_{\mathbf{w}}[f(\mathbf{x})f(\mathbf{x}')] = \mathbb{E}[b^2] + \sum_{j=1}^{m} \sum_{k=1}^{m} \mathbb{E}[v_j v_k] \mathbb{E}_{\mathbf{w}_j, \mathbf{w}_k}[\sigma(\mathbf{x}; \mathbf{w}_j)\sigma(\mathbf{x}'; \mathbf{w}_k)]$$

Our assumptions tell $\mathbb{E}[b^2] = \sigma_b^2$. If $j \neq k$, $\mathbb{E}[v_j v_k] = \mathbb{E}[v_j]\mathbb{E}[v_k] = 0 \cdot 0 = 0$. If $j = k$, $\mathbb{E}[v_j^2] = \sigma_v^2$. So the above expression simplifies to (Eq.(4.27) in \cite{Rasmussen2006GPbook}):
$$\mathbb{E}_{\mathbf{w}}[f(\mathbf{x})f(\mathbf{x}')] = \sigma_b^2 + \sum_{j=1}^{m} \sigma_v^2 \mathbb{E}_{\mathbf{w}_j}[\sigma(\mathbf{x}; \mathbf{w}_j)\sigma(\mathbf{x}'; \mathbf{w}_j)]$$

As the $\mathbf{w}_j$ are i.i.d., the expectation $\mathbb{E}_{\mathbf{w}_j}[\sigma(\mathbf{x}; \mathbf{w}_j)\sigma(\mathbf{x}'; \mathbf{w}_j)]$ is the same for all $j$. Let's denote this common expectation as $\mathbb{E}_{\mathbf{w}}[\sigma(\mathbf{x}; \mathbf{w})\sigma(\mathbf{x}'; \mathbf{w})]$, the above simplifies to (Eq.(4.28) in \cite{Rasmussen2006GPbook}):
$$\mathbb{E}_{\mathbf{w}}[f(\mathbf{x})f(\mathbf{x}')] = \sigma_b^2 + m \sigma_v^2 \mathbb{E}_{\mathbf{w}}[\sigma(\mathbf{x}; \mathbf{w})\sigma(\mathbf{x}'; \mathbf{w})]$$
Using the width-dependent variance $\sigma_v^2=\omega^2/m$, we can simplify the above to \cite{Rasmussen2006GPbook}: 
$$\mathbb{E}_{\mathbf{w}}[f(\mathbf{x})f(\mathbf{x}')] = \sigma_b^2 + \omega^2 \mathbb{E}_{\mathbf{w}}[\sigma(\mathbf{x}; \mathbf{w})\sigma(\mathbf{x}'; \mathbf{w})]$$
This scaling is important for the limit as $m \to \infty$. As $m$ increases, the individual contributions of $v_j$ decrease, keeping the overall variance finite.

Since the transfer function $\sigma$ is assumed to be bounded, all moments of the distribution will be bounded. By central limit theorem, as $m \to \infty$, the stochastic process $f(\mathbf{x})$ converges to a Gaussian process \cite{Rasmussen2006GPbook}. The covariance function of this limiting Gaussian process is:
\begin{equation} \label{eq:cov_func_NN}
    k(\mathbf{x}, \mathbf{x}') = \sigma_b^2 + \omega^2 \mathbb{E}_{\mathbf{w}}[\sigma(\mathbf{x}; \mathbf{w})\sigma(\mathbf{x}'; \mathbf{w})]
\end{equation}
which serves as the general form of the covariance function (i.e. the kernel). To obtain a concrete kernel, we need to specify the activation function $\sigma$ and the distribution of the input-to-hidden weights $p(\mathbf{w})$. In the following, we discuss a specific case in which the error function is employed as the activation function \footnote{The \textit{sigmoid kernel} $k(\mathbf{x}, \mathbf{x}') = \tanh(a + b \mathbf{x} \cdot \mathbf{x}')$ has sometimes been proposed \cite{Rasmussen2006GPbook}, but it is never positive definite and therefore not a valid covariance function \cite{Scholkopf2002KernelLearning}}.

\paragraph{Error function (erf) activation} 
If we choose the error function as the activation, i.e.
$$\sigma(z) = \text{erf}(z) = \frac{2}{\sqrt{\pi}} \int_0^z e^{-t^2} dt$$
and the argument of the error function is $z_j = u_{j0} + \sum_{i=1}^d u_{ji} x^i = \mathbf{w}_j \cdot \tilde{\mathbf{x}}$, $j=1,2,...,m$, where $\tilde{\mathbf{x}} = (1, x^1, \dots, x^d)^T$ is an augmented input vector including a bias term (offset $u_{j0}$).
Each weight vector $\mathbf{w}$ is assumed to follow a shared, zero-mean Gaussian distribution: $\mathbf{w} \sim \mathcal{N}(\mathbf{0}, \Sigma)$. The activation becomes $\sigma(\mathbf{x}; \mathbf{w}) = \text{erf}(\mathbf{w} \cdot \tilde{\mathbf{x}})$. The covariance function then involves calculating $\mathbb{E}_{\mathbf{w}}[\text{erf}(\mathbf{w} \cdot \tilde{\mathbf{x}}) \cdot \text{erf}(\mathbf{w} \cdot \tilde{\mathbf{x}}')]$. Williams \cite{Williams1998NNkernel} showed that this expectation leads to the following kernel, often referred to as the 'neural network kernel' or the 'arcsin kernel' (Eq.(4.29) in \cite{Rasmussen2006GPbook}):
\begin{equation} \label{eq:NN_kernel_erf_activation}
    k^{NN}(\mathbf{x}, \mathbf{x}') = \frac{2}{\pi} \sin^{-1}\left(\frac{2\tilde{\mathbf{x}}^T \Sigma \tilde{\mathbf{x}}'}{\sqrt{(1+2\tilde{\mathbf{x}}^T \Sigma \tilde{\mathbf{x}})(1+2\tilde{\mathbf{x}}'^T \Sigma \tilde{\mathbf{x}}')}}\right)
\end{equation}
where the variance $\omega^2$ and the bias variance $\sigma_b^2$ are implicitly absorbed or set to specific values (e.g. $\sigma_b^2=0$ and $\omega^2$ are incorporated into normalization).

Note that the covariance term from the bias, $\sigma_b^2$, can be added separately if needed. The hyperparameter matrix $\Sigma = \text{diag}(\sigma_0^2, \sigma_1^2, \dots, \sigma_D^2)$, assuming independence of the components of $\mathbf{w}$, controls the properties of the kernel: $\sigma_0^2$ controls the variance of $u_0$ (the input bias term for the hidden units), which influences the offset of the activation functions from the origin; $\sigma_i^2$ for $i \geq 1$ controls the variance of the weights for the $i$-th input feature, essentially scaling the input dimensions and determining how quickly the function varies along those dimensions. The parameters in $\Sigma$ act as hyperparameters of the kernel, controlling its behavior.

This neural network kernel Eq.\ref{eq:NN_kernel_erf_activation} is derived by considering a Bayesian neural network with one hidden layer, specific prior distributions on its weights, i.e. Gaussian for output weights and bias, and a chosen Gaussian priors for input-to-hidden weights (where the input $\mathbf{x}$ is augmented with a bias term to become $\tilde{\mathbf{x}}$): $\mathbf{w} \sim \mathcal{N}(\mathbf{0}, \Sigma)$), the specific choice of the error function $\text{erf}(z)$ as the activation $\sigma$, and then taking the limit as the number of hidden units ($m$) goes to infinity. This specific neural network kernel allows Gaussian processes to model functions with properties similar to those learned by these infinite neural networks. 

\paragraph{Dimensions of the input-to-hidden weight vector $\mathbf{w}_i$}
$\{\mathbf{w}_i\}_{i=1}^m$ is the collection of input-to-hidden weight vectors , in which $m$ is the number of neurons in the hidden layer, $\mathbf{w}_i$ is the weight vector for the $i$-th hidden unit. This vector connects all components of the augmented input $\tilde{\mathbf{x}}$ to the $i$-th hidden unit. For the $i$-th hidden unit, the pre-activation (the input to the activation function) is typically a dot product of the weight vector $\mathbf{w}_i$ and the augmented input vector $\tilde{\mathbf{x}}$:
$$z_i = \mathbf{w}_i^T \tilde{\mathbf{x}} = w_{i0} \cdot 1 + w_{i1} x^{(1)} + \dots + w_{id} x^{(d)}$$
For this operation to be defined, the weight vector $\mathbf{w}_i$ must have the same dimension as the augmented input vector $\tilde{\mathbf{x}}$, i.e. $\mathbf{w}_i$ is of dimension $(d+1)$:
\[
\mathbf{w}_i = [w_{i0}, w_{i1}, \dots, w_{id}]^T \in \mathbb{R}^{d+1}
\]
$w_{i0}$ is the bias weight for the $i$-th hidden unit (associated with the constant '1' in $\tilde{\mathbf{x}}$).
$w_{ij}$ (for $j=1, \dots, d$) is the weight connecting the $j$-th original input feature $x_j$ to the $i$-th hidden unit.

\paragraph{Dimensions of the covariance matrix $\Sigma$}
The covariance matrix $\Sigma$ describes the prior distribution from which each hidden-to-output weight vector $\mathbf{w}_i$ is drawn. Specifically, it is assumed that $\mathbf{w}_i \sim \mathcal{N}(\mathbf{0}, \Sigma)$ independently for each hidden unit $i$. Since each $\mathbf{w}_i$ is a $(d+1)$-dimensional random vector, its covariance matrix $\Sigma$ must be a square matrix of dimensions $(d+1) \times (d+1)$, i.e. $\Sigma \in \mathbb{R}^{(d+1) \times (d+1)}$. The elements of $\Sigma$, denoted $\Sigma_{jk}$, represent the covariance between the $j$-th component of $\mathbf{w}_i$ and the $k$-th component of $\mathbf{w}_i$. That is, $\Sigma_{jk} = \text{Cov}(w_{ij}, w_{ik})$.

Commonly, $\Sigma$ is assumed to be diagonal \footnote{The assumption of a diagonal $\Sigma$ implies that the individual weight components within a vector $\mathbf{w}_i$ (e.g. $w_{i0}, w_{i1}, \dots, w_{id}$) are a priori uncorrelated. A more general, non-diagonal $\Sigma$ would allow for prior correlations between these weight components.}:
$$\Sigma = \text{diag}(\sigma_0^2, \sigma_1^2, \dots, \sigma_D^2) = \begin{pmatrix} \sigma_0^2 & 0 & \dots & 0 \\ 0 & \sigma_1^2 & \dots & 0 \\ \vdots & \vdots & \ddots & \vdots \\ 0 & 0 & \dots & \sigma_d^2 \end{pmatrix}$$
where $\sigma_0^2$ is the variance of the bias weight $w_{i0}$ for any hidden unit $i$. It controls the prior uncertainty about the hidden unit's bias; $\sigma_j^2$ (for $j=1, \dots, d$) is the variance of the weight $w_{ij}$ connecting the $j$-th input feature $x_j$ to any hidden unit $i$. It controls the 'relevance' or 'length scale' associated with the $j$-th input dimension. Larger $\sigma_j^2$ allows for stronger influence from input $x_j$.

\paragraph{Neural network kernel \textit{vs} NTK}
Both the NTK (Eq.\ref{eq:NN_3layer_NTK}) and the neural network kernel (Eq.\ref{eq:NN_kernel_erf_activation}) used here are derived using a shallow, 3-layer architecture. However, they are fundamentally different and they describe different aspects of the network. The NTK in Eq.\ref{eq:NN_3layer_NTK} describes the training dynamics of a neural network in the infinite-width limit \cite{Jacot2020NTK}. It captures how the network’s output evolves during gradient descent by considering the gradients of the output with respect to all parameters. It is particularly useful for understanding the behavior of neural networks during training, as it relates to the kernel that governs the network’s evolution under gradient descent. The neural network kernel (Eq.\ref{eq:NN_kernel_erf_activation} or Eq.\ref{eq:NN_kernel_erf_activation}) models the covariance of the outputs of an infinite-width neural network with an erf activation, directly corresponding to a Gaussian process. It is not concerned with training dynamics but rather with the prior distribution over functions induced by the network architecture at initialization.

\section{SVGD as an alternative to Monte Carlo sampling} 
\label{app:SVGD}
Stein variational gradient descent (SVGD) can also be used as an alternative to Monte Carlo sampling for approximating the marginal probability $p(I) = p(z_i > 0 \ \forall i \in I)$ in Eq.\ref{eq:rule_inference_via_integral_of_marginalised_joint}. SVGD is a particle-based variational inference method \cite{Liu2016SVGD} which transports a set of particles (samples) to form the target distribution. At convergence, the particles serve as samples drawn from the target. We can therefore use SVGD to generate samples from the marginalised GP posterior, hopefully improving the efficiency of rule inference.   

\paragraph{SVGD principles}
SVGD is a particle-based, non-parametric variational inference method introduced by Liu et al. \cite{Liu2016SVGD} that approximates a target distribution by iteratively updating a set of particles (samples) using gradient-based updates. Unlike traditional Monte Carlo sampling, which generates independent samples, SVGD uses a deterministic update rule to evolve particles in a way that implicitly minimizes the Kullback-Leibler (KL) divergence between the particle distribution and the target distribution.

In our case case, the target distribution is the marginalised GP posterior $p(\mathbf{z}) = \mathcal{N}(0, K_I)$, where $K_I$ is the submatrix of the kernel matrix $K$ for the items in the itemset $I$, and $\mathbf{z}$ is the vector of latent variables for those items. SVGD starts with a set of particles $\{\mathbf{z}^{(s)}\}_{s=1}^S$ initialised from a proposal (e.g. uniform), where each particle $\mathbf{z}^{(s)}$ is a vector of length $|I|$ (the size of the itemset). SVGD then updates the particles using a gradient-based rule that balances two forces: gradient-based attraction and kernel differential based repulsion. The attractive force moves particles toward regions of high probability under the target distribution, using the gradient of the log-density $\nabla_{\mathbf{z}} \log p(\mathbf{z})$, while the repulsive force prevents particles from collapsing into a single mode, ensuring diversity among particles. The SVGD update rule for a particle $\mathbf{z}^{(s)}$ at iteration $t$ is \cite{Liu2016SVGD}:
\begin{equation} \label{eq:SVGD_dynamics}
   \mathbf{z}^{(s)}_{t+1} = \mathbf{z}^{(s)}_t + \epsilon \cdot \hat{\phi}(\mathbf{z}^{(s)}_t)
\end{equation}
Where $\epsilon$ is the step size, $\hat{\phi}(\mathbf{z})$ is the composite (variational) gradient, empirically computed as:
\[
\hat{\phi}(\mathbf{z}) = \frac{1}{S} \sum_{s=1}^S \left[ k(\mathbf{z}^{(s)}, \mathbf{z}) \nabla_{\mathbf{z}^{(s)}} \log p(\mathbf{z}^{(s)}) + \nabla_{\mathbf{z}^{(s)}} k(\mathbf{z}^{(s)}, \mathbf{z}) \right]
\]
$k(\cdot, \cdot)$ is a chosen kernel (e.g. RBF). $\nabla_{\mathbf{z}} \log p(\mathbf{z})$ is the score function of the target distribution. For a standard multivariate Gaussian target $p(\mathbf{z}) = \mathcal{N}(0, K_I)$, we have
\[
\log p(\mathbf{z}) = -\frac{1}{2} \mathbf{z}^\top K_I^{-1} \mathbf{z} - \frac{1}{2} \log \det(K_I) - \frac{|I|}{2} \log(2\pi),
\quad
\nabla_{\mathbf{z}} \log p(\mathbf{z}) = -K_I^{-1} \mathbf{z}
\]

\paragraph{Estimating the co-occurrence probability $p(I)$ using SVGD}
We first initialise a set of particles (e.g. $S=100$), and move the particles till convergence following the variational gradient descent dynamics in Eq.\ref{eq:SVGD_dynamics}, obtaining a set of particles whose final configuration approximates $\mathcal{N}(0, K_I)$. We can then estimate $p(I)$ as the fraction of particles where $\mathbf{z} > 0$:
\[
p(I) \approx \frac{1}{S} \sum_{s=1}^S \mathbf{1}_{\mathbf{z}^{(s)} > 0}
\]

\paragraph{Why SVGD} SVGD might be a candidate for approximating the GP posterior in GPAR for several reasons: (1) Efficiency. SVGD is deterministic, particles following the SVGD dynamics can converge to a good approximation of the target distribution, fewer particles are required than Monte Carlo sampling, as it actively optimizes the particle positions to match the target distribution. This can reduce the number of samples needed (e.g. $S$ can be smaller than 1000). (2) Accuracy. By implicitly minimizing KL divergence, SVGD provides a more targeted approximation of the posterior, potentially reducing variance compared to naive Monte Carlo sampling. (3) Robustness. SVGD exploits the covariance matrix $K_I$ from the trained GP, ensuring that the learned correlations (via the shifted RBF kernel) are fully utilized. Unlike methods that assume independence (e.g. fully factorized Gaussians in VI, Strategy 5), SVGD captures the full joint distribution by evolving particles in the high-probability regions of $\mathcal{N}(0, K_I)$. (3) Flexibility. SVGD is non-parametric and can approximate complex distributions, making it robust to the non-standard covariance structure introduced e.g. by the shifted RBF kernel.

However, SVGD requires computing gradients of the log-density and kernel functions, as well as iterative updates, which adds implementation complexity compared to Monte Carlo sampling. Each SVGD iteration involves computing gradients and kernel evaluations for all particles, which can be costly for large itemsets or many particles (though the total number of iterations is often small). SVGD also requires tuning the step size $\epsilon$, number of iterations, and a proper choice of kernel $k(\cdot, \cdot)$. Poor choices can lead to slow convergence or poor approximations. 

\paragraph{Implementation} 
Following settings are used in our SVGD implementation:
\begin{itemize}
    \item $S = 100$ particles (fewer than Monte Carlo’s 1000, as SVGD is more efficient per particle).
    \item \text{num\_iterations} = 50, typically a small number of iterations is needed for SVGD updates.
    \item  RBF kernel for the SVGD repulsion term, with median heuristic for the SVGD kernel length scale - a standard choice for robustness \cite{Liu2016SVGD}.
    \item \text{step\_size} = 0.1.
\end{itemize}

\paragraph{Potential impact on efficiency}  SVGD may reduce the number of particles needed, but each iteration is more expensive. Each SVGD iteration computes the score function ($\mathcal{O}(|I|^3)$ for $K_I^{-1}$, done once), kernel evaluations ($\mathcal{O}(S^2 \cdot |I|)$), and gradients ($\mathcal{O}(S^2 \cdot |I|)$), totaling $\mathcal{O}(|I|^3 + S^2 \cdot |I|)$ per iteration. For 50 iterations, the cost is $\mathcal{O}(|I|^3 + 50 \cdot S^2 \cdot |I|)$. With $S = 100$, this is $\mathcal{O}(|I|^3 + 500000 \cdot |I|)$.

As a comparison, the original Monte Carlo sampling approach costs $\mathcal{O}(|I|^3 + S \cdot |I|^2)$, with $S = 1000$, so $\mathcal{O}(|I|^3 + 1000 \cdot |I|^2)$. For $|I| = 5$ \footnote{SVGD is used for $|I| > 2$, as we have analytical formulas for the integral in low dimensions.}, Monte Carlo is $\mathcal{O}(125 + 1000 \cdot 25) = \mathcal{O}(25125)$, while SVGD is $\mathcal{O}(125 + 500000 \cdot 5) = \mathcal{O}(2500125)$, suggesting SVGD is slower per itemset. However, SVGD typically requires fewer particles ($S = 100$) to achieve comparable accuracy, and the overall runtime may be lower due to fewer total evaluations if the number of iterations is tuned.

\section{GPAR experimental setup} \label{app:GPAR_experimental_setup}

This study evaluates 4 association rule mining methods, i.e. GPAR (with its variants), Apriori, FP-Growth, and Eclat, through experiments conducted on three datasets: two synthetic datasets (\textit{Synthetic 1} and \textit{Synthetic 2}) and one real-world dataset (\textit{UK Accidents Dataset} \cite{road_safety_data}). Each method was systematically tested across a range of \textit{minimum support} thresholds, while maintaining a fixed \textit{minimum confidence} threshold.

Experiments on the 2 synthetic datasets, i.e. Synthetic 1 \& 2, were conducted on a laptop with an x86\_64 architecture, powered by an 11th Gen Intel® Core™ i7-1185G7 processor running at 3.00 GHz, equipped with 8 CPUs (4 cores, 2 threads per core, 1 socket), supporting both 32-bit and 64-bit operations, and capable of scaling from 400 MHz to a maximum of 4800 MHz. The CPU features advanced instruction sets like AVX-512, VT-x virtualization, and security mitigations for vulnerabilities such as Spectre and Meltdown. It is supported by a memory hierarchy consisting of 192 KiB L1 data cache, 128 KiB L1 instruction cache, 5 MiB L2 cache, and 12 MiB L3 cache, all across 4 instances except for the L3 cache, which has 1 instance. The system operates on a single NUMA node (node0) with CPUs 0-7, utilizing 39-bit physical and 48-bit virtual address sizes, and follows a Little Endian byte order, providing a robust platform for computational tasks.

Experiments on the real-world dataset, i.e. the \textit{UK Accidents dataset}, were conducted in \textit{Google Colab} utilized a system equipped with an Intel(R) Xeon(R) CPU operating at 2.00GHz, featuring 48 processors across 24 cores with a cache size of 39,424 KB. The CPU supports advanced features such as AVX512 and includes mitigations for vulnerabilities like Spectre and Meltdown. The system has a substantial memory capacity of approximately 350 GB, with around 337 GB free and 346 GB available during operation \footnote{Reported by using command `/proc/meminfo` in Colab.}. No swap memory is configured, and memory usage includes 9.9 GB cached and 924 MB in slab memory, ensuring efficient handling of the computational workload for the experiments.

When generating rules (i.e. rule inference), stetting the values for min\_support and min\_conf corresponds to a balance between rule quantity and quality. The decision, however, depends on the application domain and available computational resources. In literature, min\_support often ranges from 0.01 to 0.3 (1\% to 30\%), depending on the dataset’s density and size. For dense datasets (e.g. market basket data), min\_support is often low, e.g. 0.01 to 0.05, because items are frequently co-occurring; for sparse datasets (e.g. medical or web data), min\_support can be higher, e.g. 0.1 to 0.3, to avoid generating too many rules. We choose min\_support=[0.1, 0.2, 0.3, 0.4, 0.5] for \textit{Synthetic 1} ($M=10$ items, each item is characterised by $10$ features and in total $N=1000$ transaction records), and [0.1, 0.15, 0.2, 0.25, 0.3] for \textit{Synthetic 2} ($M=15$ items, each item is characterised by $10$ features and in total $N=1000$ transaction records), and [0.1, 0.2, 0.3, 0.4, 0.5] for \textit{UK Accidents} dataset ($M=39$ items, each item is characterised by $10$ features and in total $N=5231$ transaction records). High min\_support values such as 0.5 (50\%) can filter out many interesting rules, especially in a dataset with diverse items. A common range for min\_conf in many studies is 0.5 to 0.8. A min\_conf of 0.5 means that the consequent must appear in at least 50\% of the transactions where the antecedent is present, which is a reasonable threshold for identifying reliable patterns without being overly restrictive. Higher values like 0.8 ensure stronger associations but may filter out many potentially interesting rules. Roughly speaking, the value of min\_support roughly controls the number of rules generated, while min\_conf controls the reliability. Together they balance rule quantity and quality. A high min\_conf, for example min\_conf=0.7, strikes a balance between capturing a reasonable number of rules and ensuring they are reliable, that is, the consequent occurs in at least 70\% of the transactions where the antecedent is present.

\subsection{Performance comparison metrics}

For each \textit{minimum support} level, the following performance metrics were recorded to enable a comprehensive comparison across algorithms:

\begin{itemize}
    \item \textit{Runtime (seconds)}: The duration required to generate frequent itemsets and association rules.
    \item \textit{Memory usage (MB)}: The peak memory consumption during execution.
    \item \textit{Number of frequent itemsets}: The count of itemsets with support greater than or equal to the \textit{minimum support} threshold.
    \item \textit{Number of mined association rules}: The count of association rules with confidence greater than or equal to 0.5.
\end{itemize}

For GPAR, frequent itemsets were identified by estimating joint probabilities using Monte Carlo sampling with 1,000 samples, followed by the computation of confidence and lift metrics for itemset partitions to derive association rules. In contrast, Apriori and FP-Growth were implemented using the \textit{Mlxtend} Python library \cite{raschkas_2018_mlxtend}, leveraging their standard pruning techniques to efficiently identify frequent itemsets and rules. Eclat was implemented using a vertical database format, employing recursive intersection of transaction ID sets to discover frequent itemsets.

To assess, interpret and compare the mined rules, we can apply these quantitative metrics along with domain expertise. Each rule has associated support, confidence, and lift metrics. To facilitate comparison, summary tables present the top 10 generated rules, ranked by support, confidence, or lift, alongside plots visualizing runtime and memory usage across varying support thresholds. The following sections detail the performance of the GPAR variants in competition with the three classical methods, providing insights into their relative strengths and trade-offs. We describe the tests on 2 synthetic datasets in Section.\ref{app:synthetic_data_tests}, and tests on 1 real-world dataset in Section.\ref{app:accident_data_tests}, respectively.

Several other critical metrics are also considered for evaluating performance: (1) \textit{Scalability}: the algorithm's performance in terms of time and memory usage as dataset size (number of transactions, items and features) increases or support thresholds decrease. (2) \textit{Generalization}: the ease of extending the algorithm to incorporate augmented itemsets with new, unobserved items. (3) \textit{Interpretation}: the ability to identify meaningful or rare patterns that are interpretable and actionable.

\section{Synthetic data tests} \label{app:synthetic_data_tests}

We aim to compare the performances of GPAR, Apriori, FP-Growth, and Eclat in association rule mining tasks on a fair and diverse basis. To achieve this we designed 2 synthetic datasets, differing in their data generation processes (DGPs) to reflect different consumer behaviours and test the performances of these algorithms in different scenarios.

In generating the first synthetic data, we created a feature matrix $X \in \mathbb{R}^{10 \times 10}$ which represents 10 items with 10 features each, by sampling from a standard normal distribution. A radial basis function (RBF) kernel with a length scale of 10.0 was applied to compute the covariance matrix $K$. Using this covariance, $N=1000$ transactions were generated via a multivariate normal distribution, where each transaction was formed by sampling latent variables $z \sim \mathcal{N}(0, K)$ and retaining items with positive values ($z_i > 0, i=1,2,...,10$). Non-empty transactions were ensured through resampling. 

The second synthetic dataset was generated differently, consisting of 15 items organized into 3 clusters, with each cluster containing 5 items described by 10 features, yielding a feature matrix $X \in \mathbb{R}^{15 \times 10}$. A total of 1,000 transactions are generated, where each transaction is a subset of items deliberately sampled from multiple clusters. 

\subsection{Synthetic data 1}

\subsubsection{GPAR with RBF kernel} 
The maximum likelihood estimated RBF length scale $\ell=6.7078$.

\begin{table}[H]
\centering
\scriptsize
\begin{threeparttable}
\caption{Top 10 association rules mined by \textbf{RBF-based GPAR} (ranked in descending order by \textit{\color{red}{support}})}
\label{tab:RBF_GPAR_mined_rules_rankedBySupport_minSupport01_Synthetic1}
\begin{tabular}{r >{\raggedright\arraybackslash}p{4cm} >{\raggedright\arraybackslash}p{4cm} p{1cm} p{1cm} p{1cm}}
\toprule
\# & Antecedent & Consequent & \color{red}{Supp} & Conf & Lift \\
\midrule
747 & item\_5 & item\_8, item\_6 & 0.5200 & 0.9123 & 2.3392 \\
\midrule
748 & item\_6 & item\_8, item\_5 & 0.5200 & 0.8667 & 2.6289 \\
\midrule
749 & item\_8 & item\_5, item\_6 & 0.5200 & 0.8667 & 2.7778 \\
\midrule
750 & item\_5, item\_6 & item\_8 & 0.5200 & 1.6250 & 2.7013 \\
\midrule
751 & item\_8, item\_5 & item\_6 & 0.5200 & 1.2093 & 2.3810 \\
\midrule
752 & item\_8, item\_6 & item\_5 & 0.5200 & 1.2683 & 2.7368 \\
\midrule
9 & item\_0 & item\_5 & 0.4900 & 0.8305 & 1.7754 \\
\midrule
10 & item\_5 & item\_0 & 0.4900 & 0.9423 & 1.8839 \\
\midrule
57 & item\_3 & item\_8 & 0.4900 & 1.0426 & 1.8148 \\
\midrule
58 & item\_8 & item\_3 & 0.4900 & 0.9423 & 2.2182 \\
\bottomrule
\end{tabular}
\begin{tablenotes}
\item[1] Total number of transactions: 1000
\item[2] Minimum support used in mining: 0.1
\item[3] Items correspond to synthetic dataset indices.
\item[4] Optimised RBF length-scale $\ell=6.7078$.
\end{tablenotes}
\end{threeparttable}
\end{table}

\begin{table}[H]
\centering
\scriptsize
\begin{threeparttable}
\caption{Top 10 association rules mined by \textbf{RBF-based GPAR} (ranked in descending order by \textit{\color{red}{confidence}})}
\label{tab:RBF_GPAR_mined_rules_rankedByConfidence_minSupport01_Synthetic1}
\begin{tabular}{r >{\raggedright\arraybackslash}p{6cm} >{\raggedright\arraybackslash}p{3cm} p{1cm} p{1cm} p{1cm}}
\toprule
\# & Antecedent & Consequent & Supp & \color{red}{Conf} & Lift \\
\midrule
39375 & item\_1, item\_2, item\_3, item\_6, item\_7, item\_9 & item\_0, item\_4 & 0.2500 & 2.2727 & 2.1277 \\
\midrule
41942 & item\_0, item\_1, item\_2, item\_4, item\_5, item\_7, item\_8 & item\_9 & 0.2900 & 2.2308 & 2.7462 \\
\midrule
36022 & item\_1, item\_3, item\_4, item\_5, item\_6, item\_9 & item\_8 & 0.3100 & 2.2143 & 2.2513 \\
\midrule
48098 & item\_2, item\_3, item\_4, item\_5, item\_9 & item\_8, item\_6, item\_7 & 0.3300 & 2.2000 & 3.6264 \\
\midrule
24018 & item\_0, item\_1, item\_3, item\_4, item\_8 & item\_2, item\_7 & 0.3600 & 2.1176 & 2.3256 \\
\midrule
38242 & item\_0, item\_8, item\_2, item\_3 & item\_1, item\_4, item\_5, item\_7 & 0.3300 & 2.0625 & 3.4304 \\
\midrule
48123 & item\_2, item\_4, item\_6, item\_8, item\_9 & item\_3, item\_5, item\_7 & 0.3300 & 2.0625 & 3.3537 \\
\midrule
38572 & item\_0, item\_2, item\_3, item\_5, item\_9 & item\_1, item\_4, item\_7 & 0.3200 & 2.0000 & 3.2389 \\
\midrule
50726 & item\_0, item\_2, item\_3, item\_4, item\_6, item\_7, item\_8, item\_9 & item\_1 & 0.3200 & 2.0000 & 2.8520 \\
\midrule
9962 & item\_9, item\_2, item\_4, item\_5 & item\_7 & 0.4100 & 1.9524 & 4.0675 \\
\bottomrule
\end{tabular}
\begin{tablenotes}
\item[1] Total number of transactions: 1000
\item[2] Minimum support used in mining: 0.1
\item[3] Items correspond to synthetic dataset indices.
\item[4] Optimised RBF length-scale $\ell=6.7078$.
\end{tablenotes}
\end{threeparttable}
\end{table}

\begin{table}[H]
\centering
\scriptsize
\begin{threeparttable}
\caption{Top 10 association rules mined by \textbf{RBF-based GPAR} (ranked in descending order by \textit{\color{red}{lift}})}
\label{tab:RBF_GPAR_mined_rules_rankedByLift_minSupport01_Synthetic1}
\begin{tabular}{r >{\raggedright\arraybackslash}p{5cm} >{\raggedright\arraybackslash}p{5cm} p{1cm} p{1cm} p{1cm}}
\toprule
\# & Antecedent & Consequent & Supp & Conf & \color{red}{Lift} \\
\midrule
50401 & item\_0, item\_8, item\_6, item\_9 & item\_1, item\_2, item\_3, item\_4, item\_7 & 0.3200 & 1.6842 & 12.1212 \\
\midrule
49258 & item\_0, item\_2, item\_4 & item\_1, item\_3, item\_5, item\_6, item\_8, item\_9 & 0.2900 & 0.7073 & 8.7879 \\
\midrule
47999 & item\_3, item\_6, item\_7 & item\_2, item\_4, item\_5, item\_8, item\_9 & 0.3300 & 0.9706 & 8.4615 \\
\midrule
48063 & item\_9, item\_3, item\_4, item\_5 & item\_8, item\_2, item\_6, item\_7 & 0.3300 & 1.1379 & 8.0882 \\
\midrule
6035 & item\_8, item\_5 & item\_0, item\_2, item\_7 & 0.4500 & 0.8491 & 8.0357 \\
\midrule
30142 & item\_4, item\_5, item\_6 & item\_0, item\_9, item\_2, item\_3 & 0.3500 & 0.8750 & 7.7778 \\
\midrule
49479 & item\_0, item\_1, item\_3, item\_5, item\_6 & item\_8, item\_9, item\_2, item\_4 & 0.2900 & 1.2083 & 7.6720 \\
\midrule
40321 & item\_0, item\_1, item\_3, item\_8 & item\_9, item\_2, item\_5, item\_6 & 0.3400 & 1.1724 & 7.3913 \\
\midrule
42338 & item\_8, item\_1, item\_5, item\_6 & item\_0, item\_9, item\_2, item\_7 & 0.2800 & 0.7568 & 7.1429 \\
\midrule
25884 & item\_8, item\_1, item\_5 & item\_0, item\_2, item\_4, item\_7 & 0.4100 & 1.5185 & 6.9728 \\
\bottomrule
\end{tabular}
\begin{tablenotes}
\item[1] Total number of transactions: 1000
\item[2] Minimum support used in mining: 0.1
\item[3] Items correspond to synthetic dataset indices.
\item[4] Optimised RBF length-scale $\ell=6.7078$.
\end{tablenotes}
\end{threeparttable}
\end{table}

\begin{table}[H]
\centering
\scriptsize
\begin{threeparttable}
\caption{Top 10 association rules mined by \textbf{RBF-based GPAR} with item 2 on RHS (min support = 0.1, rules ranked by \color{red}{confidence})}
\label{tab:RBF_GPAR_analysis_minSupport01_Synthetic1}
\begin{tabular}{r >{\raggedright\arraybackslash}p{3cm} p{1cm} p{1cm} p{1cm} >{\raggedright\arraybackslash}p{2cm} >{\raggedright\arraybackslash}p{5cm}}
\toprule
\# & Rule & Conf & Supp (prob) & Supp Count & Example Trans. Indices & Example Transactions \\
\midrule
1 & [0, 1, 3, 4, 8] $\to$ [2, 7] & 2.1176 & 0.3600 & 488 & [1, 2, 3] & [0, 1, 2, 3, 4, 5, 6, 7, 8, 9], [0, 1, 2, 3, 4, 5, 6, 7, 8, 9], [0, 1, 2, 3, 4, 5, 6, 7, 8, 9] \\
\midrule
2 & [1, 5, 6, 7, 8] $\to$ [2, 3] & 1.8333 & 0.3300 & 514 & [1, 2, 3] & [0, 1, 2, 3, 4, 5, 6, 7, 8, 9], [0, 1, 2, 3, 4, 5, 6, 7, 8, 9], [0, 1, 2, 3, 4, 5, 6, 7, 8, 9] \\
\midrule
3 & [0, 1, 4, 5, 8] $\to$ [2, 7] & 1.7826 & 0.4100 & 498 & [1, 2, 3] & [0, 1, 2, 3, 4, 5, 6, 7, 8, 9], [0, 1, 2, 3, 4, 5, 6, 7, 8, 9], [0, 1, 2, 3, 4, 5, 6, 7, 8, 9] \\
\midrule
4 & [1, 3, 5, 9] $\to$ [2] & 1.7368 & 0.3300 & 508 & [1, 2, 3] & [0, 1, 2, 3, 4, 5, 6, 7, 8, 9], [0, 1, 2, 3, 4, 5, 6, 7, 8, 9], [0, 1, 2, 3, 4, 5, 6, 7, 8, 9] \\
\midrule
5 & [3, 5, 6, 8, 9] $\to$ [2, 4, 7] & 1.7368 & 0.3300 & 481 & [1, 2, 3] & [0, 1, 2, 3, 4, 5, 6, 7, 8, 9], [0, 1, 2, 3, 4, 5, 6, 7, 8, 9], [0, 1, 2, 3, 4, 5, 6, 7, 8, 9] \\
\midrule
6 & [9, 3, 4, 6] $\to$ [2, 5] & 1.7222 & 0.3100 & 488 & [1, 2, 3] & [0, 1, 2, 3, 4, 5, 6, 7, 8, 9], [0, 1, 2, 3, 4, 5, 6, 7, 8, 9], [0, 1, 2, 3, 4, 5, 6, 7, 8, 9] \\
\midrule
7 & [0, 1, 3, 6, 7, 8, 9] $\to$ [2, 4] & 1.6842 & 0.3200 & 467 & [1, 2, 3] & [0, 1, 2, 3, 4, 5, 6, 7, 8, 9], [0, 1, 2, 3, 4, 5, 6, 7, 8, 9], [0, 1, 2, 3, 4, 5, 6, 7, 8, 9] \\
\midrule
8 & [0, 1, 4, 6, 7, 8] $\to$ [2] & 1.6667 & 0.4000 & 493 & [1, 2, 3] & [0, 1, 2, 3, 4, 5, 6, 7, 8, 9], [0, 1, 2, 3, 4, 5, 6, 7, 8, 9], [0, 1, 2, 3, 4, 5, 6, 7, 8, 9] \\
\midrule
9 & [0, 3, 5, 6, 9] $\to$ [2, 4] & 1.5909 & 0.3500 & 480 & [1, 2, 3] & [0, 1, 2, 3, 4, 5, 6, 7, 8, 9], [0, 1, 2, 3, 4, 5, 6, 7, 8, 9], [0, 1, 2, 3, 4, 5, 6, 7, 8, 9] \\
\midrule
10 & [0, 1, 4, 5, 6, 9] $\to$ [2, 3, 7] & 1.5882 & 0.2700 & 466 & [1, 2, 3] & [0, 1, 2, 3, 4, 5, 6, 7, 8, 9], [0, 1, 2, 3, 4, 5, 6, 7, 8, 9], [0, 1, 2, 3, 4, 5, 6, 7, 8, 9] \\
\bottomrule
\end{tabular}
\begin{tablenotes}
\item[1] Total number of transactions: 1000
\item[2] Minimum support used in mining: 0.1
\item[3] Items correspond to synthetic dataset indices.
\end{tablenotes}
\end{threeparttable}
\end{table}

\subsubsection{GPAR with shifted RBF kernel}

Maximum likelihood optimised length-scale $\ell=6.7080$, shifted distance $d_0=0$, which justifying the fact that GPAR with RBF and shifted RBF kernels produce almost identical results. Kernel shift is not effective after optimisation using this dataset.

\begin{table}[H]
\centering
\scriptsize
\begin{threeparttable}
\caption{Top 10 association rules mined by \textbf{shifted RBF-based GPAR} (ranked in descending order by \textit{\color{red}{support}})}
\label{tab:shifted_RBF_GPAR_mined_rules_rankedBySupport_minSupport01_Synthetic1}
\begin{tabular}{r >{\raggedright\arraybackslash}p{4cm} >{\raggedright\arraybackslash}p{4cm} p{1cm} p{1cm} p{1cm}}
\toprule
\# & Antecedent & Consequent & \color{red}{Supp} & Conf & Lift \\
\midrule
747 & item\_5 & item\_8, item\_6 & 0.5200 & 0.9123 & 2.3392 \\
\midrule
748 & item\_6 & item\_8, item\_5 & 0.5200 & 0.8667 & 2.6289 \\
\midrule
749 & item\_8 & item\_5, item\_6 & 0.5200 & 0.8667 & 2.7778 \\
\midrule
750 & item\_5, item\_6 & item\_8 & 0.5200 & 1.6250 & 2.7013 \\
\midrule
751 & item\_8, item\_5 & item\_6 & 0.5200 & 1.2093 & 2.3810 \\
\midrule
752 & item\_8, item\_6 & item\_5 & 0.5200 & 1.2683 & 2.7368 \\
\midrule
9 & item\_0 & item\_5 & 0.4900 & 0.8305 & 1.7754 \\
\midrule
10 & item\_5 & item\_0 & 0.4900 & 0.9423 & 1.8839 \\
\midrule
57 & item\_3 & item\_8 & 0.4900 & 1.0426 & 1.8148 \\
\midrule
58 & item\_8 & item\_3 & 0.4900 & 0.9423 & 2.2182 \\
\bottomrule
\end{tabular}
\begin{tablenotes}
\item[1] Total number of transactions: 1000
\item[2] Minimum support used in mining: 0.1
\item[3] Items correspond to synthetic dataset indices.
\item[4] Optimised RBF length-scale $\ell=6.7078$.
\end{tablenotes}
\end{threeparttable}
\end{table}

\begin{table}[H]
\centering
\scriptsize
\begin{threeparttable}
\caption{Top 10 association rules mined by \textbf{shifted RBF-based GPAR} (ranked in descending order by \textit{\color{red}{confidence}})}
\label{tab:shifted_RBF_GPAR_mined_rules_rankedByConfidence_minSupport01_Synthetic1}
\begin{tabular}{r >{\raggedright\arraybackslash}p{6.5cm} >{\raggedright\arraybackslash}p{3cm} p{1cm} p{1cm} p{1cm}}
\toprule
\# & Antecedent & Consequent & Supp & \color{red}{Conf} & Lift \\
\midrule
39375 & item\_1, item\_2, item\_3, item\_6, item\_7, item\_9 & item\_0, item\_4 & 0.2500 & 2.2727 & 2.1277 \\
\midrule
41942 & item\_0, item\_1, item\_2, item\_4, item\_5, item\_7, item\_8 & item\_9 & 0.2900 & 2.2308 & 2.7462 \\
\midrule
36022 & item\_1, item\_3, item\_4, item\_5, item\_6, item\_9 & item\_8 & 0.3100 & 2.2143 & 2.2513 \\
\midrule
48098 & item\_2, item\_3, item\_4, item\_5, item\_9 & item\_8, item\_6, item\_7 & 0.3300 & 2.2000 & 3.6264 \\
\midrule
24018 & item\_0, item\_1, item\_3, item\_4, item\_8 & item\_2, item\_7 & 0.3600 & 2.1176 & 2.3256 \\
\midrule
38242 & item\_0, item\_8, item\_2, item\_3 & item\_1, item\_4, item\_5, item\_7 & 0.3300 & 2.0625 & 3.4304 \\
\midrule
48123 & item\_2, item\_4, item\_6, item\_8, item\_9 & item\_3, item\_5, item\_7 & 0.3300 & 2.0625 & 3.3537 \\
\midrule
38572 & item\_0, item\_2, item\_3, item\_5, item\_9 & item\_1, item\_4, item\_7 & 0.3200 & 2.0000 & 3.2389 \\
\midrule
50726 & item\_0, item\_2, item\_3, item\_4, item\_6, item\_7, item\_8, item\_9 & item\_1 & 0.3200 & 2.0000 & 2.8520 \\
\midrule
9962 & item\_9, item\_2, item\_4, item\_5 & item\_7 & 0.4100 & 1.9524 & 4.0675 \\
\bottomrule
\end{tabular}
\begin{tablenotes}
\item[1] Total number of transactions: 1000
\item[2] Minimum support used in mining: 0.1
\item[3] Items correspond to synthetic dataset indices.
\item[4] Optimised RBF length-scale $\ell=6.7078$.
\end{tablenotes}
\end{threeparttable}
\end{table}

\begin{table}[H]
\centering
\scriptsize
\begin{threeparttable}
\caption{Top 10 association rules mined by \textbf{shifted RBF-based GPAR} (ranked in descending order by \textit{\color{red}{lift}})}
\label{tab:shifted_RBF_GPAR_mined_rules_rankedByLift_minSupport01_Synthetic1}
\begin{tabular}{r >{\raggedright\arraybackslash}p{4.5cm} >{\raggedright\arraybackslash}p{5cm} p{1cm} p{1cm} p{1cm}}
\toprule
\# & Antecedent & Consequent & Supp & Conf & \color{red}{Lift} \\
\midrule
50401 & item\_0, item\_8, item\_6, item\_9 & item\_1, item\_2, item\_3, item\_4, item\_7 & 0.3200 & 1.6842 & 12.1212 \\
\midrule
49258 & item\_0, item\_2, item\_4 & item\_1, item\_3, item\_5, item\_6, item\_8, item\_9 & 0.2900 & 0.7073 & 8.7879 \\
\midrule
47999 & item\_3, item\_6, item\_7 & item\_2, item\_4, item\_5, item\_8, item\_9 & 0.3300 & 0.9706 & 8.4615 \\
\midrule
48063 & item\_9, item\_3, item\_4, item\_5 & item\_8, item\_2, item\_6, item\_7 & 0.3300 & 1.1379 & 8.0882 \\
\midrule
6035 & item\_8, item\_5 & item\_0, item\_2, item\_7 & 0.4500 & 0.8491 & 8.0357 \\
\midrule
30142 & item\_4, item\_5, item\_6 & item\_0, item\_9, item\_2, item\_3 & 0.3500 & 0.8750 & 7.7778 \\
\midrule
49479 & item\_0, item\_1, item\_3, item\_5, item\_6 & item\_8, item\_9, item\_2, item\_4 & 0.2900 & 1.2083 & 7.6720 \\
\midrule
40321 & item\_0, item\_1, item\_3, item\_8 & item\_9, item\_2, item\_5, item\_6 & 0.3400 & 1.1724 & 7.3913 \\
\midrule
42338 & item\_8, item\_1, item\_5, item\_6 & item\_0, item\_9, item\_2, item\_7 & 0.2800 & 0.7568 & 7.1429 \\
\midrule
25884 & item\_8, item\_1, item\_5 & item\_0, item\_2, item\_4, item\_7 & 0.4100 & 1.5185 & 6.9728 \\
\bottomrule
\end{tabular}
\begin{tablenotes}
\item[1] Total number of transactions: 1000
\item[2] Minimum support used in mining: 0.1
\item[3] Items correspond to synthetic dataset indices.
\item[4] Optimised RBF length-scale $\ell=6.7078$.
\end{tablenotes}
\end{threeparttable}
\end{table}

\begin{table}[H]
\centering
\scriptsize
\begin{threeparttable}
\caption{Top 10 association rules mined by \textbf{shifted RBF-based GPAR} with item 2 on RHS (min support = 0.1, rules ranked by \color{red}{confidence})}
\label{tab:shifted_RBF_GPAR_analysis_minSupport01_Synthetic1}
\begin{tabular}{r >{\raggedright\arraybackslash}p{3cm} p{1cm} p{1cm} p{1cm} >{\raggedright\arraybackslash}p{2cm} >{\raggedright\arraybackslash}p{5cm}}
\toprule
\# & Rule & \color{red}{Conf} & Supp (prob) & Supp Count & Example Trans. Indices & Example Transactions \\
\midrule
1 & [0, 1, 3, 4, 8] $\to$ [2, 7] & 2.1176 & 0.3600 & 488 & [1, 2, 3] & [0, 1, 2, 3, 4, 5, 6, 7, 8, 9], [0, 1, 2, 3, 4, 5, 6, 7, 8, 9], [0, 1, 2, 3, 4, 5, 6, 7, 8, 9] \\
\midrule
2 & [1, 5, 6, 7, 8] $\to$ [2, 3] & 1.8333 & 0.3300 & 514 & [1, 2, 3] & [0, 1, 2, 3, 4, 5, 6, 7, 8, 9], [0, 1, 2, 3, 4, 5, 6, 7, 8, 9], [0, 1, 2, 3, 4, 5, 6, 7, 8, 9] \\
\midrule
3 & [0, 1, 4, 5, 8] $\to$ [2, 7] & 1.7826 & 0.4100 & 498 & [1, 2, 3] & [0, 1, 2, 3, 4, 5, 6, 7, 8, 9], [0, 1, 2, 3, 4, 5, 6, 7, 8, 9], [0, 1, 2, 3, 4, 5, 6, 7, 8, 9] \\
\midrule
4 & [1, 3, 5, 9] $\to$ [2] & 1.7368 & 0.3300 & 508 & [1, 2, 3] & [0, 1, 2, 3, 4, 5, 6, 7, 8, 9], [0, 1, 2, 3, 4, 5, 6, 7, 8, 9], [0, 1, 2, 3, 4, 5, 6, 7, 8, 9] \\
\midrule
5 & [3, 5, 6, 8, 9] $\to$ [2, 4, 7] & 1.7368 & 0.3300 & 481 & [1, 2, 3] & [0, 1, 2, 3, 4, 5, 6, 7, 8, 9], [0, 1, 2, 3, 4, 5, 6, 7, 8, 9], [0, 1, 2, 3, 4, 5, 6, 7, 8, 9] \\
\midrule
6 & [9, 3, 4, 6] $\to$ [2, 5] & 1.7222 & 0.3100 & 488 & [1, 2, 3] & [0, 1, 2, 3, 4, 5, 6, 7, 8, 9], [0, 1, 2, 3, 4, 5, 6, 7, 8, 9], [0, 1, 2, 3, 4, 5, 6, 7, 8, 9] \\
\midrule
7 & [0, 1, 3, 6, 7, 8, 9] $\to$ [2, 4] & 1.6842 & 0.3200 & 467 & [1, 2, 3] & [0, 1, 2, 3, 4, 5, 6, 7, 8, 9], [0, 1, 2, 3, 4, 5, 6, 7, 8, 9], [0, 1, 2, 3, 4, 5, 6, 7, 8, 9] \\
\bottomrule
\end{tabular}
\begin{tablenotes}
\item[1] Total number of transactions: 1000
\item[2] Minimum support used in mining: 0.1
\item[3] Items correspond to synthetic dataset indices.
\end{tablenotes}
\end{threeparttable}
\end{table}

\subsubsection{GPAR with neural network kernel}

\begin{table}[H]
\centering
\scriptsize
\begin{threeparttable}
\caption{Top 10 association rules mined by \textbf{neural net kernel-based GPAR} (ranked in descending order by \textit{\color{red}{support}})}
\label{tab:nn_GPAR_mined_rules_rankedBySupport_minSupport01_Synthetic1}
\begin{tabular}{r >{\raggedright\arraybackslash}p{4cm} >{\raggedright\arraybackslash}p{4cm} p{1cm} p{1cm} p{1cm}}
\toprule
\# & Antecedent & Consequent & \color{red}{Supp} & Conf & Lift \\
\midrule
3 & item\_0 & item\_5 & 0.3500 & 0.6250 & 1.5909 \\
\midrule
4 & item\_5 & item\_0 & 0.3500 & 0.7955 & 1.5198 \\
\midrule
16 & item\_1 & item\_9 & 0.3400 & 0.6538 & 1.3878 \\
\midrule
17 & item\_9 & item\_1 & 0.3400 & 0.6415 & 1.6473 \\
\midrule
23 & item\_2 & item\_9 & 0.3300 & 0.6000 & 1.1983 \\
\midrule
24 & item\_9 & item\_2 & 0.3300 & 0.6346 & 1.6667 \\
\midrule
6 & item\_0 & item\_8 & 0.3200 & 0.7619 & 1.3640 \\
\midrule
7 & item\_8 & item\_0 & 0.3200 & 0.6809 & 1.1189 \\
\midrule
12 & item\_1 & item\_4 & 0.3200 & 0.5818 & 1.1662 \\
\midrule
13 & item\_4 & item\_1 & 0.3200 & 0.6957 & 1.0796 \\
\bottomrule
\end{tabular}
\begin{tablenotes}
\item[1] Total number of transactions: 1000
\item[2] Minimum support used in mining: 0.1
\item[3] Items correspond to synthetic dataset indices.
\end{tablenotes}
\end{threeparttable}
\end{table}

\begin{table}[H]
\centering
\scriptsize
\begin{threeparttable}
\caption{Top 10 association rules mined by \textbf{neural net kernel-based GPAR} (ranked in descending order by \textit{\color{red}{confidence}})}
\label{tab:nn_GPAR_mined_rules_rankedByConfidence_minSupport01_Synthetic1}
\begin{tabular}{r >{\raggedright\arraybackslash}p{6.5cm} >{\raggedright\arraybackslash}p{3cm} p{1cm} p{1cm} p{1cm}}
\toprule
\# & Antecedent & Consequent & Supp & \color{red}{Conf} & Lift \\
\midrule
223 & item\_8, item\_3, item\_6 & item\_0 & 0.1100 & 3.6667 & 2.0370 \\
\midrule
207 & item\_0, item\_2, item\_5 & item\_4 & 0.1200 & 2.0000 & 1.9640 \\
\midrule
320 & item\_0, item\_8, item\_3, item\_7 & item\_5 & 0.1100 & 1.8333 & 1.7544 \\
\midrule
42 & item\_0, item\_1 & item\_5 & 0.1400 & 1.7500 & 1.8229 \\
\midrule
279 & item\_3, item\_4, item\_5 & item\_8 & 0.1200 & 1.7143 & 3.4985 \\
\midrule
276 & item\_8, item\_9, item\_6 & item\_2 & 0.1000 & 1.6667 & 1.2987 \\
\midrule
319 & item\_0, item\_8, item\_3, item\_5 & item\_7 & 0.1100 & 1.5714 & 2.9255 \\
\midrule
293 & item\_9, item\_3, item\_4 & item\_7 & 0.1200 & 1.5000 & 1.7417 \\
\midrule
57 & item\_0, item\_3 & item\_7 & 0.1900 & 1.4615 & 1.2418 \\
\midrule
242 & item\_0, item\_8, item\_4 & item\_5 & 0.1300 & 1.4444 & 1.5330 \\
\bottomrule
\end{tabular}
\begin{tablenotes}
\item[1] Total number of transactions: 1000
\item[2] Minimum support used in mining: 0.1
\item[3] Items correspond to synthetic dataset indices.
\end{tablenotes}
\end{threeparttable}
\end{table}

\begin{table}[H]
\centering
\scriptsize
\begin{threeparttable}
\caption{Top 10 association rules mined by \textbf{neural net kernel-based GPAR} (ranked in descending order by \textit{\color{red}{lift}})}
\label{tab:nn_GPAR_mined_rules_rankedByLift_minSupport01_Synthetic1}
\begin{tabular}{r >{\raggedright\arraybackslash}p{4.5cm} >{\raggedright\arraybackslash}p{5cm} p{1cm} p{1cm} p{1cm}}
\toprule
\# & Antecedent & Consequent & Supp & Conf & \color{red}{Lift} \\
\midrule
309 & item\_8, item\_6, item\_7 & item\_5 & 0.1000 & 1.2500 & 6.4103 \\
\midrule
317 & item\_8, item\_3, item\_7 & item\_0, item\_5 & 0.1100 & 0.8462 & 5.8201 \\
\midrule
307 & item\_5, item\_6, item\_7 & item\_8 & 0.1000 & 0.8333 & 5.2083 \\
\midrule
258 & item\_8, item\_1, item\_3 & item\_7 & 0.1000 & 1.0000 & 5.1282 \\
\midrule
311 & item\_0, item\_3, item\_7 & item\_8, item\_5 & 0.1100 & 0.7857 & 4.1667 \\
\midrule
313 & item\_0, item\_8, item\_5 & item\_3, item\_7 & 0.1100 & 0.6111 & 4.0000 \\
\midrule
279 & item\_3, item\_4, item\_5 & item\_8 & 0.1200 & 1.7143 & 3.4985 \\
\midrule
314 & item\_0, item\_8, item\_7 & item\_3, item\_5 & 0.1100 & 0.5238 & 3.4375 \\
\midrule
103 & item\_1, item\_3 & item\_4 & 0.1900 & 1.1176 & 3.1987 \\
\midrule
209 & item\_2, item\_4, item\_5 & item\_0 & 0.1200 & 0.6316 & 3.1915 \\
\bottomrule
\end{tabular}
\begin{tablenotes}
\item[1] Total number of transactions: 1000
\item[2] Minimum support used in mining: 0.1
\item[3] Items correspond to synthetic dataset indices.
\end{tablenotes}
\end{threeparttable}
\end{table}

\begin{table}[H]
\centering
\scriptsize
\begin{threeparttable}
\caption{Top 10 association rules mined by \textbf{neural net kernel-based GPAR} with item 2 on RHS (min support = 0.1, rules ranked by \color{red}{confidence})}
\label{tab:nn_GPAR_analysis_minSupport01_Synthetic1}
\begin{tabular}{r >{\raggedright\arraybackslash}p{3cm} p{1cm} p{1cm} p{1cm} >{\raggedright\arraybackslash}p{2cm} >{\raggedright\arraybackslash}p{5cm}}
\toprule
\# & Rule & \color{red}{Conf} & Supp (prob) & Supp Count & Example Trans. Indices & Example Transactions \\
\midrule
1 & [8, 9, 6] $\to$ [2] & 1.6667 & 0.1000 & 542 & [1, 2, 3] & [0, 1, 2, 3, 4, 5, 6, 7, 8, 9], [0, 1, 2, 3, 4, 5, 6, 7, 8, 9], [0, 1, 2, 3, 4, 5, 6, 7, 8, 9] \\
\midrule
2 & [9, 6] $\to$ [2] & 1.0000 & 0.2200 & 573 & [1, 2, 3] & [0, 1, 2, 3, 4, 5, 6, 7, 8, 9], [0, 1, 2, 3, 4, 5, 6, 7, 8, 9], [0, 1, 2, 3, 4, 5, 6, 7, 8, 9] \\
\midrule
3 & [0, 4, 5] $\to$ [2] & 0.8571 & 0.1200 & 531 & [1, 2, 3] & [0, 1, 2, 3, 4, 5, 6, 7, 8, 9], [0, 1, 2, 3, 4, 5, 6, 7, 8, 9], [0, 1, 2, 3, 4, 5, 6, 7, 8, 9] \\
\midrule
4 & [9, 3, 6] $\to$ [2] & 0.8333 & 0.1000 & 543 & [1, 2, 3] & [0, 1, 2, 3, 4, 5, 6, 7, 8, 9], [0, 1, 2, 3, 4, 5, 6, 7, 8, 9], [0, 1, 2, 3, 4, 5, 6, 7, 8, 9] \\
\midrule
5 & [1, 5] $\to$ [2] & 0.7200 & 0.1800 & 565 & [1, 2, 3] & [0, 1, 2, 3, 4, 5, 6, 7, 8, 9], [0, 1, 2, 3, 4, 5, 6, 7, 8, 9], [0, 1, 2, 3, 4, 5, 6, 7, 8, 9] \\
\midrule
6 & [1] $\to$ [2] & 0.6591 & 0.2900 & 613 & [1, 2, 3] & [0, 1, 2, 3, 4, 5, 6, 7, 8, 9], [0, 1, 2, 3, 4, 5, 6, 7, 8, 9], [0, 1, 2, 3, 4, 5, 6, 7, 8, 9] \\
\midrule
7 & [0, 9] $\to$ [2] & 0.6471 & 0.1100 & 555 & [1, 2, 3] & [0, 1, 2, 3, 4, 5, 6, 7, 8, 9], [0, 1, 2, 3, 4, 5, 6, 7, 8, 9], [0, 1, 2, 3, 4, 5, 6, 7, 8, 9] \\
\midrule
8 & [6] $\to$ [2] & 0.6458 & 0.3100 & 649 & [1, 2, 3] & [0, 1, 2, 3, 4, 5, 6, 7, 8, 9], [0, 1, 2, 3, 4, 5, 6, 7, 8, 9], [0, 1, 2, 3, 4, 5, 6, 7, 8, 9] \\
\midrule
9 & [8, 6] $\to$ [2] & 0.6364 & 0.1400 & 586 & [1, 2, 3] & [0, 1, 2, 3, 4, 5, 6, 7, 8, 9], [0, 1, 2, 3, 4, 5, 6, 7, 8, 9], [0, 1, 2, 3, 4, 5, 6, 7, 8, 9] \\
\midrule
10 & [9] $\to$ [2] & 0.6346 & 0.3300 & 629 & [1, 2, 3] & [0, 1, 2, 3, 4, 5, 6, 7, 8, 9], [0, 1, 2, 3, 4, 5, 6, 7, 8, 9], [0, 1, 2, 3, 4, 5, 6, 7, 8, 9] \\
\bottomrule
\end{tabular}
\begin{tablenotes}
\item[1] Total number of transactions: 1000
\item[2] Minimum support used in mining: 0.1
\item[3] Items correspond to synthetic dataset indices.
\end{tablenotes}
\end{threeparttable}
\end{table}

\subsubsection{GPAR with NTK}

\begin{table}[H]
\centering
\scriptsize
\begin{threeparttable}
\caption{Top 10 association rules mined by \textbf{NTK-based GPAR} (ranked in descending order by \textit{\color{red}{support}})}
\label{tab:NTK_GPAR_mined_rules_rankedBySupport_minSupport01_Synthetic1}
\begin{tabular}{r >{\raggedright\arraybackslash}p{4cm} >{\raggedright\arraybackslash}p{4cm} p{1cm} p{1cm} p{1cm}}
\toprule
\# & Antecedent & Consequent & \color{red}{Supp} & Conf & Lift \\
\midrule
2 & item\_0 & item\_3 & 0.3800 & 0.6230 & 1.5833 \\
\midrule
3 & item\_3 & item\_0 & 0.3800 & 0.7308 & 1.4638 \\
\midrule
17 & item\_1 & item\_4 & 0.3700 & 0.6981 & 1.0641 \\
\midrule
18 & item\_4 & item\_1 & 0.3700 & 0.7551 & 1.4111 \\
\midrule
45 & item\_4 & item\_9 & 0.3500 & 0.7447 & 1.2963 \\
\midrule
46 & item\_9 & item\_4 & 0.3500 & 0.6731 & 1.5844 \\
\midrule
55 & item\_7 & item\_8 & 0.3500 & 0.5738 & 1.3183 \\
\midrule
56 & item\_8 & item\_7 & 0.3500 & 0.7609 & 1.3503 \\
\midrule
29 & item\_2 & item\_6 & 0.3400 & 0.6182 & 1.2346 \\
\midrule
30 & item\_6 & item\_2 & 0.3400 & 0.6538 & 1.7172 \\
\bottomrule
\end{tabular}
\begin{tablenotes}
\item[1] Total number of transactions: 1000
\item[2] Minimum support used in mining: 0.1
\item[3] Items correspond to synthetic dataset indices.
\end{tablenotes}
\end{threeparttable}
\end{table}

\begin{table}[H]
\centering
\scriptsize
\begin{threeparttable}
\caption{Top 10 association rules mined by \textbf{NTK-based GPAR} (ranked in descending order by \textit{\color{red}{confidence}})}
\label{tab:NTK_GPAR_mined_rules_rankedByConfidence_minSupport01_Synthetic1}
\begin{tabular}{r >{\raggedright\arraybackslash}p{6.5cm} >{\raggedright\arraybackslash}p{3cm} p{1cm} p{1cm} p{1cm}}
\toprule
\# & Antecedent & Consequent & Supp & \color{red}{Conf} & Lift \\
\midrule
1076 & item\_0, item\_1, item\_8, item\_9 & item\_7 & 0.1100 & 5.5000 & 5.0926 \\
\midrule
1444 & item\_0, item\_5, item\_7, item\_8, item\_9 & item\_3 & 0.1100 & 2.7500 & 2.8646 \\
\midrule
1019 & item\_0, item\_1, item\_8, item\_7 & item\_3 & 0.1000 & 2.5000 & 1.5699 \\
\midrule
1312 & item\_8, item\_2, item\_5, item\_7 & item\_4 & 0.1000 & 2.5000 & 1.9342 \\
\midrule
1350 & item\_8, item\_3, item\_4, item\_5 & item\_9 & 0.1000 & 2.5000 & 1.6667 \\
\midrule
1473 & item\_8, item\_1, item\_4, item\_6 & item\_2, item\_5 & 0.1000 & 2.5000 & 5.7143 \\
\midrule
1369 & item\_8, item\_3, item\_4, item\_6 & item\_7 & 0.1200 & 2.4000 & 2.0513 \\
\midrule
1120 & item\_0, item\_2, item\_5, item\_7 & item\_4 & 0.1100 & 2.2000 & 2.5463 \\
\midrule
859 & item\_8, item\_6, item\_7 & item\_2 & 0.1700 & 2.1250 & 2.3098 \\
\midrule
1335 & item\_8, item\_2, item\_6, item\_7 & item\_9 & 0.1200 & 2.0000 & 1.3158 \\
\bottomrule
\end{tabular}
\begin{tablenotes}
\item[1] Total number of transactions: 1000
\item[2] Minimum support used in mining: 0.1
\item[3] Items correspond to synthetic dataset indices.
\end{tablenotes}
\end{threeparttable}
\end{table}

\begin{table}[H]
\centering
\scriptsize
\begin{threeparttable}
\caption{Top 10 association rules mined by \textbf{NTK-based GPAR} (ranked in descending order by \textit{\color{red}{lift}})}
\label{tab:NTK_GPAR_mined_rules_rankedByLift_minSupport01_Synthetic1}
\begin{tabular}{r >{\raggedright\arraybackslash}p{4.5cm} >{\raggedright\arraybackslash}p{5cm} p{1cm} p{1cm} p{1cm}}
\toprule
\# & Antecedent & Consequent & Supp & Conf & \color{red}{Lift} \\
\midrule
1453 & item\_8, item\_1, item\_4 & item\_2, item\_5, item\_6 & 0.1000 & 0.7692 & 8.5470 \\
\midrule
1463 & item\_8, item\_4, item\_6 & item\_1, item\_2, item\_5 & 0.1000 & 0.7143 & 8.3333 \\
\midrule
1475 & item\_2, item\_4, item\_5, item\_6 & item\_8, item\_1 & 0.1000 & 0.5882 & 8.3333 \\
\midrule
1469 & item\_8, item\_1, item\_2, item\_5 & item\_4, item\_6 & 0.1000 & 0.7143 & 7.6923 \\
\midrule
1451 & item\_1, item\_4, item\_5 & item\_8, item\_2, item\_6 & 0.1000 & 0.6250 & 7.1429 \\
\midrule
1326 & item\_2, item\_6, item\_7 & item\_8, item\_9 & 0.1200 & 0.6667 & 7.1429 \\
\midrule
1187 & item\_0, item\_4, item\_6, item\_7 & item\_5 & 0.1300 & 0.7222 & 7.0652 \\
\midrule
1461 & item\_4, item\_5, item\_6 & item\_8, item\_1, item\_2 & 0.1000 & 0.7143 & 6.9930 \\
\midrule
1456 & item\_2, item\_4, item\_5 & item\_8, item\_1, item\_6 & 0.1000 & 1.6667 & 6.9444 \\
\midrule
1448 & item\_1, item\_2, item\_5 & item\_8, item\_4, item\_6 & 0.1000 & 0.5882 & 6.4935 \\
\bottomrule
\end{tabular}
\begin{tablenotes}
\item[1] Total number of transactions: 1000
\item[2] Minimum support used in mining: 0.1
\item[3] Items correspond to synthetic dataset indices.
\end{tablenotes}
\end{threeparttable}
\end{table}

\begin{table}[H]
\centering
\scriptsize
\begin{threeparttable}
\caption{Top 10 association rules mined by \textbf{NTK-based GPAR} with item 2 on RHS (min support = 0.1, rules ranked by \color{red}{confidence})}
\label{tab:NTK_GPAR_analysis_minSupport01_Synthetic1}
\begin{tabular}{r >{\raggedright\arraybackslash}p{3cm} p{1cm} p{1cm} p{1cm} >{\raggedright\arraybackslash}p{2cm} >{\raggedright\arraybackslash}p{5cm}}
\toprule
\# & Rule & \color{red}{Conf} & Supp (prob) & Supp Count & Example Trans. Indices & Example Transactions \\
\midrule
1 & [8, 1, 4, 6] $\to$ [2, 5] & 2.5000 & 0.1000 & 496 & [1, 2, 3] & [0, 1, 2, 3, 4, 5, 6, 7, 8, 9], [0, 1, 2, 3, 4, 5, 6, 7, 8, 9], [0, 1, 2, 3, 4, 5, 6, 7, 8, 9] \\
\midrule
2 & [8, 6, 7] $\to$ [2] & 2.1250 & 0.1700 & 564 & [1, 2, 3] & [0, 1, 2, 3, 4, 5, 6, 7, 8, 9], [0, 1, 2, 3, 4, 5, 6, 7, 8, 9], [0, 1, 2, 3, 4, 5, 6, 7, 8, 9] \\
\midrule
3 & [9, 5, 7] $\to$ [2] & 1.8333 & 0.2200 & 549 & [1, 2, 3] & [0, 1, 2, 3, 4, 5, 6, 7, 8, 9], [0, 1, 2, 3, 4, 5, 6, 7, 8, 9], [0, 1, 2, 3, 4, 5, 6, 7, 8, 9] \\
\midrule
4 & [1, 4, 6, 9] $\to$ [2] & 1.4286 & 0.1000 & 496 & [1, 2, 3] & [0, 1, 2, 3, 4, 5, 6, 7, 8, 9], [0, 1, 2, 3, 4, 5, 6, 7, 8, 9], [0, 1, 2, 3, 4, 5, 6, 7, 8, 9] \\
\midrule
5 & [9, 5, 6, 7] $\to$ [2] & 1.4286 & 0.1000 & 525 & [1, 2, 3] & [0, 1, 2, 3, 4, 5, 6, 7, 8, 9], [0, 1, 2, 3, 4, 5, 6, 7, 8, 9], [0, 1, 2, 3, 4, 5, 6, 7, 8, 9] \\
\midrule
6 & [0, 8, 3, 9] $\to$ [2] & 1.3750 & 0.1100 & 518 & [1, 2, 3] & [0, 1, 2, 3, 4, 5, 6, 7, 8, 9], [0, 1, 2, 3, 4, 5, 6, 7, 8, 9], [0, 1, 2, 3, 4, 5, 6, 7, 8, 9] \\
\midrule
7 & [0, 8, 4] $\to$ [2] & 1.3636 & 0.1500 & 529 & [1, 2, 3] & [0, 1, 2, 3, 4, 5, 6, 7, 8, 9], [0, 1, 2, 3, 4, 5, 6, 7, 8, 9], [0, 1, 2, 3, 4, 5, 6, 7, 8, 9] \\
\bottomrule
\end{tabular}
\begin{tablenotes}
\item[1] Total number of transactions: 1000
\item[2] Minimum support used in mining: 0.1
\item[3] Items correspond to synthetic dataset indices.
\end{tablenotes}
\end{threeparttable}
\end{table}

\subsubsection{Apriori}

\begin{table}[H]
\centering
\scriptsize
\begin{threeparttable}
\caption{Top 10 association rules mined by \textbf{Apriori} (ranked in descending order by \textit{\color{red}{support}})}
\label{tab:Apriori_mined_rules_rankedBySupport_minSupport01_Synthetic1}
\begin{tabular}{r >{\raggedright\arraybackslash}p{4cm} >{\raggedright\arraybackslash}p{4cm} p{1cm} p{1cm} p{1cm}}
\toprule
\# & Antecedent & Consequent & \color{red}{Supp} & Conf & Lift \\
\midrule
73 & item\_7 & item\_5 & 0.6620 & 0.9220 & 1.2700 \\
\midrule
74 & item\_5 & item\_7 & 0.6620 & 0.9118 & 1.2700 \\
\midrule
85 & item\_7 & item\_8 & 0.6620 & 0.9220 & 1.2877 \\
\midrule
86 & item\_8 & item\_7 & 0.6620 & 0.9246 & 1.2877 \\
\midrule
9 & item\_0 & item\_5 & 0.6600 & 0.9141 & 1.2591 \\
\midrule
10 & item\_5 & item\_0 & 0.6600 & 0.9091 & 1.2591 \\
\midrule
15 & item\_0 & item\_8 & 0.6560 & 0.9086 & 1.2690 \\
\midrule
16 & item\_8 & item\_0 & 0.6560 & 0.9162 & 1.2690 \\
\midrule
13 & item\_0 & item\_7 & 0.6550 & 0.9072 & 1.2635 \\
\midrule
14 & item\_7 & item\_0 & 0.6550 & 0.9123 & 1.2635 \\
\bottomrule
\end{tabular}
\begin{tablenotes}
\item[1] Total number of transactions: 1000
\item[2] Minimum support used in mining: 0.1
\item[3] Items correspond to synthetic dataset indices.
\end{tablenotes}
\end{threeparttable}
\end{table}

\begin{table}[H]
\centering
\scriptsize
\begin{threeparttable}
\caption{Top 10 association rules mined by \textbf{Apriori} (ranked in descending order by \textit{\color{red}{confidence}})}
\label{tab:Apriori_mined_rules_rankedByConfidence_minSupport01_Synthetic1}
\begin{tabular}{r >{\raggedright\arraybackslash}p{6.5cm} >{\raggedright\arraybackslash}p{3cm} p{1cm} p{1cm} p{1cm}}
\toprule
\# & Antecedent & Consequent & Supp & \color{red}{Conf} & Lift \\
\midrule
9272 & item\_4, item\_5, item\_1, item\_8 & item\_7 & 0.5290 & 0.9981 & 1.3901 \\
\midrule
9571 & item\_5, item\_1, item\_8, item\_9 & item\_7 & 0.5280 & 0.9981 & 1.3901 \\
\midrule
13047 & item\_2, item\_0, item\_1, item\_8, item\_5 & item\_7 & 0.5260 & 0.9981 & 1.3901 \\
\midrule
14908 & item\_0, item\_1, item\_8, item\_5, item\_4 & item\_7 & 0.5200 & 0.9981 & 1.3901 \\
\midrule
19929 & item\_2, item\_1, item\_3, item\_8, item\_5 & item\_7 & 0.5190 & 0.9981 & 1.3901 \\
\midrule
15527 & item\_0, item\_1, item\_8, item\_9, item\_5 & item\_7 & 0.5180 & 0.9981 & 1.3901 \\
\midrule
28111 & item\_6, item\_2, item\_0, item\_1, item\_8, item\_5 & item\_7 & 0.5160 & 0.9981 & 1.3901 \\
\midrule
22533 & item\_6, item\_1, item\_8, item\_9, item\_5 & item\_7 & 0.5140 & 0.9981 & 1.3901 \\
\midrule
36175 & item\_6, item\_2, item\_1, item\_3, item\_8, item\_5 & item\_7 & 0.5140 & 0.9981 & 1.3901 \\
\midrule
22099 & item\_1, item\_3, item\_8, item\_9, item\_5 & item\_7 & 0.5130 & 0.9981 & 1.3900 \\
\bottomrule
\end{tabular}
\begin{tablenotes}
\item[1] Total number of transactions: 1000
\item[2] Minimum support used in mining: 0.1
\item[3] Items correspond to synthetic dataset indices.
\end{tablenotes}
\end{threeparttable}
\end{table}

\begin{table}[H]
\centering
\scriptsize
\begin{threeparttable}
\caption{Top 10 association rules mined by \textbf{Apriori} (ranked in descending order by \textit{\color{red}{lift}})}
\label{tab:Apriori_mined_rules_rankedByLift_minSupport01_Synthetic1}
\begin{tabular}{r >{\raggedright\arraybackslash}p{4.5cm} >{\raggedright\arraybackslash}p{5cm} p{1cm} p{1cm} p{1cm}}
\toprule
\# & Antecedent & Consequent & Supp & Conf & \color{red}{Lift} \\
\midrule
56502 & item\_6, item\_0, item\_1, item\_9, item\_7 & item\_2, item\_3, item\_8, item\_5, item\_4 & 0.4640 & 0.9045 & 1.7982 \\
\midrule
56481 & item\_2, item\_3, item\_8, item\_5, item\_4 & item\_6, item\_0, item\_1, item\_9, item\_7 & 0.4640 & 0.9225 & 1.7982 \\
\midrule
56458 & item\_2, item\_1, item\_3, item\_8, item\_5 & item\_6, item\_0, item\_9, item\_7, item\_4 & 0.4640 & 0.8923 & 1.7954 \\
\midrule
56525 & item\_6, item\_0, item\_9, item\_7, item\_4 & item\_2, item\_1, item\_3, item\_8, item\_5 & 0.4640 & 0.9336 & 1.7954 \\
\midrule
56594 & item\_6, item\_8, item\_9, item\_7, item\_4 & item\_2, item\_0, item\_1, item\_3, item\_5 & 0.4640 & 0.9206 & 1.7911 \\
\midrule
56389 & item\_2, item\_0, item\_1, item\_3, item\_5 & item\_6, item\_8, item\_9, item\_7, item\_4 & 0.4640 & 0.9027 & 1.7911 \\
\midrule
56192 & item\_2, item\_0, item\_1, item\_3, item\_8, item\_5 & item\_6, item\_7, item\_4, item\_9 & 0.4640 & 0.9134 & 1.7910 \\
\midrule
56791 & item\_6, item\_7, item\_4, item\_9 & item\_2, item\_0, item\_1, item\_3, item\_8, item\_5 & 0.4640 & 0.9098 & 1.7910 \\
\midrule
56279 & item\_2, item\_3, item\_8, item\_5, item\_7, item\_4 & item\_6, item\_0, item\_1, item\_9 & 0.4640 & 0.9261 & 1.7879 \\
\midrule
56704 & item\_6, item\_0, item\_1, item\_9 & item\_2, item\_3, item\_8, item\_5, item\_7, item\_4 & 0.4640 & 0.8958 & 1.7879 \\
\bottomrule
\end{tabular}
\begin{tablenotes}
\item[1] Total number of transactions: 1000
\item[2] Minimum support used in mining: 0.1
\item[3] Items correspond to synthetic dataset indices.
\end{tablenotes}
\end{threeparttable}
\end{table}

\begin{table}[H]
\centering
\scriptsize
\begin{threeparttable}
\caption{Top 10 association rules mined by \textbf{Apriori} with item 2 on RHS (min support = 0.1, rules ranked by \color{red}{confidence})}
\label{tab:Apriori_analysis_minSupport01_Synthetic1}
\begin{tabular}{r >{\raggedright\arraybackslash}p{7cm} p{0.5cm} p{0.5cm} p{0.5cm} >{\raggedright\arraybackslash}p{0.8cm} >{\raggedright\arraybackslash}p{4cm}}
\toprule
\# & Rule & \color{red}{Conf} & Supp (prob) & Supp Count & Example Trans. Indices & Example Transactions \\
\midrule
1 & [item\_6, item\_1, item\_3, item\_8, item\_9, item\_7, item\_4] $\to$ [item\_2] & 0.9855 & 0.4760 & 476 & [1, 2, 3] & [0, 1, 2, 3, 4, 5, 6, 7, 8, 9], [0, 1, 2, 3, 4, 5, 6, 7, 8, 9], [0, 1, 2, 3, 4, 5, 6, 7, 8, 9] \\
\midrule
2 & [item\_6, item\_1, item\_3, item\_9, item\_5, item\_4] $\to$ [item\_2] & 0.9855 & 0.4760 & 476 & [1, 2, 3] & [0, 1, 2, 3, 4, 5, 6, 7, 8, 9], [0, 1, 2, 3, 4, 5, 6, 7, 8, 9], [0, 1, 2, 3, 4, 5, 6, 7, 8, 9] \\
\midrule
3 & [item\_6, item\_1, item\_3, item\_9, item\_5, item\_7, item\_4] $\to$ [item\_2] & 0.9854 & 0.4740 & 474 & [1, 2, 3] & [0, 1, 2, 3, 4, 5, 6, 7, 8, 9], [0, 1, 2, 3, 4, 5, 6, 7, 8, 9], [0, 1, 2, 3, 4, 5, 6, 7, 8, 9] \\
\midrule
4 & [item\_6, item\_0, item\_1, item\_3, item\_9, item\_4] $\to$ [item\_2] & 0.9854 & 0.4730 & 473 & [1, 2, 3] & [0, 1, 2, 3, 4, 5, 6, 7, 8, 9], [0, 1, 2, 3, 4, 5, 6, 7, 8, 9], [0, 1, 2, 3, 4, 5, 6, 7, 8, 9] \\
\midrule
5 & [item\_6, item\_1, item\_3, item\_8, item\_9, item\_5, item\_4] $\to$ [item\_2] & 0.9854 & 0.4720 & 472 & [1, 2, 3] & [0, 1, 2, 3, 4, 5, 6, 7, 8, 9], [0, 1, 2, 3, 4, 5, 6, 7, 8, 9], [0, 1, 2, 3, 4, 5, 6, 7, 8, 9] \\
\midrule
6 & [item\_6, item\_1, item\_3, item\_8, item\_9, item\_5, item\_7, item\_4] $\to$ [item\_2] & 0.9854 & 0.4710 & 471 & [1, 2, 3] & [0, 1, 2, 3, 4, 5, 6, 7, 8, 9], [0, 1, 2, 3, 4, 5, 6, 7, 8, 9], [0, 1, 2, 3, 4, 5, 6, 7, 8, 9] \\
\midrule
7 & [item\_6, item\_0, item\_1, item\_3, item\_9, item\_7, item\_4] $\to$ [item\_2] & 0.9853 & 0.4690 & 469 & [1, 2, 3] & [0, 1, 2, 3, 4, 5, 6, 7, 8, 9], [0, 1, 2, 3, 4, 5, 6, 7, 8, 9], [0, 1, 2, 3, 4, 5, 6, 7, 8, 9] \\
\bottomrule
\end{tabular}
\begin{tablenotes}
\item[1] Total number of transactions: 1000
\item[2] Minimum support used in mining: 0.1
\item[3] Items correspond to synthetic dataset indices.
\end{tablenotes}
\end{threeparttable}
\end{table}

\subsubsection{FP-Growth}

\begin{table}[H]
\centering
\scriptsize
\begin{threeparttable}
\caption{Top 10 association rules mined by \textbf{FP-Growth} (ranked in descending order by \textit{\color{red}{support}})}
\label{tab:FP-Growth_mined_rules_rankedBySupport_minSupport01_Synthetic1}
\begin{tabular}{r >{\raggedright\arraybackslash}p{4cm} >{\raggedright\arraybackslash}p{4cm} p{1cm} p{1cm} p{1cm}}
\toprule
\# & Antecedent & Consequent & \color{red}{Supp} & Conf & Lift \\
\midrule
51133 & item\_7 & item\_5 & 0.6620 & 0.9220 & 1.2700 \\
\midrule
51134 & item\_5 & item\_7 & 0.6620 & 0.9118 & 1.2700 \\
\midrule
52885 & item\_7 & item\_8 & 0.6620 & 0.9220 & 1.2877 \\
\midrule
52886 & item\_8 & item\_7 & 0.6620 & 0.9246 & 1.2877 \\
\midrule
51003 & item\_0 & item\_5 & 0.6600 & 0.9141 & 1.2591 \\
\midrule
51004 & item\_5 & item\_0 & 0.6600 & 0.9091 & 1.2591 \\
\midrule
52887 & item\_0 & item\_8 & 0.6560 & 0.9086 & 1.2690 \\
\midrule
52888 & item\_8 & item\_0 & 0.6560 & 0.9162 & 1.2690 \\
\midrule
51135 & item\_0 & item\_7 & 0.6550 & 0.9072 & 1.2635 \\
\midrule
51136 & item\_7 & item\_0 & 0.6550 & 0.9123 & 1.2635 \\
\bottomrule
\end{tabular}
\begin{tablenotes}
\item[1] Total number of transactions: 1000
\item[2] Minimum support used in mining: 0.1
\end{tablenotes}
\end{threeparttable}
\end{table}

\begin{table}[H]
\centering
\scriptsize
\begin{threeparttable}
\caption{Top 10 association rules mined by \textbf{FP-Growth} (ranked in descending order by \textit{\color{red}{confidence}})}
\label{tab:FP-Growth_mined_rules_rankedByConfidence_minSupport01_Synthetic1}
\begin{tabular}{r >{\raggedright\arraybackslash}p{6.5cm} >{\raggedright\arraybackslash}p{3cm} p{1cm} p{1cm} p{1cm}}
\toprule
\# & Antecedent & Consequent & Supp & \color{red}{Conf} & Lift \\
\midrule
14128 & item\_4, item\_5, item\_1, item\_8 & item\_7 & 0.5290 & 0.9981 & 1.3901 \\
\midrule
5109 & item\_5, item\_1, item\_8, item\_9 & item\_7 & 0.5280 & 0.9981 & 1.3901 \\
\midrule
2407 & item\_2, item\_0, item\_1, item\_8, item\_5 & item\_7 & 0.5260 & 0.9981 & 1.3901 \\
\midrule
14188 & item\_0, item\_1, item\_8, item\_5, item\_4 & item\_7 & 0.5200 & 0.9981 & 1.3901 \\
\midrule
1085 & item\_2, item\_1, item\_3, item\_8, item\_5 & item\_7 & 0.5190 & 0.9981 & 1.3901 \\
\midrule
7087 & item\_0, item\_1, item\_8, item\_9, item\_5 & item\_7 & 0.5180 & 0.9981 & 1.3901 \\
\midrule
2745 & item\_6, item\_2, item\_0, item\_1, item\_8, item\_5 & item\_7 & 0.5160 & 0.9981 & 1.3901 \\
\midrule
1209 & item\_6, item\_2, item\_1, item\_3, item\_8, item\_5 & item\_7 & 0.5140 & 0.9981 & 1.3901 \\
\midrule
5351 & item\_6, item\_1, item\_8, item\_9, item\_5 & item\_7 & 0.5140 & 0.9981 & 1.3901 \\
\midrule
5783 & item\_1, item\_3, item\_8, item\_9, item\_5 & item\_7 & 0.5130 & 0.9981 & 1.3900 \\
\bottomrule
\end{tabular}
\begin{tablenotes}
\item[1] Total number of transactions: 1000
\item[2] Minimum support used in mining: 0.1
\item[3] Items correspond to synthetic dataset indices.
\end{tablenotes}
\end{threeparttable}
\end{table}

\begin{table}[H]
\centering
\scriptsize
\begin{threeparttable}
\caption{Top 10 association rules mined by \textbf{FP-Growth} (ranked in descending order by \textit{\color{red}{lift}})}
\label{tab:FP-Growth_mined_rules_rankedByLift_minSupport01_Synthetic1}
\begin{tabular}{r >{\raggedright\arraybackslash}p{4.5cm} >{\raggedright\arraybackslash}p{5cm} p{1cm} p{1cm} p{1cm}}
\toprule
\# & Antecedent & Consequent & Supp & Conf & \color{red}{Lift} \\
\midrule
50452 & item\_6, item\_0, item\_1, item\_9, item\_7 & item\_2, item\_3, item\_8, item\_5, item\_4 & 0.4640 & 0.9045 & 1.7982 \\
\midrule
50431 & item\_2, item\_3, item\_8, item\_5, item\_4 & item\_6, item\_0, item\_1, item\_9, item\_7 & 0.4640 & 0.9225 & 1.7982 \\
\midrule
50408 & item\_2, item\_1, item\_3, item\_8, item\_5 & item\_6, item\_0, item\_9, item\_7, item\_4 & 0.4640 & 0.8923 & 1.7954 \\
\midrule
50475 & item\_6, item\_0, item\_9, item\_7, item\_4 & item\_2, item\_1, item\_3, item\_8, item\_5 & 0.4640 & 0.9336 & 1.7954 \\
\midrule
50544 & item\_6, item\_8, item\_9, item\_7, item\_4 & item\_2, item\_0, item\_1, item\_3, item\_5 & 0.4640 & 0.9206 & 1.7911 \\
\midrule
50339 & item\_2, item\_0, item\_1, item\_3, item\_5 & item\_6, item\_8, item\_9, item\_7, item\_4 & 0.4640 & 0.9027 & 1.7911 \\
\midrule
50142 & item\_2, item\_0, item\_1, item\_3, item\_8, item\_5 & item\_6, item\_7, item\_4, item\_9 & 0.4640 & 0.9134 & 1.7910 \\
\midrule
50741 & item\_6, item\_7, item\_4, item\_9 & item\_2, item\_0, item\_1, item\_3, item\_8, item\_5 & 0.4640 & 0.9098 & 1.7910 \\
\midrule
50229 & item\_2, item\_3, item\_8, item\_5, item\_7, item\_4 & item\_6, item\_0, item\_1, item\_9 & 0.4640 & 0.9261 & 1.7879 \\
\midrule
50654 & item\_6, item\_0, item\_1, item\_9 & item\_2, item\_3, item\_8, item\_5, item\_7, item\_4 & 0.4640 & 0.8958 & 1.7879 \\
\bottomrule
\end{tabular}
\begin{tablenotes}
\item[1] Total number of transactions: 1000
\item[2] Minimum support used in mining: 0.1
\item[3] Items correspond to synthetic dataset indices.
\end{tablenotes}
\end{threeparttable}
\end{table}

\begin{table}[H]
\centering
\scriptsize
\begin{threeparttable}
\caption{Top 10 association rules mined by \textbf{FP-Growth} with item 2 on RHS (min support = 0.1, rules ranked by \color{red}{confidence})}
\label{tab:FP-Growth_analysis_minSupport01_Synthetic1}
\begin{tabular}{r >{\raggedright\arraybackslash}p{7cm} p{0.5cm} p{0.5cm} p{0.5cm} >{\raggedright\arraybackslash}p{1cm} >{\raggedright\arraybackslash}p{4cm}}
\toprule
\# & Rule & \color{red}{Conf} & Supp (prob) & Supp Count & Example Trans. Indices & Example Transactions \\
\midrule
1 & [item\_6, item\_1, item\_3, item\_9, item\_5, item\_4] $\to$ [item\_2] & 0.9855 & 0.4760 & 476 & [1, 2, 3] & [0, 1, 2, 3, 4, 5, 6, 7, 8, 9], [0, 1, 2, 3, 4, 5, 6, 7, 8, 9], [0, 1, 2, 3, 4, 5, 6, 7, 8, 9] \\
\midrule
2 & [item\_6, item\_1, item\_3, item\_8, item\_9, item\_7, item\_4] $\to$ [item\_2] & 0.9855 & 0.4760 & 476 & [1, 2, 3] & [0, 1, 2, 3, 4, 5, 6, 7, 8, 9], [0, 1, 2, 3, 4, 5, 6, 7, 8, 9], [0, 1, 2, 3, 4, 5, 6, 7, 8, 9] \\
\midrule
3 & [item\_6, item\_1, item\_3, item\_9, item\_5, item\_7, item\_4] $\to$ [item\_2] & 0.9854 & 0.4740 & 474 & [1, 2, 3] & [0, 1, 2, 3, 4, 5, 6, 7, 8, 9], [0, 1, 2, 3, 4, 5, 6, 7, 8, 9], [0, 1, 2, 3, 4, 5, 6, 7, 8, 9] \\
\midrule
4 & [item\_6, item\_0, item\_1, item\_3, item\_9, item\_4] $\to$ [item\_2] & 0.9854 & 0.4730 & 473 & [1, 2, 3] & [0, 1, 2, 3, 4, 5, 6, 7, 8, 9], [0, 1, 2, 3, 4, 5, 6, 7, 8, 9], [0, 1, 2, 3, 4, 5, 6, 7, 8, 9] \\
\midrule
5 & [item\_6, item\_1, item\_3, item\_8, item\_9, item\_5, item\_4] $\to$ [item\_2] & 0.9854 & 0.4720 & 472 & [1, 2, 3] & [0, 1, 2, 3, 4, 5, 6, 7, 8, 9], [0, 1, 2, 3, 4, 5, 6, 7, 8, 9], [0, 1, 2, 3, 4, 5, 6, 7, 8, 9] \\
\midrule
6 & [item\_6, item\_1, item\_3, item\_8, item\_9, item\_5, item\_7, item\_4] $\to$ [item\_2] & 0.9854 & 0.4710 & 471 & [1, 2, 3] & [0, 1, 2, 3, 4, 5, 6, 7, 8, 9], [0, 1, 2, 3, 4, 5, 6, 7, 8, 9], [0, 1, 2, 3, 4, 5, 6, 7, 8, 9] \\
\midrule
7 & [item\_6, item\_0, item\_1, item\_3, item\_9, item\_7, item\_4] $\to$ [item\_2] & 0.9853 & 0.4690 & 469 & [1, 2, 3] & [0, 1, 2, 3, 4, 5, 6, 7, 8, 9], [0, 1, 2, 3, 4, 5, 6, 7, 8, 9], [0, 1, 2, 3, 4, 5, 6, 7, 8, 9] \\
\bottomrule
\end{tabular}
\begin{tablenotes}
\item[1] Total number of transactions: 1000
\item[2] Minimum support used in mining: 0.1
\item[3] Items correspond to synthetic dataset indices.
\end{tablenotes}
\end{threeparttable}
\end{table}

\subsubsection{Eclat}

\begin{table}[H]
\centering
\scriptsize
\begin{threeparttable}
\caption{Top 10 association rules mined by \textbf{Eclat} (ranked in descending order by \textit{\color{red}{support}})}
\label{tab:Eclat_mined_rules_rankedBySupport_minSupport01_Synthetic1}
\begin{tabular}{r >{\raggedright\arraybackslash}p{7cm} >{\raggedright\arraybackslash}p{2cm} p{1cm} p{1cm} p{1cm}}
\toprule
\# & Antecedent & Consequent & \color{red}{Supp} & Conf & Lift \\
\midrule
495 & item\_6, item\_2, item\_0, item\_1, item\_3, item\_8, item\_5, item\_7 & item\_4 & 0.6620 & 1.0000 & 1.4245 \\
\midrule
496 & item\_6, item\_2, item\_0, item\_1, item\_3, item\_8, item\_5, item\_4 & item\_7 & 0.6620 & 1.0491 & 1.4612 \\
\midrule
497 & item\_6, item\_2, item\_0, item\_1, item\_3, item\_8, item\_7, item\_4 & item\_5 & 0.6620 & 1.0000 & 1.3774 \\
\midrule
498 & item\_6, item\_2, item\_0, item\_1, item\_3, item\_5, item\_7, item\_4 & item\_8 & 0.6620 & 1.0764 & 1.5034 \\
\midrule
499 & item\_6, item\_2, item\_0, item\_1, item\_8, item\_5, item\_7, item\_4 & item\_3 & 0.6620 & 1.0000 & 1.3947 \\
\midrule
500 & item\_6, item\_2, item\_0, item\_3, item\_8, item\_5, item\_7, item\_4 & item\_1 & 0.6620 & 1.0000 & 1.4205 \\
\midrule
501 & item\_2, item\_0, item\_1, item\_3, item\_8, item\_5, item\_7, item\_4 & item\_6 & 0.6620 & 1.0000 & 1.3793 \\
\midrule
502 & item\_6, item\_2, item\_1, item\_3, item\_8, item\_5, item\_7, item\_4 & item\_0 & 0.6620 & 1.0000 & 1.3850 \\
\midrule
503 & item\_6, item\_0, item\_1, item\_3, item\_8, item\_5, item\_7, item\_4 & item\_2 & 0.6620 & 1.0000 & 1.3643 \\
\midrule
504 & item\_6, item\_2, item\_0, item\_1, item\_3, item\_8, item\_5 & item\_7, item\_4 & 0.6620 & 1.0491 & 1.7227 \\
\bottomrule
\end{tabular}
\begin{tablenotes}
\item[1] Total number of transactions: 1000
\item[2] Minimum support used in mining: 0.1
\end{tablenotes}
\end{threeparttable}
\end{table}

\begin{table}[H]
\centering
\scriptsize
\begin{threeparttable}
\caption{Top 10 association rules mined by \textbf{Eclat} (ranked in descending order by \textit{\color{red}{confidence}})}
\label{tab:Eclat_mined_rules_rankedByConfidence_minSupport01_Synthetic1}
\begin{tabular}{r >{\raggedright\arraybackslash}p{4cm} >{\raggedright\arraybackslash}p{6cm} p{1cm} p{1cm} p{1cm}}
\toprule
\# & Antecedent & Consequent & Supp & \color{red}{Conf} & Lift \\
\midrule
642 & item\_2, item\_0, item\_1, item\_3, item\_4 & item\_6, item\_7, item\_5, item\_8 & 0.6620 & 1.1634 & 1.7575 \\
\midrule
764 & item\_2, item\_0, item\_3, item\_4 & item\_6, item\_1, item\_8, item\_5, item\_7 & 0.6620 & 1.1634 & 1.7575 \\
\midrule
789 & item\_4, item\_2, item\_1, item\_3 & item\_6, item\_0, item\_8, item\_5, item\_7 & 0.6620 & 1.1634 & 1.7575 \\
\midrule
824 & item\_3, item\_0, item\_1, item\_4 & item\_6, item\_2, item\_8, item\_5, item\_7 & 0.6620 & 1.1634 & 1.7575 \\
\midrule
897 & item\_2, item\_3, item\_4 & item\_6, item\_0, item\_1, item\_8, item\_5, item\_7 & 0.6620 & 1.1634 & 1.7575 \\
\midrule
918 & item\_3, item\_0, item\_4 & item\_6, item\_2, item\_1, item\_8, item\_5, item\_7 & 0.6620 & 1.1634 & 1.7575 \\
\midrule
943 & item\_3, item\_1, item\_4 & item\_6, item\_2, item\_0, item\_8, item\_5, item\_7 & 0.6620 & 1.1634 & 1.7575 \\
\midrule
989 & item\_3, item\_4 & item\_6, item\_2, item\_0, item\_1, item\_8, item\_5, item\_7 & 0.6620 & 1.1634 & 1.7575 \\
\midrule
3569 & item\_2, item\_0, item\_1, item\_3, item\_4 & item\_7, item\_5 & 0.6620 & 1.1634 & 1.7575 \\
\midrule
3591 & item\_2, item\_0, item\_3, item\_4 & item\_7, item\_5, item\_1 & 0.6620 & 1.1634 & 1.7575 \\
\bottomrule
\end{tabular}
\begin{tablenotes}
\item[1] Total number of transactions: 1000
\item[2] Minimum support used in mining: 0.1
\end{tablenotes}
\end{threeparttable}
\end{table}

\begin{table}[H]
\centering
\scriptsize
\begin{threeparttable}
\caption{Top 10 association rules mined by \textbf{Eclat} (ranked in descending order by \textit{\color{red}{lift}})}
\label{tab:Eclat_mined_rules_rankedByLift_minSupport01_Synthetic1}
\begin{tabular}{r >{\raggedright\arraybackslash}p{5cm} >{\raggedright\arraybackslash}p{4cm} p{1cm} p{1cm} p{1cm}}
\toprule
\# & Antecedent & Consequent & Supp & Conf & \color{red}{Lift} \\
\midrule
562 & item\_2, item\_0, item\_1, item\_3, item\_8, item\_4 & item\_6, item\_7, item\_5 & 0.6620 & 1.1164 & 1.8152 \\
\midrule
645 & item\_2, item\_0, item\_1, item\_8, item\_4 & item\_6, item\_7, item\_5, item\_3 & 0.6620 & 1.1164 & 1.8152 \\
\midrule
651 & item\_2, item\_0, item\_3, item\_8, item\_4 & item\_6, item\_7, item\_5, item\_1 & 0.6620 & 1.1164 & 1.8152 \\
\midrule
681 & item\_2, item\_1, item\_3, item\_8, item\_4 & item\_6, item\_0, item\_5, item\_7 & 0.6620 & 1.1164 & 1.8152 \\
\midrule
716 & item\_0, item\_1, item\_3, item\_8, item\_4 & item\_6, item\_2, item\_7, item\_5 & 0.6620 & 1.1164 & 1.8152 \\
\midrule
767 & item\_4, item\_2, item\_0, item\_8 & item\_6, item\_1, item\_3, item\_5, item\_7 & 0.6620 & 1.1164 & 1.8152 \\
\midrule
792 & item\_4, item\_2, item\_1, item\_8 & item\_6, item\_0, item\_3, item\_5, item\_7 & 0.6620 & 1.1164 & 1.8152 \\
\midrule
798 & item\_4, item\_2, item\_3, item\_8 & item\_6, item\_0, item\_1, item\_5, item\_7 & 0.6620 & 1.1164 & 1.8152 \\
\midrule
827 & item\_4, item\_0, item\_1, item\_8 & item\_6, item\_2, item\_3, item\_5, item\_7 & 0.6620 & 1.1164 & 1.8152 \\
\midrule
833 & item\_4, item\_3, item\_0, item\_8 & item\_6, item\_2, item\_1, item\_5, item\_7 & 0.6620 & 1.1164 & 1.8152 \\
\bottomrule
\end{tabular}
\begin{tablenotes}
\item[1] Total number of transactions: 1000
\item[2] Minimum support used in mining: 0.1
\end{tablenotes}
\end{threeparttable}
\end{table}

\begin{table}[H]
\centering
\scriptsize
\begin{threeparttable}
\caption{Top 10 association rules mined by \textbf{Eclat} with item 2 on RHS (min support = 0.1, rules ranked by \color{red}{confidence})}
\label{tab:Eclat_analysis_minSupport01_Synthetic1}
\begin{tabular}{r >{\raggedright\arraybackslash}p{3cm} p{1cm} p{1cm} p{1cm} >{\raggedright\arraybackslash}p{2cm} >{\raggedright\arraybackslash}p{5cm}}
\toprule
\# & Rule & \color{red}{Conf} & Supp (prob) & Supp Count & Example Trans. Indices & Example Transactions \\
\midrule
1 & [item\_0, item\_1, item\_9] $\to$ [item\_2] & 1.0571 & 0.6290 & 517 & [1, 2, 3] & [0, 1, 2, 3, 4, 5, 6, 7, 8, 9], [0, 1, 2, 3, 4, 5, 6, 7, 8, 9], [0, 1, 2, 3, 4, 5, 6, 7, 8, 9] \\
\midrule
2 & [item\_1, item\_9] $\to$ [item\_2] & 1.0571 & 0.6290 & 556 & [1, 2, 3] & [0, 1, 2, 3, 4, 5, 6, 7, 8, 9], [0, 1, 2, 3, 4, 5, 6, 7, 8, 9], [0, 1, 2, 3, 4, 5, 6, 7, 8, 9] \\
\midrule
3 & [item\_6, item\_0, item\_1] $\to$ [item\_2] & 1.0468 & 0.6490 & 536 & [1, 2, 3] & [0, 1, 2, 3, 4, 5, 6, 7, 8, 9], [0, 1, 2, 3, 4, 5, 6, 7, 8, 9], [0, 1, 2, 3, 4, 5, 6, 7, 8, 9] \\
\midrule
4 & [item\_6, item\_1] $\to$ [item\_2] & 1.0468 & 0.6490 & 584 & [1, 2, 3] & [0, 1, 2, 3, 4, 5, 6, 7, 8, 9], [0, 1, 2, 3, 4, 5, 6, 7, 8, 9], [0, 1, 2, 3, 4, 5, 6, 7, 8, 9] \\
\midrule
5 & [item\_6, item\_0] $\to$ [item\_2] & 1.0401 & 0.6490 & 574 & [1, 2, 3] & [0, 1, 2, 3, 4, 5, 6, 7, 8, 9], [0, 1, 2, 3, 4, 5, 6, 7, 8, 9], [0, 1, 2, 3, 4, 5, 6, 7, 8, 9] \\
\midrule
6 & [item\_0, item\_9] $\to$ [item\_2] & 1.0362 & 0.6290 & 555 & [1, 2, 3] & [0, 1, 2, 3, 4, 5, 6, 7, 8, 9], [0, 1, 2, 3, 4, 5, 6, 7, 8, 9], [0, 1, 2, 3, 4, 5, 6, 7, 8, 9] \\
\midrule
7 & [item\_0, item\_1] $\to$ [item\_2] & 1.0200 & 0.6130 & 551 & [1, 2, 3] & [0, 1, 2, 3, 4, 5, 6, 7, 8, 9], [0, 1, 2, 3, 4, 5, 6, 7, 8, 9], [0, 1, 2, 3, 4, 5, 6, 7, 8, 9] \\
\bottomrule
\end{tabular}
\begin{tablenotes}
\item[1] Total number of transactions: 1000
\item[2] Minimum support used in mining: 0.1
\end{tablenotes}
\end{threeparttable}
\end{table}

Comparison of the seven algorithms (four GPAR methods with different kernels, Apriori, FP-Growth, Eclat), using the same \textit{Synthetic 1} data, is listed in Tables \ref{tab:synthetic1_comparison_runtime} to \ref{tab:synthetic1_comparison_noRules} and visualized in Fig.\ref{fig:synthetic1_performance_comparison}. Based on the tables and plots, the following conclusions can be drawn: (1) Runtime. GPAR with the neural network kernel is the fastest among the GPAR variants across all minimum support levels, with runtimes as low as 0.133 seconds at a minimum support of 0.1, followed by GPAR with NTK at 0.449 seconds. In contrast, GPAR with RBF and shifted RBF kernels exhibit significantly higher runtimes, ranging from 8.930 to 10.179 seconds at the same threshold, though their runtimes decrease sharply as the minimum support increases. Apriori and Eclat are consistently fast, with average runtimes below 0.635 seconds across all minimum support levels, while FP-Growth is slower, averaging around 2.019 to 2.967 seconds but still outperforming GPAR with RBF and shifted RBF kernels. (2) Memory usage. All GPAR variants (RBF, shifted RBF, Neural Network, NTK) exhibit negligible memory usage across all minimum support levels, consistently at 0.0 MB. Apriori and FP-Growth also show minimal memory usage, with Apriori at 0.0 MB and FP-Growth peaking at 0.063 MB at a minimum support of 0.1. Eclat, however, has higher memory usage at certain thresholds, peaking at 18.149 MB at a minimum support of 0.1 and 7.969 MB at 0.4, though it drops to 0.0 MB at a minimum support of 0.2. (3) Number of frequent itemsets. GPAR with RBF and shifted RBF kernels generates a comparable and high number of frequent itemsets at lower minimum support levels (1013 at 0.1 for both), significantly outperforming GPAR with Neural Network (140) and NTK (305) kernels at the same threshold. Apriori and FP-Growth consistently generate 1023 frequent itemsets across minimum support levels 0.1 to 0.4, dropping to 746 at 0.5, while Eclat consistently produces the highest number at 1033 across all levels. At higher minimum support levels (0.4–0.5), GPAR with neural network and NTK kernels generates very few to no itemsets (0–2), while RBF and shifted RBF kernels generate 65–67 and 1–2, respectively. (4) Number of rules. GPAR with RBF and shifted RBF kernels generates a high number of rules at lower minimum support levels (54036 at 0.1 for both), far exceeding GPAR with Neural Network (322) and NTK (1485) kernels at the same threshold. Apriori, FP-Growth, and Eclat generate the highest number of rules at 57002 for minimum support levels 0.1 to 0.4, with Eclat maintaining this level even at 0.5, while Apriori and FP-Growth drop to 20948 at 0.5.

The choice of kernel in GPAR significantly impacts its performance. In terms of efficiency, GPAR with the neural network kernel is the most efficient in runtime, making it ideal for scenarios where speed is important, followed closely by GPAR with NTK. However, both generate significantly fewer itemsets and rules, which may limit their utility in discovering comprehensive patterns. GPAR with RBF and shifted RBF kernels, while being the same after optimisation, generates a higher number of itemsets and rules, offering greater utility for pattern discovery but at the cost of substantially higher runtimes. Apriori, FP-Growth, and Eclat are generally faster than most GPAR variants and consistently generate more itemsets and rules, with Eclat being particularly robust, maintaining high rule counts across all minimum support thresholds. Memory usage remains low across all algorithms, with Eclat showing occasional spikes at specific thresholds, making it less memory-efficient compared to others in those cases.

An analysis of the association rules generated by the 7 algorithms (GPAR with RBF, shifted RBF, neural network kernel, NTK, Apriori, FP-Growth, and Eclat) on the \textit{Synthetic 1} dataset reveals notable similarities and differences in the mined rules, particularly in their support, confidence, and lift metrics. When ranked by support, GPAR with RBF and shifted RBF kernels produce identical top rules, such as 'item\_5 → item\_8, item\_6' with a support of 0.5200. Apriori and FP-Growth also share identical top rules by support, such as 'item\_7 → item\_5' with a support of 0.6620, reflecting their algorithmic consistency in identifying frequent itemsets, though these rules differ from GPAR's due to their distinct approaches to pattern discovery. Eclat, while also identifying high-support rules involving items 0–8, focuses on larger itemsets, such as 'item\_6, item\_2, ..., item\_7 → item\_4' with the same support of 0.6620. When ranked by confidence, Apriori, FP-Growth, and Eclat consistently prioritize rules predicting 'item\_7' (e.g. confidence of 0.9981 for Apriori and FP-Growth), suggesting a strong association with item\_7 across these methods, whereas GPAR variants with different kernels highlight diverse consequents, such as 'item\_0, item\_4' for RBF/shifted RBF (confidence 2.2727) and 'item\_7' for NTK (confidence 5.5000). Rules ranked by lift further highlight algorithmic differences: GPAR with RBF/shifted RBF achieves the highest lift (12.1212 for 'item\_0, item\_8, item\_6, item\_9 → item\_1, item\_2, item\_3, item\_4, item\_7'), indicating strong associations in complex itemsets, while NTK prioritizes smaller itemsets with high lift (e.g. 8.5470 for 'item\_8, item\_1, item\_4 → item\_2, item\_5, item\_6'). For rules with item\_2 on the RHS, GPAR with NTK emphasizes high-confidence rules involving item\_8 (e.g. confidence 2.5000 for '[8, 1, 4, 6] → [2, 5]'), while Apriori, FP-Growth, and Eclat consistently identify rules with item\_6 and item\_1 (confidence 0.9855), reflecting a shared focus on these items. Overall, while traditional algorithms like Apriori, FP-Growth, and Eclat exhibit high consistency in rule patterns, GPAR's kernel choice significantly influences the diversity and strength of the associations uncovered, with NTK favoring high-confidence predictions and RBF/shifted RBF excelling in identifying high-lift, complex patterns.

\begin{table}[H]
\centering
\scriptsize
\begin{threeparttable}
\caption{Runtime performance (in \textit{seconds}) of different AR mining algorithms (\textit{Synthetic 1})}
\label{tab:synthetic1_comparison_runtime}
\begin{tabular}{p{1cm} p{1.5cm} p{1.5cm} p{1.5cm} p{1.5cm} p{1.5cm} p{1.5cm} p{1.5cm}}
\toprule
Min Support & GPAR (RBF) & GPAR (shifted RBF) & GPAR (neural net) & GPAR (NTK) & Apriori (AVG) & FP-Growth (AVG) & Eclat (AVG) \\
\midrule
0.1 & 10.179 & 8.930 & 0.133 & 0.449 & 0.483 & 2.967 & 0.556 \\
0.2 & 8.448 & 8.697 & 0.046 & 0.077 & 0.505 & 2.734 & 0.555 \\
0.3 & 3.295 & 2.658 & 0.017 & 0.030 & 0.565 & 2.622 & 0.575 \\
0.4 & 0.215 & 0.174 & 0.002 & 0.016 & 0.570 & 2.700 & 0.558 \\
0.5 & 0.013 & 0.034 & 0.005 & 0.003 & 0.211 & 2.019 & 0.635 \\
\bottomrule
\end{tabular}
\begin{tablenotes}
\item[1] Total number of transactions: 1000.
\item[2] Runtimes are measured in seconds.
\item[3] Minimum support values range from 0.1 to 0.5.
\end{tablenotes}
\end{threeparttable}
\end{table}

\begin{table}[H]
\centering
\scriptsize
\begin{threeparttable}
\caption{Memory usage (in \textit{MB}) of different AR mining algorithms (\textit{Synthetic 1})}
\label{tab:synthetic1_comparison_memory}
\begin{tabular}{p{1cm} p{1.5cm} p{1.5cm} p{1.5cm} p{1.5cm} p{1.5cm} p{1.5cm} p{1.5cm}}
\toprule
Min Support & GPAR (RBF) & GPAR (shifted RBF) & GPAR (neural net) & GPAR (NTK) & Apriori (AVG) & FP-Growth (AVG) & Eclat (AVG) \\
\midrule
0.1 & 0.0 & 0.0 & 0.0 & 0.0 & 0.0 & 0.063 & 18.149 \\
0.2 & 0.0 & 0.0 & 0.0 & 0.0 & 0.0 & 0.0 & 0.0 \\
0.3 & 0.0 & 0.0 & 0.0 & 0.0 & 0.0 & 0.0 & 1.438 \\
0.4 & 0.0 & 0.0 & 0.0 & 0.0 & 0.0 & 0.0 & 7.969 \\
0.5 & 0.0 & 0.0 & 0.0 & 0.0 & 0.0 & 0.0 & 0.156 \\
\bottomrule
\end{tabular}
\begin{tablenotes}
\item[1] Total number of transactions: 1000.
\item[2] Memory usage is measured in megabytes (MB).
\item[3] Minimum support values range from 0.1 to 0.5.
\end{tablenotes}
\end{threeparttable}
\end{table}

\begin{table}[H]
\centering
\scriptsize
\begin{threeparttable}
\caption{No. of frequent itemsets generated by different AR mining algorithms (\textit{Synthetic 1})}
\label{tab:synthetic1_comparison_frequentItemsets}
\begin{tabular}{p{1cm} p{1.5cm} p{1.5cm} p{1.5cm} p{1.5cm} p{1.5cm} p{1.5cm} p{1.5cm}}
\toprule
Min Support & GPAR (RBF) & GPAR (shifted RBF) & GPAR (neural net) & GPAR (NTK) & Apriori (AVG) & FP-Growth (AVG) & Eclat (AVG) \\
\midrule
0.1 & 1013 & 1013 & 140 & 305 & 1023 & 1023 & 1033 \\
0.2 & 981 & 982 & 43 & 67 & 1023 & 1023 & 1033 \\
0.3 & 533 & 536 & 9 & 11 & 1023 & 1023 & 1033 \\
0.4 & 67 & 65 & 0 & 2 & 1023 & 1023 & 1033 \\
0.5 & 1 & 2 & 0 & 0 & 746 & 746 & 1033 \\
\bottomrule
\end{tabular}
\begin{tablenotes}
\item[1] Total number of transactions: 1000.
\item[2] Numbers are counts of frequent itemsets generated.
\item[3] Minimum support values range from 0.1 to 0.5.
\end{tablenotes}
\end{threeparttable}
\end{table}

\begin{table}[H]
\centering
\scriptsize
\begin{threeparttable}
\caption{No. of rules generated by different AR mining algorithms (\textit{Synthetic 1})}
\label{tab:synthetic1_comparison_noRules}
\begin{tabular}{p{1cm} p{1.5cm} p{1.5cm} p{1.5cm} p{1.5cm} p{1.5cm} p{1.5cm} p{1.5cm}}
\toprule
Min Support & GPAR (RBF) & GPAR (shifted RBF) & GPAR (neural net) & GPAR (NTK) & Apriori (AVG) & FP-Growth (AVG) & Eclat (AVG) \\
\midrule
0.1 & 54036 & 54036 & 322 & 1485 & 57002 & 57002 & 57002 \\
0.2 & 51479 & 51328 & 59 & 159 & 57002 & 57002 & 57002 \\
0.3 & 19251 & 16067 & 18 & 26 & 57002 & 57002 & 57002 \\
0.4 & 782 & 530 & 0 & 4 & 57002 & 57002 & 57002 \\
0.5 & 2 & 8 & 0 & 0 & 20948 & 20948 & 57002 \\
\bottomrule
\end{tabular}
\begin{tablenotes}
\item[1] Total number of transactions: 1000.
\item[2] Rules are counts of association rules generated.
\item[3] Minimum support values range from 0.1 to 0.5.
\end{tablenotes}
\end{threeparttable}
\end{table}

\begin{figure}[H]
\centering
\includegraphics[width=1.0\textwidth]{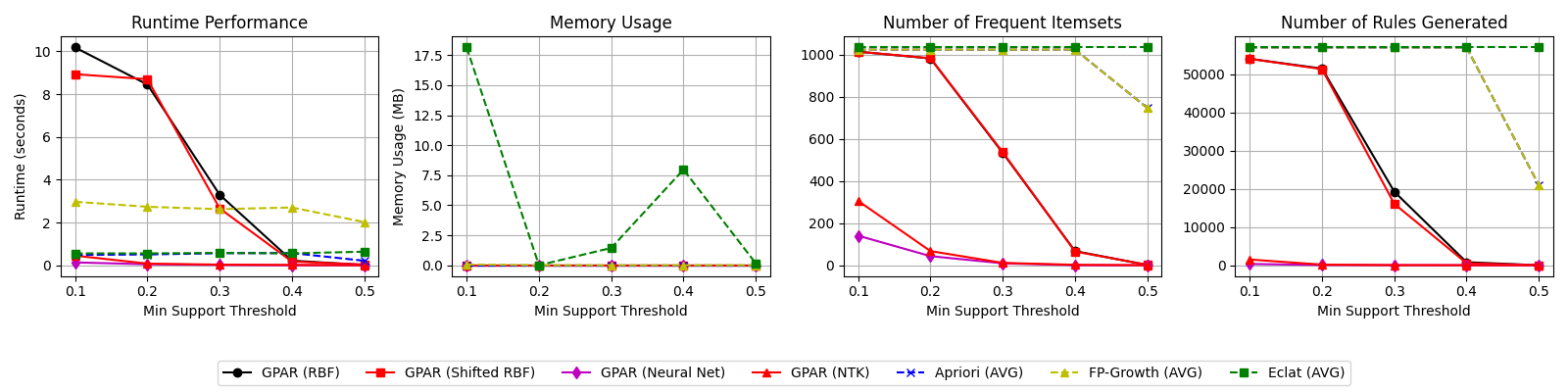}
\caption{Compare the performances of the 7 algorithms.}
\label{fig:synthetic1_performance_comparison}
\end{figure}

\subsection{Synthetic data 2}

\paragraph{Data generation} We aim to construct a controlled, interpretable dataset that emulates real-world transactional patterns while facilitating a fair evaluation of association rule (AR) mining algorithms. The synthetic dataset consists of two matrices: a binary transaction matrix 
$\mathcal{T}_{1000 \times 15}$, which records the presence of 15 items across 1,000 transactions, and a feature matrix $X_{15 \times 10}$, which characterizes each of the 15 items with 10 features. These features represent attributes such as shape, size, price, color, ingredients, manufacturing characteristics, and other relevant properties, encoded as distinct numeric values. The 15 items are generated from 3 Gaussian clusters in the feature space (see Fig.\ref{fig:synthetic2_3clusters}):
\begin{itemize}
    \item Cluster 1 (food, items 0-4): centroid at $\boldsymbol{\mu_1} = [0, 0, \dots, 0]$.
    \item Cluster 2 (drinks, items 5-9): centroid at $\boldsymbol{\mu_2} = [5, 0, \dots, 0]$.
    \item Cluster 3 (daily consumables, items 10-14): centroid at $\boldsymbol{\mu_3} = [20, 20, \dots, 20]$.
\end{itemize}
reflecting different goods categories in a supermarket: food (items 0-4, e.g. bread, cereal), drinks (items 5-9, e.g. water, milk, cola, beer), and daily consumables (items 10-14, e.g. toilet paper, diapers). This clustering ensures that items within the same category exhibit greater similarity in their feature space.  Inter-cluster distances are carefully controlled \footnote{The distance between Cluster 1 and Cluster 2 is: $\sqrt{(5 - 0)^2 + 0 + \dots + 0} = 5$. The distance between Cluster 2 and Cluster 3 is: $\sqrt{(20 - 5)^2 + 20^2 + \dots + 20^2} \approx 61.85$.}: the centroids of the food and drinks clusters are positioned close to each other (distance of 5 units), reflecting their semantic proximity as consumable goods, while the daily consumables cluster is placed farther away (distance $\sim$63 units), emphasizing its distinct category.

This time, transactions are generated using a stratified random sampling approach to ensure diversity in purchasing patterns. Specifically, each transaction contains 2 to 6 items, sampled from 2 or 3 clusters, guaranteeing that items from different categories can be frequently purchased together, ensuring diversity and co-occurrence of complementary products, and enabling the discovery of meaningful cross-category purchasing patterns. This method avoids the bias toward same-cluster items that would result from feature similarity-based sampling, instead promoting cross-cluster co-occurrences that mirror realistic shopping behaviors, such as purchasing food and drinks, beer and diapers together. 

\begin{figure}
    \centering
    \includegraphics[width=0.5\linewidth]{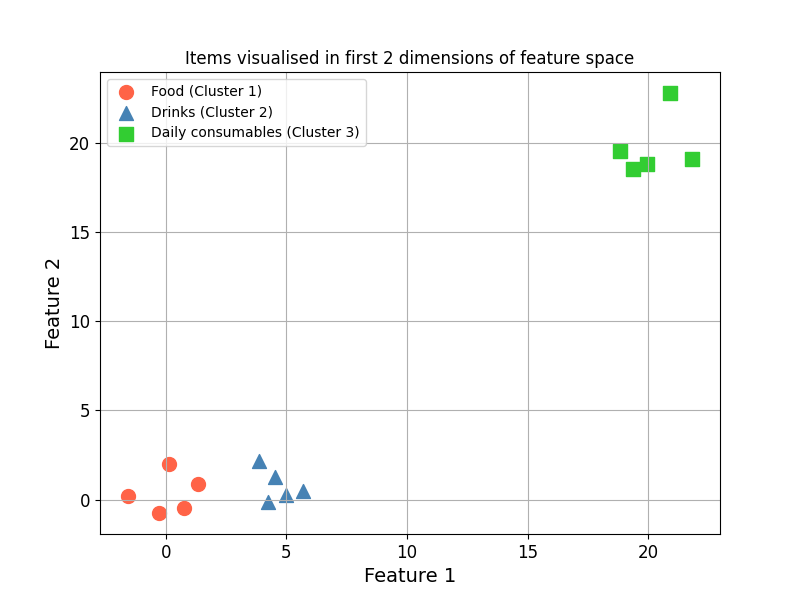}
    \caption{3 Gaussian clusters in 2D feature space for item generation for Synthetic 2 dataset.}
    \label{fig:synthetic2_3clusters}
\end{figure}

\subsubsection{GPAR with RBF kernel} 
The maximum likelihood estimated RBF length scale $\ell=6.7078$.

\begin{table}[H]
\centering
\scriptsize
\begin{threeparttable}
\caption{Top 10 association rules mined by \textbf{RBF-based GPAR} (ranked in descending order by \textit{\color{red}{support}})}
\label{tab:RBF_GPAR_mined_rules_rankedBySupport_minSupport01_Synthetic2}
\begin{tabular}{r >{\raggedright\arraybackslash}p{4cm} >{\raggedright\arraybackslash}p{4cm} p{1cm} p{1cm} p{1cm}}
\toprule
\# & Antecedent & Consequent & \color{red}{Supp} & Conf & Lift \\
\midrule
3 & item\_0 & item\_2 & 0.4600 & 0.9020 & 2.5471 \\
\midrule
4 & item\_2 & item\_0 & 0.4600 & 0.8679 & 1.5780 \\
\midrule
79 & item\_3 & item\_4 & 0.4600 & 0.8214 & 1.6523 \\
\midrule
80 & item\_4 & item\_3 & 0.4600 & 0.8846 & 1.8586 \\
\midrule
143 & item\_6 & item\_9 & 0.4600 & 1.0698 & 1.8254 \\
\midrule
144 & item\_9 & item\_6 & 0.4600 & 0.7931 & 2.2244 \\
\midrule
206 & item\_12 & item\_14 & 0.4500 & 0.8036 & 1.8382 \\
\midrule
207 & item\_14 & item\_12 & 0.4500 & 0.9783 & 1.6642 \\
\midrule
192 & item\_10 & item\_12 & 0.4300 & 0.7818 & 1.4874 \\
\midrule
193 & item\_12 & item\_10 & 0.4300 & 1.0238 & 1.9162 \\
\bottomrule
\end{tabular}
\begin{tablenotes}
\item[1] Total number of transactions: 1000
\item[2] Minimum support used in mining: 0.1
\item[3] Items correspond to synthetic dataset indices.
\item[4] Optimised RBF length-scale $\ell=5.4428$.
\end{tablenotes}
\end{threeparttable}
\end{table}

\begin{table}[H]
\centering
\scriptsize
\begin{threeparttable}
\caption{Top 10 association rules mined by \textbf{RBF-based GPAR} (ranked in descending order by \textit{\color{red}{confidence}})}
\label{tab:RBF_GPAR_mined_rules_rankedByConfidence_minSupport01_Synthetic2}
\begin{tabular}{r >{\raggedright\arraybackslash}p{6cm} >{\raggedright\arraybackslash}p{3cm} p{1cm} p{1cm} p{1cm}}
\toprule
\# & Antecedent & Consequent & Supp & \color{red}{Conf} & Lift \\
\midrule
866292 & item\_1, item\_2, item\_3, item\_4, item\_5, item\_11 & item\_10, item\_12, item\_13 & 0.1500 & 15.0000 & 3.9063 \\
\midrule
928979 & item\_2, item\_3, item\_5, item\_8, item\_9, item\_10, item\_11 & item\_12, item\_13 & 0.1500 & 15.0000 & 5.0676 \\
\midrule
340668 & item\_1, item\_2, item\_5, item\_7, item\_9, item\_14 & item\_8 & 0.1400 & 14.0000 & 6.6038 \\
\midrule
682458 & item\_3, item\_5, item\_6, item\_10, item\_11, item\_14 & item\_4, item\_7 & 0.1400 & 14.0000 & 3.6458 \\
\midrule
916267 & item\_4, item\_5, item\_7, item\_9, item\_12, item\_13, item\_14 & item\_1, item\_10 & 0.1400 & 14.0000 & 4.3210 \\
\midrule
917227 & item\_1, item\_4, item\_5, item\_8, item\_11, item\_12, item\_14 & item\_9, item\_13 & 0.1400 & 14.0000 & 6.7308 \\
\midrule
242173 & item\_0, item\_1, item\_3, item\_11, item\_13, item\_14 & item\_12 & 0.1300 & 13.0000 & 2.2298 \\
\midrule
304000 & item\_0, item\_4, item\_5, item\_9, item\_11, item\_13 & item\_6 & 0.1300 & 13.0000 & 3.6932 \\
\midrule
490273 & item\_0, item\_1, item\_3, item\_5, item\_10, item\_13, item\_14 & item\_12 & 0.1300 & 13.0000 & 4.5139 \\
\midrule
603124 & item\_1, item\_2, item\_8, item\_11, item\_12, item\_14 & item\_9, item\_3 & 0.1300 & 13.0000 & 2.5490 \\
\bottomrule
\end{tabular}
\begin{tablenotes}
\item[1] Total number of transactions: 1000
\item[2] Minimum support used in mining: 0.1
\item[3] Items correspond to synthetic dataset indices.
\item[4] Optimised RBF length-scale $\ell=5.4428$.
\end{tablenotes}
\end{threeparttable}
\end{table}

\begin{table}[H]
\centering
\scriptsize
\begin{threeparttable}
\caption{Top 10 association rules mined by \textbf{RBF-based GPAR} (ranked in descending order by \textit{\color{red}{lift}})}
\label{tab:RBF_GPAR_mined_rules_rankedByLift_minSupport01_Synthetic2}
\begin{tabular}{r >{\raggedright\arraybackslash}p{5cm} >{\raggedright\arraybackslash}p{5cm} p{1cm} p{1cm} p{1cm}}
\toprule
\# & Antecedent & Consequent & Supp & Conf & \color{red}{Lift} \\
\midrule
937598 & item\_2, item\_7, item\_10, item\_12, item\_13 & item\_9, item\_11, item\_4, item\_14 & 0.1000 & 1.1111 & 142.8571 \\
\midrule
808521 & item\_11, item\_9, item\_3, item\_13 & item\_0, item\_2, item\_5, item\_6, item\_10 & 0.1300 & 1.1818 & 130.0000 \\
\midrule
984930 & item\_1, item\_2, item\_4, item\_5, item\_11 & item\_0, item\_7, item\_9, item\_13, item\_14 & 0.1300 & 1.1818 & 130.0000 \\
\midrule
989010 & item\_2, item\_6, item\_8, item\_10, item\_12 & item\_0, item\_1, item\_7, item\_9, item\_11 & 0.1100 & 1.5714 & 122.2222 \\
\midrule
725355 & item\_2, item\_11, item\_14 & item\_0, item\_1, item\_3, item\_5, item\_9, item\_12 & 0.1200 & 0.7059 & 120.0000 \\
\midrule
567861 & item\_11, item\_12, item\_6 & item\_0, item\_4, item\_5, item\_13, item\_14 & 0.1300 & 0.5652 & 118.1818 \\
\midrule
996127 & item\_0, item\_4, item\_6, item\_8, item\_11, item\_12 & item\_1, item\_3, item\_14, item\_7 & 0.1200 & 1.2000 & 109.0909 \\
\midrule
567346 & item\_0, item\_13, item\_6 & item\_4, item\_5, item\_8, item\_9, item\_12 & 0.1300 & 0.7647 & 108.3333 \\
\midrule
973125 & item\_0, item\_1, item\_6, item\_12, item\_13, item\_14 & item\_10, item\_9, item\_2, item\_3 & 0.1200 & 1.7143 & 100.0000 \\
\midrule
973010 & item\_2, item\_3, item\_6, item\_10, item\_13 & item\_0, item\_1, item\_9, item\_12, item\_14 & 0.1200 & 1.0909 & 100.0000 \\
\bottomrule
\end{tabular}
\begin{tablenotes}
\item[1] Total number of transactions: 1000
\item[2] Minimum support used in mining: 0.1
\item[3] Items correspond to synthetic dataset indices.
\item[4] Optimised RBF length-scale $\ell=5.4428$.
\end{tablenotes}
\end{threeparttable}
\end{table}

\begin{table}[H]
\centering
\scriptsize
\begin{threeparttable}
\caption{Top 10 association rules mined by \textbf{RBF-based GPAR} with item 2 on RHS (min support = 0.1, rules ranked by \color{red}{confidence})}
\label{tab:RBF_GPAR_analysis_minSupport01_Synthetic2}
\begin{tabular}{r >{\raggedright\arraybackslash}p{6cm} p{1cm} p{1cm} p{1cm} >{\raggedright\arraybackslash}p{2cm} >{\raggedright\arraybackslash}p{2cm}}
\toprule
\# & Rule & \color{red}{Conf} & Supp (prob) & Supp Count & Example Trans. Indices & Example Transactions \\
\midrule
1 & [0, 3, 4, 6, 8, 12, 13] $\to$ [2, 11] & 13.0000 & 0.1300 & 0 & [] & [] \\
\midrule
2 & [0, 3, 7, 8, 9, 13, 14] $\to$ [2, 12] & 13.0000 & 0.1300 & 0 & [] & [] \\
\midrule
3 & [0, 1, 4, 5, 6, 7, 8, 9, 11, 14] $\to$ [2] & 12.0000 & 0.1200 & 0 & [] & [] \\
\midrule
4 & [0, 4, 5, 7, 9, 10, 11, 14] $\to$ [2, 12, 6] & 12.0000 & 0.1200 & 0 & [] & [] \\
\midrule
5 & [0, 1, 3, 4, 6, 7, 8, 9, 11, 12] $\to$ [2, 5] & 12.0000 & 0.1200 & 0 & [] & [] \\
\midrule
6 & [0, 1, 3, 4, 6, 7, 10, 11] $\to$ [2] & 11.0000 & 0.1100 & 0 & [] & [] \\
\midrule
7 & [0, 1, 3, 9, 12, 13, 14] $\to$ [2, 4] & 11.0000 & 0.1100 & 0 & [] & [] \\
\midrule
8 & [0, 1, 3, 9, 10, 13] $\to$ [2, 11, 7] & 11.0000 & 0.1100 & 0 & [] & [] \\
\midrule
9 & [0, 3, 4, 5, 9, 10, 11, 13] $\to$ [2] & 11.0000 & 0.1100 & 0 & [] & [] \\
\midrule
10 & [0, 3, 6, 11, 13] $\to$ [2, 5, 14, 7] & 11.0000 & 0.1100 & 0 & [] & [] \\
\bottomrule
\end{tabular}
\begin{tablenotes}
\item[1] Total number of transactions: 1000
\item[2] Minimum support used in mining: 0.1
\item[3] Items correspond to synthetic dataset indices.
\end{tablenotes}
\end{threeparttable}
\end{table}

\subsubsection{Apriori}

\begin{table}[H]
\centering
\scriptsize
\begin{threeparttable}
\caption{Top 10 association rules mined by \textbf{Apriori} (ranked in descending order by \textit{\color{red}{support}})}
\label{tab:Apriori_mined_rules_rankedBySupport_minSupport01_Synthetic2}
\begin{tabular}{r >{\raggedright\arraybackslash}p{4cm} >{\raggedright\arraybackslash}p{4cm} p{1cm} p{1cm} p{1cm}}
\toprule
\# & Antecedent & Consequent & \color{red}{Supp} & Conf & Lift \\
\midrule
7 & item\_3 & item\_2 & 0.1100 & 0.3642 & 1.3009 \\
\midrule
8 & item\_2 & item\_3 & 0.1100 & 0.3929 & 1.3009 \\
\midrule
17 & item\_11 & item\_12 & 0.1090 & 0.4082 & 1.4632 \\
\midrule
18 & item\_12 & item\_11 & 0.1090 & 0.3907 & 1.4632 \\
\midrule
11 & item\_9 & item\_5 & 0.1050 & 0.3608 & 1.3266 \\
\midrule
12 & item\_5 & item\_9 & 0.1050 & 0.3860 & 1.3266 \\
\midrule
1 & item\_0 & item\_3 & 0.1030 & 0.3801 & 1.2585 \\
\midrule
2 & item\_3 & item\_0 & 0.1030 & 0.3411 & 1.2585 \\
\midrule
3 & item\_1 & item\_2 & 0.1030 & 0.4055 & 1.4483 \\
\midrule
4 & item\_2 & item\_1 & 0.1030 & 0.3679 & 1.4483 \\
\bottomrule
\end{tabular}
\begin{tablenotes}
\item[1] Total number of transactions: 1000
\item[2] Minimum support used in mining: 0.1
\item[3] Items correspond to synthetic dataset indices.
\end{tablenotes}
\end{threeparttable}
\end{table}

\begin{table}[H]
\centering
\scriptsize
\begin{threeparttable}
\caption{Top 10 association rules mined by \textbf{Apriori} (ranked in descending order by \textit{\color{red}{confidence}})}
\label{tab:Apriori_mined_rules_rankedByConfidence_minSupport01_Synthetic2}
\begin{tabular}{r >{\raggedright\arraybackslash}p{6.5cm} >{\raggedright\arraybackslash}p{3cm} p{1cm} p{1cm} p{1cm}}
\toprule
\# & Antecedent & Consequent & Supp & \color{red}{Conf} & Lift \\
\midrule
17 & item\_11 & item\_12 & 0.1090 & 0.4082 & 1.4632 \\
\midrule
3 & item\_1 & item\_2 & 0.1030 & 0.4055 & 1.4483 \\
\midrule
5 & item\_1 & item\_3 & 0.1020 & 0.4016 & 1.3297 \\
\midrule
14 & item\_6 & item\_9 & 0.1020 & 0.4000 & 1.3746 \\
\midrule
9 & item\_4 & item\_3 & 0.1020 & 0.3969 & 1.3142 \\
\midrule
8 & item\_2 & item\_3 & 0.1100 & 0.3929 & 1.3009 \\
\midrule
18 & item\_12 & item\_11 & 0.1090 & 0.3907 & 1.4632 \\
\midrule
12 & item\_5 & item\_9 & 0.1050 & 0.3860 & 1.3266 \\
\midrule
16 & item\_10 & item\_11 & 0.1020 & 0.3835 & 1.4362 \\
\midrule
15 & item\_11 & item\_10 & 0.1020 & 0.3820 & 1.4362 \\
\bottomrule
\end{tabular}
\begin{tablenotes}
\item[1] Total number of transactions: 1000
\item[2] Minimum support used in mining: 0.1
\item[3] Items correspond to synthetic dataset indices.
\end{tablenotes}
\end{threeparttable}
\end{table}

\begin{table}[H]
\centering
\scriptsize
\begin{threeparttable}
\caption{Top 10 association rules mined by \textbf{Apriori} (ranked in descending order by \textit{\color{red}{lift}})}
\label{tab:Apriori_mined_rules_rankedByLift_minSupport01_Synthetic2}
\begin{tabular}{r >{\raggedright\arraybackslash}p{4.5cm} >{\raggedright\arraybackslash}p{5cm} p{1cm} p{1cm} p{1cm}}
\toprule
\# & Antecedent & Consequent & Supp & Conf & \color{red}{Lift} \\
\midrule
17 & item\_11 & item\_12 & 0.1090 & 0.4082 & 1.4632 \\
\midrule
18 & item\_12 & item\_11 & 0.1090 & 0.3907 & 1.4632 \\
\midrule
3 & item\_1 & item\_2 & 0.1030 & 0.4055 & 1.4483 \\
\midrule
4 & item\_2 & item\_1 & 0.1030 & 0.3679 & 1.4483 \\
\midrule
16 & item\_10 & item\_11 & 0.1020 & 0.3835 & 1.4362 \\
\midrule
15 & item\_11 & item\_10 & 0.1020 & 0.3820 & 1.4362 \\
\midrule
14 & item\_6 & item\_9 & 0.1020 & 0.4000 & 1.3746 \\
\midrule
13 & item\_9 & item\_6 & 0.1020 & 0.3505 & 1.3746 \\
\midrule
5 & item\_1 & item\_3 & 0.1020 & 0.4016 & 1.3297 \\
\midrule
6 & item\_3 & item\_1 & 0.1020 & 0.3377 & 1.3297 \\
\bottomrule
\end{tabular}
\begin{tablenotes}
\item[1] Total number of transactions: 1000
\item[2] Minimum support used in mining: 0.1
\item[3] Items correspond to synthetic dataset indices.
\end{tablenotes}
\end{threeparttable}
\end{table}

\begin{table}[H]
\centering
\scriptsize
\begin{threeparttable}
\caption{Top 10 association rules mined by \textbf{Apriori} with item 2 on RHS (min support = 0.1, rules ranked by \color{red}{confidence})}
\label{tab:Apriori_analysis_minSupport01_Synthetic2}
\begin{tabular}{r >{\raggedright\arraybackslash}p{7cm} p{0.5cm} p{0.5cm} p{0.5cm} >{\raggedright\arraybackslash}p{0.8cm} >{\raggedright\arraybackslash}p{4cm}}
\toprule
\# & Rule & \color{red}{Conf} & Supp (prob) & Supp Count & Example Trans. Indices & Example Transactions \\
\midrule
1 & [item\_1] $\to$ [item\_2] & 0.4055 & 0.1030 & 103 & [0, 10, 16] & [0, 1, 2, 3, 4, 9], [0, 1, 2], [1, 2, 3, 11, 13] \\
\midrule
2 & [item\_3] $\to$ [item\_2] & 0.3642 & 0.1100 & 110 & [0, 6, 11] & [0, 1, 2, 3, 4, 9], [2, 3, 4, 5, 14], [2, 3, 13] \\
\bottomrule
\end{tabular}
\begin{tablenotes}
\item[1] Total number of transactions: 1000
\item[2] Minimum support used in mining: 0.1
\item[3] Items correspond to synthetic dataset indices.
\end{tablenotes}
\end{threeparttable}
\end{table}

\subsubsection{FP-Growth}

\begin{table}[H]
\centering
\scriptsize
\begin{threeparttable}
\caption{Top 10 association rules mined by \textbf{FP-Growth} (ranked in descending order by \textit{\color{red}{support}})}
\label{tab:FP-Growth_mined_rules_rankedBySupport_minSupport01_Synthetic2}
\begin{tabular}{r >{\raggedright\arraybackslash}p{4cm} >{\raggedright\arraybackslash}p{4cm} p{1cm} p{1cm} p{1cm}}
\toprule
\# & Antecedent & Consequent & \color{red}{Supp} & Conf & Lift \\
\midrule
1 & item\_3 & item\_2 & 0.1100 & 0.3642 & 1.3009 \\
\midrule
2 & item\_2 & item\_3 & 0.1100 & 0.3929 & 1.3009 \\
\midrule
11 & item\_11 & item\_12 & 0.1090 & 0.4082 & 1.4632 \\
\midrule
12 & item\_12 & item\_11 & 0.1090 & 0.3907 & 1.4632 \\
\midrule
17 & item\_9 & item\_5 & 0.1050 & 0.3608 & 1.3266 \\
\midrule
18 & item\_5 & item\_9 & 0.1050 & 0.3860 & 1.3266 \\
\midrule
3 & item\_0 & item\_3 & 0.1030 & 0.3801 & 1.2585 \\
\midrule
4 & item\_3 & item\_0 & 0.1030 & 0.3411 & 1.2585 \\
\midrule
7 & item\_1 & item\_2 & 0.1030 & 0.4055 & 1.4483 \\
\midrule
8 & item\_2 & item\_1 & 0.1030 & 0.3679 & 1.4483 \\
\bottomrule
\end{tabular}
\begin{tablenotes}
\item[1] Total number of transactions: 1000
\item[2] Minimum support used in mining: 0.1
\end{tablenotes}
\end{threeparttable}
\end{table}

\begin{table}[H]
\centering
\scriptsize
\begin{threeparttable}
\caption{Top 10 association rules mined by \textbf{FP-Growth} (ranked in descending order by \textit{\color{red}{confidence}})}
\label{tab:FP-Growth_mined_rules_rankedByConfidence_minSupport01_Synthetic2}
\begin{tabular}{r >{\raggedright\arraybackslash}p{6.5cm} >{\raggedright\arraybackslash}p{3cm} p{1cm} p{1cm} p{1cm}}
\toprule
\# & Antecedent & Consequent & Supp & \color{red}{Conf} & Lift \\
\midrule
11 & item\_11 & item\_12 & 0.1090 & 0.4082 & 1.4632 \\
\midrule
7 & item\_1 & item\_2 & 0.1030 & 0.4055 & 1.4483 \\
\midrule
9 & item\_1 & item\_3 & 0.1020 & 0.4016 & 1.3297 \\
\midrule
16 & item\_6 & item\_9 & 0.1020 & 0.4000 & 1.3746 \\
\midrule
5 & item\_4 & item\_3 & 0.1020 & 0.3969 & 1.3142 \\
\midrule
2 & item\_2 & item\_3 & 0.1100 & 0.3929 & 1.3009 \\
\midrule
12 & item\_12 & item\_11 & 0.1090 & 0.3907 & 1.4632 \\
\midrule
18 & item\_5 & item\_9 & 0.1050 & 0.3860 & 1.3266 \\
\midrule
14 & item\_10 & item\_11 & 0.1020 & 0.3835 & 1.4362 \\
\midrule
13 & item\_11 & item\_10 & 0.1020 & 0.3820 & 1.4362 \\
\bottomrule
\end{tabular}
\begin{tablenotes}
\item[1] Total number of transactions: 1000
\item[2] Minimum support used in mining: 0.1
\item[3] Items correspond to synthetic dataset indices.
\end{tablenotes}
\end{threeparttable}
\end{table}

\begin{table}[H]
\centering
\scriptsize
\begin{threeparttable}
\caption{Top 10 association rules mined by \textbf{FP-Growth} (ranked in descending order by \textit{\color{red}{lift}})}
\label{tab:FP-Growth_mined_rules_rankedByLift_minSupport01_Synthetic2}
\begin{tabular}{r >{\raggedright\arraybackslash}p{4.5cm} >{\raggedright\arraybackslash}p{5cm} p{1cm} p{1cm} p{1cm}}
\toprule
\# & Antecedent & Consequent & Supp & Conf & \color{red}{Lift} \\
\midrule
11 & item\_11 & item\_12 & 0.1090 & 0.4082 & 1.4632 \\
\midrule
12 & item\_12 & item\_11 & 0.1090 & 0.3907 & 1.4632 \\
\midrule
7 & item\_1 & item\_2 & 0.1030 & 0.4055 & 1.4483 \\
\midrule
8 & item\_2 & item\_1 & 0.1030 & 0.3679 & 1.4483 \\
\midrule
14 & item\_10 & item\_11 & 0.1020 & 0.3835 & 1.4362 \\
\midrule
13 & item\_11 & item\_10 & 0.1020 & 0.3820 & 1.4362 \\
\midrule
16 & item\_6 & item\_9 & 0.1020 & 0.4000 & 1.3746 \\
\midrule
15 & item\_9 & item\_6 & 0.1020 & 0.3505 & 1.3746 \\
\midrule
9 & item\_1 & item\_3 & 0.1020 & 0.4016 & 1.3297 \\
\midrule
10 & item\_3 & item\_1 & 0.1020 & 0.3377 & 1.3297 \\
\bottomrule
\end{tabular}
\begin{tablenotes}
\item[1] Total number of transactions: 1000
\item[2] Minimum support used in mining: 0.1
\item[3] Items correspond to synthetic dataset indices.
\end{tablenotes}
\end{threeparttable}
\end{table}

\begin{table}[H]
\centering
\scriptsize
\begin{threeparttable}
\caption{Top 10 association rules mined by \textbf{FP-Growth} with item 2 on RHS (min support = 0.1, rules ranked by \color{red}{confidence})}
\label{tab:FP-Growth_analysis_minSupport01_Synthetic2}
\begin{tabular}{r >{\raggedright\arraybackslash}p{3cm} p{0.5cm} p{0.5cm} p{0.5cm} >{\raggedright\arraybackslash}p{2cm} >{\raggedright\arraybackslash}p{6cm}}
\toprule
\# & Rule & \color{red}{Conf} & Supp (prob) & Supp Count & Example Trans. Indices & Example Transactions \\
\midrule
1 & [item\_1] $\to$ [item\_2] & 0.4055 & 0.1030 & 103 & [0, 10, 16] & [0, 1, 2, 3, 4, 9], [0, 1, 2], [1, 2, 3, 11, 13] \\
\midrule
2 & [item\_3] $\to$ [item\_2] & 0.3642 & 0.1100 & 110 & [0, 6, 11] & [0, 1, 2, 3, 4, 9], [2, 3, 4, 5, 14], [2, 3, 13] \\
\bottomrule
\end{tabular}
\begin{tablenotes}
\item[1] Total number of transactions: 1000
\item[2] Minimum support used in mining: 0.1
\item[3] Items correspond to synthetic dataset indices.
\end{tablenotes}
\end{threeparttable}
\end{table}

\subsubsection{Eclat}

\begin{table}[H]
\centering
\scriptsize
\begin{threeparttable}
\caption{Top 10 association rules mined by \textbf{Eclat} (ranked in descending order by \textit{\color{red}{support}})}
\label{tab:Eclat_mined_rules_rankedBySupport_minSupport01_Synthetic2}
\begin{tabular}{r >{\raggedright\arraybackslash}p{7cm} >{\raggedright\arraybackslash}p{2cm} p{1cm} p{1cm} p{1cm}}
\toprule
\# & Antecedent & Consequent & \color{red}{Supp} & Conf & Lift \\
\midrule
5 & item\_1, item\_2 & item\_3 & 0.1100 & 1.0680 & 3.5363 \\
\midrule
6 & item\_1, item\_3 & item\_2 & 0.1100 & 1.0784 & 3.8515 \\
\midrule
7 & item\_3, item\_2 & item\_1 & 0.1100 & 1.0000 & 3.9370 \\
\midrule
8 & item\_1 & item\_3, item\_2 & 0.1100 & 0.4331 & 3.9370 \\
\midrule
9 & item\_2 & item\_1, item\_3 & 0.1100 & 0.3929 & 3.8515 \\
\midrule
10 & item\_3 & item\_1, item\_2 & 0.1100 & 0.3642 & 3.5363 \\
\midrule
13 & item\_3 & item\_2 & 0.1100 & 0.3642 & 1.3009 \\
\midrule
14 & item\_2 & item\_3 & 0.1100 & 0.3929 & 1.3009 \\
\midrule
23 & item\_11 & item\_12 & 0.1090 & 0.4082 & 1.4632 \\
\midrule
24 & item\_12 & item\_11 & 0.1090 & 0.3907 & 1.4632 \\
\bottomrule
\end{tabular}
\begin{tablenotes}
\item[1] Total number of transactions: 1000
\item[2] Minimum support used in mining: 0.1
\end{tablenotes}
\end{threeparttable}
\end{table}

\begin{table}[H]
\centering
\scriptsize
\begin{threeparttable}
\caption{Top 10 association rules mined by \textbf{Eclat} (ranked in descending order by \textit{\color{red}{confidence}})}
\label{tab:Eclat_mined_rules_rankedByConfidence_minSupport01_Synthetic2}
\begin{tabular}{r >{\raggedright\arraybackslash}p{4cm} >{\raggedright\arraybackslash}p{6cm} p{1cm} p{1cm} p{1cm}}
\toprule
\# & Antecedent & Consequent & Supp & \color{red}{Conf} & Lift \\
\midrule
6 & item\_1, item\_3 & item\_2 & 0.1100 & 1.0784 & 3.8515 \\
\midrule
5 & item\_1, item\_2 & item\_3 & 0.1100 & 1.0680 & 3.5363 \\
\midrule
7 & item\_3, item\_2 & item\_1 & 0.1100 & 1.0000 & 3.9370 \\
\midrule
8 & item\_1 & item\_3, item\_2 & 0.1100 & 0.4331 & 3.9370 \\
\midrule
23 & item\_11 & item\_12 & 0.1090 & 0.4082 & 1.4632 \\
\midrule
3 & item\_1 & item\_2 & 0.1030 & 0.4055 & 1.4483 \\
\midrule
11 & item\_1 & item\_3 & 0.1020 & 0.4016 & 1.3297 \\
\midrule
20 & item\_6 & item\_9 & 0.1020 & 0.4000 & 1.3746 \\
\midrule
15 & item\_4 & item\_3 & 0.1020 & 0.3969 & 1.3142 \\
\midrule
9 & item\_2 & item\_1, item\_3 & 0.1100 & 0.3929 & 3.8515 \\
\bottomrule
\end{tabular}
\begin{tablenotes}
\item[1] Total number of transactions: 1000
\item[2] Minimum support used in mining: 0.1
\end{tablenotes}
\end{threeparttable}
\end{table}

\begin{table}[H]
\centering
\scriptsize
\begin{threeparttable}
\caption{Top 10 association rules mined by \textbf{Eclat} (ranked in descending order by \textit{\color{red}{lift}})}
\label{tab:Eclat_mined_rules_rankedByLift_minSupport01_Synthetic2}
\begin{tabular}{r >{\raggedright\arraybackslash}p{5cm} >{\raggedright\arraybackslash}p{4cm} p{1cm} p{1cm} p{1cm}}
\toprule
\# & Antecedent & Consequent & Supp & Conf & \color{red}{Lift} \\
\midrule
7 & item\_3, item\_2 & item\_1 & 0.1100 & 1.0000 & 3.9370 \\
\midrule
8 & item\_1 & item\_3, item\_2 & 0.1100 & 0.4331 & 3.9370 \\
\midrule
6 & item\_1, item\_3 & item\_2 & 0.1100 & 1.0784 & 3.8515 \\
\midrule
9 & item\_2 & item\_1, item\_3 & 0.1100 & 0.3929 & 3.8515 \\
\midrule
10 & item\_3 & item\_1, item\_2 & 0.1100 & 0.3642 & 3.5363 \\
\midrule
5 & item\_1, item\_2 & item\_3 & 0.1100 & 1.0680 & 3.5363 \\
\midrule
23 & item\_11 & item\_12 & 0.1090 & 0.4082 & 1.4632 \\
\midrule
24 & item\_12 & item\_11 & 0.1090 & 0.3907 & 1.4632 \\
\midrule
3 & item\_1 & item\_2 & 0.1030 & 0.4055 & 1.4483 \\
\midrule
4 & item\_2 & item\_1 & 0.1030 & 0.3679 & 1.4483 \\
\bottomrule
\end{tabular}
\begin{tablenotes}
\item[1] Total number of transactions: 1000
\item[2] Minimum support used in mining: 0.1
\end{tablenotes}
\end{threeparttable}
\end{table}

\begin{table}[H]
\centering
\scriptsize
\begin{threeparttable}
\caption{Top 10 association rules mined by \textbf{Eclat} with item 2 on RHS (min support = 0.1, rules ranked by \color{red}{confidence})}
\label{tab:Eclat_analysis_minSupport01_Synthetic2}
\begin{tabular}{r >{\raggedright\arraybackslash}p{3cm} p{1cm} p{1cm} p{1cm} >{\raggedright\arraybackslash}p{2cm} >{\raggedright\arraybackslash}p{5cm}}
\toprule
\# & Rule & \color{red}{Conf} & Supp (prob) & Supp Count & Example Trans. Indices & Example Transactions \\
\midrule
1 & [item\_1, item\_3] $\to$ [item\_2] & 1.0784 & 0.1100 & 45 & [0, 16, 21] & [0, 1, 2, 3, 4, 9], [1, 2, 3, 11, 13], [0, 1, 2, 3, 4, 14] \\
\midrule
2 & [item\_1] $\to$ [item\_2] & 0.4055 & 0.1030 & 103 & [0, 10, 16] & [0, 1, 2, 3, 4, 9], [0, 1, 2], [1, 2, 3, 11, 13] \\
\midrule
3 & [item\_3] $\to$ [item\_2] & 0.3642 & 0.1100 & 110 & [0, 6, 11] & [0, 1, 2, 3, 4, 9], [2, 3, 4, 5, 14], [2, 3, 13] \\
\bottomrule
\end{tabular}
\begin{tablenotes}
\item[1] Total number of transactions: 1000
\item[2] Minimum support used in mining: 0.1
\end{tablenotes}
\end{threeparttable}
\end{table}

Comparison of the 4 algorithms (RBF-based GPAR, Apriori, FP-Growth, and Eclat) on the Synthetic 2 dataset is detailed in Tables \ref{tab:synthetic2_comparison_runtime} to \ref{tab:synthetic2_comparison_noRules} and illustrated in Fig.\ref{fig:synthetic2_performance_comparison_Synthetic2}. There are some distinct differences in their performances: (1) Runtime: Apriori, FP-Growth, and Eclat demonstrate superior runtime efficiency compared to RBF-based GPAR across all minimum support levels. Apriori achieves the fastest runtimes, averaging 0.004 seconds, followed by FP-Growth at 0.020 seconds, and Eclat at 0.095 seconds. In contrast, RBF-based GPAR exhibits significantly higher runtimes, ranging from 231.274 seconds at a minimum support of 0.1 to 1.565 seconds at 0.3, reflecting its computational intensity in inference time. (2) Memory usage: RBF-based GPAR shows substantial memory usage at a minimum support of 0.1 (276.090 MB), dropping to 0.125 MB at 0.15 and 0.0 MB at higher thresholds. Apriori, FP-Growth, and Eclat consistently exhibit negligible memory usage (0.0 MB) across all levels, making them highly memory-efficient. (3) Number of frequent itemsets: RBF-based GPAR generates a significantly higher number of frequent itemsets at lower minimum support levels (10215 at 0.1), decreasing to 157 at 0.3. Apriori and FP-Growth produce fewer itemsets (24 at 0.1, dropping to 1 at 0.3), while Eclat generates slightly more (40 at 0.1, dropping to 2 at 0.3). (4) Number of rules: RBF-based GPAR generates an exceptionally high number of rules at a minimum support of 0.1 (1200793), decreasing to 1230 at 0.3. In contrast, Apriori and FP-Growth produce far fewer rules (18 at 0.1, dropping to 0 at higher thresholds), while Eclat generates 24 rules at 0.1, also dropping to 0 at higher thresholds.

A key insight from the \textit{Synthetic 2} dataset is the unique capability of RBF-based GPAR to identify association rules with high confidence and support, even in the absence of supporting transactions (e.g. rules like '[0, 3, 4, 6, 8, 12, 13] → [2, 11]' with confidence 13.0000 and support 0.1300, despite a support count of 0). This ability arises from GPAR's kernel-based approach, which leverages item feature similarities in the feature matrix $X_{15 \times 10}$ to infer latent patterns not present in the transaction matrix $\mathcal{T}_{1000 \times 15}$. Traditional methods like Apriori, FP-Growth, and Eclat, which rely solely on observed transactional data, produce fewer rules with lower confidence (e.g. Apriori's top rule '[item\_1] → [item\_2]' with confidence 0.4055). When ranked by lift, RBF-based GPAR achieves exceptionally high values (e.g. 142.8571 for '[item\_2, item\_7, item\_10, item\_12, item\_13] → [item\_9, item\_11, item\_4, item\_14]'), uncovering strong, complex associations, while Apriori, FP-Growth, and Eclat focus on simpler patterns with lower lift (e.g. 1.4632 for '[item\_11] → [item\_12]'). However, the computational cost of RBF-based GPAR, in terms of runtime and memory usage, makes it less practical for applications prioritizing efficiency, where Apriori, FP-Growth, and Eclat offer a better balance of performance and resource usage.

\begin{table}[H]
\centering
\scriptsize
\begin{threeparttable}
\caption{Runtime performance (in \textit{seconds}) of different AR mining algorithms (\textit{Synthetic 2})}
\label{tab:synthetic2_comparison_runtime}
\begin{tabular}{p{1.5cm} p{1.5cm} p{1.5cm} p{1.5cm} p{1.5cm}}
\toprule
Min Support & GPAR (RBF) & Apriori & FP-Growth & Eclat \\
\midrule
0.1  & 231.274 & 0.006 & 0.046 & 0.087 \\
0.15 & 24.032  & 0.004 & 0.014 & 0.092 \\
0.2  & 7.002   & 0.003 & 0.017 & 0.104 \\
0.25 & 3.066   & 0.003 & 0.014 & 0.089 \\
0.3  & 1.565   & 0.004 & 0.009 & 0.104 \\
\bottomrule
\end{tabular}
\begin{tablenotes}
\item[1] Total number of transactions: 1000.
\item[2] Runtimes are measured in seconds.
\item[3] Minimum support values range from 0.1 to 0.3.
\end{tablenotes}
\end{threeparttable}
\end{table}

\begin{table}[H]
\centering
\scriptsize
\begin{threeparttable}
\caption{Memory usage (in \textit{MB}) of different AR mining algorithms (\textit{Synthetic 2})}
\label{tab:synthetic2_comparison_memory}
\begin{tabular}{p{1.5cm} p{1.5cm} p{1.5cm} p{1.5cm} p{1.5cm}}
\toprule
Min Support & GPAR (RBF) & Apriori & FP-Growth & Eclat \\
\midrule
0.1  & 276.090 & 0.0 & 0.0 & 0.0 \\
0.15 & 0.125   & 0.0 & 0.0 & 0.0 \\
0.2  & 0.0     & 0.0 & 0.0 & 0.0 \\
0.25 & 0.0     & 0.0 & 0.0 & 0.0 \\
0.3  & 0.0     & 0.0 & 0.0 & 0.0 \\
\bottomrule
\end{tabular}
\begin{tablenotes}
\item[1] Total number of transactions: 1000.
\item[2] Memory usage is measured in megabytes (MB).
\item[3] Minimum support values range from 0.1 to 0.3.
\end{tablenotes}
\end{threeparttable}
\end{table}

\begin{table}[H]
\centering
\scriptsize
\begin{threeparttable}
\caption{No. of frequent itemsets generated by different AR mining algorithms (\textit{Synthetic 2})}
\label{tab:synthetic2_comparison_frequentItemsets}
\begin{tabular}{p{1.5cm} p{1.5cm} p{1.5cm} p{1.5cm} p{1.5cm}}
\toprule
Min Support & GPAR (RBF) & Apriori & FP-Growth & Eclat \\
\midrule
0.1  & 10215 & 24 & 24 & 40 \\
0.15 & 2779  & 15 & 15 & 30 \\
0.2  & 992   & 15 & 15 & 30 \\
0.25 & 368   & 14 & 14 & 28 \\
0.3  & 157   & 1  & 1  & 2  \\
\bottomrule
\end{tabular}
\begin{tablenotes}
\item[1] Total number of transactions: 1000.
\item[2] Numbers are counts of frequent itemsets generated.
\item[3] Minimum support values range from 0.1 to 0.3.
\end{tablenotes}
\end{threeparttable}
\end{table}

\begin{table}[H]
\centering
\scriptsize
\begin{threeparttable}
\caption{No. of rules generated by different AR mining algorithms (\textit{Synthetic 2})}
\label{tab:synthetic2_comparison_noRules}
\begin{tabular}{p{1.5cm} p{1.5cm} p{1.5cm} p{1.5cm} p{1.5cm}}
\toprule
Min Support & GPAR (RBF) & Apriori & FP-Growth & Eclat \\
\midrule
0.1  & 1200793 & 18 & 18 & 24 \\
0.15 & 123415  & 0  & 0  & 0  \\
0.2  & 26228   & 0  & 0  & 0  \\
0.25 & 5620    & 0  & 0  & 0  \\
0.3  & 1230    & 0  & 0  & 0  \\
\bottomrule
\end{tabular}
\begin{tablenotes}
\item[1] Total number of transactions: 1000.
\item[2] Rules are counts of association rules generated.
\item[3] Minimum support values range from 0.1 to 0.3.
\end{tablenotes}
\end{threeparttable}
\end{table}

\begin{figure}[H]
\centering
\includegraphics[width=0.9\textwidth]{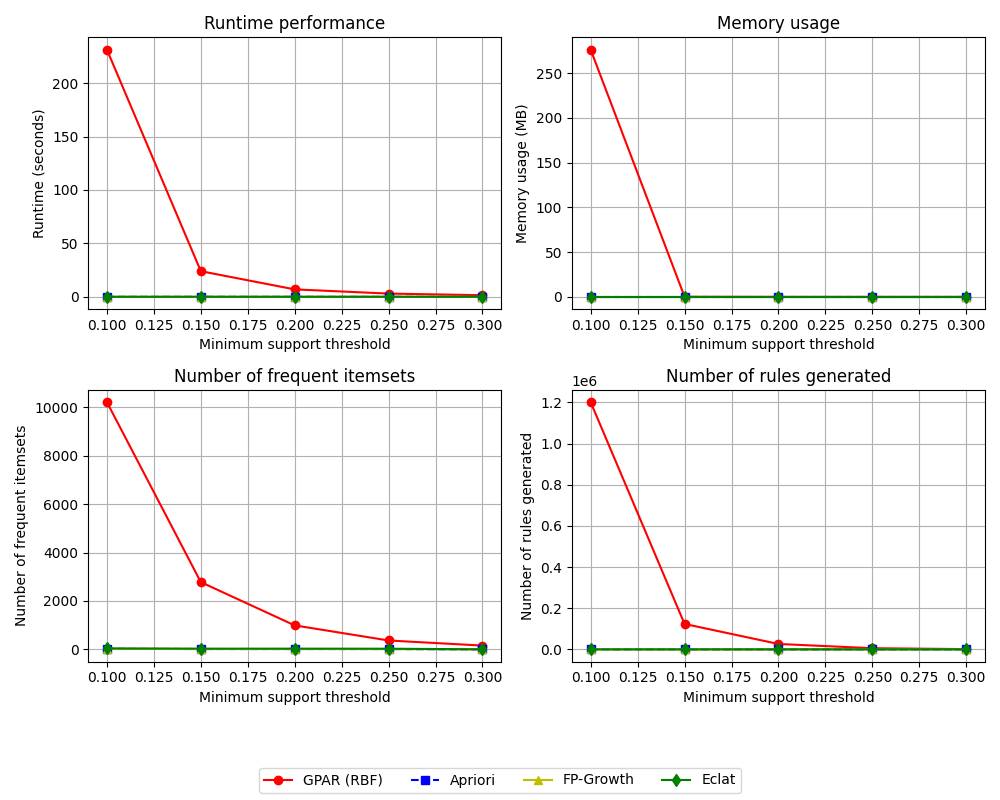}
\caption{Compare the performances of the 4 algorithms (\textit{Synthetic 2}).}
\label{fig:synthetic2_performance_comparison_Synthetic2}
\end{figure}

\section{The \textit{UK Accident} data tests} \label{app:accident_data_tests}
Here we evaluate the performances of the 4 AR mining methods i.e. RBF-based GPAR, Apriori, FP-Growth, Eclat, using a real-world dataset. 

\subsection{Data source}
This study utilizes a real-world dataset of road accident records sourced from the UK government's public data repository \footnote{The dataset is also aggregated and available on Kaggle: \url{https://www.kaggle.com/datasets/silicon99/dft-accident-data/data}.}, specifically the \textit{Road Safety Data} collection maintained by the Department for Transport \cite{road_safety_data}. Known as \textit{STATS19} \footnote{Over the years, the \textit{STATS19} specification has evolved through periodic reviews (e.g. 2005 and 2011 specifications), improving data accuracy (such as increasing coordinate precision from 10 meters to 1 meter) and adapting to new road safety needs, as outlined in the DfT's guidance \cite{stats19_review_2018,stats20_guidance_2005}.}, this dataset has been collected by UK police forces since 1979, capturing detailed information on road accidents involving personal injury. The data used in this experiment spans from 2005 to 2015, covering a total of 1,048,575 records and 32 columns. Each record includes key attributes such as the accident's location (\textit{Longitude}, \textit{Latitude}), date and time (\textit{Date}, \textit{Time}, \textit{Day\_of\_Week}), 
severity (\textit{Accident\_Severity}), number of vehicles and casualties involved (\textit{Number\_of\_Vehicles}, \textit{Number\_of\_Casualties}), 
road conditions (\textit{Road\_Type}, \textit{Speed\_limit}, \textit{Weather\_Conditions}), and additional contextual details (\textit{Junction\_Detail}, \textit{Urban\_or\_Rural\_Area}). Meanings of these variables can be found on \href{https://data.dft.gov.uk/road-accidents-safety-data/Understanding-historical-road-safety-data.docx}{the government public dataset web}.

\subsection{Data pre-processing}
To ensure computational feasibility for GPAR's Monte Carlo sampling-based, marginalised posterior probability estimation, we reduce its size while preserving its representativeness for pattern mining. The dataset was first filtered to include only fatal (\textit{Accident\_Severity} = 1) and serious (\textit{Accident\_Severity} = 2) accidents, focusing on significant incidents likely to yield critical patterns. Non-informative feature such as \textit{Accident\_Index column} was removed, resulting in an intermediate dataset, `\textit{Accidents\_fatal\_serious.csv}`, with 5,231 records and 31 columns, which is used in Section.\ref{app:GPAR_RBF_kernel_results}. In a brief feature engineering practice, time information, e.g. \textit{Hour\_of\_Day}, was extracted as a feature. No missing values found in the dataset.

\subsection{Feature and transaction matrix construction}
The pre-processed dataset was partitioned into two components tailored to comply with GPAR's working mechanism: a feature matrix for GP modeling and a transaction record matrix for itemset mining. Seven categorical columns were selected as \textbf{items} for mining due to their discrete nature and relevance to accident patterns: 
\begin{itemize}
    \item \textit{Accident\_Severity}
    \item \textit{Road\_Type}
    \item \textit{Weather\_Conditions}
    \item \textit{Light\_Conditions}
    \item \textit{Road\_Surface\_Conditions}
    \item \textit{Junction\_Detail}
    \item \textit{Urban\_or\_Rural\_Area}
\end{itemize}

These columns generated 39 unique items for `\textit{Accidents\_fatal\_serious.csv}`, e.g. `\textit{Accident\_Severity}=2`, `\textit{Road\_Type}=6`), as each unique value in a column was encoded as a \textit{distinct} item. The transaction matrix $\mathcal{T} \in \{0, 1\}^{N \times M}$, where $N$ is the number of accidents (rows) and $M$ is the number of items (columns), was constructed by setting $\mathcal{T}_{i,j} = 1$ if the $j$-th item appeared in the $i$-th accident record, and $\mathcal{T}_{i,j} = 0$ otherwise. Each row in $\mathcal{T}_{N \times M}$ represents an accident (similar to a transaction record) recording the appearance of all $M$ items. In our case, we have $\mathcal{T}_{5231 \times 39}$.

For GPAR's GP modeling, 10 columns were selected as input features (also termed input indices) due to their continuous or ordinal nature and relevance to spatial, temporal, and contextual relationships: 
\begin{itemize}
    \item \textit{Longitude}
    \item \textit{Latitude}
    \item \textit{Speed\_limit}
    \item \textit{Number\_of\_Vehicles}
    \item \textit{Day\_of\_Week}
    \item \textit{Hour\_of\_Day}
    \item \textit{Junction\_Control}
    \item \textit{Pedestrian\_Crossing-Human\_Control}
    \item \textit{Pedestrian\_Crossing-Physical\_Facilities}
    \item \textit{1st\_Road\_Class}
\end{itemize}

For each item, a feature vector was computed as the mean of the feature values across all records containing that item, resulting in a feature matrix $X \in \mathbb{R}^{M \times d}$, where $d = 10$ is the number of features. This matrix enables GPAR to model probabilistic relationships between items by leveraging feature similarities, a key distinction from traditional AR mining approaches.

\subsection{RBF-based GPAR \textit{vs} Apriori, FP-Growth, Eclat} \label{app:GPAR_RBF_kernel_results}

The maximum likelihood estimated length scale $\ell=0.2703$. 

\begin{table}[H]
\centering
\scriptsize
\begin{threeparttable}
\caption{Top 10 association rules mined by \textbf{RBF based GPAR} (ranked in descending order by \textit{\color{red}{support}})}
\label{tab:RBF_GPAR_mined_rules_rankedBySupport_minSupport01_Accident}
\begin{tabular}{r >{\raggedright\arraybackslash}p{4cm} >{\raggedright\arraybackslash}p{4cm} p{1cm} p{1cm} p{1cm}}
\toprule
\# & Antecedent & Consequent & Supp & Conf & \color{red}{Lift} \\
\midrule
6 & Accident\_Severity=2 & Weather\_Conditions=1 & 0.4610 & 0.9111 & 1.7182 \\
\midrule
7 & Weather\_Conditions=1 & Accident\_Severity=2 & 0.4610 & 0.9039 & 1.8853 \\
\midrule
25 & Accident\_Severity=2 & Road\_Surface\_Conditions=1 & 0.3500 & 0.6744 & 1.4373 \\
\midrule
26 & Road\_Surface\_Conditions=1 & Accident\_Severity=2 & 0.3500 & 0.6518 & 1.4758 \\
\midrule
4 & Accident\_Severity=2 & Weather\_Conditions=2 & 0.3170 & 0.6084 & 1.2382 \\
\midrule
5 & Weather\_Conditions=2 & Accident\_Severity=2 & 0.3170 & 0.6353 & 1.2315 \\
\midrule
294 & Weather\_Conditions=2 & Weather\_Conditions=1 & 0.3020 & 0.5603 & 1.1896 \\
\midrule
295 & Weather\_Conditions=1 & Weather\_Conditions=2 & 0.3020 & 0.5945 & 1.2662 \\
\midrule
1255 & Accident\_Severity=2 & Weather\_Conditions=1, Road\_Surface\_Conditions=1 & 0.2900 & 0.5765 & 1.9608 \\
\midrule
1256 & Weather\_Conditions=1 & Accident\_Severity=2, Road\_Surface\_Conditions=1 & 0.2900 & 0.5967 & 1.6951 \\
\bottomrule
\end{tabular}
\begin{tablenotes}
\item[1] Total number of accident records: 5231
\item[2] Minimum support used in mining: 0.1
\item[3] Code explanation see later table.
\end{tablenotes}
\end{threeparttable}
\end{table}

\begin{table}[H]
\centering
\scriptsize
\begin{threeparttable}
\caption{Top 10 association rules mined by \textbf{RBF based GPAR} (ranked in descending order by \textit{\color{red}{confidence}})}
\label{tab:RBF_GPAR_mined_rules_rankedByConfidence_minSupport01_Accident}
\begin{tabular}{r >{\raggedright\arraybackslash}p{4cm} >{\raggedright\arraybackslash}p{4cm} p{1cm} p{1cm} p{1cm}}
\toprule
\# & Antecedent & Consequent & Supp & \color{red}{Conf} & Lift \\
\midrule
1 & Accident\_Severity=2, Weather\_Conditions=8, Road\_Type=7 & Weather\_Conditions=1 & 0.1220 & 1.2079 & 1.9717 \\
\midrule
2 & Weather\_Conditions=1, Road\_Surface\_Conditions=2, Weather\_Conditions=7 & Accident\_Severity=2 & 0.1100 & 1.1828 & 1.7332 \\
\midrule
3 & Accident\_Severity=2, Road\_Surface\_Conditions=2, Junction\_Detail=6 & Weather\_Conditions=1 & 0.1240 & 1.1810 & 1.7018 \\
\midrule
4 & Weather\_Conditions=1, Weather\_Conditions=8, Weather\_Conditions=3 & Accident\_Severity=2 & 0.1180 & 1.1800 & 1.8922 \\
\midrule
5 & Accident\_Severity=2, Light\_Conditions=1, Weather\_Conditions=4 & Weather\_Conditions=1 & 0.1140 & 1.1753 & 1.8352 \\
\midrule
6 & Junction\_Detail=7, Weather\_Conditions=1, Road\_Surface\_Conditions=1 & Accident\_Severity=2 & 0.1630 & 1.1727 & 2.0757 \\
\midrule
7 & Accident\_Severity=2, Accident\_Severity=1, Weather\_Conditions=9 & Weather\_Conditions=1 & 0.1040 & 1.1685 & 1.7576 \\
\midrule
8 & Accident\_Severity=2, Road\_Type=2, Junction\_Detail=3 & Weather\_Conditions=1 & 0.1190 & 1.1553 & 1.5351 \\
\midrule
9 & Accident\_Severity=2, Weather\_Conditions=6, Light\_Conditions=6 & Weather\_Conditions=1 & 0.1150 & 1.1500 & 1.7651 \\
\midrule
10 & Accident\_Severity=2, Road\_Type=3, Junction\_Detail=6 & Weather\_Conditions=1 & 0.1440 & 1.1429 & 1.9084 \\
\bottomrule
\end{tabular}
\begin{tablenotes}
\item[1] Total number of accident records: 5231
\item[2] Minimum support used in mining: 0.1
\item[3] Code explanation see later table.
\end{tablenotes}
\end{threeparttable}
\end{table}

\begin{table}[H]
\centering
\scriptsize
\begin{threeparttable}
\caption{Top 10 association rules mined by \textbf{RBF based GPAR} (ranked in descending order by \textit{\color{red}{lift}})}
\label{tab:RBF_GPAR_mined_rules_rankedByLift_minSupport01_Accident}
\begin{tabular}{r >{\raggedright\arraybackslash}p{4cm} >{\raggedright\arraybackslash}p{4cm} p{1cm} p{1cm} p{1cm}}
\toprule
\# & Antecedent & Consequent & Supp & Conf & \color{red}{Lift} \\
\midrule
1 & Accident\_Severity=2, Weather\_Conditions=2, Junction\_Detail=3 & Weather\_Conditions=1, Road\_Surface\_Conditions=1 & 0.1000 & 0.6173 & 2.7236 \\
\midrule
2 & Weather\_Conditions=1, Light\_Conditions=7, Light\_Conditions=6 & Accident\_Severity=2 & 0.1300 & 0.9701 & 2.6912 \\
\midrule
3 & Weather\_Conditions=1, Road\_Type=3 & Accident\_Severity=2, Junction\_Detail=6 & 0.1440 & 0.5926 & 2.6098 \\
\midrule
4 & Weather\_Conditions=2, Weather\_Conditions=1, Weather\_Conditions=5, Road\_Surface\_Conditions=1 & Accident\_Severity=2 & 0.1010 & 1.0745 & 2.6054 \\
\midrule
5 & Accident\_Severity=2, Junction\_Detail=3, Urban\_or\_Rural\_Area=2 & Weather\_Conditions=1 & 0.1230 & 0.9111 & 2.5826 \\
\midrule
6 & Weather\_Conditions=1, Weather\_Conditions=8, Road\_Surface\_Conditions=1 & Accident\_Severity=2, Weather\_Conditions=2 & 0.1040 & 0.6887 & 2.5520 \\
\midrule
7 & Weather\_Conditions=1, Junction\_Detail=6 & Accident\_Severity=2, Road\_Type=3 & 0.1440 & 0.6457 & 2.5273 \\
\midrule
8 & Weather\_Conditions=1, Road\_Surface\_Conditions=5, Road\_Type=2 & Accident\_Severity=2 & 0.1220 & 0.9531 & 2.4792 \\
\midrule
9 & Accident\_Severity=2, Weather\_Conditions=8 & Weather\_Conditions=1, Road\_Type=7 & 0.1220 & 0.5083 & 2.4646 \\
\midrule
10 & Weather\_Conditions=1, Weather\_Conditions=5, Road\_Surface\_Conditions=1 & Accident\_Severity=2, Weather\_Conditions=2 & 0.1010 & 0.8347 & 2.4631 \\
\bottomrule
\end{tabular}
\begin{tablenotes}
\item[1] Total number of accident records: 5231
\item[2] Minimum support used in mining: 0.1
\item[3] Code explanation see later table.
\end{tablenotes}
\end{threeparttable}
\end{table}

\begin{table}[H]
\centering
\scriptsize
\begin{threeparttable}
\caption{Top 10 association rules mined by \textbf{Apriori} (ranked in descending order by \textit{\color{red}{support}})}
\label{tab:Apriori_mined_rules_rankedBySupport_minSupport01_Accident}
\begin{tabular}{r >{\raggedright\arraybackslash}p{4cm} >{\raggedright\arraybackslash}p{4cm} p{1cm} p{1cm} p{1cm}}
\toprule
\# & Antecedent & Consequent & \color{red}{Supp} & Conf & Lift \\
\midrule
1 & Weather\_Conditions=1 & Accident\_Severity=2 & 0.7985 & 0.9280 & 1.0005 \\
\midrule
2 & Accident\_Severity=2 & Weather\_Conditions=1 & 0.7985 & 0.8609 & 1.0005 \\
\midrule
3 & Accident\_Severity=2 & Urban\_or\_Rural\_Area=1 & 0.7790 & 0.8398 & 1.0123 \\
\midrule
4 & Urban\_or\_Rural\_Area=1 & Accident\_Severity=2 & 0.7790 & 0.9389 & 1.0123 \\
\midrule
5 & Road\_Surface\_Conditions=1 & Weather\_Conditions=1 & 0.7514 & 0.9884 & 1.1488 \\
\midrule
6 & Weather\_Conditions=1 & Road\_Surface\_Conditions=1 & 0.7514 & 0.8733 & 1.1488 \\
\midrule
7 & Road\_Type=6 & Accident\_Severity=2 & 0.7325 & 0.9308 & 1.0035 \\
\midrule
8 & Accident\_Severity=2 & Road\_Type=6 & 0.7325 & 0.7897 & 1.0035 \\
\midrule
9 & Weather\_Conditions=1 & Urban\_or\_Rural\_Area=1 & 0.7197 & 0.8364 & 1.0082 \\
\midrule
10 & Urban\_or\_Rural\_Area=1 & Weather\_Conditions=1 & 0.7197 & 0.8675 & 1.0082 \\
\bottomrule
\end{tabular}
\begin{tablenotes}
\item[1] Total number of accident records: 5231
\item[2] Minimum support used in mining: 0.1
\item[3] Code explanation see later table.
\end{tablenotes}
\end{threeparttable}
\end{table}

\begin{table}[H]
\centering
\scriptsize
\begin{threeparttable}
\caption{Top 10 association rules mined by \textbf{Apriori} (ranked in descending order by \textit{\color{red}{confidence}})}
\label{tab:Apriori_mined_rules_rankedByConfidence_minSupport01_Accident}
\begin{tabular}{r >{\raggedright\arraybackslash}p{4cm} >{\raggedright\arraybackslash}p{4cm} p{1cm} p{1cm} p{1cm}}
\toprule
\# & Antecedent & Consequent & Supp & \color{red}{Conf} & Lift \\
\midrule
1 & Road\_Surface\_Conditions=1, Junction\_Detail=6 & Weather\_Conditions=1 & 0.1036 & 0.9982 & 1.1601 \\
\midrule
2 & Weather\_Conditions=2 & Road\_Surface\_Conditions=2 & 0.1025 & 0.9981 & 4.4016 \\
\midrule
3 & Road\_Type=6, Road\_Surface\_Conditions=1, Light\_Conditions=1, Junction\_Detail=3 & Weather\_Conditions=1 & 0.1732 & 0.9945 & 1.1558 \\
\midrule
4 & Road\_Surface\_Conditions=1, Light\_Conditions=1, Junction\_Detail=3 & Weather\_Conditions=1 & 0.2002 & 0.9943 & 1.1556 \\
\midrule
5 & Accident\_Severity=2, Road\_Surface\_Conditions=1, Light\_Conditions=1, Road\_Type=6, Junction\_Detail=3 & Weather\_Conditions=1 & 0.1635 & 0.9942 & 1.1555 \\
\midrule
6 & Road\_Surface\_Conditions=1, Light\_Conditions=1, Urban\_or\_Rural\_Area=1, Road\_Type=6, Junction\_Detail=3 & Weather\_Conditions=1 & 0.1591 & 0.9940 & 1.1553 \\
\midrule
7 & Road\_Surface\_Conditions=1, Light\_Conditions=1, Accident\_Severity=2, Junction\_Detail=3 & Weather\_Conditions=1 & 0.1872 & 0.9939 & 1.1551 \\
\midrule
8 & Road\_Surface\_Conditions=1, Light\_Conditions=1, Junction\_Detail=3, Urban\_or\_Rural\_Area=1 & Weather\_Conditions=1 & 0.1836 & 0.9938 & 1.1550 \\
\midrule
9 & Accident\_Severity=2, Road\_Surface\_Conditions=1, Light\_Conditions=1, Urban\_or\_Rural\_Area=1, Road\_Type=6, Junction\_Detail=3 & Weather\_Conditions=1 & 0.1511 & 0.9937 & 1.1549 \\
\midrule
10 & Accident\_Severity=2, Road\_Surface\_Conditions=1, Light\_Conditions=1, Urban\_or\_Rural\_Area=1, Junction\_Detail=3 & Weather\_Conditions=1 & 0.1730 & 0.9934 & 1.1546 \\
\bottomrule
\end{tabular}
\begin{tablenotes}
\item[1] Total number of accident records: 5231
\item[2] Minimum support used in mining: 0.1
\item[3] Code explanation see later table.
\end{tablenotes}
\end{threeparttable}
\end{table}

\begin{table}[H]
\centering
\scriptsize
\begin{threeparttable}
\caption{Top 10 association rules mined by \textbf{Apriori} (ranked in descending order by \textit{\color{red}{lift}})}
\label{tab:Apriori_mined_rules_rankedByLift_minSupport01_Accident}
\begin{tabular}{r >{\raggedright\arraybackslash}p{4cm} >{\raggedright\arraybackslash}p{4cm} p{1cm} p{1cm} p{1cm}}
\toprule
\# & Antecedent & Consequent & Supp & Conf & \color{red}{Lift} \\
\midrule
1 & Weather\_Conditions=2 & Road\_Surface\_Conditions=2 & 0.1025 & 0.9981 & 4.4016 \\
\midrule
2 & Urban\_or\_Rural\_Area=2 & Junction\_Detail=0 & 0.1084 & 0.6364 & 1.6109 \\
\midrule
3 & Weather\_Conditions=1, Light\_Conditions=1, Accident\_Severity=2, Junction\_Detail=3 & Road\_Type=6, Road\_Surface\_Conditions=1, Urban\_or\_Rural\_Area=1 & 0.1511 & 0.7439 & 1.4388 \\
\midrule
4 & Road\_Surface\_Conditions=1, Junction\_Detail=3, Accident\_Severity=2 & Road\_Type=6, Light\_Conditions=1, Weather\_Conditions=1, Urban\_or\_Rural\_Area=1 & 0.1511 & 0.5579 & 1.4345 \\
\midrule
5 & Light\_Conditions=1, Weather\_Conditions=1, Junction\_Detail=3 & Road\_Type=6, Road\_Surface\_Conditions=1, Accident\_Severity=2, Urban\_or\_Rural\_Area=1 & 0.1511 & 0.6930 & 1.4202 \\
\midrule
6 & Road\_Surface\_Conditions=1, Junction\_Detail=3 & Accident\_Severity=2, Light\_Conditions=1, Urban\_or\_Rural\_Area=1, Road\_Type=6, Weather\_Conditions=1 & 0.1511 & 0.5208 & 1.4134 \\
\midrule
7 & Light\_Conditions=1, Weather\_Conditions=1, Junction\_Detail=3 & Road\_Type=6, Road\_Surface\_Conditions=1, Urban\_or\_Rural\_Area=1 & 0.1591 & 0.7298 & 1.4116 \\
\midrule
8 & Road\_Surface\_Conditions=1, Junction\_Detail=3 & Road\_Type=6, Light\_Conditions=1, Weather\_Conditions=1, Urban\_or\_Rural\_Area=1 & 0.1591 & 0.5485 & 1.4102 \\
\midrule
9 & Road\_Surface\_Conditions=1, Light\_Conditions=1, Accident\_Severity=2, Junction\_Detail=3 & Road\_Type=6, Weather\_Conditions=1, Urban\_or\_Rural\_Area=1 & 0.1511 & 0.8020 & 1.3987 \\
\midrule
10 & Weather\_Conditions=1, Accident\_Severity=2, Junction\_Detail=3 & Road\_Type=6, Road\_Surface\_Conditions=1, Light\_Conditions=1, Urban\_or\_Rural\_Area=1 & 0.1511 & 0.5051 & 1.3978 \\
\bottomrule
\end{tabular}
\begin{tablenotes}
\item[1] Total number of accident records: 5231
\item[2] Minimum support used in mining: 0.1
\item[3] Code explanation see later table.
\end{tablenotes}
\end{threeparttable}
\end{table}

\begin{table}[H]
\centering
\scriptsize
\begin{threeparttable}
\caption{Top 10 association rules mined by \textbf{FP-Growth} (ranked in descending order by \textit{\color{red}{support}})}
\label{tab:FPGrowth_mined_rules_rankedBySupport_minSupport01_Accident}
\begin{tabular}{r >{\raggedright\arraybackslash}p{4cm} >{\raggedright\arraybackslash}p{4cm} p{1cm} p{1cm} p{1cm}}
\toprule
\# & Antecedent & Consequent & \color{red}{Supp} & Conf & Lift \\
\midrule
1 & Weather\_Conditions=1 & Accident\_Severity=2 & 0.7985 & 0.9280 & 1.0005 \\
\midrule
2 & Accident\_Severity=2 & Weather\_Conditions=1 & 0.7985 & 0.8609 & 1.0005 \\
\midrule
3 & Accident\_Severity=2 & Urban\_or\_Rural\_Area=1 & 0.7790 & 0.8398 & 1.0123 \\
\midrule
4 & Urban\_or\_Rural\_Area=1 & Accident\_Severity=2 & 0.7790 & 0.9389 & 1.0123 \\
\midrule
5 & Road\_Surface\_Conditions=1 & Weather\_Conditions=1 & 0.7514 & 0.9884 & 1.1488 \\
\midrule
6 & Weather\_Conditions=1 & Road\_Surface\_Conditions=1 & 0.7514 & 0.8733 & 1.1488 \\
\midrule
7 & Road\_Type=6 & Accident\_Severity=2 & 0.7325 & 0.9308 & 1.0035 \\
\midrule
8 & Accident\_Severity=2 & Road\_Type=6 & 0.7325 & 0.7897 & 1.0035 \\
\midrule
9 & Weather\_Conditions=1 & Urban\_or\_Rural\_Area=1 & 0.7197 & 0.8364 & 1.0082 \\
\midrule
10 & Urban\_or\_Rural\_Area=1 & Weather\_Conditions=1 & 0.7197 & 0.8675 & 1.0082 \\
\bottomrule
\end{tabular}
\begin{tablenotes}
\item[1] Total number of accident records: 5231
\item[2] Minimum support used in mining: 0.1
\item[3] Code explanation see later table.
\end{tablenotes}
\end{threeparttable}
\end{table}

\begin{table}[H]
\centering
\scriptsize
\begin{threeparttable}
\caption{Top 10 association rules mined by \textbf{FP-Growth}  (ranked in descending order by \textit{\color{red}{confidence}})}
\label{tab:FPGrowth_mined_rules_rankedByConfidence_minSupport01_Accident}
\begin{tabular}{r >{\raggedright\arraybackslash}p{4cm} >{\raggedright\arraybackslash}p{4cm} p{1cm} p{1cm} p{1cm}}
\toprule
\# & Antecedent & Consequent & Supp & \color{red}{Conf} & Lift \\
\midrule
1 & Road\_Surface\_Conditions=1, Junction\_Detail=6 & Weather\_Conditions=1 & 0.1036 & 0.9982 & 1.1601 \\
\midrule
2 & Weather\_Conditions=2 & Road\_Surface\_Conditions=2 & 0.1025 & 0.9981 & 4.4016 \\
\midrule
3 & Road\_Type=6, Road\_Surface\_Conditions=1, Light\_Conditions=1, Junction\_Detail=3 & Weather\_Conditions=1 & 0.1732 & 0.9945 & 1.1558 \\
\midrule
4 & Road\_Surface\_Conditions=1, Light\_Conditions=1, Junction\_Detail=3 & Weather\_Conditions=1 & 0.2002 & 0.9943 & 1.1556 \\
\midrule
5 & Accident\_Severity=2, Road\_Surface\_Conditions=1, Light\_Conditions=1, Road\_Type=6, Junction\_Detail=3 & Weather\_Conditions=1 & 0.1635 & 0.9942 & 1.1555 \\
\midrule
6 & Road\_Surface\_Conditions=1, Light\_Conditions=1, Urban\_or\_Rural\_Area=1, Road\_Type=6, Junction\_Detail=3 & Weather\_Conditions=1 & 0.1591 & 0.9940 & 1.1553 \\
\midrule
7 & Road\_Surface\_Conditions=1, Light\_Conditions=1, Accident\_Severity=2, Junction\_Detail=3 & Weather\_Conditions=1 & 0.1872 & 0.9939 & 1.1551 \\
\midrule
8 & Road\_Surface\_Conditions=1, Light\_Conditions=1, Junction\_Detail=3, Urban\_or\_Rural\_Area=1 & Weather\_Conditions=1 & 0.1836 & 0.9938 & 1.1550 \\
\midrule
9 & Accident\_Severity=2, Road\_Surface\_Conditions=1, Light\_Conditions=1, Urban\_or\_Rural\_Area=1, Road\_Type=6, Junction\_Detail=3 & Weather\_Conditions=1 & 0.1511 & 0.9937 & 1.1549 \\
\midrule
10 & Accident\_Severity=2, Road\_Surface\_Conditions=1, Light\_Conditions=1, Urban\_or\_Rural\_Area=1, Junction\_Detail=3 & Weather\_Conditions=1 & 0.1730 & 0.9934 & 1.1546 \\
\bottomrule
\end{tabular}
\begin{tablenotes}
\item[1] Total number of accident records: 5231
\item[2] Minimum support used in mining: 0.1
\item[3] Code explanation see later table.
\end{tablenotes}
\end{threeparttable}
\end{table}

\begin{table}[H]
\centering
\scriptsize
\begin{threeparttable}
\caption{Top 10 association rules mined by \textbf{FP-Growth}  (ranked in descending order by \textit{\color{red}{confidence}})}
\label{tab:FPGrowth_mined_rules_rankedByLift_minSupport01_Accident}
\begin{tabular}{r >{\raggedright\arraybackslash}p{4cm} >{\raggedright\arraybackslash}p{4cm} p{1cm} p{1cm} p{1cm}}
\toprule
\# & Antecedent & Consequent & Supp & Conf & \color{red}{Lift} \\
\midrule
1 & Weather\_Conditions=2 & Road\_Surface\_Conditions=2 & 0.1025 & 0.9981 & 4.4016 \\
\midrule
2 & Urban\_or\_Rural\_Area=2 & Junction\_Detail=0 & 0.1084 & 0.6364 & 1.6109 \\
\midrule
3 & Weather\_Conditions=1, Light\_Conditions=1, Accident\_Severity=2, Junction\_Detail=3 & Road\_Type=6, Road\_Surface\_Conditions=1, Urban\_or\_Rural\_Area=1 & 0.1511 & 0.7439 & 1.4388 \\
\midrule
4 & Road\_Surface\_Conditions=1, Junction\_Detail=3, Accident\_Severity=2 & Road\_Type=6, Light\_Conditions=1, Weather\_Conditions=1, Urban\_or\_Rural\_Area=1 & 0.1511 & 0.5579 & 1.4345 \\
\midrule
5 & Light\_Conditions=1, Weather\_Conditions=1, Junction\_Detail=3 & Road\_Type=6, Road\_Surface\_Conditions=1, Accident\_Severity=2, Urban\_or\_Rural\_Area=1 & 0.1511 & 0.6930 & 1.4202 \\
\midrule
6 & Road\_Surface\_Conditions=1, Junction\_Detail=3 & Accident\_Severity=2, Light\_Conditions=1, Urban\_or\_Rural\_Area=1, Road\_Type=6, Weather\_Conditions=1 & 0.1511 & 0.5208 & 1.4134 \\
\midrule
7 & Light\_Conditions=1, Weather\_Conditions=1, Junction\_Detail=3 & Road\_Type=6, Road\_Surface\_Conditions=1, Urban\_or\_Rural\_Area=1 & 0.1591 & 0.7298 & 1.4116 \\
\midrule
8 & Road\_Surface\_Conditions=1, Junction\_Detail=3 & Road\_Type=6, Light\_Conditions=1, Weather\_Conditions=1, Urban\_or\_Rural\_Area=1 & 0.1591 & 0.5485 & 1.4102 \\
\midrule
9 & Road\_Surface\_Conditions=1, Light\_Conditions=1, Accident\_Severity=2, Junction\_Detail=3 & Road\_Type=6, Weather\_Conditions=1, Urban\_or\_Rural\_Area=1 & 0.1511 & 0.8020 & 1.3987 \\
\midrule
10 & Weather\_Conditions=1, Accident\_Severity=2, Junction\_Detail=3 & Road\_Type=6, Road\_Surface\_Conditions=1, Light\_Conditions=1, Urban\_or\_Rural\_Area=1 & 0.1511 & 0.5051 & 1.3978 \\
\bottomrule
\end{tabular}
\begin{tablenotes}
\item[1] Total number of accident records: 5231
\item[2] Minimum support used in mining: 0.1
\item[3] Code explanation see later table.
\end{tablenotes}
\end{threeparttable}
\end{table}

\begin{table}[H]
\centering
\scriptsize
\begin{threeparttable}
\caption{Top 10 association rules mined by \textbf{Eclat} (ranked in descending order by \textit{\color{red}{support}})}
\label{tab:Eclat_mined_rules_rankedBySupport_minSupport01_Accident}
\begin{tabular}{r >{\raggedright\arraybackslash}p{4cm} >{\raggedright\arraybackslash}p{4cm} p{1cm} p{1cm} p{1cm}}
\toprule
\# & Antecedent & Consequent & \color{red}{Supp} & Conf & Lift \\
\midrule
1 & Weather\_Conditions=1 & Accident\_Severity=2 & 0.7985 & 0.9280 & 1.0005 \\
\midrule
2 & Accident\_Severity=2 & Weather\_Conditions=1 & 0.7985 & 0.8609 & 1.0005 \\
\midrule
3 & Accident\_Severity=2 & Urban\_or\_Rural\_Area=1 & 0.7790 & 0.8398 & 1.0123 \\
\midrule
4 & Urban\_or\_Rural\_Area=1 & Accident\_Severity=2 & 0.7790 & 0.9389 & 1.0123 \\
\midrule
5 & Road\_Type=6, Road\_Surface\_Conditions=1, Weather\_Conditions=1 & Accident\_Severity=2 & 0.7514 & 1.0000 & 1.0781 \\
\midrule
6 & Road\_Type=6, Road\_Surface\_Conditions=1, Accident\_Severity=2 & Weather\_Conditions=1 & 0.7514 & 1.2576 & 1.4616 \\
\midrule
7 & Road\_Type=6, Accident\_Severity=2, Weather\_Conditions=1 & Road\_Surface\_Conditions=1 & 0.7514 & 1.1067 & 1.4558 \\
\midrule
8 & Road\_Surface\_Conditions=1, Accident\_Severity=2, Weather\_Conditions=1 & Road\_Type=6 & 0.7514 & 1.0000 & 1.2707 \\
\midrule
9 & Road\_Type=6, Road\_Surface\_Conditions=1 & Accident\_Severity=2, Weather\_Conditions=1 & 0.7514 & 1.2576 & 1.5750 \\
\midrule
10 & Road\_Type=6, Weather\_Conditions=1 & Road\_Surface\_Conditions=1, Accident\_Severity=2 & 0.7514 & 1.1067 & 1.5597 \\
\bottomrule
\end{tabular}
\begin{tablenotes}
\item[1] Total number of accident records: 5231
\item[2] Minimum support used in mining: 0.1
\item[3] Code explanation see later table.
\end{tablenotes}
\end{threeparttable}
\end{table}

\begin{table}[H]
\centering
\scriptsize
\begin{threeparttable}
\caption{Top 10 association rules mined by \textbf{Eclat} (ranked in descending order by \textit{\color{red}{confidence}})}
\label{tab:Eclat_mined_rules_rankedByConfidence_minSupport01_Accident}
\begin{tabular}{r >{\raggedright\arraybackslash}p{4cm} >{\raggedright\arraybackslash}p{4cm} p{1cm} p{1cm} p{1cm}}
\toprule
\# & Antecedent & Consequent & Supp & \color{red}{Conf} & Lift \\
\midrule
2654 & Road\_Type=3, Road\_Surface\_Conditions=1, Accident\_Severity=2 & Weather\_Conditions=1 & 0.7514 & 6.6050 & 7.6765 \\
\midrule
2657 & Road\_Type=3, Road\_Surface\_Conditions=1 & Accident\_Severity=2, Weather\_Conditions=1 & 0.7514 & 6.6050 & 8.2721 \\
\midrule
5213 & Road\_Type=3, Road\_Surface\_Conditions=1 & Weather\_Conditions=1 & 0.7514 & 6.6050 & 7.6765 \\
\midrule
2699 & Road\_Type=3, Accident\_Severity=2, Urban\_or\_Rural\_Area=1 & Weather\_Conditions=1 & 0.7197 & 6.5121 & 7.5685 \\
\midrule
2703 & Road\_Type=3, Urban\_or\_Rural\_Area=1 & Accident\_Severity=2, Weather\_Conditions=1 & 0.7197 & 6.5121 & 8.1557 \\
\midrule
5234 & Road\_Type=3, Urban\_or\_Rural\_Area=1 & Weather\_Conditions=1 & 0.7197 & 6.5121 & 7.5685 \\
\midrule
2655 & Road\_Type=3, Accident\_Severity=2, Weather\_Conditions=1 & Road\_Surface\_Conditions=1 & 0.7514 & 5.9726 & 7.8564 \\
\midrule
2658 & Road\_Type=3, Weather\_Conditions=1 & Road\_Surface\_Conditions=1, Accident\_Severity=2 & 0.7514 & 5.9726 & 8.4174 \\
\midrule
5214 & Road\_Type=3, Weather\_Conditions=1 & Road\_Surface\_Conditions=1 & 0.7514 & 5.9726 & 7.8564 \\
\midrule
2676 & Road\_Type=3, Accident\_Severity=2, Urban\_or\_Rural\_Area=1 & Road\_Surface\_Conditions=1, Weather\_Conditions=1 & 0.6511 & 5.8910 & 7.8397 \\
\bottomrule
\end{tabular}
\begin{tablenotes}
\item[1] Total number of accident records: 5231
\item[2] Minimum support used in mining: 0.1
\item[3] Code explanation see later table.
\end{tablenotes}
\end{threeparttable}
\end{table}

\begin{table}[H]
\centering
\scriptsize
\begin{threeparttable}
\caption{Top 10 association rules mined by \textbf{Eclat} (ranked in descending order by \textit{\color{red}{lift}})}
\label{tab:Eclat_mined_rules_rankedByLift_minSupport01_Accident}
\begin{tabular}{r >{\raggedright\arraybackslash}p{4cm} >{\raggedright\arraybackslash}p{4cm} p{1cm} p{1cm} p{1cm}}
\toprule
\# & Antecedent & Consequent & Supp & Conf & \color{red}{Lift} \\
\midrule
1 & Road\_Type=6, Road\_Surface\_Conditions=2, Accident\_Severity=2, Weather\_Conditions=1 & Light\_Conditions=1, Junction\_Detail=0, Urban\_or\_Rural\_Area=1 & 0.2866 & 2.8230 & 9.8493 \\
\midrule
2 & Road\_Type=6, Road\_Surface\_Conditions=2, Weather\_Conditions=1 & Junction\_Detail=0, Light\_Conditions=1, Accident\_Severity=2, Urban\_or\_Rural\_Area=1 & 0.2866 & 2.8230 & 9.8493 \\
\midrule
3 & Road\_Surface\_Conditions=2, Accident\_Severity=2, Weather\_Conditions=1 & Road\_Type=6, Light\_Conditions=1, Junction\_Detail=0, Urban\_or\_Rural\_Area=1 & 0.2866 & 2.8230 & 9.8493 \\
\midrule
4 & Road\_Surface\_Conditions=2, Weather\_Conditions=1 & Accident\_Severity=2, Light\_Conditions=1, Urban\_or\_Rural\_Area=1, Road\_Type=6, Junction\_Detail=0 & 0.2866 & 2.8230 & 9.8493 \\
\midrule
5 & Road\_Type=6, Road\_Surface\_Conditions=2, Accident\_Severity=2, Weather\_Conditions=1 & Junction\_Detail=0, Urban\_or\_Rural\_Area=1 & 0.2866 & 2.8230 & 9.8493 \\
\midrule
6 & Road\_Type=6, Road\_Surface\_Conditions=2, Weather\_Conditions=1 & Junction\_Detail=0, Accident\_Severity=2, Urban\_or\_Rural\_Area=1 & 0.2866 & 2.8230 & 9.8493 \\
\midrule
7 & Road\_Surface\_Conditions=2, Accident\_Severity=2, Weather\_Conditions=1 & Road\_Type=6, Junction\_Detail=0, Urban\_or\_Rural\_Area=1 & 0.2866 & 2.8230 & 9.8493 \\
\midrule
8 & Road\_Surface\_Conditions=2, Weather\_Conditions=1 & Road\_Type=6, Junction\_Detail=0, Accident\_Severity=2, Urban\_or\_Rural\_Area=1 & 0.2866 & 2.8230 & 9.8493 \\
\midrule
9 & Road\_Surface\_Conditions=2, Accident\_Severity=2, Weather\_Conditions=1 & Light\_Conditions=1, Junction\_Detail=0, Urban\_or\_Rural\_Area=1 & 0.2866 & 2.8230 & 9.8493 \\
\midrule
10 & Road\_Surface\_Conditions=2, Weather\_Conditions=1 & Junction\_Detail=0, Light\_Conditions=1, Accident\_Severity=2, Urban\_or\_Rural\_Area=1 & 0.2866 & 2.8230 & 9.8493 \\
\bottomrule
\end{tabular}
\begin{tablenotes}
\item[1] Total number of accident records: 5231
\item[2] Minimum support used in mining: 0.1
\item[3] Code explanation see later table.
\end{tablenotes}
\end{threeparttable}
\end{table}

\begin{table}[H]
\centering
\scriptsize
\caption{Codes for \textbf{items} from the \textit{Accident} dataset}
\label{tab:item_codes_Accident}
\begin{tabular}{>{\raggedright\arraybackslash}p{4cm} >{\raggedright\arraybackslash}p{2cm} >{\raggedright\arraybackslash}p{8cm}}
\toprule
Item & Code & Label \\
\midrule
Accident\_Severity & 1 & Fatal \\
Accident\_Severity & 2 & Serious \\
Accident\_Severity & 3 & Slight \\
Accident\_Severity & -1 & Data missing or out of range \\
\midrule
Road\_Type & 1 & Roundabout \\
Road\_Type & 2 & One way street \\
Road\_Type & 3 & Dual carriageway \\
Road\_Type & 6 & Single carriageway \\
Road\_Type & 7 & Slip road \\
Road\_Type & 9 & Unknown \\
Road\_Type & 12 & One way street/Slip road \\
Road\_Type & -1 & Data missing or out of range \\
\midrule
Weather\_Conditions & 1 & Fine no high winds \\
Weather\_Conditions & 2 & Raining no high winds \\
Weather\_Conditions & 3 & Snowing no high winds \\
Weather\_Conditions & 4 & Fine + high winds \\
Weather\_Conditions & 5 & Raining + high winds \\
Weather\_Conditions & 6 & Snowing + high winds \\
Weather\_Conditions & 7 & Fog or mist \\
Weather\_Conditions & 8 & Other \\
Weather\_Conditions & 9 & Unknown \\
Weather\_Conditions & -1 & Data missing or out of range \\
\midrule
Light\_Conditions & 1 & Daylight \\
Light\_Conditions & 4 & Darkness - lights lit \\
Light\_Conditions & 5 & Darkness - lights unlit \\
Light\_Conditions & 6 & Darkness - no lighting \\
Light\_Conditions & 7 & Darkness - lighting unknown \\
Light\_Conditions & -1 & Data missing or out of range \\
\midrule
Road\_Surface\_Conditions & 1 & Dry \\
Road\_Surface\_Conditions & 2 & Wet or damp \\
Road\_Surface\_Conditions & 3 & Snow \\
Road\_Surface\_Conditions & 4 & Frost or ice \\
Road\_Surface\_Conditions & 5 & Flood over 3cm. deep \\
Road\_Surface\_Conditions & 6 & Oil or diesel \\
Road\_Surface\_Conditions & 7 & Mud \\
Road\_Surface\_Conditions & -1 & Data missing or out of range \\
Road\_Surface\_Conditions & 9 & Unknown (self reported) \\
\midrule
Junction\_Detail & 0 & Not at junction or within 20 metres \\
Junction\_Detail & 1 & Roundabout \\
Junction\_Detail & 2 & Mini-roundabout \\
Junction\_Detail & 3 & T or staggered junction \\
Junction\_Detail & 5 & Slip road \\
Junction\_Detail & 6 & Crossroads \\
Junction\_Detail & 7 & More than 4 arms (not roundabout) \\
Junction\_Detail & 8 & Private drive or entrance \\
Junction\_Detail & 9 & Other junction \\
Junction\_Detail & -1 & Data missing or out of range \\
Junction\_Detail & 99 & Unknown (self reported) \\
\midrule
Urban\_or\_Rural\_Area & 1 & Urban \\
Urban\_or\_Rural\_Area & 2 & Rural \\
Urban\_or\_Rural\_Area & 3 & Unallocated \\
Urban\_or\_Rural\_Area & -1 & Data missing or out of range \\
\bottomrule
\end{tabular}
\end{table}

\begin{table}[H]
\centering
\scriptsize
\caption{Codes for \textbf{feature variables} from the \textit{Accident} dataset}
\label{tab:feature_codes_Accident}
\begin{tabular}{>{\raggedright\arraybackslash}p{5.5cm} >{\raggedright\arraybackslash}p{1cm} >{\raggedright\arraybackslash}p{7cm}}
\toprule
Variable & Code & Label \\
\midrule
Longitude & - & Numeric coordinate (Null if not known) \\
\midrule
Latitude & - & Numeric coordinate (Null if not known) \\
\midrule
Speed\_limit & 20 & 20 mph \\
Speed\_limit & 30 & 30 mph \\
Speed\_limit & 40 & 40 mph \\
Speed\_limit & 50 & 50 mph \\
Speed\_limit & 60 & 60 mph \\
Speed\_limit & 70 & 70 mph \\
Speed\_limit & -1 & Data missing or out of range \\
Speed\_limit & 99 & Unknown (self reported) \\
\midrule
Number\_of\_Vehicles & - & Numeric count of vehicles \\
\midrule
Day\_of\_Week & 1 & Sunday \\
Day\_of\_Week & 2 & Monday \\
Day\_of\_Week & 3 & Tuesday \\
Day\_of\_Week & 4 & Wednesday \\
Day\_of\_Week & 5 & Thursday \\
Day\_of\_Week & 6 & Friday \\
Day\_of\_Week & 7 & Saturday \\
\midrule
Hour\_of\_Day & - & Hour from 00 to 23 (Null if not known) \\
\midrule
Junction\_Control & 0 & Not at junction or within 20 metres \\
Junction\_Control & 1 & Authorised person \\
Junction\_Control & 2 & Auto traffic signal \\
Junction\_Control & 3 & Stop sign \\
Junction\_Control & 4 & Give way or uncontrolled \\
Junction\_Control & -1 & Data missing or out of range \\
Junction\_Control & 9 & Unknown (self reported) \\
\midrule
Pedestrian\_Crossing-Human\_Control & 0 & None within 50 metres \\
Pedestrian\_Crossing-Human\_Control & 1 & Control by school crossing patrol \\
Pedestrian\_Crossing-Human\_Control & 2 & Control by other authorised person \\
Pedestrian\_Crossing-Human\_Control & -1 & Data missing or out of range \\
Pedestrian\_Crossing-Human\_Control & 9 & Unknown (self reported) \\
\midrule
Pedestrian\_Crossing-Physical\_Facilities & 0 & No physical crossing facilities within 50 metres \\
Pedestrian\_Crossing-Physical\_Facilities & 1 & Zebra \\
Pedestrian\_Crossing-Physical\_Facilities & 4 & Pelican, puffin, toucan or similar non-junction pedestrian light crossing \\
Pedestrian\_Crossing-Physical\_Facilities & 5 & Pedestrian phase at traffic signal junction \\
Pedestrian\_Crossing-Physical\_Facilities & 7 & Footbridge or subway \\
Pedestrian\_Crossing-Physical\_Facilities & 8 & Central refuge \\
Pedestrian\_Crossing-Physical\_Facilities & -1 & Data missing or out of range \\
Pedestrian\_Crossing-Physical\_Facilities & 9 & Unknown (self reported) \\
\midrule
1st\_Road\_Class & 1 & Motorway \\
1st\_Road\_Class & 2 & A(M) \\
1st\_Road\_Class & 3 & A \\
1st\_Road\_Class & 4 & B \\
1st\_Road\_Class & 5 & C \\
1st\_Road\_Class & 6 & Unclassified \\
1st\_Road\_Class & -1 & Data missing or out of range \\
\bottomrule
\end{tabular}
\end{table}

Comparison of the 4 methods, i.e. RBF-based GPAR, Apriori, FP-Growth, and Eclat, is presented in Table.\ref{tab:accident_comparison_runtime} to Table.\ref{tab:accident_comparison_noRules}. RBF-based GPAR exhibits significantly higher computational demands, with a runtime of 1027.79 seconds and memory usage of 341.23 MB at a minimum support of 0.1, though both metrics decrease sharply at higher thresholds (e.g. 0.10 seconds and 0.0 MB at 0.5). In contrast, Apriori, FP-Growth, and Eclat are far more efficient, with runtimes ranging from 0.01 to 3.47 seconds and memory usage peaking at 22.95 MB for Eclat at 0.1, dropping to 0.0 MB at higher thresholds. Regarding the number of frequent itemsets and rules, RBF-based GPAR generates the most at lower thresholds (10366 itemsets and 16869 rules at 0.1), but this number diminishes significantly at higher thresholds (0 itemsets and rules at 0.4 and 0.5). Apriori and FP-Growth consistently produce fewer itemsets (277 at 0.1, dropping to 39 at 0.5) and rules (2052 at 0.1, dropping to 170 at 0.5), while Eclat generates the highest among the traditional methods (479 itemsets and 5794 rules at 0.1, reducing to 69 itemsets and 602 rules at 0.5).

Analysis of the mined rules reveals distinct strengths in pattern discovery. When ranked by support, RBF-based GPAR identifies rules with strong associations between accident severity and weather conditions (e.g. 'Accident\_Severity=2 → Weather\_Conditions=1' with support 0.4610), reflecting its ability to leverage feature similarities in the feature matrix $X_{39 \times 10}$. Apriori, FP-Growth, and Eclat, relying solely on transactional data, consistently identify high-support rules involving similar items (e.g. Apriori identifies 'Weather\_Conditions=1 → Accident\_Severity=2' with support 0.7985), but Eclat emphasizes larger itemsets (e.g. 'Road\_Type=6, Road\_Surface\_Conditions=1, Weather\_Conditions=1 → Accident\_Severity=2' with support 0.7514). When ranked by confidence, RBF-based GPAR achieves higher confidence scores (e.g. 'Accident\_Severity=2, Weather\_Conditions=8, Road\_Type=7 → Weather\_Conditions=1 with confidence 1.2079), capturing nuanced patterns, while Apriori and FP-Growth prioritize rules with near-perfect confidence (e.g. Apriori identifies 'Road\_Surface\_Conditions=1, Junction\_Detail=6 → Weather\_Conditions=1' with confidence 0.9982). Eclat excels in confidence for rules involving specific road types (e.g. 'Road\_Type=3, Road\_Surface\_Conditions=1, Accident\_Severity=2 → Weather\_Conditions=1' with confidence 6.6050). For lift, Eclat identifies the strongest associations (e.g. 'Road\_Type=6, Road\_Surface\_Conditions=2, Accident\_Severity=2, Weather\_Conditions=1 → Light\_Conditions=1, Junction\_Detail=0, Urban\_or\_Rural\_Area=1" with lift 9.8493), followed by Apriori and FP-Growth (e.g. 'Weather\_Conditions=2 → Road\_Surface\_Conditions=2' with lift 4.4016), while RBF-based GPAR achieves a maximum lift of 2.7236. Overall, while RBF-based GPAR uncovers more complex and diverse patterns, its computational cost is also high, which may make it less practical compared to the efficiency of Apriori, FP-Growth, and Eclat in real-world applications.

\begin{table}[H]
\centering
\scriptsize
\begin{threeparttable}
\caption{Runtime performance (in \textit{seconds}) of different AR mining algorithms (\textit{Accident})}
\label{tab:accident_comparison_runtime}
\begin{tabular}{p{2cm} p{1.5cm} p{1.5cm} p{2cm} p{1.5cm}}
\toprule
Min support & GPAR & Apriori & FP-Growth & Eclat \\
\midrule
0.1 & 1027.79 & 0.08 & 3.47 & 0.56 \\
0.2 & 18.24 & 0.03 & 1.85 & 0.57 \\
0.3 & 1.59 & 0.02 & 0.87 & 0.46 \\
0.4 & 0.09 & 0.01 & 0.66 & 0.45 \\
0.5 & 0.10 & 0.01 & 0.44 & 0.45 \\
\bottomrule
\end{tabular}
\begin{tablenotes}
\item[1] Total number of accident records: 5231.
\item[2] Runtimes are measured in seconds.
\item[3] Minimum support values range from 0.1 to 0.5.
\end{tablenotes}
\end{threeparttable}
\end{table}

\begin{table}[H]
\centering
\scriptsize
\begin{threeparttable}
\caption{Memory usage (in \textit{MB}) of different AR mining algorithms (\textit{Accident})}
\label{tab:accident_comparison_memory}
\begin{tabular}{p{2cm} p{1.5cm} p{1.5cm} p{2cm} p{1.5cm}}
\toprule
Min Support & GPAR & Apriori & FP-Growth & Eclat \\
\midrule
0.1 & 341.23 & 20.48 & 1.18 & 22.95 \\
0.2 & 0.00 & 0.00 & 0.00 & 0.00 \\
0.3 & 0.00 & 0.00 & 0.00 & 0.00 \\
0.4 & 0.00 & 0.00 & 0.00 & 0.00 \\
0.5 & 0.00 & 0.00 & 0.00 & 0.00 \\
\bottomrule
\end{tabular}
\begin{tablenotes}
\item[1] Total number of accident records: 5231.
\item[2] Memory usage is measured in megabytes (MB).
\item[3] Minimum support values range from 0.1 to 0.5.
\end{tablenotes}
\end{threeparttable}
\end{table}

\begin{table}[H]
\centering
\scriptsize
\begin{threeparttable}
\caption{No. of frequent itemsets generated by different AR mining algorithms (\textit{Accident})}
\label{tab:accident_comparison_frequentItemsets}
\begin{tabular}{p{1.5cm} p{1.5cm} p{1.5cm} p{1.5cm} p{1.5cm}}
\toprule
Min support & GPAR & Apriori & FP-Growth & Eclat \\
\midrule
0.1 & 10366 & 277 & 277 & 479 \\
0.2 & 770 & 153 & 153 & 235 \\
0.3 & 4 & 75 & 75 & 97 \\
0.4 & 0 & 55 & 55 & 69 \\
0.5 & 0 & 39 & 39 & 69 \\
\bottomrule
\end{tabular}
\begin{tablenotes}
\item[1] Total number of accident records: 5231.
\item[2] Numbers are counts of frequent itemsets generated.
\item[3] Minimum support values range from 0.1 to 0.5.
\end{tablenotes}
\end{threeparttable}
\end{table}

\begin{table}[H]
\centering
\scriptsize
\begin{threeparttable}
\caption{No. of rules generated by different AR mining algorithms (\textit{Accident})}
\label{tab:accident_comparison_noRules}
\begin{tabular}{p{2cm} p{1.5cm} p{1.5cm} p{1.5cm} p{1.5cm}}
\toprule
Min Support & GPAR & Apriori & FP-Growth & Eclat \\
\midrule
0.1 & 16869 & 2052 & 2052 & 5794 \\
0.2 & 834 & 987 & 987 & 2661 \\
0.3 & 8 & 566 & 566 & 692 \\
0.4 & 0 & 370 & 370 & 602 \\
0.5 & 0 & 170 & 170 & 602 \\
\bottomrule
\end{tabular}
\begin{tablenotes}
\item[1] Total number of accident records: 5231.
\item[2] Rules are counts of association rules generated.
\item[3] Minimum support values range from 0.1 to 0.5.
\end{tablenotes}
\end{threeparttable}
\end{table}

\begin{figure}[H]
\centering
\includegraphics[width=0.9\textwidth]{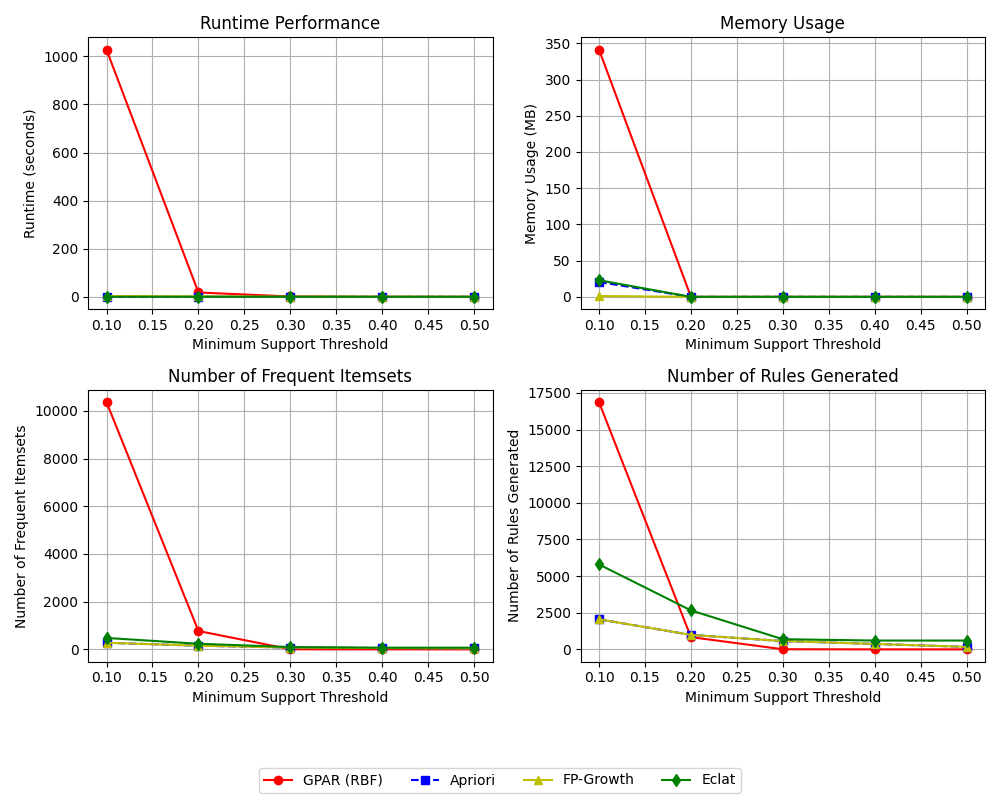}
\caption{Compare the performances of the 4 algorithms (\textit{UK Accident}).}
\label{fig:synthetic2_performance_comparison_UKAccident}
\end{figure}

\section{Comparison: GPAR \textit{vs} BARM} \label{app:compare_GPAR_and_BARM}

Here we compare GPAR (Gaussian process-based association rule mining) and BARM (Bayesian association rule mining), in terms of model, training, computational complexity and uncertainty representation.

\paragraph{Modelling} 
GPAR uses a Gaussian process (GP) to model latent variables $\mathbf{z}$, where each item’s presence in a transaction is determined by a threshold function ($\mathbf{t}_j[i] = 1$ if $z_i > 0$, else 0), with posterior inferred by attribute similarities in feature space; BARM directly models the presence probabilities $p_i$ for each item $i$ as random variables with priors, e.g. $p_i \sim \text{Beta}(\alpha_i, \beta_i)$. Both GPAR and BARM can incorporate item dependencies using a correlation function e.g. $k(\mathbf{x}_i, \mathbf{x}_j) = \exp(||\mathbf{x}_i - \mathbf{x}_j||^2 / \ell^2)$. A kernel (covariance) function $k(\mathbf{x}_i, \mathbf{x}_j)$, e.g. a RBF kernel, computes the covariance matrix $K \in \mathbb{R}^{M \times M}$, capturing dependencies between items based on their feature vectors $\mathbf{x}$. This correlation, output by the data optimised kernel function, is used in the Gaussian process model in GPAR to generate the posterior covariance; while in BARM, it can be used to correct the likelihood (Eq.\ref{eq:BARM_likelihood}) in the posterior inference and the probability in rule inference (Eq.\ref{eq:BARM_subset_prob_estimate_via_MCMC_posterior_samples}). In GPAR the kernel hyper-parameter $\ell$ is point estimated, BARM puts a prior over it, e.g. $\ell \sim \text{Gamma}(a, b)$. However, the notion of a covariance structure is not essential for BARM, see a dependency-free version BARM in Section.\ref{sec:BARM_variant}. Another difference lies in the way how the kernel $k(\cdot,\cdot)$ is optimised: GPAR optimises the kernel hyper-parameters (e.g. length-scale $\ell$) via maximum likelihood (Eq.\ref{eq:GP_training_MLE}) (i.e. adjust the hyper-parameters to match data fitting); in BARM, the correlation adjustment factor $g(\{\rho_{ij}\})$, as used in likelihood adjustment (Eq.\ref{eq:BARM_likelihood}) and rule inference adjustment (Eq.\ref{eq:BARM_subset_prob_estimate_via_MCMC_posterior_samples}) is assumed a known priori (no training happens), either derived from a Gaussian copula or a simplified pairwise model in the feature space. Accurately modeling the correlation adjustment factor $g(\cdot)$ could be challenging. 

\paragraph{Training and inference} 
GPAR optimizes kernel hyperparameters (e.g. length scale $\ell$) by minimizing the negative loglikelihood of the transaction data over $T$ iterations, producing a point estimate for $\ell$. Co-occurrence probabilities $p(I)$ for an itemset $I$ are estimated via Monte Carlo sampling from a multivariate Gaussian $\mathcal{N}(0, K_I)$, where $K_I$ is the submatrix of $K$ corresponding to items in $I$.
BARM employs Bayesian inference via MCMC sampling to draw $S_{\text{MCMC}}$ samples from the posterior $p(\{p_i\}, \ell | \mathcal{T}, X)$. co-occurrence probabilities $p(I)$ are computed by averaging over $S$ posterior samples, i.e. $p(I) = \frac{1}{S} \sum_{s=1}^{S} p(I | \{p_i^{(s)}\}, \ell^{(s)})$, capturing the full posterior distribution of probabilities.

\paragraph{Complexity} 
In GPAR, the initial GP inference step has a complexity of $\mathcal{O}(T (M^2d + NM^2 + M^3))$ where $T$ is the number of optimisation iterations, $N$ number of transactions, $M$ the number of items This cost is dominated by covariance matrix operations. In the rule inference stage, probability estimation per itemset costs $\mathcal{O}(m^3 + S m^2)$, where $m = |I|$ and $S$ is the number of Monte Carlo samples used in estimating the marginalised joint probability $p(I)$. The overall complexity of GPAR, combining GP inference and marginal density evaluation, is $\mathcal{O}(M^2d + T(M^2d + NM^2 + M^3) + \sum_{m=2}^M \mathcal{C}^M_m \times (m^3 + Sm^2))$, which is dominated by the last term $\mathcal{O}(2^M (M^3 + SM^2))$.

In BARM with item dependency, if we pre-compute the pairwise correlation adjustment factor $g(\cdot)$, then MCMC sampling costs $\mathcal{O}(S_{\text{MCMC}} \cdot N \cdot M)$, and probability estimation per itemset costs $\mathcal{O}(S \cdot m)$, which is generally less expensive per itemset than GPAR’s Monte Carlo sampling due to the absence of Cholesky decomposition ($\mathcal{O}(m^3)$). The total cost of BARM is $\mathcal{O}(S_{\text{MCMC}} \cdot N \cdot M + 2^M \cdot S \cdot M)$ (ignoring the cost of computing the adjustment factor), or $\mathcal{O}(S_{\text{MCMC}} \cdot N \cdot M^2 \cdot d + 2^M \cdot S \cdot M^2)$ (considering the cost of computing the adjustment factor).

\paragraph{Uncertainty representation} 
In GPAR, uncertainties for the kernel hyper-parameters are not accounted for; it uses e.g. maximum likelihood to yield point estimates (the optimisation step, i.e. Step 2 in Algo.\ref{algo:GPAR}, yields a single optimal $\ell$, which is used to construct a fixed covariance matrix $K$). The co-occurrence probability $p(I)$ for an itemset $I$ is estimated as (Eq.\ref{eq:GPAR_marginal_posterior_MC_approx}):
\[
p(I) = \frac{1}{S} \sum_{s=1}^{S} \mathbb{I}(\mathbf{z}_{s,I} > 0), \quad \mathbf{z}_{s,I} \sim \mathcal{N}(0, K_I)
\]
where $K_I$ is fixed based on the optimized $\ell$. The uncertainty in $p(I)$ arises solely from the Monte Carlo sampling variability (i.e. the variance across the $S$ samples of $\mathbf{z}_s$). For rare rules (e.g. itemsets with low $p(I)$), the Monte Carlo estimate may have high variance if $S$ is insufficient.

BARM uses a fully Bayesian approach, treating both the item presence probabilities $\{p_i\}$ and the hyper-parameters (e.g. length scale $\ell$) as random variables with prior distributions. MCMC sampling generates $S_{\text{MCMC}}$ samples from the joint posterior $p(\{p_i\}, \ell | \mathcal{T}, X)$, and the co-occurrence probability $p(I)$ is computed as (Eq.\ref{eq:BARM_cooccurrence_probs}):
\[
p(I) = \frac{1}{S} \sum_{s=1}^{S} p(I | \{p_i^{(s)}\}, \ell^{(s)})
\]
where each $p(I | \{p_i^{(s)}\}, \ell^{(s)})$ is calculated using the sampled probabilities and correlations (Eq.\ref{eq:BARM_subset_prob_estimate_via_MCMC_posterior_samples}):
\[
p(I | \{p_i^{(s)}\}, \ell^{(s)}) \approx \prod_{i \in I} p_i^{(s)} \cdot g(\{\rho_{ij}^{(s)} \mid i, j \in I\})
\]
This way yields a distribution over $p(I)$, as each of the $S$ samples produces a different value of $p(I | \{p_i^{(s)}\}, \ell^{(s)})$, reflecting the uncertainty in both $\{p_i\}$ and $\ell$. The variance of these $S$ values directly quantifies the uncertainty in $p(I)$, providing a richer representation of uncertainty compared to GPAR’s point estimate.

Therefore, BARM provides a full posterior distribution over presence probabilities, offering better uncertainty quantification than GPAR’s point estimates, which is important for interpreting \textit{rare} or \textit{uncertain} rules. Consider an itemset $I = \{\text{bread 1, bread 2}\}$ with a true co-occurrence probability that is very low (e.g. rare rule). GPAR estimates $p(I)$ using Monte Carlo sampling with a fixed $K_I$. If $S = 1000$, the GPAR estimate might be $p(I) = 0.02$, but the uncertainty is only reflected in the sampling error (e.g. standard error $\sqrt{\frac{p(I)(1-p(I))}{S}} \approx 0.0044$). There is no mechanism to account for uncertainty in $\ell$, which could significantly affect $K_I$ and thus $p(I)$. For a rare rule, this lack of uncertainty quantification makes it difficult to assess the reliability of the estimate, potentially leading to \textit{over-confidence} in the mined rule. 

For the same itemset $I = \{\text{bread 1, bread 2}\}$, BARM computes $p(I)$ by averaging over $S = 1000$ posterior samples, each reflecting a different ($\{p_i^{(s)}\}, \ell^{(s)}$). Suppose the resulting values of $p(I | \{p_i^{(s)}\}, \ell^{(s)})$ range from 0.01 to 0.05, with a mean of 0.03 and a standard deviation of 0.01. This distribution over $p(I)$ directly quantifies the uncertainty, showing that the co-occurrence probability is uncertain within the range [0.01, 0.05]. For rare rules, this is important, as it allows the user to assess the reliability (or robustness) of the rule (e.g. a high variance indicates low confidence in the estimate) and make informed decisions about its interpretation or use.

\section{Beta distribution} \label{app:Beta_dist}

\begin{figure} [H]
\centering
\includegraphics[width=0.5\linewidth]{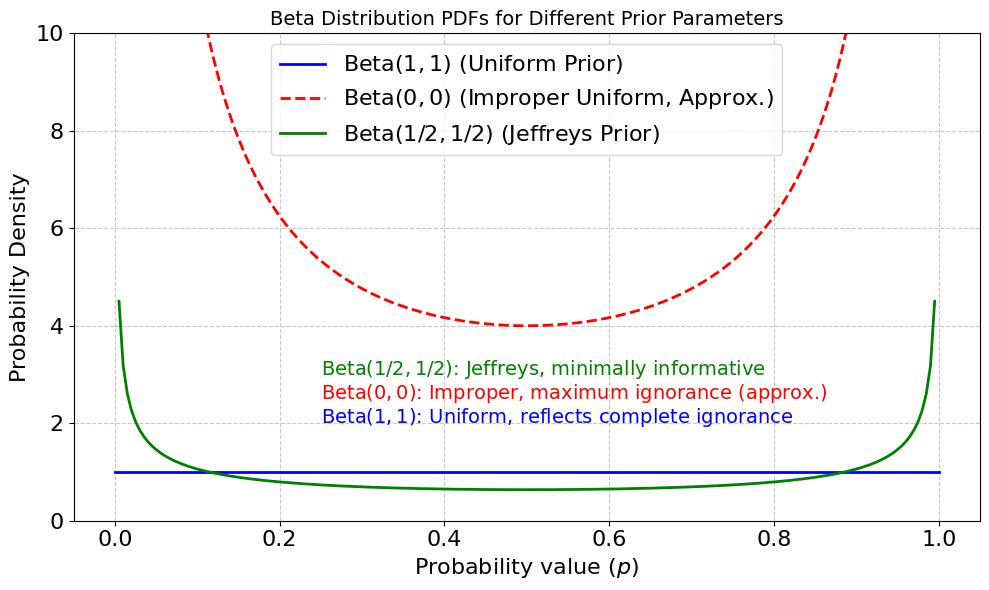}
\caption{Beta distributions.}
\label{fig:beta_distributions}
\end{figure}

Beta distribution, with examples shown in Fig.\ref{fig:beta_distributions}, is a continuous probability distribution defined on the interval $[0, 1]$, which makes it particularly suitable for modeling random variables that represent probabilities or proportions. Its probability density function (\textit{pdf}) is given by:
$$f(x; \alpha, \beta) = \frac{x^{\alpha - 1} (1 - x)^{\beta - 1}}{B(\alpha, \beta)}$$
where $ x \in [0, 1] $, $ \alpha > 0 $ and $ \beta > 0 $ are shape parameters, and $ B(\alpha, \beta) = \frac{\Gamma(\alpha) \Gamma(\beta)}{\Gamma(\alpha + \beta)} $ is the Beta function, serving as a normalizing constant, $\Gamma$ is the gamma function. The cumulative distribution function (\textit{CDF}) is:
$$F(x; \alpha, \beta) = \frac{\int_0^{x} t^{\alpha - 1} (1 - t)^{\beta - 1} dt}{B(\alpha, \beta)}$$
The moments of Beta distribution provide insight into its shape: the mean is $ \mathbb{E}[x] = \frac{\alpha}{\alpha + \beta} $, and variance $ \text{Var}(x) = \frac{\alpha \beta}{(\alpha + \beta)^2 (\alpha + \beta + 1)} $, and higher moments (e.g. skewness, kurtosis) can be derived from the moment-generating function (MGF) or recursive relations, though the mean and variance are most commonly used for prior specification. These statistical properties endow Beta distribution with flexibility to model a wide range of behaviors, from uniform to highly skewed distributions, depending on the values of $ \alpha $ and $ \beta $.

Beta distribution is often used as a Bayesian prior for probability values due to its conjugacy with the Bernoulli likelihood, support on $[0, 1]$, flexibility in encoding prior beliefs (e.g. $ \text{Beta}(1, 1) $, $ \text{Beta}(0, 0) $, $ \text{Beta}(1/2, 1/2) $ for expressing ignorance), analytical tractability, and robustness to small datasets. These properties make it a natural and efficient choice for modeling item presence probabilities \cite{Mackay2003Information}, as used in our BARM framework: it is an ideal choice for the dependency-free BARM framework, because each item’s presence is modeled as an independent Bernoulli trial, and the absence of correlation adjustments ($ g(\cdot) = 1 $) aligns with the assumption of independence. The use of $ \text{Beta}(1, 1) $ to represent ignorance is justified when no prior information about item frequencies is available, providing a neutral starting point that the data can update. The analytical posterior update further supports its practicality, eliminating the computational overhead of MCMC sampling in this simplified model.

\paragraph{Support on the unit interval [0, 1]}
Beta distribution is defined over the interval $[0, 1]$, which is the natural range for probability values. Its probability density function, $ p(x) = \frac{x^{\alpha - 1} (1 - x)^{\beta - 1}}{B(\alpha, \beta)} $ with $ B(\cdot, \cdot) $ being the Beta function, allows it to model any probability distribution flexibly by adjusting the shape parameters $ \alpha > 0$ and $ \beta > 0$. This makes it inherently suitable as a prior for parameters such as the probability of an item’s presence.

\paragraph{Flexibility in representing prior beliefs}
Beta distribution can represent a wide range of prior beliefs through its parameters. For example, with $ \alpha = 1 $ and $ \beta = 1 $ (i.e. $\text{Beta}(1, 1)$), the Beta \textit{pdf} becomes $ p(x) = 1 $ for $ x \in [0, 1] $, which reduces to the uniform distribution, reflecting complete ignorance or a non-informative (vague) prior, assuming all probability values are equally likely; $\text{Beta}(0, 0)$ represents an improper uniform prior with density $ p(x) \propto 1 $, which is not integrable over $[0, 1]$. $\text{Beta}(0, 0)$ is used to express maximum ignorance, though it requires careful handling to ensure a proper posterior. It can be justified in cases where no prior information is available, but its use is less common due to potential numerical instability; With $\alpha = 1/2 $ and $ \beta = 1/2$ (i.e. $\text{Beta}(1/2, 1/2)$), the Beta \textit{pdf} becomes $p(x) \propto x^{-1/2} (1 - x)^{-1/2}$, which is the \textit{Jeffreys prior} for a Bernoulli parameter. This prior is invariant under reparameterization and is often chosen to represent ignorance in a way that is minimally informative while remaining proper, especially for small sample sizes. These options allow the analyst to encode varying degrees of prior knowledge or ignorance, making Beta distribution versatile for different scenarios.

\paragraph{Analytical tractability}
Beta distribution’s CDF and moments (e.g. mean $ \mathbb{E}[x] = \frac{\alpha}{\alpha + \beta} $, variance $ \text{Var}(x) = \frac{\alpha \beta}{(\alpha + \beta)^2 (\alpha + \beta + 1)} $) are analytically computable, which facilitates prior specification and posterior analysis. This tractability is particularly valuable for BARM, where posterior samples can be drawn efficiently from the updated Beta distribution.

\paragraph{Conjugacy with Bernoulli and Binomial likelihoods}
Beta distribution is the conjugate prior for the Bernoulli and binomial likelihoods, which are commonly used to model binary outcomes such as the presence or absence of an item in a transaction. Conjugacy implies that the posterior distribution, after updating with observed data, remains within the same family as the prior distribution, facilitating analytical computation of the posterior parameters. For a Bernoulli likelihood $ p(n | x) = x^{n} (1 - x)^{1 - n} $, combined with a Beta prior $ x \sim \text{Beta}(\alpha, \beta) $, the posterior is $ x | \mathcal{T} \sim \text{Beta}(\alpha + n, \beta + N - n) $, where $ n $ is the number of successes (item presences). This analytical tractability eliminates the need for numerical methods like MCMC, enhancing computational efficiency.

\section{Upper confidence bound (UCB)} \label{app:UCB}

The UCB algorithm is a method employed in multi-armed bandit problems to address the exploration-exploitation dilemma, where an agent must balance selecting actions with the highest known reward (exploitation) and exploring lesser-known actions to improve future decision-making (exploration). Exploitation involves choosing the action with the largest current estimated value, maximizing immediate rewards, while exploration entails selecting non-greedy actions to gather additional information, potentially yielding long-term benefits. A purely greedy strategy may lead to sub-optimal outcomes by overlooking potentially superior actions, necessitating a balanced approach.

UCB resolves this dilemma by incorporating uncertainty into action-value estimates. At time $t$, the algorithm selects the action $A_t$ that maximizes the following expression:

\[
A_t = \arg\max_a \left[ \overset{\text{exploit}}{Q_t(a)} + \overset{\text{explore}}{c \sqrt{\frac{\ln(t)}{N_t(a)}}} \right]
\]
where $Q_t(a)$ is the current estimated value of action $a$, $N_t(a)$ is the number of times action $a$ has been taken, and $c$ is a constant controlling the exploration rate. The term $c \sqrt{\frac{\ln(t)}{N_t(a)}}$ represents the upper confidence bound, encouraging exploration of actions with higher uncertainty (larger confidence intervals) while favoring exploitation of actions with reliable estimates (smaller confidence intervals). This confidence interval, centered around the estimated action value $Q(a)$, reflects the range within which the true action value is likely to lie, with the upper bound guiding optimistic selection under uncertainty.

Initially, UCB explores more to reduce uncertainty across all actions, gradually shifting toward exploitation as confidence in estimates increases. For example, if multiple actions have associated uncertainties, UCB selects the one with the highest upper bound, either maximizing reward if the action is optimal or enhancing knowledge if it is less explored. This is exemplified in Fig.\ref{fig:UCB1}, where the action A with the highest upper bound is chosen.

\begin{figure}[H]
    \centering
    \includegraphics[width=0.45\linewidth]{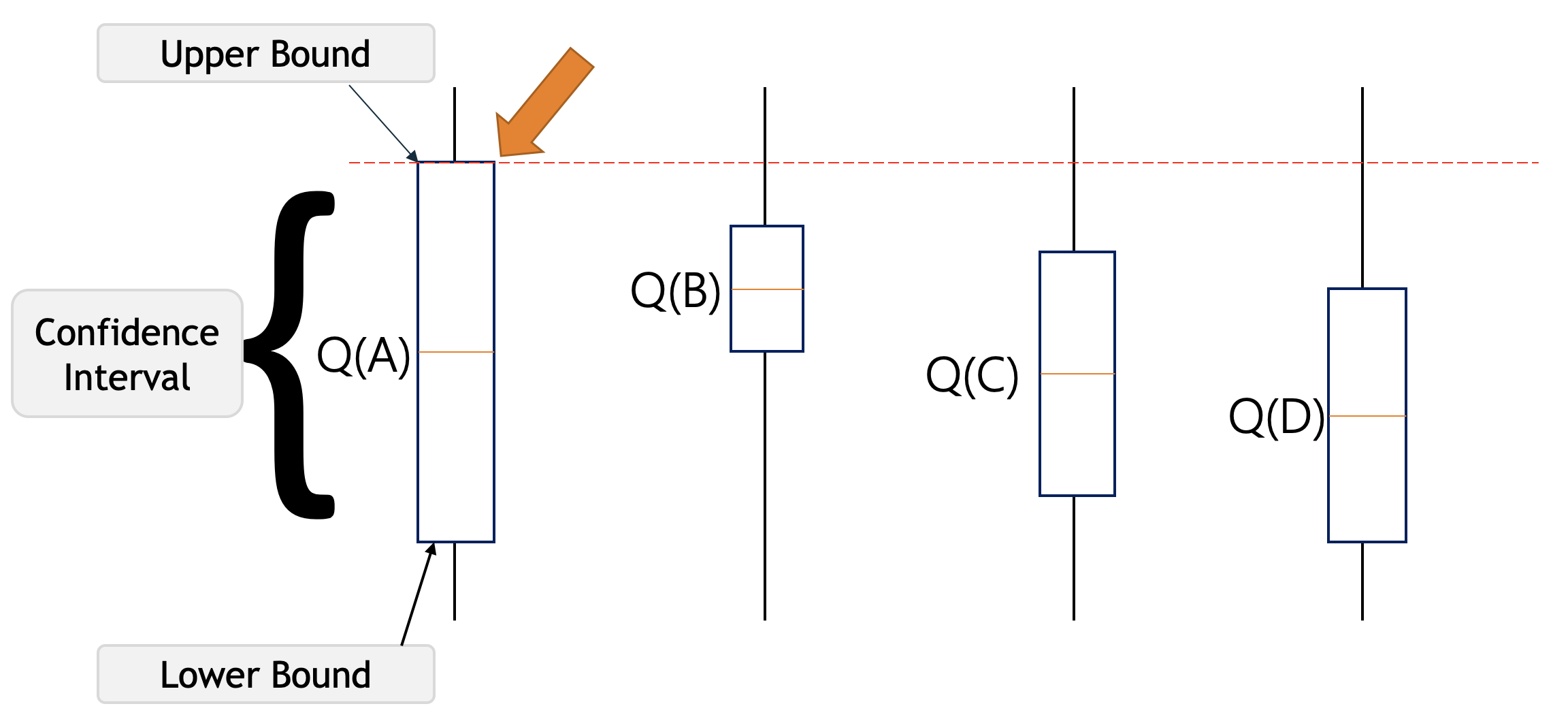}
    \caption{Confidence intervals for actions A, B, C, and D, with UCB selecting action A due to the highest upper bound. Figure taken from \cite{UCB_figures}.}
    \label{fig:UCB1}
\end{figure}

Assume that after selecting action A, the updated state is depicted in  Fig.\ref{fig:UCB2}, where action B now has the highest upper-confidence bound due to its elevated action-value estimate.

\begin{figure}[H]
    \centering
    \includegraphics[width=0.45\linewidth]{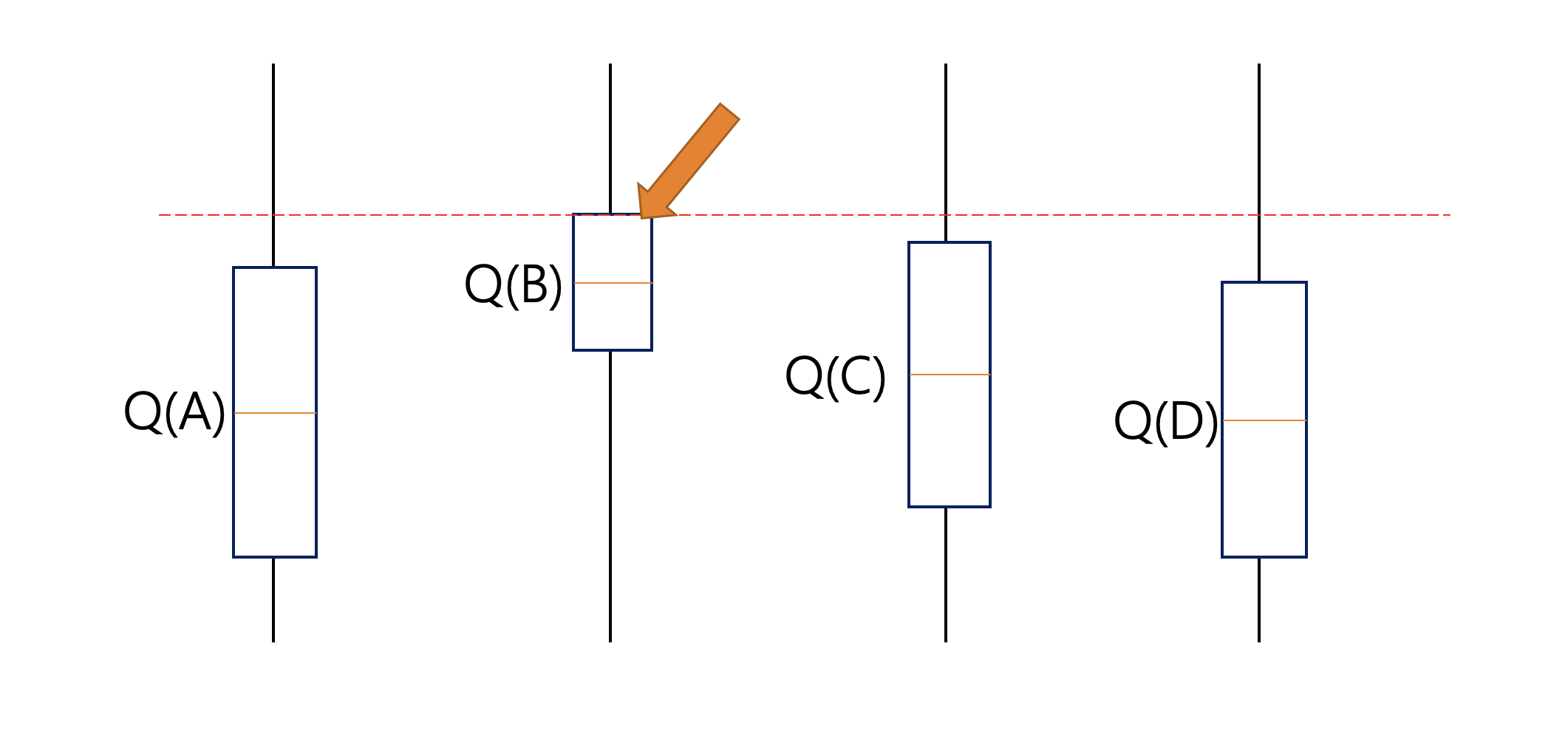}
    \caption{Updated confidence intervals after selecting A, with UCB now selecting action B. Figure taken from \cite{UCB_figures}.}
    \label{fig:UCB2}
\end{figure}

Over time, this approach typically yields higher average rewards compared to alternatives such as \textit{epsilon-greedy} or \textit{optimistic initial values}, as it systematically balances exploration and exploitation based on empirical data.

\section{MAB-ARM: computational complexity} \label{app:MAB-ARM_complexity}

We analyse the costs of the MAB-ARM algorithm in its two main phases: the initialization stage and the main loop.

\paragraph{Initialization stage} 
This stage involves generating all candidate itemsets $\mathcal{I}$ for sizes $m$ from $m_{min}$ to $m_{max}$. The cost is given as $\mathcal{O}(\sum_{m=m_{\text{min}}}^{m_{\text{max}}} \mathcal{C}_m^M)$, where $\mathcal{C}_m^M = \binom{M}{m}$ is the binomial coefficient representing the number of itemsets of size $m$. The total number of itemsets in $\mathcal{I}$ is:
\[
|\mathcal{I}| = \sum_{m=m_{\text{min}}}^{m_{\text{max}}} \binom{M}{m}
\]
This cost depends on the range $[m_{\text{min}}, m_{\text{max}}]$. For example, if $m_{\text{min}} = 2$ and $m_{\text{max}} = M$, it approaches $\mathcal{O}(2^M)$, but typically $m_{max}$ is set to a smaller value to limit the search space.

\paragraph{Main loop} 
The main loop runs for $T_{\text{max}}$ iterations. Each iteration involves: (1) computing UCB scores and selecting an itemset, which costs $\mathcal{O}(|\mathcal{I}|)$; (2) Estimating the co-occurrence probability for the selected itemset, which costs $\mathcal{O}(N \cdot m)$, where $m \leq m_{\text{max}}$, and in the worst case, $m \leq M$. The total cost of the main loop is therefore:
\[
\mathcal{O}(T_{\text{max}} \cdot (|\mathcal{I}| + N \cdot M))
\]
as $m \leq M$, so $\mathcal{O}(N \cdot m)$ is bounded by $\mathcal{O}(N \cdot M)$.

\paragraph{Other costs} 
Initializing counters for each itemset in $\mathcal{I}$ costs $\mathcal{O}(|\mathcal{I}|)$, which is subsumed by the main loop cost. Rule generation and pruning steps within the loop do not add significant complexity beyond the probability estimation.

\paragraph{Overall cost} 
The overall computational cost of MAB-ARM is the sum of the initialization and the main loop costs:
\[
\mathcal{O}\left( \sum_{m=m_{\text{min}}}^{m_{\text{max}}} \mathcal{C}_m^M + T_{\text{max}} \cdot \left( \sum_{m=m_{\text{min}}}^{m_{\text{max}}} \mathcal{C}_m^M + N \cdot M \right) \right)
\]

As the main loop cost typically dominates for large $T_{\text{max}}$, the overall cost can be simplified to:
\[
\mathcal{O}\left( T_{\text{max}} \cdot \left( \sum_{m=m_{\text{min}}}^{m_{\text{max}}} \mathcal{C}_m^M + N \cdot M \right) \right)
\]

\paragraph{Comparison with GPAR and BARM}
For GPAR, the overall cost is upper bounded by $\mathcal{O}(M^2d + T(M^2d + NM^2 + M^3) + 2^M (M^3 + SM^2))$, with $T$ being the number of optimisation iterations for optimising kernel hyper-parameters and $S$ the number of Monte Carlo samples used for rules evaluation. The first two terms correspond to the kernel optimisation stage cost, and the last term corresponds to the cost of rules evaluation. The major cost of GPAR is dominated by the last term $\mathcal{O}(2^M \cdot (M^3 + S M^2))$, i.e. the exponential number of itemsets and the cubic cost of GP operations (Cholesky decomposition of the covariance matrix when sampling).  

For the dependency-free BARM, analytical posterior inference costs $\mathcal{O}(N \cdot M)$ to update the beliefs, the rules evaluation costs $\mathcal{O}(2^M \cdot S \cdot M)$. The overall cost of the dependency-free BARM is upper bounded by $\mathcal{O}(N \cdot M + 2^M \cdot S \cdot M)$, which is dominated by the last term $\mathcal{O}(2^M \cdot S \cdot M)$, which is also exponential in $M$, but with a lower per-itemset cost than GPAR. If non-conjugate prior and likelihood are used, we seek to MCMC sampling for posterior inference, which can be costive. The total cost of BARM would be $\mathcal{O}(S_{\text{MCMC}} \cdot N \cdot M + 2^M \cdot S \cdot M)$, or $\mathcal{O}(S_{\text{MCMC}} \cdot N \cdot M^2 \cdot d + 2^M \cdot S \cdot M^2)$ if considering computing the pairwise correlation adjustment. Depending on the MCMC iterations $S_{\text{MCMC}}$, this MCMC sampling based BARM may or may not be more scalable than GPAR, although it avoids the $\mathcal{O}(M^3)$ cost of GP covariance matrix operations.

For MAB-ARM, the cost is $\mathcal{O}(T_{\text{max}} \cdot (\sum_{m=m_{\text{min}}}^{m_{\text{max}}} \mathcal{C}_m^M + N \cdot M)) \leq \mathcal{O}(T_{\text{max}} \cdot (2^M + N \cdot M))$, with $T_{\text{max}}$ being the number of MAB iterations. $T_{\text{max}}$ can be controlled to be much smaller than $S \cdot M$, $M^3$ or $SM^2$, potentially making MAB-ARM more efficient than GPAR or BARM.

\section{The MCTS-ARM algorithm} \label{app:MCTS-ARM}

Following the discussion in Section.\ref{subsec:MCTS-ARM}, we present the pseudo code for the MCTS-ARM method. The main workflow is outlined in the pseudo codes below, with implementation details provided in Algo.\ref{algo:MCTS-ARM}.

{\centering \textbf{MCTS-ARM pseudocode}\par}
\begin{lstlisting}
initialize tree with root node (empty itemset)
while not stopping_criteria_met:
    node = select_node(root)  # Traverse using UCB
    if not fully_expanded(node):
        expand_node(node)      # Add new child itemsets
        new_node = select_new_child(node)
    else:
        new_node = node
    reward = simulate(new_node)          # Random rollout and evaluate
    backpropagate(new_node, reward)      # Update rewards and visits
extract_rules(tree)                      # Output high-reward itemsets and rules
\end{lstlisting}

\begin{algorithm}[H]
\scriptsize
\caption{MCTS-ARM: Monte Carlo tree search for association rule mining}
\label{algo:MCTS-ARM}
\textbf{Input:} A set of items $\{1, 2, \dots, M\}$; a set of transactions $\mathcal{T} = \{\mathbf{t}_1, \mathbf{t}_2, \dots, \mathbf{t}_N\}$, where each $\mathbf{t}_j \in \{0, 1\}^M$; a minimum support threshold $min\_supp$; a minimum confidence threshold $min\_conf$; maximum number of iterations $T_{\text{max}}$; maximum itemset size $m_{\text{max}}$; UCB exploration parameter $c$. \\
\textbf{Output:} A set of association rules $\mathcal{R}$ where each rule $r \in \mathcal{R}$ is of the form $A \rightarrow B$ with $A, B \subseteq \{1, 2, \dots, M\}$, satisfying the support and confidence thresholds.

\vspace{1mm}\hrule\vspace{1mm}

\begin{algorithmic}[1]
\STATE Initialize an empty set $\mathcal{R} = \emptyset$ to store association rules. \hfill\textit{$\mathcal{O}(1)$}
\STATE Initialize a tree with root node $root$ where $root.itemset = \emptyset$, $root.visits = 0$, $root.total\_reward = 0$, $root.children = \emptyset$. \hfill\textit{$\mathcal{O}(1)$}

\FOR{$t = 1$ to $T_{\text{max}}$}
    \STATE $node \gets$ SELECT-NODE($root$, $c$, $t$). \hfill\textit{$\mathcal{O}(m_{\text{max}} \cdot M)$}
    
    \IF{not FULLY-EXPAND($node$, $m_{\text{max}}$, $\mathcal{T}$, $min\_supp$)}
        \STATE EXPAND-NODE($node$, $\{1, 2, \dots, M\}$, $m_{\text{max}}$, $\mathcal{T}$, $min\_supp$). \hfill\textit{$\mathcal{O}(M \cdot N)$}
        \STATE $new\_node \gets$ SELECT-NEW-CHILD($node$). \hfill\textit{$\mathcal{O}(M)$}
    \ELSE
        \STATE $new\_node \gets node$. \hfill\textit{$\mathcal{O}(1)$}
    \ENDIF
    
    \STATE $reward \gets$ SIMULATE($new\_node$, $\mathcal{T}$, $min\_supp$, $min\_conf$, $m_{\text{max}}$). \hfill\textit{$\mathcal{O}(N \cdot m_{\text{max}} \cdot 2^{m_{\text{max}}})$}
    
    \STATE BACKPROPAGATE($new\_node$, $reward$). \hfill\textit{$\mathcal{O}(m_{\text{max}})$}
\ENDFOR

\STATE $\mathcal{R} \gets$ EXTRACT-RULES($root$, $\mathcal{T}$, $min\_supp$, $min\_conf$). \hfill\textit{$\mathcal{O}(V \cdot 2^{m_{\text{max}}} \cdot N)$}
\RETURN $\mathcal{R}$. \hfill\textit{$\mathcal{O}(1)$}

\vspace{0.3cm}
\STATE \textbf{Function} SELECT-NODE($node$, $c$, $t$):
\begin{ALC@g}
    \WHILE{$node.children \neq \emptyset$}
        \IF{any $child \in node.children$ has $child.visits = 0$}
            \RETURN $child$. \hfill\textit{$\mathcal{O}(|node.children|)$}
        \ENDIF
        \FOR{each $child \in node.children$}
            \STATE Compute $\text{UCB} = child.average\_reward + c \cdot \sqrt{\frac{\ln(t)}{child.visits}}$. \hfill\textit{$\mathcal{O}(1)$}
        \ENDFOR
        \STATE $node \gets \arg\max_{child \in node.children} \text{UCB}$. \hfill\textit{$\mathcal{O}(|node.children|)$}
    \ENDWHILE
    \RETURN $node$. \hfill\textit{$\mathcal{O}(1)$}
\end{ALC@g}
\vspace{0.3cm}

\STATE \textbf{Function} FULLY-EXPAND($node$, $m_{\text{max}}$, $\mathcal{T}$, $min\_supp$):
\begin{ALC@g}
    \IF{$|node.itemset| \geq m_{\text{max}}$}
        \RETURN True. \hfill\textit{$\mathcal{O}(1)$}
    \ENDIF
    \FOR{each $i \in \{1, 2, \dots, M\} \setminus node.itemset$}
        \STATE $new\_itemset \gets node.itemset \cup \{i\}$.
        \STATE Compute $supp = \frac{1}{N} \sum_{j=1}^N \mathbb{I}(\mathbf{t}_j[k] = 1 \ \forall k \in new\_itemset)$.
        \IF{$supp \geq min\_supp$ and $new\_itemset \notin \{child.itemset \mid child \in node.children\}$}
            \RETURN False. \hfill\textit{$\mathcal{O}(N \cdot m_{\text{max}})$}
        \ENDIF
    \ENDFOR
    \RETURN True. \hfill\textit{$\mathcal{O}(M \cdot N \cdot m_{\text{max}})$}
\end{ALC@g}
\vspace{0.3cm}
\end{algorithmic}
\end{algorithm}

\newpage

\setcounter{algorithm}{6} 
\begin{algorithm}[H]
\scriptsize
\caption{MCTS-ARM: Monte Carlo tree search for association rule mining (Cont'd)}
\label{algo:MCTS-ARM-continued}
\ContinuedFloat

\begin{algorithmic}[1]
\STATE \textbf{Function} EXPAND-NODE($node$, $items$, $m_{\text{max}}$, $\mathcal{T}$, $min\_supp$):
\begin{ALC@g}
    \IF{$|node.itemset| < m_{\text{max}}$}
        \FOR{each $i \in items \setminus node.itemset$}
            \STATE $new\_itemset \gets node.itemset \cup \{i\}$.
            \STATE Compute $supp = \frac{1}{N} \sum_{j=1}^N \mathbb{I}(\mathbf{t}_j[k] = 1 \ \forall k \in new\_itemset)$.
            \IF{$supp \geq min\_supp$}
                \STATE Create $new\_child$ with $itemset = new\_itemset$, $visits = 0$, $total\_reward = 0$, $average\_reward = 0$, $parent = node$, $children = \emptyset$.
                \STATE Add $new\_child$ to $node.children$. \hfill\textit{$\mathcal{O}(1)$}
            \ENDIF
        \ENDFOR
    \ENDIF \hfill\textit{$\mathcal{O}(M \cdot N \cdot m_{\text{max}})$}
\end{ALC@g}
\vspace{0.3cm}

\STATE \textbf{Function} SELECT-NEW-CHILD($node$):
\begin{ALC@g}
    \STATE Select $child \in node.children$ with $child.visits = 0$ (randomly if multiple).
    \RETURN $child$. \hfill\textit{$\mathcal{O}(|node.children|)$}
\end{ALC@g}
\vspace{0.3cm}

\STATE \textbf{Function} SIMULATE($node$, $\mathcal{T}$, $min\_supp$, $min\_conf$, $m_{\text{max}}$):
\begin{ALC@g}
    \STATE $current\_itemset \gets node.itemset$.
    \WHILE{$|current\_itemset| < m_{\text{max}}$}
        \STATE Randomly select $i \in \{1, 2, \dots, M\} \setminus current\_itemset$.
        \STATE $current\_itemset \gets current\_itemset \cup \{i\}$.
        \STATE Compute $supp = \frac{1}{N} \sum_{j=1}^N \mathbb{I}(\mathbf{t}_j[k] = 1 \ \forall k \in current\_itemset)$.
        \IF{$supp < min\_supp$}
            \RETURN 0. \hfill\textit{$\mathcal{O}(N \cdot m_{\text{max}})$}
        \ENDIF
    \ENDWHILE
    \STATE Initialize $valid\_rules \gets \emptyset$.
    \FOR{each $(A, B)$ split of $current\_itemset$ with $B \neq \emptyset$}
        \STATE Compute $p_A = \frac{1}{N} \sum_{j=1}^N \mathbb{I}(\mathbf{t}_j[k] = 1 \ \forall k \in A)$.
        \STATE Compute $conf = \frac{supp}{p_A}$.
        \IF{$conf \geq min\_conf$}
            \STATE Add $(A \rightarrow B)$ to $valid\_rules$. \hfill\textit{$\mathcal{O}(N \cdot m_{\text{max}})$}
        \ENDIF
    \ENDFOR
    \IF{$valid\_rules = \emptyset$}
        \RETURN 0. \hfill\textit{$\mathcal{O}(1)$}
    \ELSE
        \RETURN $\max(\{conf \mid (A \rightarrow B) \in valid\_rules\})$. \hfill\textit{$\mathcal{O}(2^{m_{\text{max}}})$}
    \ENDIF \hfill\textit{$\mathcal{O}(N \cdot m_{\text{max}} \cdot 2^{m_{\text{max}}})$}
\end{ALC@g}
\vspace{0.3cm}

\STATE \textbf{Function} BACKPROPAGATE($node$, $reward$):
\begin{ALC@g}
    \WHILE{$node \neq$ null}
        \STATE $node.visits \gets node.visits + 1$.
        \STATE $node.total\_reward \gets node.total\_reward + reward$.
        \STATE $node.average\_reward \gets \frac{node.total\_reward}{node.visits}$.
        \STATE $node \gets node.parent$. \hfill\textit{$\mathcal{O}(1)$}
    \ENDWHILE \hfill\textit{$\mathcal{O}(m_{\text{max}})$}
\end{ALC@g}
\vspace{0.3cm}

\STATE \textbf{Function} EXTRACT-RULES($root$, $\mathcal{T}$, $min\_supp$, $min\_conf$):
\begin{ALC@g}
    \STATE Initialize $\mathcal{R} \gets \emptyset$.
    \FOR{each $node$ in tree (via depth-first search) with $node.visits > 0$}
        \STATE Compute $supp = \frac{1}{N} \sum_{j=1}^N \mathbb{I}(\mathbf{t}_j[k] = 1 \ \forall k \in node.itemset)$.
        \IF{$supp \geq min\_supp$}
            \FOR{each $(A, B)$ split of $node.itemset$ with $B \neq \emptyset$}
                \STATE Compute $p_A = \frac{1}{N} \sum_{j=1}^N \mathbb{I}(\mathbf{t}_j[k] = 1 \ \forall k \in A)$.
                \STATE Compute $conf = \frac{supp}{p_A}$.
                \IF{$conf \geq min\_conf$}
                    \STATE Add $(A \rightarrow B)$ to $\mathcal{R}$. \hfill\textit{$\mathcal{O}(N \cdot m_{\text{max}})$}
                \ENDIF
            \ENDFOR
        \ENDIF
    \ENDFOR
    \RETURN $\mathcal{R}$. \hfill\textit{$\mathcal{O}(V \cdot 2^{m_{\text{max}}} \cdot N)$}
\end{ALC@g}
\vspace{0.3cm}
\end{algorithmic}
\end{algorithm}

\end{document}